\documentclass[12pt,oneside,letterpaper]{memoir}
\usepackage{dsfont}
\usepackage{lipsum}

\usepackage{ragged2e,verbatimbox,multirow}
\IfFileExists{natbib.sty}{%
\usepackage{natbib}%
}{}

\IfFileExists{appendix.sty}{%
\usepackage{appendix}%
}{}

\DisemulatePackage{setspace}
\usepackage[onehalfspacing]{setspace}

\IfFileExists{enumitem.sty}{%
\usepackage[loadonly]{enumitem}[2007/06/30]%
\newlist{hanglist}{enumerate}{1}%
\setlist[hanglist]{label=\arabic*.}%
\setlist[hanglist,1]{leftmargin=0pt}%
}{%
}

\usepackage{amssymb}
\usepackage{amsmath}
\usepackage{amsthm}

\usepackage{wrapfig}
\usepackage[dvipsnames]{xcolor}
\usepackage{color, colortbl}

\usepackage[colorlinks,
linkcolor=Periwinkle,
citecolor=Periwinkle]{hyperref}

\IfFileExists{mathpartir.sty}{%
\usepackage{mathpartir}%
}{}

\IfFileExists{listings.sty}{%
\usepackage{listings}%
}{}

\hypersetup{pdfborder = 0 0 0}

\IfFileExists{microtype.sty}{%
\usepackage[protrusion=true,expansion=true]{microtype}%
}{}

\usepackage{boxedminipage}

\IfFileExists{booktabs.sty}{%
\usepackage{booktabs}%
}{}

\usepackage{ifpdf}

\usepackage{lmodern}
\usepackage[LY1]{fontenc}

\IfFileExists{tgpagella.sty}{%
\usepackage{tgpagella}%
}{}

\usepackage[LY1]{eulervm}

\ifpdf
\usepackage[pdftex]{graphicx}
\else
\usepackage{graphicx}
\fi

\usepackage{makeidx}
\makeindex

{\theoremstyle{plain}

}
{\theoremstyle{remark}

}

\newtheoremstyle{customsty1}
{3pt}%
{3pt}%
{}
{}
{\bfseries}
{:}
{.5em}
{}

{\theoremstyle{customsty1} 
}

\makeatletter
\newenvironment{wb-bib}[1]{%
  \chapter*{References}
  \relax\markboth{References}{References}
\ifnobibintoc\else 
\phantomsection 
\addcontentsline{toc}{chapter}{References} 
\fi 
\prebibhook
  \begin{bibitemlist}{#1}}{\end{bibitemlist}\postbibhook}

\AtBeginDocument{%
  \@ifpackageloaded{natbib}{
    \addtodef{\endthebibliography}{}{\vskip-\lastskip\postbibhook}
    \@ifpackagewith{natbib}{sectionbib}{
      \renewcommand{\bibsection}{\@memb@bsec}}%
      {\renewcommand{\bibsection}{\@memb@bchap}}}%
  {}
  \@ifpackagewith{chapterbib}{sectionbib}{%
    \renewcommand{\sectionbib}[2]{}
    \renewcommand{\bibsection}{\@memb@bsec}}{}
}
\makeatother

\makeatletter
\newcommand{\wb@episource}{}
\newenvironment{wbepi}[1]{\begin{quote}\renewcommand{\wb@episource}{#1}\itshape}{\par\upshape \raggedleft --- \textsc{\wb@episource}\\ \end{quote}}
\makeatother

\IfFileExists{svn-multi.sty}{%
\usepackage{svn-multi}%
\newcommand{\vcinfo}{}%
}{%
\newcommand{\svnidlong}[4]{}%
\newcommand{\vcinfo}{}%
}

\makeatletter

\newcommand\@thesistitlemedskip{0.2in}
\newcommand\@thesistitlebigskip{0.6in}
\newcommand{\degree}[1]{\gdef\@degree{#1}}
\newcommand{\project}{\gdef\@doctype{A masters project report}}
\newcommand{\prelim}{\gdef\@doctype{A preliminary report}}
\newcommand{\thesis}{\gdef\@doctype{A thesis}}
\newcommand{\dissertation}{\gdef\@doctype{A dissertation}}
\newcommand{\department}[1]{\gdef\@department{(#1)}}

\newenvironment{titlepage}
 {\@restonecolfalse\if@twocolumn\@restonecoltrue\onecolumn
  \else \newpage \fi \thispagestyle{empty}
}{\if@restonecol\twocolumn \else \newpage \fi}

\gdef\@degree{Doctor of Philosophy}    
\gdef\@doctype{A dissertation}         

\gdef\@department{(Electrical Engineering)} 
\gdef\@defensedate{05/16/2025}
\gdef\@committee{
Yong Jae Lee, Associate Professor, Computer Sciences\\
  Robert Nowak, Professor, Electrical and Computer Engineering\\
  Jerry Zhu, Professor, Computer Sciences\\
  Sharon Y. Li (Advisor), Assistant Professor, Computer Sciences
  }

\renewcommand{\maketitle}{%
  \begin{titlepage}
    \def\thanks##1{\typeout{Warning: `thanks' deleted from thesis titlepage.}}
    \let\footnotesize\small \let\footnoterule\relax \setcounter{page}{1}
    \begin{center}
      {\textbf{\expandafter\expandafter{\@title}}} \\[\@thesistitlebigskip]
       by \\[\@thesistitlemedskip]
      \@author \\[\@thesistitlebigskip]
      \@doctype\ submitted in partial fulfillment of \\
      the requirements for the degree of\\[\@thesistitlebigskip]
      \@degree \\[\@thesistitlemedskip]
      \@department \\[\@thesistitlebigskip]
      at the \\[\@thesistitlemedskip]
      UNIVERSITY OF WISCONSIN--MADISON\\[\@thesistitlemedskip]
      \@date
    \end{center}
    \hspace*{-0.7in}Date of final oral examination: \@defensedate \\[\@thesistitlemedskip]
    \hspace*{-0.7in}The dissertation is approved by the following members of the 
    Final Oral Committee:\\
    \@committee
  \end{titlepage}
  \setcounter{footnote}{0}
  \setcounter{page}{1} 
  \let\thanks\relax
  \let\maketitle\relax \let\degree\relax \let\project\relax \let\prelim\relax
  \let\department\relax
  \gdef\@thanks{}\gdef\@degree{}\gdef\@doctype{}
  \gdef\@department{}
}

\def\abstract{
  \chapter*{Abstract}
  \addcontentsline{toc}{chapter}{Abstract}
  \relax\markboth{Abstract}{Abstract}
  }

\def\advisortitle#1{\gdef\@advisortitle{#1}}
\def\advisorname#1{\gdef\@advisorname{#1}}
\gdef\@advisortitle{Professor}
\gdef\@advisorname{Cheer E.\ Place}

\def\umiabstract{
             \thispagestyle{empty}
                  \addtocounter{page}{-1}
                \begin{center}
                  {\textbf{\expandafter\uppercase\expandafter{\@title}}}\\
                  \vspace{12pt}
                  \@author \\
                  \vspace{12pt}
                  Under the supervision of \@advisortitle\ \@advisorname\\
                  At the University of Wisconsin-Madison
                \end{center}
}

\def\endumiabstract{\vfill \hfill\@advisorname\par\newpage}

\def\verbatimfile#1{\begingroup \singlespace
                    \@verbatim \frenchspacing \@vobeyspaces
                    \input#1 \endgroup
}

\def\copyrightpage{
  \newpage
  \thispagestyle{empty}    
  \addtocounter{page}{-1}
  \chapter*{}            
  \begin{center}
   \vfill
   \copyright\ Copyright by \@author\ \@date\\
   All Rights Reserved
  \end{center}}

\def\acks{
\makeatletter
\renewcommand{\ps@plain}{%
  \let\@oddhead\@empty
  \let\@evenhead\@empty
  \def\@oddhead{\hfill\thepage}   
  \let\@evenhead\@oddhead         
  \let\@oddfoot\@empty            
  \let\@evenfoot\@oddfoot
}
\makeatother
  \chapter*{Acknowledgments}
}

\def\dedication{
  \newpage
  \null\vfil
  \begin{center}}
\def\enddedication{\end{center}\par\vfil\newpage}

\newcount\@testday
\def\today{\@testday=\day
  \ifnum\@testday>30 \advance\@testday by -30
  \else\ifnum\@testday>20 \advance\@testday by -20
  \fi\fi
  \number\day\ \
  \ifcase\month\or
    January \or February \or March \or April \or May \or June \or
    July \or August \or September \or October \or November \or December
    \fi\ \number\year
}

\newtheorem{theorem}{Theorem}[chapter]

\newtheorem{definition}[theorem]{Definition}

\newtheorem{lemma}[theorem]{Lemma}

\def\@bibchaptitle{Bibliography}
\def\altbibtitle{\def\@bibchaptitle{Bibliography}}
\def\thebibliography#1{
  \global\@bibpresenttrue
  \chapter*{\@bibchaptitle\markboth{\@bibchaptitle}{\@bibchaptitle}}
  \addcontentsline{toc}{chapter}{\@bibchaptitle}
  \vspace{0.375in}    
  \interlinepenalty=10000 
  \singlespace\list
  {[\arabic{enumi}]}{\settowidth\labelwidth{[#1]}\leftmargin\labelwidth
    \advance\leftmargin\labelsep \usecounter{enumi}}
  \def\newblock{\hskip .11em plus .33em minus -.07em}
  \sloppy
  \sfcode`\.=1000\relax}
\let\endthebibliography=\endlist

\makeatother

 \usepackage{algorithm}
\usepackage[ruled,vlined,algo2e]{algorithm2e}
 \def\*#1{\mathbf{#1}}
\usepackage{pifont}
\usepackage{arydshln}
\usepackage{makecell}
\usepackage[breakable, skins]{tcolorbox}
\newtheorem{Remark}{Remark}
\DeclareMathOperator*{\argmin}{arg\,min}
\DeclareMathOperator*{\argmax}{arg\,max}
\definecolor{greyC}{RGB}{180,180,180}
\definecolor{greyL}{RGB}{235,235,235}
\newtheorem{theorem*}{Theorem}
\newcommand{\xmark}{\ding{55}}%
\newtheorem{Definition}{Definition}
\newtheorem{Proposition}{Proposition}

\newtheorem{assumption}{Assumption}
\svnidlong{$LastChangedBy$}{$LastChangedRevision$}{$LastChangedDate$}{$HeadURL: http://freevariable.com/dissertation/branches/diss-template/dissertation.tex $} 

\clearpage

\title{Foundations of Unknown-aware Machine Learning}
\author{Xuefeng Du}
\department{Computer Sciences}

\date{2025}

\begin{document}


\ifpdf
\DeclareGraphicsExtensions{.pdf, .jpg, .tif}
\else
\DeclareGraphicsExtensions{.eps, .jpg}
\fi

\maketitle


\svnidlong{$LastChangedBy$}{$LastChangedRevision$}{$LastChangedDate$}{$HeadURL: http://freevariable.com/dissertation/branches/diss-template/frontmatter/frontmatter.tex $}
\vcinfo{}


\copyrightpage

\begin{dedication}
	\emph{To my dearest parents, whose unwavering belief in me lit the way,\\
To my admired advisor, whose wisdom and kindness steadied my steps,\\
And to the many souls—too numerous to name—who held me in their hearts when the days were dark.\\
This work is a tribute to your love, your faith, and your quiet strength.}
\end{dedication}




\begin{acks}
\begin{wbepi}{William Shakespeare (Twelfth Night, Act III, Scene 3)}
I can no other answer make but thanks, and thanks, and ever thanks.
\end{wbepi}

\noindent This dissertation would not have been possible without the support of many individuals.

First and foremost, I want to express my deepest gratitude to my advisor, Dr. Sharon Y. Li, the best advisor in the world. She mentored me from scratch, teaching me various different things not only on research with great patience. She always gave me the greatest support for every paper I wrote, every award I applied to, and every decision I made. These supports, even if just a "Great job, Xuefeng!!" in the chat, are invaluable to me. I couldn't imagine this close mentorship relationship I built with Sharon before I came to this lab as a PhD student.

I am also grateful to my committee members, Dr. Robert Nowak, Dr. Yong Jae Lee, and Dr. Jerry Zhu, for their thoughtful feedback and encouragement that improved my thesis work. Besides, I would like to thank them for making time to attend my thesis defense despite their busy schedule.

I appreciate my labmates, my mentors and collaborators at UTS, CMU and HKBU for their mentorship and support with attention to detail. I am also thankful for the financial support provided by the grants from my advisor, Jane Street Graduate Research Fellowship and the Ivanisevic Award at UW-Madison, which made this research possible.

On a personal level, I extend my heartfelt thanks to my family and friends for their unwavering belief in me. To my parents, your unconditional and altruistic encouragement and love have been my foundation for pursuing this PhD degree very far away from home. To my friends, who always accompanied me and reminded me to laugh even during the toughest moments—thank you!

\end{acks}

\renewcommand{\printtoctitle}[1]{\chapter*{#1}}
\renewcommand{\printloftitle}[1]{\chapter*{#1}}
\renewcommand{\printlottitle}[1]{\chapter*{#1}}

\renewcommand{\tocmark}{}
\renewcommand{\lofmark}{}
\renewcommand{\lotmark}{}

\renewcommand{\tocheadstart}{}
\renewcommand{\lofheadstart}{}
\renewcommand{\lotheadstart}{}

\renewcommand{\aftertoctitle}{}
\renewcommand{\afterloftitle}{}
\renewcommand{\afterlottitle}{}

\renewcommand{\cftchapterfont}{\normalfont} 
\renewcommand{\cftsectionfont}{\itshape} 
\renewcommand{\cftchapterpagefont}{\normalfont} 
\renewcommand{\cftchapterpresnum}{\bfseries} 
\renewcommand{\cftchapterleader}{} 
\renewcommand{\cftsectionleader}{} 
\renewcommand{\cftchapterafterpnum}{\cftparfillskip} 
\renewcommand{\cftsectionafterpnum}{\cftparfillskip} 



\makeatletter
\renewcommand{\ps@plain}{%
  \let\@oddhead\@empty
  \let\@evenhead\@empty
  \def\@oddhead{\hfill\thepage}   
  \let\@evenhead\@oddhead         
  \let\@oddfoot\@empty            
  \let\@evenfoot\@oddfoot
}
\makeatother

\tableofcontents

\clearpage
\makeatletter
\renewcommand{\ps@plain}{%
  \let\@oddhead\@empty
  \let\@evenhead\@empty
  \def\@oddhead{\hfill\thepage}   
  \let\@evenhead\@oddhead         
  \let\@oddfoot\@empty            
  \let\@evenfoot\@oddfoot
}
\makeatother

\listoftables

\clearpage
\makeatletter
\renewcommand{\ps@plain}{%
  \let\@oddhead\@empty
  \let\@evenhead\@empty
  \def\@oddhead{\hfill\thepage}   
  \let\@evenhead\@oddhead         
  \let\@oddfoot\@empty            
  \let\@evenfoot\@oddfoot
}
\makeatother
\listoffigures

\clearpage


\begin{abstract}


Ensuring the reliability and safety of machine learning models in open-world deployment is a central challenge in AI safety. This thesis, which focuses on developing both algorithms and theoretical foundations, addresses key reliability issues that arise under distributional uncertainty and unknown classes, from conventional neural networks to modern foundation models, like large language models (LLMs).

The key challenge of the thesis lies in reliability characterization of the reliability of off-the-shelf machine learning algorithms, which typically minimize errors on in-distribution (ID) data without accounting for uncertainties that
could arise outside out of distribution (OOD). For instance, the widely used empirical risk minimization (ERM), operates under
the closed-world assumption (i.e., no distribution shift between training and inference). Models optimized
with ERM are known to produce overconfidence predictions on OOD data, since the decision boundary is
not conservative. To address this challenge, our works developed novel frameworks that jointly optimize for both: (1)
accurate prediction of samples from ID, and (2) reliable handling of data from outside ID. 

To solve this challenge, we propose an unknown-aware learning framework  that enables models to recognize and handle novel inputs without explicit prior knowledge of those unknowns. In particular, this thesis begins by developing novel outlier synthesis paradigms, i.e., VOS, NPOS and DREAM-OOD, to generate representative "unknown" examples during training, which improves out-of-distribution detection without requiring any labeled OOD data. Building on top of this, our works propose new algorithms and theoretical analyses for unknown-aware learning in the wild (SAL), leveraging unlabeled deployment data to enhance model reliability to OOD samples. These methods provide formal guarantees and show that abundant unlabeled data can be harnessed to detect and adapt to unforeseen inputs, which significantly improves reliability under real-world conditions.

In addition, we advance the reliability of large-scale foundation models, including state-of-the-art text-only and multimodal large language models (LLMs). It presents techniques for detecting hallucinations in generated outputs (HaloScope), defending against malicious prompts (MLLMGuard), and alignment data cleaning to remove noisy or biased feedback data. By mitigating such failure modes, the thesis ensures safer interactions of the cutting edge AI systems. 

The contributions of this research are not only novel in methodology but also broad in impact: they collectively strengthen reliable decision-making in AI and pave the way toward unknown-aware learning as a standard paradigm. We hope this can inspire future OOD research for advanced AI systems with minimal human efforts.

\end{abstract}
\clearpage\pagenumbering{arabic}

\begin{intro}
    Artificial Intelligence (AI) and its subfield of machine learning (ML) have become increasingly instrumental in driving innovation across numerous domains, from computer vision~\citep{ren2015faster} and natural language processing~\citep{devlin2018bert} to healthcare~\citep{bajwa2021artificial} and autonomous driving~\citep{hu2023gaia}. At the same time, the reliability and safety of ML models remain central concerns, particularly as these systems move from controlled laboratory settings to wide-ranging real-world applications. Traditional ML models, which often rely on the assumption that training and test data arise from the same underlying distribution~\citep{vapnik1999nature}, face significant challenges when confronted with unfamiliar conditions or novel inputs—phenomena known broadly as distribution shifts or out-of-distribution (OOD) inputs~\citep{liu2020energy,yang2021generalized, fang2022learnable}.

When ML systems fail to recognize their own limitations, the consequences can be severe. For instance, as shown in Figure~\ref{fig:fig1} (a), an autonomous vehicle’s \textit{discriminative} vision algorithm might confidently misclassify an unusual object on the road, such as a helicopter, as a known object. Such failures not
only raise concerns about model reliability but also pose serious risks in safety-critical deployments.  Large-scale \textit{generative} models, including Large Language Models (LLMs)~\citep{openai2023gpt4,touvron2023llama,grattafiori2024llama} and Multimodal Large Language Models (MLLMs)~\citep{liu2023visual,bai2023qwen}, can produce untruthful or harmful responses if they are not adequately aligned with human norms~\citep{ji2023survey,zhang2023siren}. These vulnerabilities underscore the urgent need for reliability-oriented techniques—methods that can robustly detect and respond to OOD data, maintain calibration under distributional shifts, and mitigate unsafe behaviors in powerful foundation models.

\begin{figure}
    \centering
    \includegraphics[width=1.0\textwidth]{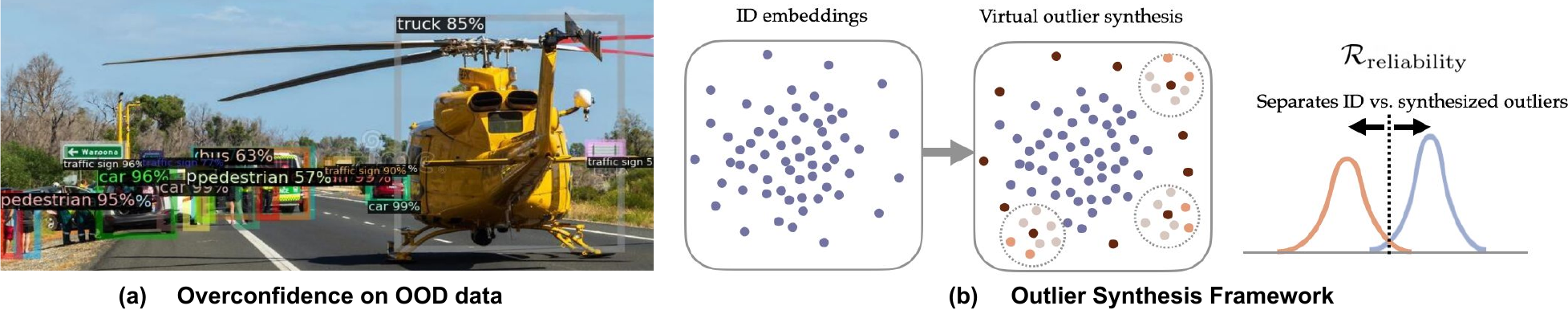}
    \vspace{0em}
    \caption[Illustration of the overconfident predictions of an objection detection model and the proposed outlier synthesis framework.]{\small (a) An object detection model trained on BDD-100k dataset~\citep{DBLP:conf/cvpr/YuCWXCLMD20} produces overconfident predictions for OOD objects (e.g., helicopter), highlighting reliability concerns in ML models during deployment. Test images are sampled from MS-COCO~\citep{lin2014microsoft}. (b) Overview of my proposed outlier synthesis framework for unknown-aware learning. }
    \label{fig:fig1}
    \vspace{-1.2em}
\end{figure}

Reliable ML introduces core challenges in characterizing the reliability of off-the-shelf learning algorithms, which typically minimize errors on in-distribution (ID) data from $\mathbb{P}_{\text{in}}$ without accounting for uncertainties that could arise outside $\mathbb{P}_{\text{in}}$.  For instance, the widely used empirical risk minimization (ERM)~\citep{vapnik1999nature}, operates under the closed-world assumption (\emph{i.e.}, no distribution shift between training and inference). Models optimized with ERM are known to produce overconfidence predictions on OOD data~\citep{nguyen2015deep}, since the decision boundary is not conservative. To address this challenge, my PhD research developed novel frameworks that jointly optimize for both: \textbf{(1)} accurate prediction of samples from $\mathbb{P}_\text{in}$, and \textbf{(2)} reliable handling of data from outside $\mathbb{P}_\text{in}$. Given a weighting factor $\alpha$, this can be formalized as follows:
\vspace{-0.5em}
\begin{align}
\label{eq:obj}
    \text{argmin}~~ [ \mathcal{R}_\text{accuracy}+ ~\alpha \cdot \mathcal{R}_\text{reliability}]. \vspace{-0.8em}
\end{align}

As an example, $\mathcal{R}_\text{accuracy}$ can be the risk that classifies ID samples into known classes while $ \mathcal{R}_\text{reliability}$ aims to distinguish ID vs. OOD. The introduction of the reliability risk term $\mathcal{R}_\text{reliability}$ is crucial to prevent overconfident predictions on unknown data and improve test-time reliability when encountering unknowns. However, incorporating this reliability risk requires large-scale human annotations, e.g., binary ID and OOD labels, which could limit the practical usage of the proposed framework. Therefore, my research contributes to \textbf{\textit{developing the foundations of reliable machine learning  with minimal human supervision}}, which spans three key aspects:

\begin{enumerate}
    \item I developed novel unknown-aware learning frameworks that teach the models what they don't know without having explicit knowledge about unknowns (Figure~\ref{fig:fig1} (b)). The framework enables tractable learning from the unknowns by adaptively generating virtual outliers from the low-likelihood region in both the feature~\citep{du2022towards,du2022stud,tao2023nonparametric}
and input space~\citep{du2023dream}, and shows strong efficacy and interpretability for regularizing the model to discriminate the boundaries between known and
unknown data. 
\item I designed algorithms and theoretical analysis  for unknown-aware learning by leveraging unlabeled data collected from the models' deployment environment. This wild data is a mixture of ID and OOD data by an unknown mixing ratio. Methods I designed such as \textit{gradient SVD score}~\citep{du2024sal,bai2024out} and \textit{constrained optimization}~\citep{bai2023feed} can facilitate OOD detection and generalization on these real-world reliability challenges.
\item  I built reliable foundation models by investigating the reliability blind spots of language models, such as untruthful generations~\citep{du2024halo}, malicious prompts~\citep{du2024vlmguard}, and noisy alignment data~\citep{yeh2024reliable}. My work seeks to fundamentally understand the sources of these issues by developing algorithms that leverage unlabeled data to identify and mitigate the unintended information, which ensures safer human-AI interactions.
\end{enumerate}

 My thesis research has led to impactful publications in top-tier ML and vision venues and has been recognized by \href{https://datascience.ucsd.edu/rising-stars-in-data-science/}{\textcolor{Periwinkle}{Rising Stars in Data Science}} and \href{https://www.janestreet.com/join-jane-street/programs-and-events/grf-profiles-2023/}{\textcolor{Periwinkle}{Jane Street Graduate Research Fellowship}} programs. Many of my works has been integrated into the OpenOOD benchmark~\citep{yang2022openood,zhang2023openood}, and have received considerable follow-ups from worldwide major industry labs, such as Google~\citep{liu2023fast}, Microsoft~\citep{narayanaswamy2023exploring}, Amazon~\citep{constantinou2024out}, Apple~\citep{zang2024overcoming}, Adobe~\citep{gu2023critical}, Air Force Research~\citep{inkawhich2024tunable}, Toyota~\citep{seifi2024ood}, LG~\citep{yoon2024diffusion}, Alibaba~\citep{lang2022estimating} etc. The scientific impact of reliable ML is profound, I am excited to explore interdisciplinary collaborations across computer science, statistics, biology science, and policy to push the boundaries of reliable ML as a machine learning researcher in the future. 
 
 The outline of this thesis is as follows: \textbf{Chapter 2} states the background of the thesis research by introducing the problem setup, and reviewing the literature on out-of-distribution detection and reliable foundation models, which provide the broader conceptual framework for unknown-aware learning. \textbf{Chapter 3} provides an overview of the first piece of my PhD research on \textit{foundations of unknown-aware learning}. \textbf{Chapters 4-6} present in order the three representative foundational works that (1) discuss the tractable learning foundation by outlier synthesis that is primarily based on the publications at ICLR'22~\citep{du2022towards} and ICLR'23~\citep{tao2023nonparametric}; (2) introduce interpretable outlier synthesis to allow human-compatiable interpretation~\citep{du2023dream}; and (3) understand the impact of in-distribution data particularly on the effect of a compact representation space~\citep{du2022siren}. \textbf{Chapter 7} summarizes the contributions of the second piece of my PhD thesis research on \textit{learning in the wild with unlabeled data}.  \textbf{Chapter 8} discusses algorithmic and theoretical advances in leveraging unlabeled wild data. We describe the proposed learning algorithm, optimization procedures, and generalization analyses. \textbf{Chapter 9} contains the overview of the contributions for the final piece of my PhD thesis research on \textit{towards reliable foundation models}.  \textbf{Chapter 10} expands the focus to large-scale language and multimodal models, examining issues of hallucination, and malicious user prompt attacks along with proposed solutions. Finally, \textbf{Chapter 11} concludes the thesis, summarizing the key findings and envisioning how unknown-aware learning can further push the boundary of AI reliability.

\end{intro}

\begin{back}
    In this chapter, we will first introduce the problem formulation in Section~\ref{sec:pf},  and then discuss the related work to this thesis in Section~\ref{sec:rw}.
\section{Problem Formulation}
\label{sec:pf}
\subsection{Out-of-distribution Detection}

The key idea in the unknown-aware learning framework is to perform OOD detection, which identifies data shifts, such as the data that belongs to semantic classes different from training, by performing thresholding on certain scoring functions. Formally, we describe the data setup, models and losses and learning goal.

\textbf{Labeled ID data and ID distribution.} \textcolor{black}{Let $\mathcal{X}$ be the input space, and $\mathcal{Y}=\{1,...,K\}$ be the label space for ID data. Given an unknown ID joint distribution $\mathbb{P}_{\mathcal{X}\mathcal{Y}}$ defined over $\mathcal{X}\times \mathcal{Y}$, the labeled ID data $\mathcal{S}^{\text{in}} = \{(\*x_1, y_1),...,(\*x_n, y_n)\}$ are drawn independently and identically from $\mathbb{P}_{\mathcal{X}\mathcal{Y}}$.  We also denote $\mathbb{P}_{\text{in}}$ as the marginal distribution of $\mathbb{P}_{\mathcal{X}\mathcal{Y}}$ on $\mathcal{X}$, which is referred to as the ID distribution. }

\textbf{Out-of-distribution detection.} Our framework concerns a common real-world scenario in which the algorithm is trained on the \textcolor{black}{labeled} ID data, but will then be deployed in environments containing OOD data from {unknown} class, i.e., $y\notin \mathcal{Y}$, and therefore should not be predicted by the model. 
At test time, the goal is to decide whether a test-time input is from  ID  or not (OOD).

\textbf{Unlabeled wild data.} A key challenge in OOD detection is the lack of labeled OOD data. In particular, the sample space for potential OOD data can be prohibitively large, making it expensive to collect labeled OOD data.  In this thesis (particularly in Chapter~\ref{sec:cha8}), to model the realistic environment, we incorporate unlabeled wild
data $\mathcal{S}_{\text{wild}}=\left\{\tilde{\mathbf{x}}_1,...,\tilde{\mathbf{x}}_m \right\}$ into our learning framework. Wild data
 consists of both ID and OOD data, and can be collected   freely upon deploying an existing model trained on $\mathcal{S}^{\text{in}}$. Following~\cite{katzsamuels2022training}, we use the Huber contamination model to characterize the marginal distribution
of the wild data
\begin{equation}
\mathbb{P}_\text{wild} := (1-\pi) \mathbb{P}_\text{in} + \pi \mathbb{P}_\text{out}, 
\label{eq:wild_w_cor}
\end{equation}
where $\pi \in (0,1]$ and \textcolor{black}{$\mathbb{P}_\text{out}$ is the OOD distribution defined over $\mathcal{X}$}.  Note that the case $\pi=0$  is
straightforward since no novelties occur.

\textbf{Models and losses.} We denote by $\*h_{\*w}: \mathcal{X} \mapsto \mathbb{R}^K$  a predictor for ID classification 
 with parameter $\*w \in \mathcal{W}$, where $\mathcal{W}$ is the parameter space. $\*h_{\*w}$ returns the soft classification output. We consider the loss function $\ell: \mathbb{R}^{K} \times \mathcal{Y}\mapsto \mathbb{R}$ on the labeled ID data. In addition, we denote the OOD classifier $\*g_{\boldsymbol{\theta}}:  \mathcal{X} \mapsto \mathbb{R}$ with parameter ${\boldsymbol{\theta}} \in \Theta$, where $\Theta$ is the parameter space. We use $\ell_{\text{b}}(\*g_{\boldsymbol{\theta}}(\mathbf{x}),y_{\text{b}})$ to denote the binary loss function \emph{w.r.t.} $\*g_{\boldsymbol{\theta}}$ and binary label $y_{\text{b}}\in \mathcal{Y}_{\text{b}}:=\{y_{+},y_{-}\}$, where $y_+ \in \mathbb{R}_{>0}$ and $y_-\in \mathbb{R}_{<0}$ correspond to the ID class and the OOD class, respectively.

 \textbf{Learning goal.} Our learning framework aims to build the OOD classifier $\*g_{\boldsymbol{\theta}}$ by leveraging data from either $\mathcal{S}^{\text{in}}$ only (Chapters~\ref{sec:cha3}-\ref{sec:cha6}) or the joint set of $\mathcal{S}^{\text{in}}$ and $\mathcal{S}_{\text{wild}}$ (Chapter~\ref{sec:cha8}). In evaluating our
model, we are interested in the following measurements:
\begin{equation}
\begin{aligned}
&(1)~\downarrow \text{FPR}(\*g_{\boldsymbol{\theta}};\lambda):=\mathbb{E}_{\*x \sim \mathbb{P}_\text{out}}(\mathds{1}{\{ \*g_{\boldsymbol{\theta}} (\*x)>\lambda\}}), \\
 &(2)~\uparrow \text{TPR}(\*g_{\boldsymbol{\theta}};\lambda):=\mathbb{E}_{\*x \sim \mathbb{P}_\text{in}}(\mathds{1}{\{ \*g_{\boldsymbol{\theta}}(\*x)>\lambda\}}),
\end{aligned}
\end{equation}
where $\lambda$ is a threshold, typically chosen so that a high fraction of ID data is correctly classified. In Chapters~\ref{sec:cha4} and~\ref{sec:cha6}, we include the object-level OOD detection task during evaluation, where we will discuss the problem setup more concretely there.

\subsection{Hallucination Detection}
The third piece of this PhD thesis focuses on LLM safety (Chapter~\ref{sec:cha10}), specifically on LLM hallucination detection with a safety emphasis on the model outputs compared to OOD detection.  Formally, we describe the LLM generation and the problem of hallucination detection.

\textbf{LLM generation.} We consider an $L$-layer causal LLM, which takes a sequence of $n$ tokens $\*x_{\text{prompt}} = \{x_1,...,x_n\}$, and generates an output $\*x_o=\{ x_{n+1},...,x_{n+m}\}$ in an autoregressive manner. 
  Each output token $x_i, i  \in [n+1,...,n+m]$ is sampled from a distribution
over the model vocabulary $\mathcal{V}$, conditioned on the prefix $\{x_1,...,x_{i-1}\}$:
\begin{equation}
    x_i = \operatorname{argmax}_{x\in \mathcal{V}} P(x | \{x_1,...,x_{i-1}\}),
\end{equation}
and the probability $P$ is calculated as:
\begin{equation}
    P(x | \{x_1,...,x_{i-1}\}) = \operatorname{softmax}(\*w_o \*f_L(x) + \*b_o),
\end{equation}
where $\*f_{L}(x) \in \mathbb{R}^{d}$ denotes the representation at the $L$-th layer of LLM for token $x$, and $\*w_o, \*b_o$ are the weight and bias parameters at the final output layer. 
  
\textbf{Hallucination detection.}
    We denote $\mathbb{P}_{\text{true}}$ as the \textcolor{black}{joint} distribution over the truthful \textcolor{black}{input and generation pairs}, which is referred to as truthful distribution. For any given generated text $\*x_o$ \textcolor{black}{and its corresponding input prompt $\*x_{\text{prompt}}$ where $(\*x_{\text{prompt}}, \*x_o) \in \mathcal{X}_{LLM}$}, the goal of hallucination detection is to learn a binary predictor $G: \mathcal{X}_{LLM} \rightarrow \{0,1\}$ such that
    \begin{equation}
\label{eq:hal-def}
    G(\textcolor{black}{\*x_{\text{prompt}}, \*x_o}) = \begin{cases}
        \;1, \hspace{4mm} &\text{if } \; \textcolor{black}{(\*x_{\text{prompt}}, \*x_o)} \sim \mathbb{P}_\text{true} \\
        \;0,         &\text{otherwise}
    \end{cases}
\end{equation}

\section{Related Work}
\label{sec:rw}
This section includes an introduction to related work in out-of-distribution
detection and LLM hallucination detection, which touches both the algorithmic and theoretical research results. Additionally, each chapter includes discussions on other research areas relevant to its specific
topic.

\subsection{Literature on OOD Detection}
 \textit{OOD detection} has attracted a surge of interest in recent years~\citep{fort2021exploring,yang2021generalized,fang2022learnable,zhu2022boosting,ming2022delving,ming2022spurious,yang2022openood,wang2022outofdistribution,galil2023a,djurisic2023extremely,zheng2023out,wang2022watermarking,wang2023outofdistribution,narasimhan2023learning,yang2023auto,uppaal2023fine,zhu2023diversified,zhu2023unleashing,ming2023finetune,zhang2023openood,ghosal2024how}. One line of work performs OOD detection by devising scoring functions, including confidence-based methods~\citep{bendale2016towards,hendrycks2016baseline,liang2018enhancing}, energy-based score~\citep{liu2020energy,wang2021canmulti,wu2023energybased}, distance-based approaches~\citep{lee2018simple,tack2020csi,DBLP:journals/corr/abs-2106-09022,2021ssd,sun2022out, du2022siren, ming2023cider,ren2023outofdistribution}, gradient-based score~\citep{huang2021importance}, and {Bayesian approaches~\citep{gal2016dropout,lakshminarayanan2017simple,maddox2019simple,dpn19nips,Wen2020BatchEnsemble,kristiadi2020being}.} Another  line of work addressed OOD detection by training-time regularization~\citep{bevandic2018discriminative,malinin2018predictive,geifman2019selectivenet,hein2019relu,meinke2019towards,DBLP:conf/nips/JeongK20,liu2020simple,DBLP:conf/icml/AmersfoortSTG20,DBLP:conf/iccv/YangWFYZZ021,DBLP:conf/icml/WeiXCF0L22,du2022stud,du2023dream,wang2023learning}.  For example, the model is
regularized to produce lower confidence~\citep{lee2018training} or higher energy~\citep{liu2020energy,du2022towards} on a set of clean OOD data~\citep{hendrycks2018deep,DBLP:conf/icml/MingFL22}, wild data~\citep{zhou2021step,katzsamuels2022training,he2023topological,bai2023feed, du2024sal} and synthetic outliers~\citep{du2023dream,tao2023nonparametric,park2023powerfulness}. 

\textit{OOD detection for object detection} is a rising topic with very few existing works. For Faster R-CNN, My work VOS~\citep{du2022towards} proposed to synthesize virtual outliers in the feature space for model regularization. \cite{du2022stud} explored unknown-aware object detection by leveraging videos in the wild, whereas we focus on settings with still images only.
For transformer-based object detection model \textsc{detr},~\cite{ow_detr} adopted unmatched object queries that are with high confidence as unknowns, which did not focus on regularizing the model for desirable representations.
Several works~\citep{DBLP:journals/corr/abs-2108-03614,DBLP:conf/wacv/DhamijaGVB20, DBLP:conf/wacv/0003DSZMCCAS20,DBLP:conf/icra/MillerDMS19, DBLP:conf/icra/MillerNDS18} used approximate Bayesian methods, such as MC-Dropout~\citep{gal2016dropout} for OOD detection. They require multiple inference passes to generate the uncertainty score, which are computationally expensive on larger datasets and models.

\subsection{Literature on OOD Detection Theory}
Recent studies have begun to focus on the theoretical understanding of OOD detection. ~\citet{fang2022learnable} studied the generalization of
OOD detection by PAC learning and they found a necessary condition for the learnability of OOD
detection.~\citet{morteza2022provable} derived a novel OOD score and provided a provable understanding of the OOD detection result using that score. My work at ICLR'24~\citep{du2024sal} theoretically studied the impact of unlabeled data for OOD detection. Later at ICML'24, my work~\citep{duand} formally analyzed the impact of ID labels on OOD detection, which has not been studied in the past.

\subsection{Literature on Hallucination Detection}
 \textit{Hallucination detection} has gained interest recently for ensuring LLMs' safety and reliability~\citep{guerreiro2022looking,huang2023survey,ji2023survey,zhang2023siren,xu2024hallucination,zhang2023enhancing,chern2023factool,min2023factscore,huang2023look,ren2023outofdistribution,wang2023hallucination}. The majority of work performs hallucination detection by devising uncertainty scoring functions, including those based on the logits~\citep{malinin2021uncertainty,kuhn2023semantic,duan2023shifting} that assumed hallucinations would be generated by flat token log probabilities, and methods that are based on the output texts, which either measured the consistency of multiple generated texts ~\citep{manakul2023selfcheckgpt,agrawal2023language,mundler2023self,xiong2023can,cohen2023lm} or prompted LLMs to evaluate the confidence on their generations~\citep{kadavath2022language,xiong2023can,ren2023self,lin2022teaching,tian2023just,zhou2023navigating}. Additionally, there is growing interest in exploring the LLM activations to determine whether an LLM generation is true or false~\citep{su2024unsupervised,yin2024characterizing,rateike2023weakly}. For example, ~\cite{chen2024inside} performed eigendecomposition with activations but the decomposition was done on the covariance matrix that required multiple generation steps to measure the consistency.~\cite{zou2023representation} explored probing meaningful direction from neural activations.  Another branch of works, such as~\citep{li2023inference,duan2024llms,azaria2023internal}, employed labeled data for extracting truthful directions, which differs from the scope on harnessing unlabeled LLM generations that is explored in this thesis. Note that our studied problem is different from the research on hallucination mitigation~\citep{lee2022factuality,tian2019sticking,zhang2023alleviating,kai2024sh2,shi2023trusting,chuang2024dola}, which aims to enhance the truthfulness of LLMs' decoding process. Some of my thesis works, such as~\citep{bai2024out,du2024sal,bai2023feed} can be closely connected with hallucination detection with unlabeled LLM generations, which utilized unlabeled data for out-of-distribution detection. However, their approach and problem formulation are different.

\end{back}

\begin{overviewone}
    \label{sec:cha3}
\textbf{Motivation.} Ensuring safe and reliable AI systems requires addressing a critical issue: the overconfident predictions made on the OOD inputs~\citep{nguyen2015deep}. These inputs arise from unknown categories and should ideally be excluded from model predictions. For example, in self-driving car applications, my research is \textbf{the first} to discover that an object detection model trained on ID objects (e.g., cars, pedestrians) might confidently misidentify an unusual object, such as a helicopter on a highway, as a known object; see Figure~\ref{fig:fig1} (a). Such failures not only raise concerns about model reliability but also pose serious risks in safety-critical deployments.

The vulnerability to OOD inputs stems from the lack of explicit knowledge of unknowns during training, as neural networks are typically optimized only on ID data. While this approach effectively captures ID tasks, the resulting decision boundaries can be inadequate for OOD detection. Ideally, a model should maintain high confidence for ID data and exhibit high uncertainty for OOD samples, yet achieving this goal is challenging due to the absence of labeled outliers. My research tackles this challenge  for unknown-aware learning through an automated outlier generation paradigm, which offers greater feasibility and flexibility than approaches requiring extensive human annotations~\citep{hendrycks2018deep}. I outline three core  fundamental contributions between Chapter~\ref{sec:cha4} and Chapter~\ref{sec:cha6}:

\textbf{Tractable learning foundation by outlier synthesis.} My work VOS~\citep{du2022towards} (ICLR'22) \textit{laid the foundation of a learning framework called virtual outlier synthesis to regularize the models' decision boundary}. This approach is based on modeling ID features as Gaussians, reject sampling to synthesize virtual outliers from low-likelihood regions, and a novel unknown-aware training objective that contrastively shapes the uncertainty energy surface between ID data and synthesized outliers. Additionally, VOS delivers the insight that synthesizing outliers in the feature space is more tractable than generating high-dimensional pixels~\citep{lee2018training}. My subsequent work, NPOS~\citep{tao2023nonparametric} (ICLR'23), relaxed the Gaussian assumption through a non-parametric synthesis approach, yielding improved results on language models and larger datasets. The STUD method~\citep{du2022stud}, presented at CVPR'22 as an \textbf{oral}, further demonstrated the efficacy of this approach in \textbf{real-world practice}, i.e., video object detection, distilling unknown objects in both spatial and temporal dimensions to regularize model decision boundaries. Particularly, Chapter~\ref{sec:cha4} is going to mainly discuss the work of VOS.

\textbf{Interpretable outlier synthesis.} While feature-space synthesis is effective, it doesn't allow \textit{human-compatible interpretation like visual pixels}. To address this, my NeurIPS'23 paper Dream-OOD~\citep{du2023dream} introduced a framework to \textit{comprehensively study the interactions between feature-space and pixel-space synthesis}. The method learns a text-conditioned visual latent space, enabling outlier sampling and decoding by diffusion models, which not only enhances interpretability but also achieves strong results on OOD detection benchmarks. It has garnered quite a few interests from community, prompting follow-up research on pixel-space outlier synthesis~\citep{yoon2024diffusion,um2024self,liu2024can}. Chapter~\ref{sec:cha5} will cover the work of Dream-OOD.

 \textbf{Understanding the impact of in-distribution data.} Beyond focusing on reliability risks in the unknown-aware learning framework, it's crucial to address the in-distribution accuracy term $\mathcal{R}_\text{accuracy}$ in Equation~\ref{eq:obj} during training. My work, SIREN~\citep{du2022siren} (NeurIPS'22) and a subsequent ICML'24 paper~\citep{duand}, fundamentally investigated the influence of \textit{compact representation space} and \textit{ID label supervision} on identifying OOD samples. These insights contribute to designing better training strategies on ID data and enhancing overall model reliability. In Chapter~\ref{sec:cha6}, we will focus on the work of SIREN.
\end{overviewone}

\begin{vos}
    \label{sec:cha4}
\textbf{Publication Statement.} This chapter is joint work with Zhaoning Wang, Mu Cai and
Yixuan Li. The paper version of this chapter appeared in ICLR'22~\citep{du2022towards}.

\noindent \textbf{Abstract.} OOD detection has received much attention lately due to its importance in the safe deployment of neural networks. One of the key challenges is that models lack supervision signals from unknown data, and as a result, can produce overconfident predictions on OOD data. Previous approaches rely on real outlier datasets for model regularization, which can be costly and sometimes infeasible to obtain in practice. 
In this chapter, we present {VOS}, a novel framework for OOD detection by adaptively synthesizing virtual outliers that can meaningfully regularize the model's decision boundary during training. Specifically, VOS samples virtual outliers from the low-likelihood region of the class-conditional distribution estimated in the {feature space}. Alongside,  we introduce a novel unknown-aware training objective,  which contrastively shapes the uncertainty space between the ID data and synthesized outlier data. VOS achieves competitive performance on both object detection and image classification models, reducing the  FPR95 by up to 9.36\% compared to the previous best method on object detectors. Code is available at \url{https://github.com/deeplearning-wisc/vos}. 

\section{Introduction}

\begin{figure}
\centering

\includegraphics[width=0.95\linewidth,height=0.35\linewidth]{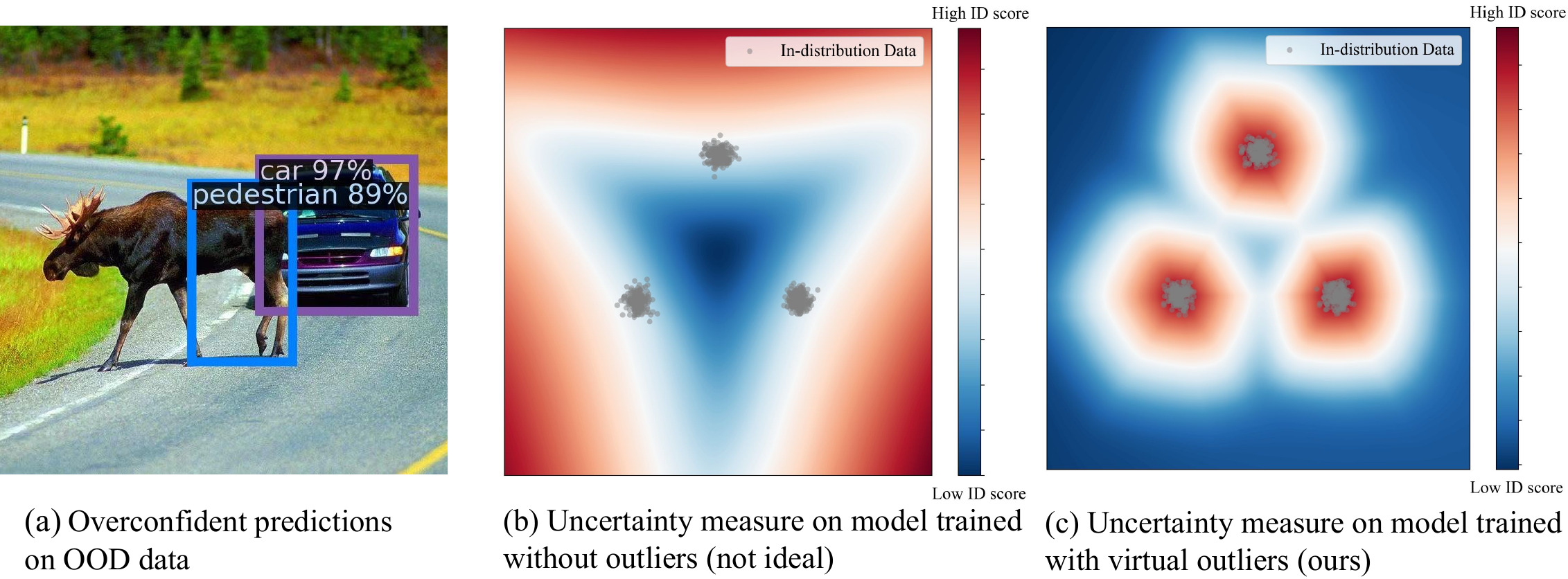}\label{fig:teaser}

\caption[Teasers for the work VOS]{ (a) A Faster-RCNN~\citep{ren2015faster} model trained on BDD-100k dataset~\citep{DBLP:conf/cvpr/YuCWXCLMD20} produces overconfident predictions for OOD object (\emph{e.g.}, moose). (b)-(c) The uncertainty measurement with and without virtual outlier training. The in-distribution data $\*x\in \mathcal{X}=\mathbb{R}^2$ is sampled from a Gaussian mixture model). Regularizing the model with virtual outliers (c) better captures the OOD uncertainty  than without (b).  
}

\label{fig::GMM_toy}
\end{figure}

Modern deep neural networks have achieved unprecedented success in known contexts for which they are trained, yet they often struggle to handle the unknowns. In particular, neural networks have been shown to produce high posterior probability for out-of-distribution (OOD) test inputs~\citep{nguyen2015deep}, which arise from unknown categories and should not be predicted by the model. Taking self-driving car as an example, an object detection model trained to recognize in-distribution objects (\emph{e.g.,} cars, stop signs) can produce a high-confidence prediction for an unseen object of a {moose}; see Figure~\ref{fig::GMM_toy}(a). Such a failure case raises concerns in model reliability, and worse, may lead to catastrophe when deployed in safety-critical applications.

The vulnerability to OOD inputs arises due to the lack explicit knowledge of unknowns during training time. In particular, neural networks are typically  optimized only on the in-distribution (ID) data. The resulting decision boundary, despite being useful on ID tasks such as classification, can be ill-fated for OOD detection. We illustrate this in Figure~\ref{fig::GMM_toy}. The ID data (gray) consists of three class-conditional Gaussians, on which a three-way softmax classifier is trained. The resulting classifier is overconfident for regions far away from the ID data (see the red shade in Figure~\ref{fig::GMM_toy}(b)), causing trouble for OOD detection. Ideally, a model should learn a more compact decision boundary that produces low uncertainty for the ID data, with high OOD uncertainty elsewhere (\emph{e.g.}, Figure~\ref{fig::GMM_toy}(c)). However, achieving this goal is non-trivial due to the lack of supervision signal of unknowns. This motivates the question: \emph{Can we synthesize virtual outliers for effective model regularization?}

In this chapter, we propose a novel unknown-aware learning framework dubbed  \textbf{VOS} (\textbf{V}irtual \textbf{O}utlier \textbf{S}ynthesis), which optimizes the dual objectives of both ID task and OOD detection performance. In a nutshell, {VOS} consists of three  components tackling challenges of outlier synthesis and effective model regularization with synthesized outliers. To synthesize the outliers, we estimate the class-conditional distribution in the \emph{feature  
space}, and sample outliers from the low-likelihood region of ID classes (Section~\ref{sec:gda}). Key to our method, we show that sampling in the feature space is more tractable than synthesizing images in the high-dimensional pixel space~\citep{lee2018training}. Alongside, we propose a novel unknown-aware training objective, which contrastively shapes the uncertainty surface between the ID data and synthesized outliers (Section~\ref{sec:training}). During training, \texttt{VOS} simultaneously performs the ID task (\emph{e.g.}, classification or object detection) as well as the OOD uncertainty regularization. During inference time, the uncertainty estimation branch produces a larger probabilistic score for ID data and vice versa, which enables effective OOD detection (Section~\ref{sec:inference}).

\texttt{VOS} offers several compelling advantages compared to existing solutions. \textbf{(1)} \texttt{VOS} is a \emph{general} learning framework that is effective for both object detection and image classification tasks, whereas previous methods were primarily driven by image classification. Image-level detection can be limiting as an image could be OOD in certain regions while being in-distribution elsewhere. Our work bridges a critical research gap since OOD detection for object detection is timely yet underexplored in literature.  \textbf{(2)}  \texttt{VOS} enables \emph{adaptive} outlier synthesis, which can be flexibly and conveniently used for any ID data {without} manual data collection or cleaning. In contrast, previous methods using outlier exposure~\citep{hendrycks2018deep} require an auxiliary image dataset that is sufficiently diverse, which can be arguably prohibitive to obtain. Moreover, one needs to perform careful data cleaning to ensure the auxiliary outlier dataset does not overlap with ID data. \textbf{(3)} \texttt{VOS} synthesizes outliers that can estimate a compact decision boundary between ID and OOD data. 
In contrast, existing solutions use outliers that are either too trivial to regularize the OOD estimator, or too hard to be separated from ID data, resulting in sub-optimal performance. 

Our {key contributions and results} are summarized as follows:
\begin{itemize}
 \item 
  We propose a new framework \texttt{VOS} addressing a pressing issue---unknown-aware deep learning that optimizes for both ID and OOD performance. \texttt{VOS} establishes \emph{state-of-the-art} results on a challenging object detection task. Compared to the best method, \texttt{VOS} reduces the FPR95 by up to 9.36\% while preserving the accuracy on the  ID task. 
 \item We conduct extensive ablations and reveal important insights by contrasting different outlier synthesis approaches. We show that \texttt{VOS} is more advantageous than generating outliers directly in the high-dimensional pixel space (\emph{e.g.}, using GAN~\citep{lee2018training}) or using noise as outliers. 
    \item We comprehensively evaluate our method on common OOD detection benchmarks, along with a more challenging yet underexplored task in the context of object detection. Our effort facilitates future research to evaluate OOD detection in a real-world setting. 
\end{itemize}

\section{Problem Setup}
We start by formulating the problem of OOD detection in the setting of object detection. Our framework can be easily generalized to image classification when the bounding box is the entire image (see Section~\ref{sec:cls}). Most previous formulations of OOD detection treat entire images as anomalies, which can lead to ambiguity shown in Figure~\ref{fig::GMM_toy}. In particular, natural
images are composed
of numerous objects and components. Knowing which regions of
an image are anomalous could allow for safer handling of
unfamiliar objects. This setting is more realistic in practice, yet also more challenging as it requires reasoning OOD uncertainty at the fine-grained object level. 

Specifically, we denote the input and label space by $\mathcal{X}=\mathbb{R}^d$ and $\mathcal{Y}=\{1,2,...,K\}$, respectively. Let $\*x \in \mathcal{X}$ be the input image, $\*b \in \mathbb{R}^4$ be the bounding box coordinates associated with object instances in the image, and $y \in \mathcal{Y} $ be the semantic label for $K$-way classification.  An object detection model is trained on in-distribution data $\mathcal{D}=\left\{\left(\*x_{i},\mathbf{b}_i, {y}_{i}\right)\right\}_{i=1}^{N}$
drawn from an unknown joint distribution $mathcal{P}$. We use neural networks with parameters $\theta$ to model the bounding box regression $p_\theta(\*b |\*x)$ and the classification $p_\theta(y|\*x, \*b)$.

The OOD detection can be formulated as a binary classification problem, which distinguishes between the in- vs. out-of-distribution objects. Let $P_{\mathcal{X}}$ denote the marginal probability distribution on $\mathcal{X}$. Given a test input $\*x^*\sim P_\mathcal{X}$, as well as an object 
instance $\*b^*$ predicted by the object detector, the goal is to predict $p_\theta(g \vert \*x^*, \*b^*)$. We use $g=1$ to indicate a detected object being in-distribution, and $g=0$ being out-of-distribution, with semantics outside the support of $\mathcal{Y}$. 

\begin{figure}[t]
    \centering
    \includegraphics[width=1.0\textwidth]{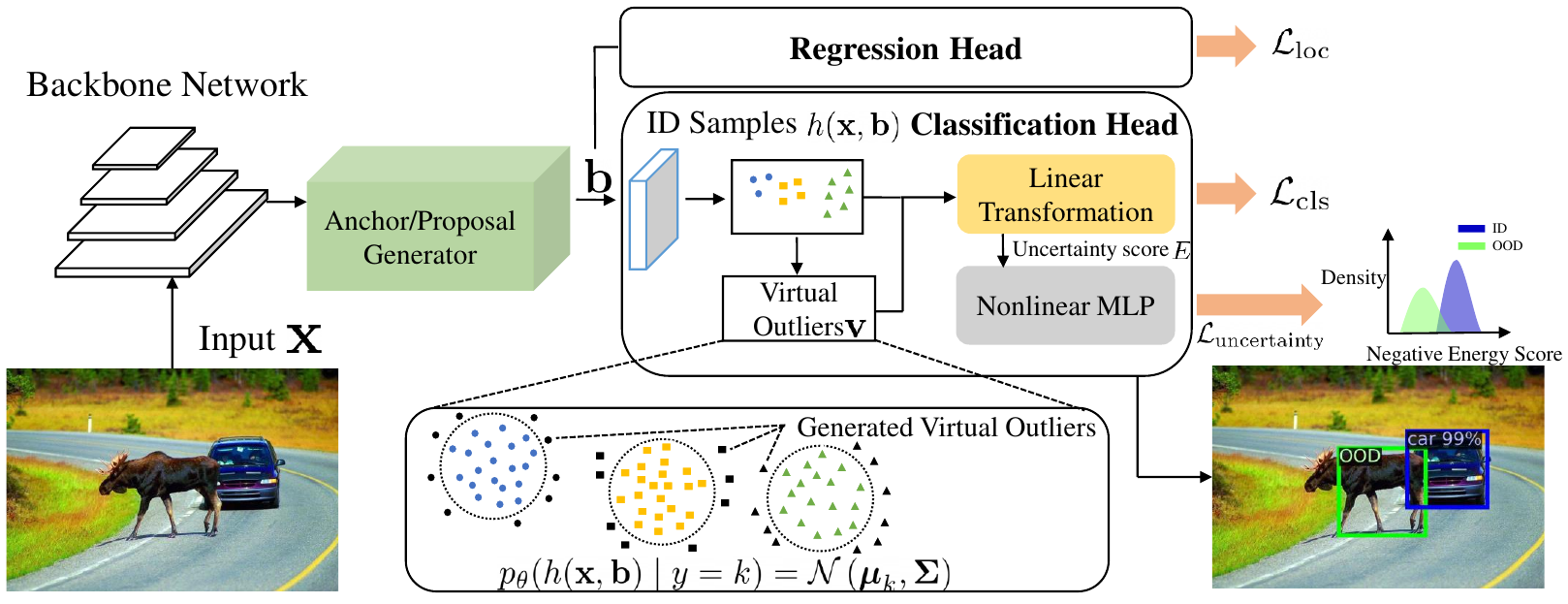}
    \caption[The framework of {VOS}]{ The framework of \texttt{VOS}. We model the feature representation of ID objects as class-conditional Gaussians, and sample virtual outliers $\mathbf{v}$ from the low-likelihood region. The virtual outliers, along with the ID objects, are used to produce the uncertainty loss for regularization. The uncertainty estimation branch ($\mathcal{L}_{\mathrm{uncertainty}}$) is jointly trained with the object detection loss ($\mathcal{L}_{\mathrm{loc}},\mathcal{L}_{\mathrm{cls}}$).}
    \label{fig:overview}
\end{figure}

\section{Method}
\label{sec:method}
Our novel unknown-aware learning framework is illustrated in Figure~\ref{fig:overview}.
Our framework encompasses three novel components and addresses the following questions: (1) how to synthesize the virtual outliers (Section~\ref{sec:gda}), (2) how to leverage the synthesized outliers for effective model regularization (Section~\ref{sec:training}), and (3) how to perform OOD detection during inference time (Section~\ref{sec:inference})?

\subsection{VOS: Virtual Outlier Synthesis}
\label{sec:gda}

Our framework \texttt{VOS} generates {virtual} outliers for model regularization, without relying on external data. While a straightforward idea is to train generative models such as GANs~\citep{goodfellow2014generative, lee2018training}, synthesizing images in the high-dimensional \emph{pixel space} can be difficult to optimize. 
Instead, our key idea is to synthesize virtual outliers in the \emph{feature space}, which is more tractable given lower dimensionality. 
Moreover, our method is based on a discriminatively trained classifier in the object detector, which circumvents the difficult optimization process in training generative models. 

Specifically, we assume the feature representation of object instances forms a class-conditional multivariate Gaussian distribution (see Figure~\ref{fig:umap}): $$p_\theta(h(\*x,\*{b}) | y=k)=\mathcal{N}(\boldsymbol\mu_{k}, \boldsymbol{\Sigma}),$$
where $\boldsymbol\mu_k$ is the Gaussian mean of class $k \in\{1,2, . ., \mathrm{K}\}$, $\boldsymbol{\Sigma}$ is the tied covariance matrix, and $h(\*x,\*{b}) \in \mathbb{R}^m$ is the latent representation of an object instance $(\*x,\*b)$.  To extract the latent representation, we use the penultimate layer of the neural network. The dimensionality $m$ is significantly smaller than the input dimension $d$.

\begin{wrapfigure}{r}{0.25\textwidth}
  \begin{center}
  \vspace{-0.9cm}
   {\includegraphics[width=\linewidth]{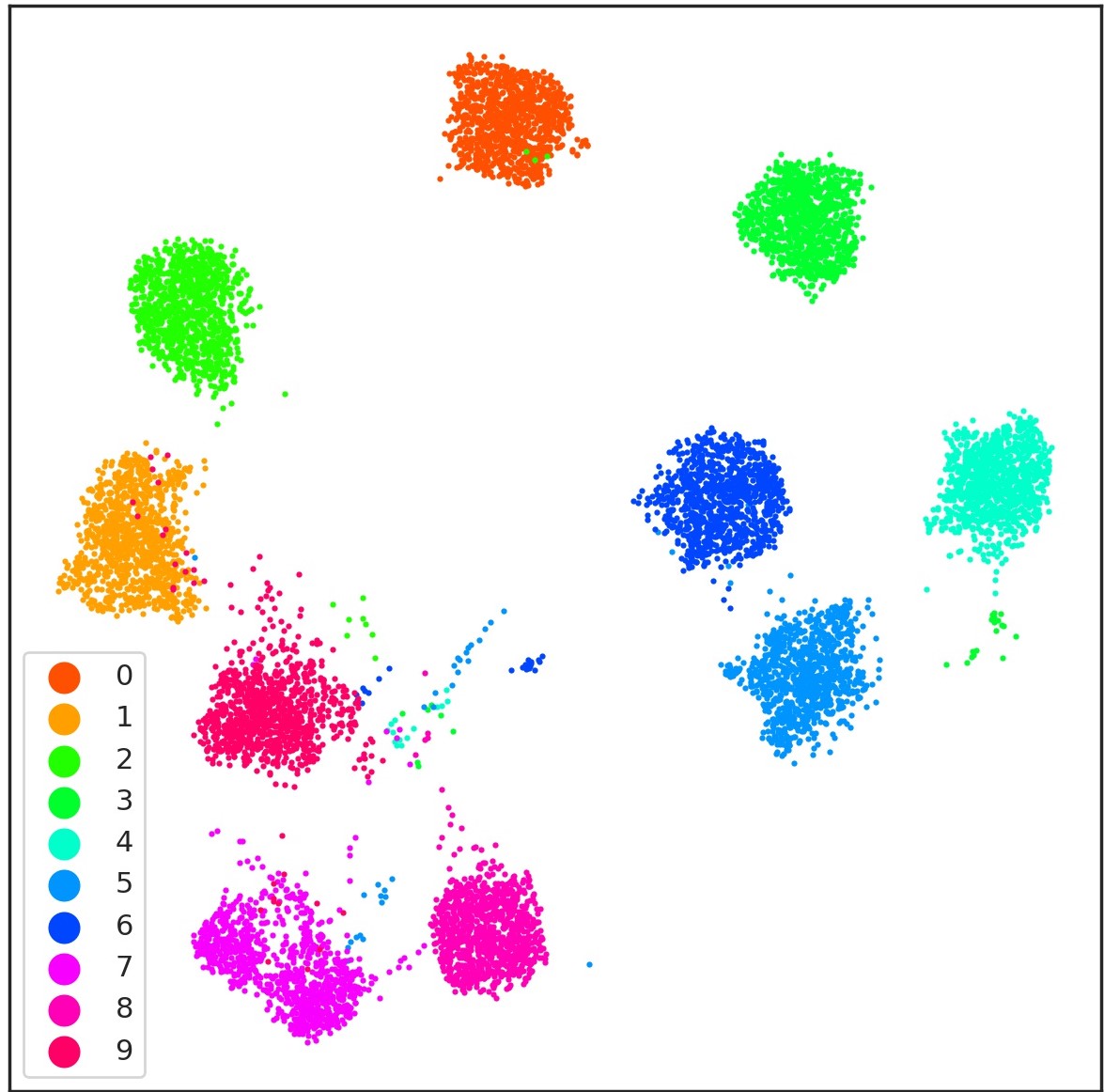}}
  \end{center}
  \vspace{-1em}
  \caption[UMAP visualization of feature embeddings of PASCAL-VOC]{ \small UMAP visualization of feature embeddings of PASCAL-VOC (on a subset of 10 classes).}
     \vspace{-3.5cm}
  \label{fig:umap}
\end{wrapfigure}

To estimate the parameters of the class-conditional Gaussian, we compute empirical class mean $\widehat{\boldsymbol{\mu}}_{k}$ and covariance $\widehat{\boldsymbol{\Sigma}}$ of  training samples $\left\{\left(\*x_{i},\mathbf{b}_i, {y}_{i}\right)\right\}_{i=1}^{N}$:

\begin{align}
    \widehat{\boldsymbol\mu}_{k}&=\frac{1}{N_{k}} \sum_{i: y_{i}=k} h(\*x_i, \*{b}_i) \\
    \widehat{\boldsymbol{\Sigma}}&=\frac{1}{N} \sum_{k} \sum_{i: y_{i}=k}\left(h(\*x_i, \*{b}_i)-\widehat{\boldsymbol\mu}_{k}\right)\left(h(\*x_i,\*{b}_i)-\widehat{\boldsymbol\mu}_{k}\right)^{\top},
    \label{eq:mean_cov}
\end{align}

where $N_k$ is the number of objects in class $k$, and $N$ is the total number of objects. We use online estimation for efficient training, where we maintain a class-conditional queue with $|Q_k|$ object instances from each class.  In each iteration, {we enqueue the embeddings of objects to their corresponding class-conditional queues, and dequeue the same number of object embeddings.}

\textbf{Sampling from the feature representation space.} We propose sampling the virtual outliers from the feature representation space, using the multivariate distributions estimated above. Ideally, these virtual outliers should help estimate a more compact decision boundary between ID and OOD data.

To achieve this, we propose sampling the virtual outliers $\mathcal{V}_k$ from the $\epsilon$-likelihood region of the estimated class-conditional distribution:
\begin{align}
  \mathcal{V}_k= \{ \*v_k &\vert \frac{1}{(2 \pi)^{m / 2}|\widehat{\*{\boldsymbol{\Sigma}}}|^{1 / 2}} \exp \left(-\frac{1}{2}(\*v_k-\widehat{\boldsymbol\mu}_k)^{\top} \widehat{\boldsymbol{\Sigma}}^{-1}(\*v_k-\widehat{\boldsymbol\mu}_k)\right) < \epsilon\},
   \label{eq:virtual}
\end{align}
where $\*v_k \sim \mathcal{N}(\widehat{\boldsymbol\mu}_k,\widehat{\boldsymbol{\Sigma}})$ denotes the sampled virtual outliers for class $k$, which are in the sublevel set based on the likelihood. $\epsilon$ is sufficiently small so that the sampled outliers are near class boundary. 

\textbf{Classification outputs for virtual outliers.} For a given sampled virtual outlier $\*v\in \mathbb{R}^m$, the output of the classification branch can be derived through a linear transformation:
\begin{equation}
    f(\*v; \theta)= W_\text{cls}^\top\*v , 
\end{equation}
where $W_\text{cls}\in \mathbb{R}^{m\times K}$ is the weight of the last fully connected layer. We proceed with describing how to regularize the output of virtual outliers for improved OOD detection. 

\subsection{Unknown-aware Training Objective}
\label{sec:training}
We now introduce a new training objective for unknown-aware learning, leveraging the virtual outliers in Section~\ref{sec:gda}. The key idea is to perform visual recognition task while regularizing the model to produce a low OOD score for ID data, and a high OOD score for the synthesized outlier. 

\vspace{-0.3cm}
\paragraph{Uncertainty regularization for classification.} For simplicity, we first describe the regularization in the multi-class classification setting. The regularization loss should ideally optimize for the separability between the ID vs. OOD data under some function that captures the data density. However, directly estimating $\log p(\*x)$ can be computationally intractable as it requires sampling from the entire space $\mathcal{X}$. We note that the log partition function $E(\*x;\theta) := - \log \sum_{k=1}^K e^{f_k(\*x;\theta)}$ is proportional to  $\log p(\*x)$ with some unknown factor, which can be seen from the following:
$$
p(y | \*x)  = \frac{p(\*x,y)}{p(\*x)} = \frac{e^{f_y(\*x;\theta)}}{\sum_{k=1}^K e^{f_k(\*x;\theta)}},
$$
where $f_y(\*x; {\theta})$ denotes the $y$-th element of logit output corresponding to the label $y$. The negative log partition function is also known as the {free energy}, which was shown to be an effective uncertainty measurement for OOD detection~\citep{liu2020energy}. 

Our idea is to explicitly perform a level-set estimation based on the energy function (threshold  at 0), where the ID data has negative energy values and the synthesized outlier has positive energy:
\begin{align*}
        \mathcal{L}_\text{uncertainty} = \mathbb{E}_{\*v\sim \mathcal{V}}~~\mathds{1} \{E(\*v;\theta) > 0\} + \mathbb{E}_{\*x\sim \mathcal{D}}~~\mathds{1}\{E(\*x;\theta) \le 0\}
\end{align*}
This is a simpler objective  than estimating density. Since the $0/1$ loss is intractable,  we replace it with the binary sigmoid loss, a smooth approximation of the $0/1$ loss, yielding the following:
\begin{equation}
  \mathcal{L}_\text{uncertainty}=\mathbb{E}_{\*v\sim \mathcal{V}} \left[-\log \frac{1}{1+\exp ^{- \phi(E(\*v;\theta))}} \right]+\mathbb{E}_{{\*x \sim \mathcal{D}}} \left[-\log \frac{\exp ^{- \phi(E(\*x;\theta))}}{1+\exp ^{-  \phi(E(\*x;\theta))}} \right].
  \label{eq:reg_loss}
\end{equation}
Here $\phi(\cdot)$ is a nonlinear MLP function, which allows learning flexible energy surface. The learning process shapes the uncertainty surface, which predicts high probability for ID data and low probability for virtual outliers $\*v$. ~\citet{liu2020energy} employed energy for model uncertainty regularization, however, the loss function is based on the squared hinge loss and requires tuning two margin hyperparameters. In contrast, our uncertainty regularization loss is completely \emph{hyperparameter-free} and is much easier to use in practice. Moreover, \texttt{VOS} produces probabilistic score for OOD detection, whereas \cite{liu2020energy} relies on non-probabilistic energy score. 

\vspace{-0.3cm}
\paragraph{Object-level energy score.} In case of object detection, we can replace the image-level energy with object-level energy score. For ID object $(\*x, \*b)$, the energy is defined as:
\vspace{-0.4em}
\begin{equation}
E(\*x,\*b;\theta) = -\log \sum_{k=1}^K w_k\cdot \exp^{f_k((\*x,\*{b});\theta)},
\label{eq:energy}
\vspace{-0.4em}
\end{equation}
where $f_k((\*x,\*{b});\theta)=W^\top_\text{cls}h(\*x,\*b)$ is the logit output for class $k$ in the classification branch. The energy score for the virtual outlier can be defined in a similar way as above.  
In particular, we will show in Section~\ref{sec:experiment} that a learnable $\mathbf{w}$ is more flexible than a constant $\*w$, given the inherent class imbalance in object detection datasets. {Additional analysis on $w_k$ is in Appendix~\ref{sec:app_visual_weight}}.

\paragraph{Overall training objective.} In the case of object detection, the overall training objective combines the standard object detection loss, along with a regularization loss in terms of uncertainty:
\begin{equation}
    \min _{\theta} \mathbb{E}_{(\*x, \mathbf{b}, {y}) \sim \mathcal{D}}~~\left[\mathcal{L}_\text{cls}+\mathcal{L}_\text{loc}\right]+\beta \cdot \mathcal{L}_\text{uncertainty},
    \label{eq:all_loss}
\end{equation}
where $\beta$ is the weight of the uncertainty regularization. $\mathcal{L}_\text{cls}$ and $\mathcal{L}_\text{loc}$ are losses for classification and bounding box regression, respectively. This can be simplified to classification task without $\mathcal{L}_\text{loc}$. We provide ablation studies in Section~\ref{sec:exp_baseline} demonstrating the superiority of our loss function. 

\begin{algorithm}[H]
\SetAlgoLined
\textbf{Input:} ID data $\mathcal{D}=\left\{\left(\*x_{i}, \mathbf{b}_i,{y}_{i}\right)\right\}_{i=1}^{N}$, randomly initialized detector with parameter $\theta$, queue size $|Q_k|$ for Gaussian density estimation, weight for uncertainty regularization $\beta$, and $\epsilon$.\\
\textbf{Output:} Object detector with parameter $\theta^{*}$, and OOD detector $G$.\\
\While{train}{
Update the ID queue $Q_k$ with the training objects $\{(\*x,{\mathbf{b}},y)\}$.\;

  Estimate the multivariate distributions based on ID training objects using Equation~1 and~\ref{eq:mean_cov}.\;
  
  Sample virtual outliers $\*v$ using Equation~\ref{eq:virtual}.
  
  Calculate the regularization loss using Equation~\ref{eq:reg_loss}, update the parameters $\theta$ based on Equation~\ref{eq:all_loss}.
  }
 
  \While{eval}{
  Calculate the OOD uncertainty score using Equation~\ref{eq:ood_uncertainty}.\;
  
  Perform thresholding comparison using Equation~\ref{eq:ood_detection}.

 }
 \caption{VOS: Virtual Outlier Synthesis for OOD detection}
 \label{alg:algo}
\end{algorithm}

\subsection{Inference-time OOD Detection}
\label{sec:inference}
During inference, we use the output of the logistic regression uncertainty branch for OOD detection. In particular, given a test input $\*x^*$, the object detector produces a bounding box prediction $\*b^*$. The OOD uncertainty score for the predicted object $(\*x^*, \*b^*)$ is given by:
\begin{align}
    p_\theta(g \mid \*x^*, \*b^*) = \frac{\exp^{-  \phi(E(\*x^*,\*b^*))}}{1+\exp^{-  \phi(E(\*x^*,\*b^*))}}.
    \label{eq:ood_uncertainty}
    \vspace{-0.5em}
\end{align}
For OOD detection, one can exercise the thresholding mechanism to  distinguish
between ID and OOD objects:
\begin{equation}
    G(\*x^*,\*b^*)=\left\{\begin{array}{ll}
1 & \text { if }p_\theta(g \mid \*x^*, \*b^*)\geq \gamma, \\
0 & \text { if }p_\theta(g \mid \*x^*, \*b^*) <\gamma.
\end{array}\right.
\label{eq:ood_detection}
\end{equation}
The threshold $\gamma$ is typically chosen so that a high
fraction of ID data (e.g., 95\%) is correctly classified. Our framework \texttt{VOS} is summarized in Algorithm~\ref{alg:algo}.


\section{Experimental  Results}
\label{sec:experiment}
In this section, we present empirical evidence to validate
the effectiveness of \texttt{VOS} on several real-world tasks, including both object detection~(Section~\ref{subsec:obj}) and image classification (Section~\ref{subsec:img}).
\subsection{Evaluation on Object Detection}
\label{subsec:obj}

\begin{table}[t]
\centering
\small
\scalebox{0.64}{
\begin{tabular}{lllll}
    \toprule
    {In-distribution $\mathcal{D}$}  & 
     
   {\textbf{Method }}  &\textbf{FPR95} $\downarrow$  & \textbf{AUROC} $\uparrow$ & \textbf{ mAP (ID)}$\uparrow$  \\
   \hline
    && \multicolumn{2}{c}{OOD: MS-COCO / OpenImages}\\
     \cline{3-4}
    \multirow{10}{0.2\linewidth}{{\textbf{ PASCAL-VOC} }} 
    & MSP \citep{hendrycks2016baseline}
   &  70.99 / 73.13  & 83.45 /  81.91& 48.7  \\
     & ODIN \citep{liang2018enhancing}
    & 59.82 / 63.14 &  82.20 / 82.59 & 48.7 \\
     & Mahalanobis 
    \citep{lee2018simple}

    & 96.46 / 96.27 & 59.25 / 57.42 & 48.7 \\
          & Energy score~\citep{liu2020energy}
    & 56.89 / 58.69  & 83.69 / 82.98& 48.7\\ 
        & Gram matrices~\citep{DBLP:conf/icml/SastryO20}&62.75 / 67.42& 79.88 /	77.62& 48.7 \\

     & Generalized ODIN~\citep{hsu2020generalized} & 59.57 / 70.28& 83.12 / 79.23&48.1
     \\
  & CSI~\citep{tack2020csi}& 59.91 / 57.41  & 81.83 / 82.95&48.1 \\
    & GAN-synthesis~\citep{lee2018training}&	60.93  / 59.97 &  83.67 / 82.67	& 48.5\\

      &\cellcolor{Gray}\textbf{VOS}-ResNet50  (ours) &\cellcolor{Gray}\textbf{47.53}$\pm$2.9 / 51.33$\pm$1.6&\cellcolor{Gray}88.70$\pm$1.2 / 85.23$\pm$ 0.6 & \cellcolor{Gray}48.9$\pm$0.2 \\

           & \cellcolor{Gray}\textbf{VOS}-RegX4.0 (ours)  & \cellcolor{Gray}47.77$\pm$1.1 / \textbf{48.33}$\pm$1.6 & \cellcolor{Gray}\textbf{89.00}$\pm$0.4 / \textbf{87.59}$\pm$0.2 & \cellcolor{Gray}\textbf{51.6}$\pm$0.1  
     \\
     \midrule
 \multirow{10}{0.2\linewidth}{\textbf{Berkeley DeepDrive-100k} } 
     & MSP \citep{hendrycks2016baseline}
    & 80.94 / 79.04&  75.87 / 77.38	& 31.2 \\
     & ODIN \citep{liang2018enhancing}
    & 62.85 / 58.92 & 74.44 / 76.61 & 31.2 \\
     & Mahalanobis 
    \citep{lee2018simple}
  
   	   & 57.66 / 60.16 & 84.92 / 86.88 & 31.2 \\
   	              & Energy score~\citep{liu2020energy}
    & 60.06 / 54.97& 77.48 / 79.60& 31.2\\ 
   	    & Gram matrices~\citep{DBLP:conf/icml/SastryO20}&	60.93 / 77.55 & 74.93 / 59.38	& 31.2  \\
   
      &Generalized ODIN~\citep{hsu2020generalized} &  57.27 / 50.17 & 85.22 / 87.18 &31.8
     \\
     &  CSI~\citep{tack2020csi}& 47.10 / 37.06 & 84.09 / 87.99& 30.6 \\
    & GAN-synthesis~\citep{lee2018training}& 57.03 / 50.61& 78.82 / 81.25 & 31.4 \\

    &\cellcolor{Gray}\textbf{VOS}-ResNet50   (ours) 
    &\cellcolor{Gray}44.27$\pm$2.0 / 35.54$\pm$1.7&\cellcolor{Gray}86.87$\pm$2.1 / 88.52$\pm$1.3	& \cellcolor{Gray}31.3$\pm$0.0  \\
 
      &\cellcolor{Gray}\textbf{VOS}-RegX4.0 (ours)  & \cellcolor{Gray}\textbf{36.61}$\pm$0.9 / \textbf{27.24}$\pm$1.3&\cellcolor{Gray}\textbf{89.08}$\pm$0.6 / \textbf{92.13}$\pm$0.5& \cellcolor{Gray}\textbf{32.5}$\pm$0.1
    \\
        \bottomrule
\end{tabular}}
        \caption[Main results for VOS]{ \textbf{Main results.} Comparison with competitive out-of-distribution detection methods. All baseline methods are based on a model trained on \textbf{ID data only} using ResNet-50 as the backbone, without using any real outlier data.  $\uparrow$ indicates larger values are better and $\downarrow$ indicates smaller values are better. All values are percentages. \textbf{Bold} numbers are superior results.  We report standard deviations estimated across 3 runs. RegX4.0 denotes the backbone of RegNetX-4.0GF~\citep{DBLP:conf/cvpr/RadosavovicKGHD20} for the object detector.}
        \label{tab:baseline}
\end{table}

\vspace{-0.1cm}
\textbf{Experimental details.} We use \textbf{PASCAL VOC}\footnotemark[1]~\citep{DBLP:journals/ijcv/EveringhamGWWZ10} and Berkeley DeepDrive (\textbf{BDD-100k}\footnotemark[2])~\citep{DBLP:conf/cvpr/YuCWXCLMD20} datasets as the ID training data. For both tasks, we evaluate on two OOD datasets that contain subset of images from: \textbf{MS-COCO}~\citep{lin2014microsoft} and \textbf{OpenImages} (validation set)~\citep{kuznetsova2020open}. We manually examine the OOD images to ensure they do not contain ID category. We have open-sourced our benchmark data that allows the community to easily evaluate future methods on object-level OOD detection. 

We use the Detectron2 library~\citep{Detectron2018} and train on two backbone architectures: ResNet-50~\citep{he2016identity} and RegNetX-4.0GF~\citep{DBLP:conf/cvpr/RadosavovicKGHD20}. We employ a two-layer MLP with a ReLU nonlinearity for $\phi$ in Equation~\ref{eq:reg_loss}, with hidden layer dimension of 512. For each in-distribution class, we use 1,000 samples to estimate the class-conditional Gaussians. Since the threshold $\epsilon$ can be infinitesimally small, we instead choose $\epsilon$ based on the $t$-th smallest likelihood in a pool of 10,000 samples (per-class), generated from the class-conditional Gaussian distribution. A larger $t$ corresponds to a larger threshold $\epsilon$. As shown in Table~\ref{tab:ablation_appendix2_detection}, a smaller  $t$ yields good performance. We set $t=1$ for all our experiments. \emph{Extensive details on the datasets are described in Appendix~\ref{sec:dataset}, along with a comprehensive sensitivity analysis of each hyperparameter (including the queue size $|Q_k|$, coefficient $\beta$, and threshold $\epsilon$) in Appendix~\ref{sec:ablation}}.
\footnotetext[1]{\color{black}PASCAL-VOC consists of the following ID labels: {Person, Car, Bicycle, Boat, Bus, Motorbike, Train, Airplane, Chair, Bottle, Dining Table, Potted Plant, TV, Sofa, Bird, Cat, Cow, Dog, ~Horse, Sheep}.}
\footnotetext[2]{\color{black}BDD-100k consists of ID labels: {Pedestrian, Rider, Car, Truck, Bus, Train, Motorcycle, Bicycle, Traffic light, Traffic sign}.}

\textbf{Metrics.} For evaluating the OOD detection performance, we report: (1) the false positive rate (FPR95) of
OOD samples when the true positive rate of ID
samples is at 95\%; (2) the area under the receiver operating characteristic curve (AUROC). For evaluating the object detection performance on the ID task, we report the common metric of mAP.

\label{sec:exp_baseline}

\textbf{VOS outperforms existing approaches.} In Table~\ref{tab:baseline}, we compare \texttt{VOS} with competitive OOD detection methods in literature. For a fair comparison, all the methods only use ID data without using auxiliary outlier dataset.  Our proposed method, \texttt{VOS}, outperforms competitive baselines, including Maximum Softmax Probability~\citep{hendrycks2016baseline}, ODIN~\citep{liang2018enhancing},  energy score~\citep{liu2020energy}, Mahalanobis distance~\citep{lee2018simple}, Generalized ODIN~\citep{hsu2020generalized}, CSI~\citep{tack2020csi}
and Gram matrices~\citep{DBLP:conf/icml/SastryO20}.
These approaches rely on a classification model trained primarily for the ID classification task, and can be naturally extended to the object detection model due to the existence of a classification head. The comparison precisely highlights the benefits of incorporating synthesized outliers for model regularization.

Closest to our work is the GAN-based approach for synthesizing outliers~\citep{lee2018training}. Compare to GAN-synthesis, \texttt{VOS}  improves the OOD detection performance (FPR95) by \textbf{12.76}\% on BDD-100k and \textbf{13.40}\% on Pascal VOC (COCO as OOD). Moreover, we show in Table~\ref{tab:baseline} that \texttt{VOS} achieves stronger OOD detection performance while preserving a high accuracy on the original in-distribution task (measured by mAP). This is in contrast with CSI, which displays degradation, with mAP decreased by 0.7\% on BDD-100k.~{Details of reproducing baselines are in Appendix~\ref{sec:reproduce_baseline}}. 

\textbf{Ablation on outlier synthesis approaches.} We compare \texttt{VOS} with different synthesis approaches in Table~\ref{tab:synthesis}. Specifically, we consider three types of synthesis approach:  (i$^\diamond$) synthesizing outliers in the pixel space, (ii$^\natural$) using noise as outliers, and (iii$^\clubsuit$) using negative proposals from RPN as outliers. For type I, we consider GAN-based~\citep{lee2018training} and mixup~\citep{DBLP:conf/iclr/ZhangCDL18} methods. The outputs of the classification branch for outliers are forced to be closer to a uniform distribution. For mixup, we consider two different beta distributions $\operatorname{Beta}(0.4)$ and $\operatorname{Beta}(1)$, and interpolate ID objects in the pixel space. For Type II, we use noise perturbation to create virtual outliers. We consider adding fixed Gaussian noise to the ID features, adding trainable noise to the ID features where the noise is trained to push the outliers away from ID features, and using fixed Gaussian noise as outliers. Lastly, for type III, we directly use the negative proposals in the ROI head as the outliers for Equation~\ref{eq:reg_loss}, similar to~\cite{DBLP:journals/corr/abs-2103-02603}. We consider three variants: randomly sampling $n$ negative proposals ($n$ is the number of positive proposals), sampling $n$ negative proposals with a larger probability, and using all the negative proposals. \textcolor{black}{All methods are trained under the same setup, with PASCAL-VOC as in-distribution data and ResNet-50 as the backbone. The loss function is the same as Equation~\ref{eq:all_loss} for all variants, with the only difference being the synthesis method.}

\begin{table}[t]
\centering
\scalebox{0.7}{
\begin{tabular}{llcc}
    \toprule
     &  \multirow{2}{0.08\linewidth}{\textbf{Method }}   & \textbf{AUROC} $\uparrow$   & \textbf{mAP}  $\uparrow$  \\
    \hline
 \multirow{3}{0.2\linewidth}{\textbf{Image synthesis} } 
 &  $^\diamond$GAN~\citep{lee2018training} & 83.67& 48.5 \\
 &  $^\diamond$Mixup~\citep{DBLP:conf/iclr/ZhangCDL18} (mixing ratio $0.4$) & 61.23& 44.3 \\
 &  $^\diamond$Mixup~\citep{DBLP:conf/iclr/ZhangCDL18} (mixing ratio $1$) &63.99& 46.9 \\
 \midrule
 \multirow{3}{0.2\linewidth}{\textbf{Noise as outliers} } 
 & $^\natural$Additive Gaussian noise to ID features & 68.02 & 48.7\\
 & $^\natural$Trainable noise added to the ID features &  66.67& 48.6\\
 & $^\natural$Gaussian noise &  85.98& 48.5\\
      \midrule
   \multirow{3}{0.2\linewidth}{{\textbf{ Negative proposals}  }} 
      & $^\clubsuit$All negative proposals &63.45 & 48.1 \\
  & $^\clubsuit$Random negative proposals &  66.03& 48.5 \\
  & $^\clubsuit$Proposals with large background prob~\citep{DBLP:journals/corr/abs-2103-02603} & 77.26& 48.5 \\
 \midrule 
    & \cellcolor{Gray} \textbf{VOS}   (ours) 
    & \cellcolor{Gray}\textbf{88.70}	&\cellcolor{Gray}48.9\\
        \bottomrule
\end{tabular}}
        \caption[Ablation on outlier synthesis approaches]{ {Ablation on outlier synthesis approaches (on backbone of ResNet-50, COCO is the OOD data). } }
        \label{tab:synthesis}
\end{table}

The results are summarized in Table~\ref{tab:synthesis}, where \texttt{VOS} outperforms alternative synthesis approaches both in the feature space ($\clubsuit$, $\natural$) or the pixel space ($\diamond$). Generating outliers in the pixel space ($\diamond$) is either unstable (GAN) or harmful for the object detection performance (mixup). Introducing noise ($\natural$), especially using Gaussian noise as outliers is promising. However, Gaussian noise outliers are relatively simple, and may not effectively regularize the decision boundary between ID and OOD as \texttt{VOS} does. Exploiting the negative proposals ($\clubsuit$) is not effective, because they are distributionally close to the ID data. 

\textbf{Ablation on the uncertainty loss.} We perform ablation on several variants of \texttt{VOS}, trained with different uncertainty loss $\mathcal{L}_\text{uncertainty}$. Particularly, we consider: (1) using the squared hinge loss for regularization as in~\citeauthor{liu2020energy}, (2) using constant weight $\*w=[1,1,...,1]^\top$ for energy score in Equation~\ref{eq:energy}, and (3) classifying the virtual outliers as an additional $K+1$ class in the classification branch. The performance comparison is summarized in Table~\ref{tab:ablation_loss}.  Compared to the hinge loss, our proposed logistic loss reduces the FPR95 by 10.02\% on BDD-100k. While the squared hinge loss in ~\citeauthor{liu2020energy} requires tuning the hyperparameters, our uncertainty loss is completely \emph{hyperparameter free}. In addition, we find that a learnable $\mathbf{w}$ for energy score is more desirable than a constant $\*w$, given the inherent class imbalance in object detection datasets. Finally, classifying the virtual outliers as an additional class increases the difficulty of object classification, which does not outperform either. This ablation demonstrates the superiority of the uncertainty loss employed by \texttt{VOS}.

\textbf{VOS is effective on alternative architecture.} Lastly, we demonstrate that \texttt{VOS} is effective on alternative neural network architectures. In particular, using RegNet~\citep{DBLP:conf/cvpr/RadosavovicKGHD20} as backbone yields both better ID accuracy and OOD detection performance. We also explore using intermediate layers for outlier synthesis, where we show using \texttt{VOS} on the penultimate layer is the most effective. This is expected since the feature representations are the most discriminative at deeper layers. {We provide details in Appendix~\ref{sec:intermediate}}.

\begin{table}[t]
\centering
\scalebox{0.8}{
\begin{tabular}{llrr|r}
    \toprule
     \multirow{2}{0.08\linewidth}{$\mathcal{D}$}  &  \multirow{2}{0.08\linewidth}{\textbf{Method }}   &\textbf{FPR95}  & \textbf{AUROC}  & \textbf{object detection mAP (ID)}  \\
     &  &  $\downarrow$  & $\uparrow$ & $\uparrow$ \\  
    \hline
    \multirow{4}{0.2\linewidth}{{\textbf{ PASCAL-VOC} }} 
        & VOS w/ hinge loss & 49.75 & 87.90 & 46.5\\
    & VOS w/ constant $\*w$ & 51.59 & 88.64& 48.9\\
     & VOS w/ $K+1$ class & 65.25 & 85.26  &47.0 \\
    & \cellcolor{Gray}\textbf{VOS} (ours) &  \cellcolor{Gray}\textbf{47.53}&\cellcolor{Gray}\textbf{88.70}	& \cellcolor{Gray} 48.9\\

     \midrule
 \multirow{4}{0.2\linewidth}{\textbf{Berkeley DeepDrive-100k} } 
     & VOS w/ hinge loss & 54.29&83.47& 29.5\\
    & VOS w/ constant $\*w$ & 49.25&	85.35 & 30.9\\
      & VOS w/ $K+1$ class & 52.98& 85.91&30.1 \\
     &\cellcolor{Gray}\textbf{VOS} (ours)   &  \cellcolor{Gray}\textbf{44.27}&	\cellcolor{Gray}\textbf{86.87}	&\cellcolor{Gray}  31.3\\
        \bottomrule
\end{tabular}}
        \caption[Ablation study results for VOS]{ \textbf{Ablation study.} Comparison with different regularization loss functions (on backbone of ResNet-50, COCO is the OOD data). }
        \label{tab:ablation_loss}
\end{table}

\textbf{Comparison with training on real outlier data.} \textcolor{black}{We also compare with Outlier Exposure~\citep{hendrycks2018deep} (OE). OE serves as a strong baseline since it relies on the \emph{real} outlier data.}  We train the object detector on PASCAL-VOC using the same architecture ResNet-50, and use the OE objective for the classification branch. The real outliers for OE training are sampled from the OpenImages dataset~\citep{kuznetsova2020open}. We perform careful deduplication to ensure there is no overlap between the outlier training data and PASCAL-VOC. Our method achieves OOD detection performance on COCO (AUROC: 88.70\%) that favorably matches OE (AUROC: 90.18\%), and does not require external data. 

\subsection{Evaluation on Image Classification}
\label{subsec:img}

\label{sec:cls}

Going beyond object detection, we show that \texttt{VOS} is also suitable and effective on common image classification benchmark. We use CIFAR-10~\citep{cifar} as the ID training data, with standard train/val splits. We train on WideResNet-40~\citep{zagoruyko2016wide} and DenseNet-101~\citep{huang2017densely}, where we substitute the object detection loss in Equation~\ref{eq:all_loss} with the cross-entropy loss. \begin{wraptable}{r}{0.55\linewidth}
    \centering
    
    \tabcolsep 0.04in\renewcommand\arraystretch{0.745}{}
    \vspace{0em}
    \scalebox{0.9}{
    \begin{tabular}{c|cccc}
    \toprule

\multirow{1}{*}{ \textbf{Method}} & \textbf{FPR95} $\downarrow$ & \textbf{AUROC} $\uparrow$\\
 
        \midrule
        &\multicolumn{2}{c}{WideResNet / DenseNet}\\
        \cline{2-3}
        MSP& 51.05 / 48.73	 & 90.90 / 92.46\\
        ODIN& 35.71 / 24.57 & 91.09 / 93.71	\\
         Mahalanobis& 37.08 / 36.26	 & 93.27 / 87.12\\
                          Energy& 33.01 / 27.44	& 91.88 / 94.51\\
                 Gram Matrices & 27.33 / 23.13 &93.00 / 89.83\\
        Generalized ODIN &39.94 / 26.97&92.44 / 93.76\\
        CSI &35.66 / 47.83 & 92.45 / 85.31 \\

    GAN-synthesis & 37.30 / 83.71	&	89.60 / 54.14\\
        \midrule
       \rowcolor{Gray}  \textbf{VOS} (ours) & \textbf{24.87} / \textbf{22.47} & \textbf{94.06} /	\textbf{95.33}\\

        \bottomrule
    \end{tabular}
}
      \caption[OOD detection results of {VOS} on classification models]{ OOD detection results of \texttt{VOS} and comparison with competitive baselines on two architectures: WideResNet-40 and DenseNet-101.}
    \label{tab:baseline_cls}
\end{wraptable}We evaluate on six OOD datasets: 
\texttt{Textures}~\citep{cimpoi2014describing}, \texttt{SVHN}~\citep{netzer2011reading}, \texttt{Places365}~\citep{DBLP:journals/pami/ZhouLKO018}, \texttt{LSUN-C}~\citep{DBLP:journals/corr/YuZSSX15}, \texttt{LSUN-Resize}~\citep{DBLP:journals/corr/YuZSSX15}, and
\texttt{iSUN}~\citep{DBLP:journals/corr/XuEZFKX15}.  The comparisons are shown in Table~\ref{tab:baseline_cls}, with results averaged over six test datasets. \texttt{VOS} demonstrates competitive OOD detection results on both architectures without sacrificing the ID test classification accuracy (94.84\% on pre-trained WideResNet vs. 94.68\% using \texttt{VOS}).


\subsection{Qualitative Analysis}

In Figure~\ref{fig:visual}, we visualize the prediction on several OOD images, using object detection models trained without virtual outliers (top) and with \texttt{VOS} (bottom), respectively. The in-distribution data is BDD-100k. \texttt{VOS} performs better in identifying OOD objects (in green) than a vanilla object detector, and reduces false positives among detected objects. Moreover, the confidence score of the false-positive objects of \texttt{VOS} is lower than that of the vanilla model (see the truck in the 3rd column). \emph{\textcolor{black}{Additional visualizations are in Appendix~\ref{sec:app_visual}}.}  
\begin{figure}[t]
    \centering
    \includegraphics[width=1.0\textwidth]{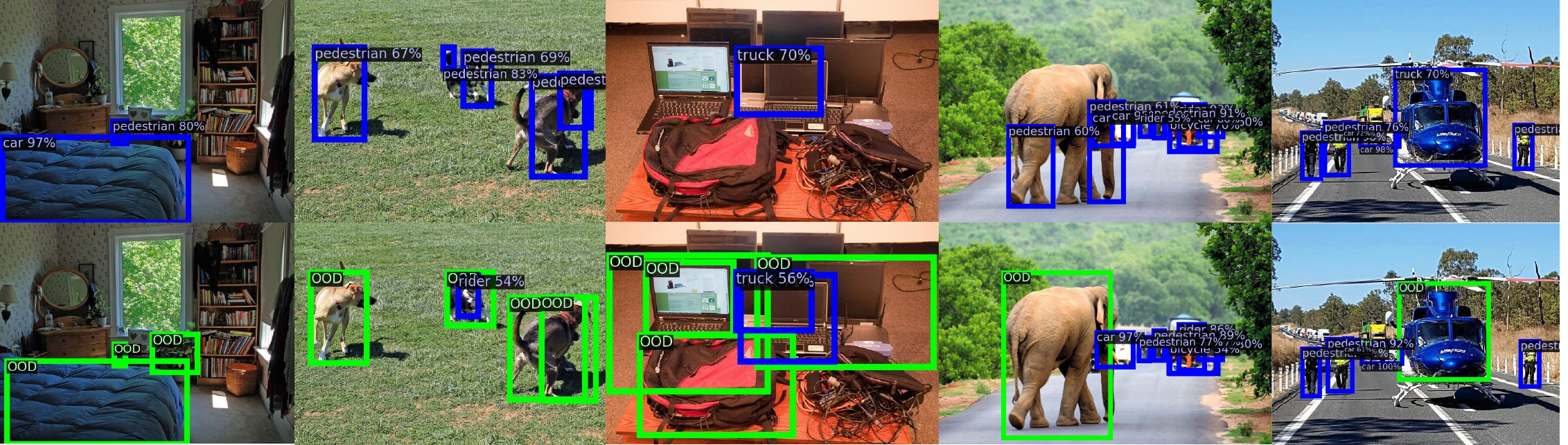}
    \caption[Visualization of detected objects on the OOD images for VOS]{ Visualization of detected objects on the OOD images (from MS-COCO) by a vanilla Faster-RCNN (\emph{top}) and \texttt{VOS} (\emph{bottom}). The in-distribution is BDD-100k dataset. \textbf{Blue}: Objects detected and classified as one of the ID classes. \textbf{Green}: OOD objects detected by \texttt{VOS}, which reduce false positives among detected objects.}
\label{fig:visual}
\end{figure}

\section{Summary}
In this chapter, we propose \texttt{VOS}, a novel unknown-aware training framework for OOD detection. Different from methods that require real outlier data, \texttt{VOS} adaptively synthesizes outliers during training by sampling virtual outliers from the low-likelihood region of the class-conditional distributions. The synthesized outliers meaningfully improve the decision boundary between the ID data and OOD data, resulting in superior OOD detection performance while preserving the performance of the ID task. \texttt{VOS} is effective and suitable for both object detection and classification tasks. We hope our work will inspire future research on unknown-aware deep learning in real-world settings.


\end{vos}

\begin{dream}
    \label{sec:cha5}
\textbf{Publication Statement.} This chapter is joint work with Yiyou Sun, Jerry Zhu and
Yixuan Li. The paper version of this chapter appeared in NeurIPS'23~\citep{du2023dream}.

\noindent \textbf{Abstract.} Utilizing auxiliary outlier datasets to regularize the machine learning model has demonstrated promise for out-of-distribution (OOD) detection and safe prediction. 
Due to the labor intensity in data collection and cleaning, automating outlier data generation has been a long-desired alternative. Despite the appeal, generating photo-realistic outliers in the high dimensional pixel space has been an open challenge for the field.  To tackle the problem, this chapter proposes a new framework \textsc{Dream-ood}, which enables imagining photo-realistic outliers by way of diffusion models, provided with only the in-distribution (ID) data and classes. Specifically, \textsc{Dream-ood} learns a  text-conditioned latent space based on ID data, and then samples outliers in the low-likelihood region via the latent, which can be decoded into images by the diffusion model. Different from prior works~\citep{du2022towards, tao2023nonparametric}, \textsc{Dream-ood} enables visualizing and understanding the imagined outliers, directly in the pixel space. We conduct comprehensive quantitative and qualitative studies to understand the efficacy of \textsc{Dream-ood}, and show that training with the samples generated by \textsc{Dream-ood} can benefit  OOD detection performance. Code
is publicly available at \url{https://github.com/deeplearning-wisc/dream-ood}.

\section{Introduction}
\label{sec:intro}

Out-of-distribution (OOD) detection is critical for deploying machine learning models in the wild,
where samples from novel classes can naturally emerge and should be flagged for caution. Concerningly, modern neural networks are
shown to produce overconfident and therefore untrustworthy
predictions for unknown OOD inputs~\citep{nguyen2015deep}. To mitigate the issue, recent works have explored training
with an auxiliary outlier dataset, where the model is regularized to learn a more conservative decision boundary around in-distribution (ID) data~\citep{hendrycks2018deep, DBLP:conf/icml/Katz-SamuelsNNL22,liu2020energy, DBLP:conf/icml/MingFL22}. These methods have demonstrated encouraging OOD detection
performance over the counterparts without auxiliary data.

Despite the promise, preparing auxiliary data can be labor-intensive and inflexible,
and necessitates careful human intervention, such as data cleaning, to ensure the auxiliary
outlier data does not overlap with the ID data. Automating outlier data generation has thus been a long-desired alternative. Despite the appeal, generating photo-realistic outliers has been extremely challenging due to the high dimensional space. Recent works including VOS and NPOS~\citep{du2022towards, tao2023nonparametric} proposed sampling outliers in the low-dimensional feature space and directly employed the latent-space outliers to regularize the model. However, these latent-space methods do not allow us to understand the outliers in a human-compatible way. Today, the field still lacks an automatic mechanism to generate high-resolution outliers  in the \emph{pixel space}. 

In this chapter, we propose a new framework \textcolor{Periwinkle}{\textsc{Dream-ood}} that enables imagining photo-realistic outliers by way of diffusion models, provided with only ID data and classes (see Figure~\ref{fig:teaser}). Harnessing the power of diffusion models for outlier imagination is non-trivial, since one cannot easily describe the exponentially many possibilities of outliers using text prompts. It can be particularly challenging to characterize informative outliers that lie on the boundary of ID data, which have been shown to be the most effective in regularizing the ID classifier and its decision boundary~\citep{DBLP:conf/icml/MingFL22}. After all, it is almost impossible to describe something in words without knowing what it looks like. 

Our framework circumvents the above challenges by: \textbf{(1)} learning compact visual representations for the ID data, conditioned on the textual latent space of the diffusion model (Section~\ref{sec:distill}), and \textbf{(2)} sampling new visual embeddings in the text-conditioned latent space, which are then decoded to pixel-space images by the diffusion model (Section~\ref{sec:synthesis}). Concretely, to learn the text-conditioned latent space, we train an image classifier to produce image embeddings that have a higher probability to be aligned with the corresponding class token embedding. The resulting feature embeddings thus form a compact and informative distribution that encodes the ID data.  Equipped with the text-conditioned latent space, we sample new embeddings from the low-likelihood region, which can be decoded into the images via the diffusion model. The rationale is if the sampled embedding is distributionally far away from the in-distribution embeddings, the generated image will have a large semantic discrepancy from the ID images and vice versa.

\begin{figure*}[t]
    \centering
    \scalebox{1.0}{
    \includegraphics[width=\linewidth]{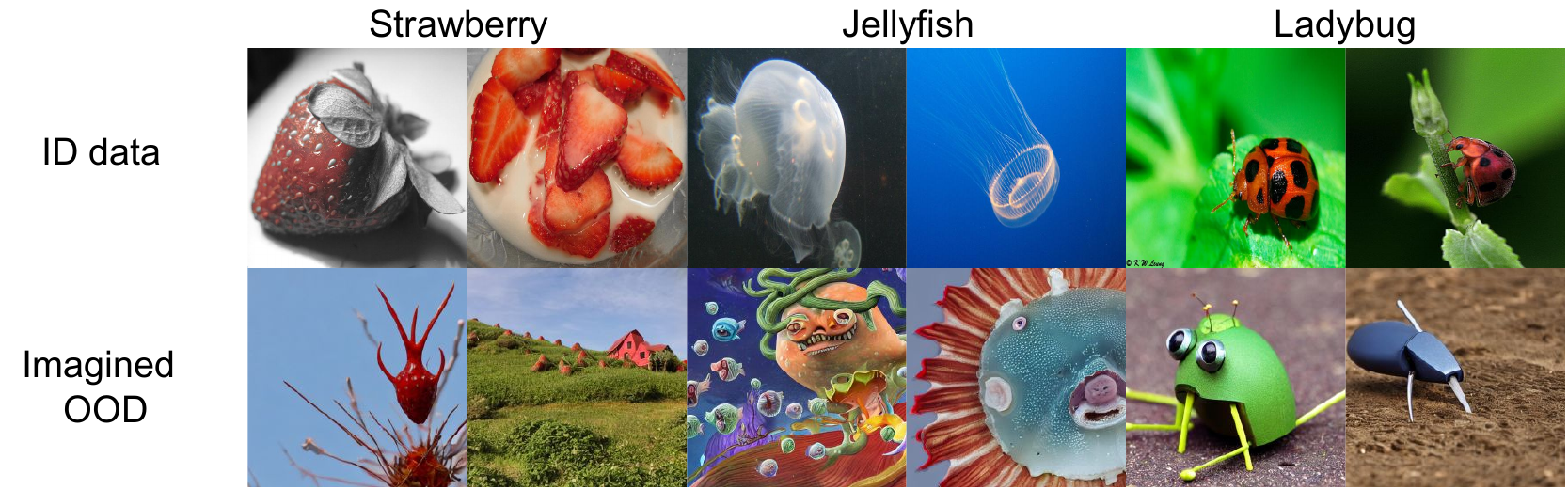}
    }
    \caption[Visualization of the images generated by Dream-OOD and the ImageNet dataset]{\small \textbf{Top}: Original ID training data in \textsc{ImageNet}~\citep{deng2009imagenet}. \textbf{Bottom}: Samples generated by our method \textsc{Dream-ood}, which deviate from the ID data.}
    \label{fig:teaser}

\end{figure*}

We demonstrate that our proposed framework creatively imagines OOD samples conditioned on a given dataset, and as a result, helps improve the OOD detection performance.  On \textsc{Imagenet} dataset, training with samples generated by \textsc{Dream-ood} improves the OOD detection on a comprehensive suite of OOD datasets. Different from~\citep{du2022towards,tao2023nonparametric}, our method allows visualizing and understanding the imagined outliers, covering a wide spectrum of near-OOD and far-OOD. Note that \textsc{Dream-ood} enables leveraging off-the-shelf diffusion models for OOD detection, rather than modifying the diffusion model (which is an actively studied area on its own~\citep{nichol2021improved}). In other words, this work's core contribution is to {leverage generative modeling to improve discriminative learning, establishing innovative connections between the diffusion model and outlier data generation.}

Our key contributions are summarized as follows:
\begin{enumerate}
    \item To the best of our knowledge, \textsc{Dream-ood} is the first to enable the generation of photo-realistic high-resolution outliers for OOD detection. \textsc{Dream-ood} establishes promising performance on common benchmarks and can benefit OOD detection. 
\item We conduct comprehensive analyses to understand the efficacy of \textsc{Dream-ood}, both quantitatively and qualitatively. The results provide insights into
outlier imagination with diffusion models. 
\item  As an \emph{extension}, we show that our synthesis method can be used to automatically generate ID samples, and as a result, improves the generalization performance of the ID task itself.

\end{enumerate}

\section{Preliminaries}

\begin{figure*}[t]
    \centering
    \scalebox{1.0}{
    \includegraphics[width=\linewidth]{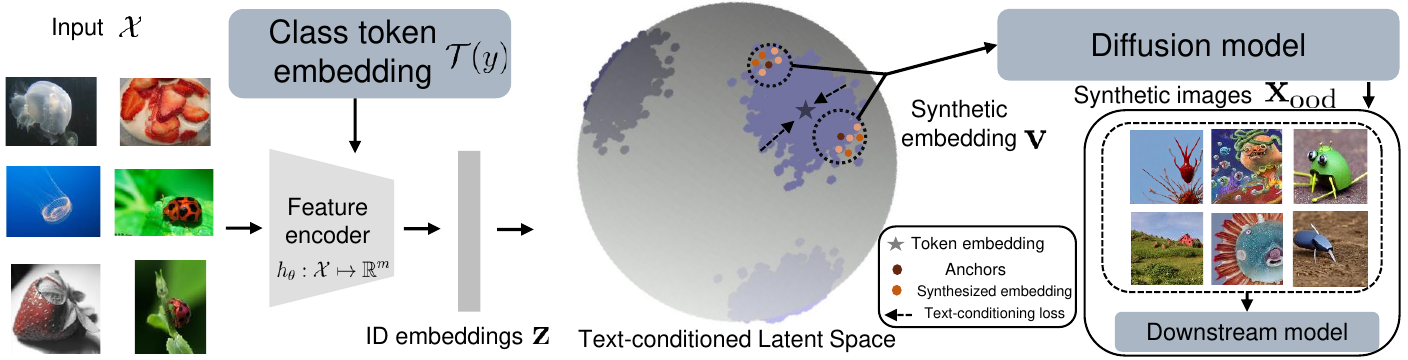}
    }
      \vspace{-1em}
    \caption[Illustration of our proposed outlier imagination  framework {Dream-OOD}]{\small \textbf{Illustration of our proposed outlier imagination  framework \textsc{Dream-ood}.} \textsc{Dream-ood} first learns a text-conditioned space to produce compact image embeddings aligned with the token embedding $\mathcal{T}(y)$ of the diffusion model. It then samples new embeddings in the latent
space, which can be decoded into pixel-space outlier images $\*x_{\text{ood}}$ by diffusion model. The newly generated samples can help improve  OOD detection. Best viewed in color. }
    \label{fig:framework}
\end{figure*}

We consider a training set $\mathcal{D}=\{(\*x_i, y_i)\}_{i=1}^n$, drawn \emph{i.i.d.} from the joint data distribution $P_{\mathcal{X}\mathcal{Y}}$. $\mathcal{X}$ denotes the input space and $\mathcal{Y}\in \{1,2,...,C\}$ denotes the label space. Let $\mathbb{P}_\text{in}$ denote the marginal distribution on $\mathcal{X}$, which is also referred to as the \emph{in-distribution}. Let $f_\theta: \mathcal{X} \mapsto \mathbb{R}^{C}$  denote a multi-class classifier, which predicts the label of an input sample with parameter $\theta$. To obtain an optimal classifier $f^*$, a standard approach is to perform empirical risk minimization (ERM)~\citep{vapnik1999nature}:
    $f^* = \text{argmin}_{f\in \mathcal{F}}  \frac{1}{n}\sum_{i=1}^n \ell (f(\*x_i),y_i)$ where $\ell$ is the loss function and $\mathcal{F}$ is the hypothesis space. 

\textbf{Out-of-distribution detection.} When deploying a machine
model in the real world, a reliable classifier should not only
accurately classify known in-distribution samples, but also identify OOD input from \emph{unknown} class $y\notin \mathcal{Y}$. This can be
achieved by having an OOD detector, in tandem with the classification model $f_{\theta}$. At its core, OOD detection can be formulated as a binary classification
problem.  At test time, the goal is to decide whether a test-time input is from  ID  or not (OOD). We denote $g_\theta:\mathcal{X} \mapsto \{\text{in}, \text{out}\}$  as the function
mapping for OOD detection.

\textbf{Denoising diffusion models} have emerged as a promising generative
modeling framework, pushing the state-of-the-art in image generation~\citep{ramesh2022hierarchical,saharia2022photorealistic}. Inspired by non-equilibrium thermodynamics, diffusion probabilistic models~\citep{sohl2015deep, song2020denoising, ho2020denoising} define a forward Gaussian Markov transition kernel of diffusion steps to gradually corrupt training data until the data distribution is transformed into a simple
noisy distribution. The model then learns to reverse this process by learning a denoising transition
kernel parameterized by a neural network.

Diffusion models can be
conditional, for example, on class labels or text descriptions~\citep{ramesh2022hierarchical,nichol2021glide,saharia2022image}. In particular, Stable Diffusion~\citep{rombach2022high} is a text-to-image model that enables synthesizing new images guided by the text prompt. The model was trained on 5 billion pairs of images and captions taken from LAION-5B~\citep{schuhmann2022laion}, a publicly available dataset derived from Common Crawl data scraped from the web. Given a class name $y$, the generation process can be mathematically denoted by:
\begin{equation}
    \*x \sim P(\*x \vert \*z_y),
    \label{eq:vanilla}
\end{equation}
where $\*z_y = \mathcal{T}(y)$ is the textual representation of label $y$ with prompting (e.g., {``A high-quality photo of a [$y$]}''). In Stable Diffusion, $\mathcal{T}(\cdot)$ is the  text encoder of the CLIP model~\citep{radford2021learning}.

\section{\textsc{Dream-ood}: Outlier Imagination with Diffusion Models}

In this chapter, we propose a novel framework that enables synthesizing photo-realistic outliers with respect to a given ID dataset (see Figure~\ref{fig:teaser}). The synthesized outliers can be useful for regularizing the ID classifier to be less confident in the OOD region.  
Recall that the vanilla diffusion generation takes as input the textual representation. While it is easy to encode the ID classes $y \in \mathcal{Y}$ into textual latent space via $\mathcal{T}(y)$, {one cannot trivially generate text prompts for outliers}. It can be particularly challenging to characterize informative outliers that lie on the boundary of ID data, which have been shown to be most effective in regularizing the ID classifier and its decision boundary~\citep{DBLP:conf/icml/MingFL22}. After all, it is almost impossible to concretely describe something in words without knowing what it looks like. 

\paragraph{Overview.} As illustrated in Figure~\ref{fig:framework}, our framework circumvents the challenge by: \textbf{(1)} learning compact visual representations for the ID data, conditioned on the textual latent space of the diffusion model (Section~\ref{sec:distill}), and \textbf{(2)} sampling new visual embeddings in the text-conditioned latent space, which are then decoded into the images by diffusion model (Section~\ref{sec:synthesis}). We demonstrate in Section~\ref{sec:experiments} that, our proposed outlier synthesis framework produces meaningful out-of-distribution samples conditioned on a given dataset, and as a result, significantly improves the OOD detection performance.

\subsection{Learning  the Text-Conditioned Latent Space}
\label{sec:feature}
\label{sec:distill}

Our key idea is to first train a classifier on ID data $\mathcal{D}$ that produces image embeddings, conditioned on the token embeddings $\mathcal{T}(y)$, with $y \in \mathcal{Y}$. 
To learn the text-conditioned visual latent space, we train the image classifier to produce image embeddings that have a higher probability of being aligned with the corresponding class token embedding, and vice versa.

\begin{wrapfigure}{r}{0.4\textwidth}
  \begin{center}
     \vspace{-0.9cm}
   {\includegraphics[width=\linewidth]{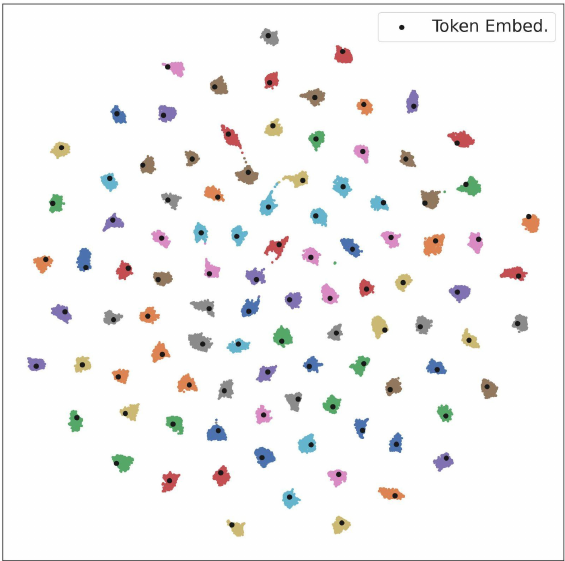}}
  \end{center}
   \vspace{-1em}
  \caption[TSNE visualization of learned feature embeddings using $\mathcal{L}$ for Dream-OOD]{\small \textbf{TSNE visualization of learned feature embeddings using $\mathcal{L}$.} Black dots indicate token embeddings, one for each class. }
     \vspace{-1em}
  \label{fig:embedding}
\end{wrapfigure}
Specifically, denote $h_\theta:\mathcal{X} \mapsto \mathbb{R}^{m} $ as a feature encoder that maps an input $\*x \in \mathcal{X}$ to the image embedding $h_{\theta}(\*x)$, and $\mathcal{T}: \mathcal{Y} \mapsto \mathbb{R}^{m}$ as the text encoder that takes a class name $y$ and outputs its token embedding $\mathcal{T}(y)$. Here $\mathcal{T}(\cdot)$ is a fixed text encoder of the diffusion model. Only the image feature encoder needs to be trained, with learnable parameters $\theta$. Mathematically, the loss function for learning the visual representations is formulated as follows:
\begin{equation}
    \mathcal{L}=  \mathbb{E}_{(\*x,y) \sim \mathcal{D}} [- \log \frac{ \exp \left( \mathcal{T}(y)^{\top}  \mathbf{z} / t \right)}{\sum_{j=1}^{C}  \exp \left( \mathcal{T}(y_j)^{\top}  \mathbf{z}/ t\right)}],
    \label{eq:cosine_loss}
\end{equation}
where $\*z = h_{\theta}(\*x)/ \| h_\theta(\*x)\|_2$ is the $L_2$-normalized image embedding, and $t$ is temperature.

\noindent \textbf{Theoretical interpretation of loss.}
Formally, our loss function directly promotes the class-conditional von Mises Fisher (vMF) distribution~\citep{du2022siren, mardia2000directional, ming2023cider}.
vMF is analogous to spherical Gaussian distributions for features with unit norms ($\|\*z\|^2=1$). The probability density function of $\mathbf{z}\in\mathbb{R}^{m}$ in class $c$ is:
\begin{equation}
    p_m(\mathbf{z} ; \boldsymbol{\mu}_{c}, \kappa)=Z_{m}(\kappa) \exp \left(\kappa \boldsymbol{\mu}_{c}^\top\*z\right),
    \label{eq:vmf_density}
\end{equation}

where $\boldsymbol{\mu}_{c}$ is the class centroid with unit norm, $\kappa\ge 0$ controls the extent of class concentration, and $Z_{m}(\kappa)$ is the normalization factor detailed in the Appendix~\ref{sec:zm}. The probability of the feature vector $\*z$ belonging to class $c$ is:
\vspace{-1em}
\begin{align}
    {P}(y=c | \mathbf{z}; \{\kappa,\boldsymbol{\mu}_{j}\}_{j=1}^C)
    & = \frac{Z_{m}\left(\kappa\right) \exp \left(\kappa \boldsymbol{\mu}_{c}^\top \mathbf{z}\right) }{\sum_{j=1}^{C} Z_{m}\left(\kappa\right) \exp \left(\kappa \boldsymbol{\mu}_{j}^\top \mathbf{z}\right) } \nonumber\\
    & = \frac{ \exp \left(\boldsymbol{\mu}_{c}^\top \mathbf{z}/t\right) }{\sum_{j=1}^{C} \exp \left( \boldsymbol{\mu}_{j}^\top \mathbf{z}/t\right) },
\end{align}
where $\kappa = {1\over t}$. Therefore, by encouraging features to be aligned with its class token embedding,  our loss function $\mathcal{L}$ (Equation~\eqref{eq:cosine_loss}) maximizes the log-likelihood of the class-conditional vMF distributions and promotes compact clusters on the hypersphere (see Figure~\ref{fig:embedding}). The highly compact representations can benefit the sampling of new embeddings, as we introduce next in Section~\ref{sec:synthesis}.

\subsection{Outlier Imagination via Text-Conditioned Latent}
\label{sec:synthesis}
Given the well-trained compact representation space that encodes the information of $\mathbb{P}_{\text{in}}$, we propose to generate outliers by sampling new embeddings in the text-conditioned latent space, and then decoding via  diffusion model. 
The rationale is that if the sampled embeddings are distributionally far away from the ID embeddings, the decoded images will have a large semantic discrepancy with the ID images and vice versa. 

Recent works~\citep{du2022towards, tao2023nonparametric} proposed sampling outlier embeddings and directly employed the latent-space outliers to regularize the model. In contrast, our method focuses on generating \emph{pixel-space} photo-realistic images, which allows us to directly inspect the generated outliers in a human-compatible way. 
Despite the appeal, generating high-resolution outliers has been extremely challenging due to the high dimensional space.  To tackle the issue, our generation procedure constitutes two steps:

\begin{enumerate}
    \item \emph{Sample OOD in the latent space}: draw new embeddings $\*v$ that are in the low-likelihood region of the text-conditioned latent space. 
    \item \emph{Image generation:} decode  $\*v$ into a pixel-space OOD image via diffusion model. 
\end{enumerate}

\begin{algorithm}[t]
\SetAlgoLined
\textbf{Input:} In-distribution training data $\mathcal{D}=\left\{\left(\*x_{i}, {y}_{i}\right)\right\}_{i=1}^{n}$, initial model parameters ${\theta}$ for learning the text-conditioned latent space, diffusion model.
\\
\textbf{Output:} Synthetic images $\*x_{\text{ood}}$.\\
\textbf{Phases:} Phase 1: Learning the Text-conditioned Latent Space. Phase 2: Outlier Imagination via Text-Conditioned Latent. \\
\While{Phase 1}{
        1. Extract token embeddings $\mathcal{T}(y)$ of the ID label  $ y\in \mathcal{Y}$.\\
    2.  Learn the text-conditioned latent representation space by Equation~\eqref{eq:cosine_loss}. 
    }
    \While{Phase 2}{
    1. Sample a set of outlier embeddings $V_i$ in the low-likelihood region of the text-conditioned latent space as in Section~\ref{sec:synthesis}. \\
    2. Decode the outlier embeddings into the pixel-space OOD images via diffusion model by Equation~\eqref{eq:outlier_image_gene}.
  }

 \caption{ \textsc{Dream-ood}: Outlier Imagination with Diffusion Models }
 \label{alg:algo_dream}
\end{algorithm}

\begin{wrapfigure}{r}{0.35\textwidth}
  \begin{center}
     \vspace{-0.9cm}
   {\includegraphics[width=\linewidth]{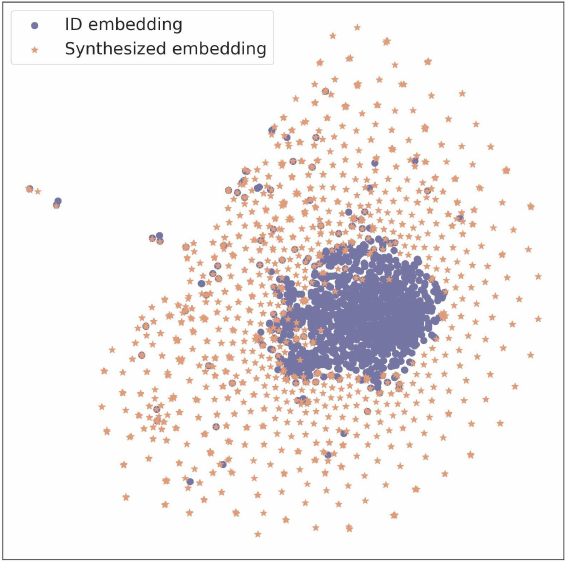}}
  \end{center}
   \vspace{-1em}
  \caption[TSNE visualization of ID embeddings and the sampled outlier embeddings for Dream-OOD]{\small{TSNE visualization of ID embeddings (purple) and the sampled outlier embeddings ({orange})}, for the class ``hen" in \textsc{Imagenet}.}
     \vspace{-3em}
  \label{fig:outliers}
\end{wrapfigure}

\textbf{Sampling OOD embedding.} Our goal here is to sample low-likelihood embeddings based on the learned feature representations (see Figure~\ref{fig:outliers}). The sampling procedure can be instantiated by different approaches. For example, a recent work by ~\cite{tao2023nonparametric} proposed a latent non-parametric sampling method,  which does not make any distributional assumption on the ID embeddings and offers stronger flexibility compared to the parametric sampling approach~\citep{du2022towards}. Concretely, we can select the boundary ID anchors by leveraging the non-parametric nearest neighbor distance, and then draw new embeddings around that boundary point.

Denote the $L_2$-normalized embedding set of training data as $\mathbb{Z} = (\*z_1, \*z_2, ..., \*z_n )$, where $\*z_i=h_{\theta}(\*x_i)/\|h_{\theta}(\*x_i)\|_2$. For any embedding $\*z' \in \mathbb{Z}$, we calculate the $k$-NN distance \emph{w.r.t.} $\mathbb{Z}$:
\begin{equation}
    d_k(\*z', \mathbb{Z}) = \lVert\*z' - \*z_{(k)}\rVert_2,
    \label{eq:knn_distance}
\end{equation}
where $\*z_{(k)}$ is the $k$-th nearest neighbor in $\mathbb{Z}$.
If an embedding has a large $k$-NN distance, it is likely to be on the boundary of the ID data and vice versa.

Given a boundary ID point, we then draw new embedding sample $\*v \in \mathbb{R}^m$  from a Gaussian kernel\footnote{\small\textcolor{black}{The choice of kernel function form (\emph{e.g.}, Gaussian vs. Epanechnikov) is not influential, while the kernel bandwidth parameter is~\citep{larrynotes}}.} centered at $\*z_i$ with covariance $\sigma^2 \mathbf{I}$: $\*v \sim \mathcal{N}(\*z_i, \sigma^2 \mathbf{I})$.
In addition, to ensure that the outliers are sufficiently far away from the ID data, we repeatedly sample multiple outlier embeddings from the Gaussian kernel $\mathcal{N}(\*z_i, \sigma^2 \mathbf{I})$, which produces a set $V_i$,  and further perform a filtering process by selecting the outlier embedding in $V_i$ with the largest $k$-NN distance \emph{w.r.t.} $\mathbb{Z}$. Detailed ablations on the sampling parameters are provided in Section~\ref{sec:ablation_main}.

\paragraph{Outlier image generation.} Lastly, to obtain the outlier images in the pixel space, we decode the sampled outlier embeddings $\*v$ via the diffusion model. In practice, this can be done by replacing the original token embedding $\mathcal{T}(y)$ with the sampled new embedding $\*v$\footnote{\small In  the implementation, we re-scale $\*v$ by multiplying the norm of the original token embedding to preserve the magnitude.}. Different from the vanilla prompt-based generation (\emph{c.f.} Equation~\eqref{eq:vanilla}) , our outlier imagination is mathematically reflected by:
\begin{equation}
    \*x_\text{ood} \sim P(\*x \vert \*v),
    \label{eq:outlier_image_gene}
\end{equation}
where ${\*x}_\text{ood}$ denotes the generated outliers in the pixel space. Importantly, $\*v \sim S\circ h_\theta \circ (\mathbb{P}_\text{in})$ is dependent on the in-distribution data, which enables generating images that deviate from $\mathbb{P}_\text{in}$. $S(\cdot)$ denotes the sampling procedure. Our framework \textsc{Dream-ood} is summarized in Algorithm~\ref{alg:algo_dream}.

\textbf{Learning with imagined outlier images.} The generated synthetic OOD images $\*x_\text{ood}$ can be used for regularizing the training of the classification model~\citep{du2022towards}:
\begin{equation}
\mathcal{L}_\text{ood}=\mathbb{E}_{\*x_\text{ood}}\left[-\log \frac{1}{1+\exp ^{\phi(E(f_{\theta}(\*x_\text{ood})))}}\right]
    +\mathbb{E}_{\*x \sim \mathbb{P}_\text{in}}\left[-\log \frac{\exp ^{\phi(E(f_{\theta}(\*x)))}}{1+\exp ^{\phi(E(f_{\theta}(\*x)))}}\right],
\label{eq:r_open}
\end{equation}
where $\phi(\cdot)$ is a three-layer nonlinear MLP function with the same architecture as VOS~\citep{du2022towards}, $E(\cdot)$ denotes the energy function, and $f_{\theta}(\*x)$ denotes the logit output of the classification model. In other words, the loss function takes both the ID and generated OOD images, and learns to separate them explicitly. The overall training objective combines the standard cross-entropy loss, along with an additional loss in terms of OOD regularization $\mathcal{L}_\text{CE}+\beta \cdot \mathcal{L}_\text{ood}$, where $\beta$ is the weight of the OOD regularization. $\mathcal{L}_\text{CE}$ denotes the cross-entropy loss on the ID training data.  In testing, we use the output of the binary logistic classifier for OOD detection.
 
\section{Experiments and Analysis}
\label{sec:experiments}
In this section, we present empirical evidence to validate the effectiveness of our proposed outlier imagination framework. In what follows, we show that \textsc{Dream-ood} produces meaningful OOD images, and as a result, significantly improves  OOD detection (Section~\ref{sec:ood_detect}) performance.  We provide comprehensive ablations and qualitative studies in Section~\ref{sec:qualitative}. In addition, we showcase an \emph{extension} of our framework for improving generalization by leveraging the synthesized inliers (Section~\ref{sec:generalization}). 

\subsection{Evaluation on OOD Detection Performance }
\label{sec:ood_detect}

\begin{table*}[t]
  \centering
 
 
  \scalebox{0.45}{
    \begin{tabular}{cccccccccccc}
    \toprule
    \multirow{3}[4]{*}{Methods} & \multicolumn{10}{c}{OOD Datasets}                                                             & \multirow{3}[4]{*}{ID ACC} \\
    \cmidrule{2-11}
          & \multicolumn{2}{c}{\textsc{iNaturalist}} & \multicolumn{2}{c}{\textsc{Places}} & \multicolumn{2}{c}{\textsc{Sun}} & \multicolumn{2}{c}{\textsc{Textures}}  & \multicolumn{2}{c}{Average} &  \\
\cmidrule{2-11}          & FPR95$\downarrow$ & AUROC$\uparrow$ & FPR95$\downarrow$ & AUROC$\uparrow$ & FPR95$\downarrow$ & AUROC$\uparrow$ & FPR95$\downarrow$ & AUROC$\uparrow$ & FPR95$\downarrow$ & AUROC$\uparrow$ &  \\
\hline

MSP~\citep{hendrycks2016baseline}& 31.80 & 94.98 & 47.10 & 90.84  &47.60 & 90.86 & 65.80 & 83.34 &  48.08 & 90.01& 87.64  \\
ODIN~\citep{liang2018enhancing}& 24.40 & 95.92& 50.30 & 90.20& 44.90 & 91.55 &61.00 & 81.37&  45.15 & 89.76 & 87.64\\
Mahalanobis~\citep{lee2018simple}& 91.60 & 75.16& 96.70 & 60.87&  97.40 & 62.23& 36.50 & 91.43& 80.55 & 72.42& 87.64 \\
Energy~\citep{liu2020energy}& 32.50 & 94.82& 50.80 & 90.76 & 47.60 & 91.71  & 63.80 & 80.54 & 48.68 & 89.46 & 87.64\\
GODIN~\citep{hsu2020generalized} & 39.90 & 93.94 & 59.70 & 89.20 &58.70 & 90.65 & 39.90 & 92.71 & 49.55 & 91.62 & 87.38\\
KNN~\citep{sun2022out}& 28.67 & 95.57& 65.83 & 88.72& 58.08 & 90.17& {12.92} & 90.37&  41.38 & 91.20& 87.64  \\
ViM~\citep{wang2022vim}& 75.50 & 87.18& 88.30 & 81.25& 88.70 & 81.37& 15.60 & {96.63}& 67.03 & 86.61& 87.64 \\
ReAct~\citep{sun2021react} &22.40 & 96.05 &  45.10 & 92.28  & 37.90 & 93.04 & 59.30 & 85.19  & 41.17 & 91.64  & 87.64 \\

DICE~\citep{sun2022dice}& 37.30 & 92.51 & 53.80 & 87.75 &45.60 & 89.21 & 50.00 & 83.27 &46.67 & 88.19 &87.64 \\
\midrule
\emph{Synthesis-based methods}\\
GAN~\citep{lee2018training}&  83.10 & 71.35 & 83.20 & 69.85 &84.40 & 67.56 &91.00 & 59.16 & 85.42 & 66.98 & 79.52 \\
VOS~\citep{du2022towards}&  43.00 & 93.77 &47.60 & 91.77 &39.40 & 93.17 &66.10 & 81.42 &  49.02 & 90.03 &87.50\\
 
 NPOS~\citep{tao2023nonparametric}  &53.84 & 86.52 &  59.66 & 83.50 & 53.54 & 87.99 & \textbf{8.98} &  \textbf{98.13} & 44.00 &  89.04 &  85.37\\
  \rowcolor{Gray}  \textbf{\textsc{Dream-ood} (Ours)} & \textbf{24.10}{\scriptsize$\pm$0.2} & \textbf{96.10}{\scriptsize$\pm$0.1}  & \textbf{39.87}{\scriptsize$\pm$0.1} & \textbf{93.11}{\scriptsize$\pm$0.3} & \textbf{36.88}{\scriptsize$\pm$0.4} & \textbf{93.31}{\scriptsize$\pm$0.4}  &  53.99{\scriptsize$\pm$0.6} & 85.56{\scriptsize$\pm$0.9}  & \textbf{38.76}{\scriptsize$\pm$0.2} & \textbf{92.02}{\scriptsize$\pm$0.4}& 87.54{\scriptsize$\pm$0.1}\\
            \hline
    \end{tabular}%
    }
     \caption[OOD detection results for {ImageNet-100} as the in-distribution data for {Dream-OOD}]{\small {OOD detection results for \textsc{Imagenet-100} as the in-distribution data. }We report standard deviations estimated across 3 runs. {Bold} numbers are superior results. }
  \label{tab:ood_results}%
\end{table*}

\textbf{Datasets.} Following ~\cite{tao2023nonparametric}, we use the \textsc{Cifar-100} and the large-scale \textsc{Imagenet} dataset~\citep{deng2009imagenet} as the ID training data. For \textsc{Cifar-100}, we use a suite of natural image datasets as OOD including \textsc{Textures}~\citep{cimpoi2014describing}, \textsc{Svhn}~\citep{netzer2011reading}, \textsc{Places365}~\citep{zhou2017places}, \textsc{iSun}~\citep{DBLP:journals/corr/XuEZFKX15} \& \textsc{Lsun}~\citep{DBLP:journals/corr/YuZSSX15}.
For \textsc{Imagenet-100}, we adopt the OOD test data as in~\citep{huang2021mos}, including  subsets of \textsc{iNaturalist} \citep{van2018inaturalist}, \textsc{Sun} \citep{xiao2010sun}, \textsc{Places} \citep{zhou2017places}, and \textsc{Textures} \citep{cimpoi2014describing}.     For each OOD dataset, the categories are disjoint from the ID dataset.  We provide the details of the datasets and categories in  Appendix~\ref{sec:app_dataset}.

\textbf{Training details.} We use ResNet-34~\citep{he2016deep} as the network architecture for both  \textsc{Cifar-100} and \textsc{Imagenet-100} datasets. We train the model using stochastic gradient descent for 100 epochs with the cosine learning rate decay schedule, a momentum of 0.9, and a weight decay of $5e^{-4}$. The initial learning rate is set to 0.1 and the batch size is set to 160. We generate $1,000$ OOD samples per class using Stable Diffusion v1.4, which results in $100,000$ synthetic images in total.  $\beta$ is set to 1.0 for \textsc{ImageNet-100} and 2.5 for \textsc{Cifar-100}. To learn the feature encoder $h_{\theta}$, we set the temperature $t$ in Equation~\eqref{eq:cosine_loss} to 0.1. Extensive ablations on hyperparameters $\sigma$, $k$ and $\beta$ are provided in  Section~\ref{sec:qualitative}.

\textbf{Evaluation metrics.} We report the following metrics: (1) the false positive rate (FPR95) of OOD samples when the true positive rate of ID samples is 95\%, (2) the area under the receiver operating characteristic curve (AUROC), and (3) ID accuracy (ID ACC).

\textbf{\textsc{Dream-ood} significantly improves the OOD detection performance.} As shown in Table~\ref{tab:ood_results} and Table~\ref{tab:ood_results_c100}, we compare our method with the competitive baselines, including  Maximum Softmax Probability \citep{hendrycks2016baseline}, ODIN score \citep{liang2018enhancing},  Mahalanobis score~\citep{lee2018simple},   Energy score \citep{liu2020energy},  Generalized ODIN~\citep{hsu2020generalized}, KNN distance~\citep{sun2022out}, ViM score~\citep{wang2022vim}, ReAct~\citep{sun2021react}, and DICE~\citep{sun2022dice}. Closely related to ours, we contrast with three synthesis-based methods, including latent-based outlier synthesis (VOS~\citep{du2022towards} \& NPOS~\citep{tao2023nonparametric}),  and GAN-based synthesis~\citep{lee2018training}, showcasing the effectiveness of our approach. For example, \textsc{Dream-ood} achieves an FPR95 of 39.87\% on \textsc{Places} with the ID data of \textsc{Imagenet-100}, which is a 19.79\% improvement from the best baseline NPOS. 

In particular, \textsc{Dream-ood} advances both VOS  and NPOS by allowing us to understand the synthesized outliers in a human-compatible way,  which was infeasible for the feature-based outlier sampling in VOS and NPOS.  {Compared with the feature-based synthesis approaches, \textsc{Dream-ood} can generate high-resolution outliers in the pixel space. The higher-dimensional pixel space offers much more knowledge about the unknowns, which provides the model with high variability and fine-grained details for the unknowns that are missing in VOS and NPOS. Since \textsc{Dream-ood} is more photo-realistic and better for humans, the generated images can be naturally better constrained for neural networks (for example, things may be more on the natural image manifolds).} We provide comprehensive qualitative results (Section~\ref{sec:qualitative}) to facilitate the understanding of generated outliers. {As we will show in Figure~\ref{fig:visual_ood}, the generated outliers are more precise in characterizing OOD data and thus improve the empirical performance.}

\begin{wraptable}{r}{0.6\linewidth}
    \centering
    \small
    \tabcolsep 0.04in\renewcommand\arraystretch{0.745}{}
    \vspace{-1em}
    \scalebox{0.6}{
    \begin{tabular}{c|cccc}
    \toprule 
\multirow{1}{*}{ {Method}} & {FPR95} $\downarrow$ & {AUROC} $\uparrow$  & {FPR95} $\downarrow$ & {AUROC} $\uparrow$\\
 
        \midrule
     &  \multicolumn{2}{c}{\textsc{Imagenet-100} as ID} &  \multicolumn{2}{c}{\textsc{Cifar-100} as ID}\\
      (I)  Add gaussian noise&41.35& 89.91  &45.33 & 88.83 \\
       (II) Add learnable noise& 	42.48& 91.45 &48.05&  87.72 \\
       (III)  Interpolate embeddings&41.35& 90.82 &43.36 &  87.09  \\
          (IV)                Disjont class names& 43.55 &87.84 & 49.89	&85.87
\\
 
        \midrule
       \rowcolor{Gray}  \textbf{\textsc{Dream-ood} (ours)} &\textbf{38.76} & \textbf{92.02} & \textbf{40.31} &  \textbf{90.15} \\

        \bottomrule
    \end{tabular}
}
      \caption[Comparison of {Dream-OOD} with different outlier embedding synthesis methods using diffusion models]{\small Comparison of \textsc{Dream-ood} with different outlier embedding synthesis methods using diffusion models.}
\vspace{-1em}
    \label{tab:ablation_different_synthesis}
\end{wraptable}

\paragraph{Comparison with other outlier synthesis approaches.} We compare \textsc{Dream-ood} with different outlier embedding synthesis approaches in Table~\ref{tab:ablation_different_synthesis}: (I) synthesizing outlier embeddings by adding multivariate Gaussian noise $\mathcal{N}(\textbf{0}, \sigma_1^2\mathbf{I})$ to the token embeddings, (II) adding learnable noise to the token embeddings where the noise is trained to push the outliers away from ID features,   (III) interpolating token embeddings from different classes by $\alpha \mathcal{T}(y_1) + (1-\alpha) \mathcal{T}(y_2)$, {and (IV) generating outlier images by using embeddings of new class names outside ID classes.} For (I), we set the optimal variance values $\sigma_1^2$ to 0.03 by sweeping from $\{0.01,0.02,0.03,...,0.10\}$. For (III), we choose the interpolation factor $\alpha$ to be $0.5$ from $\{0.1,0.2,...,0.9\}$. {For (IV), we use the remaining 900 classes in \textsc{Imagenet-1k} (exclude the 100 classes in \textsc{Imagenet-100}) as the disjoint class names for outlier generation. We generate the same amount of images as ours for all the variants to ensure a fair comparison.}

  \begin{table}[t]
  \centering
  \small
 
    \scalebox{0.4}{
    \begin{tabular}{cccccccccccccc}
    \toprule
    \multirow{3}[4]{*}{Methods} & \multicolumn{12}{c}{OOD Datasets}                                                             & \multirow{3}[4]{*}{ID ACC} \\
    \cmidrule{2-13}
          & \multicolumn{2}{c}{\textsc{Svhn}} & \multicolumn{2}{c}{\textsc{Places365}} & \multicolumn{2}{c}{\textsc{Lsun}} & \multicolumn{2}{c}{\textsc{iSun}} & \multicolumn{2}{c}{\textsc{Textures}} & \multicolumn{2}{c}{Average} &  \\
\cmidrule{2-13}          &FPR95$\downarrow$ & AUROC$\uparrow$ & FPR95$\downarrow$ & AUROC$\uparrow$ & FPR95$\downarrow$ & AUROC$\uparrow$ & FPR95$\downarrow$ & AUROC$\uparrow$ & FPR95$\downarrow$ & AUROC$\uparrow$  & FPR95$\downarrow$ & AUROC$\uparrow$   \\
    \midrule
    MSP~\citep{hendrycks2016baseline}   &  87.35 & 69.08 & 81.65 & 76.71  & 76.40 & 80.12& 76.00 & 78.90 & 79.35 & 77.43 & 80.15 & 76.45  & 79.04 \\
        ODIN~\citep{liang2018enhancing}  & 90.95 & 64.36 & 79.30 & 74.87 & 75.60 & 78.04& 53.10 & 87.40  & 72.60 & 79.82 & 74.31 & 76.90  & 79.04\\
        Mahalanobis~\citep{lee2018simple}& 87.80 & 69.98 & 76.00 & 77.90  & 56.80 & 85.83 & 59.20 & 86.46& 62.45 & 84.43& 68.45 & 80.92  & 79.04\\
    Energy~\citep{liu2020energy} & 84.90 & 70.90 & 82.05 & 76.00& 81.75 & 78.36 & 73.55 & 81.20 &  78.70 & 78.87  & 80.19 & 77.07 & 79.04\\
    GODIN~\citep{hsu2020generalized} & 63.95 & 88.98 & 80.65 & 77.19  & 60.65 & 88.36  &  51.60 & 92.07 & 71.75 & 85.02  & 65.72 & 86.32 & 76.34\\
   KNN~\citep{sun2022out}   &81.12 & 73.65   & 79.62 & 78.21 & 63.29 & 85.56 & 73.92 & 79.77& 73.29 & 80.35 & 74.25 & 79.51 & 79.04 \\

  ViM~\citep{wang2022vim}   &   81.20 & 77.24 &  79.20 & 77.81  & 43.10 & 90.43& 74.55 & 83.02 & 61.85 & 85.57 &  67.98 & 82.81  & 79.04 \\
  ReAct~\citep{sun2021react} & 82.85 & 70.12 & 81.75 & 76.25 & 80.70 & 83.03 &67.40 & 83.28 &74.60 & 81.61 & 77.46 & 78.86 &79.04\\

DICE~\citep{sun2022dice}& 83.55 & 72.49  & 85.05 & 75.92 & 94.05 & 73.59 &75.20 & 80.90 &79.80 & 77.83 &83.53 & 76.15 & 79.04 \\
\hline
\emph{Synthesis-based methods}\\
 GAN~\citep{lee2018training}& 89.45 & 66.95 & 88.75 & 66.76 &82.35 & 75.87 &83.45 & 73.49 & 92.80 & 62.99 &87.36 & 69.21 &70.12\\
         VOS~\citep{du2022towards}   & 78.50 & 73.11& 84.55 & 75.85 & 59.05 & 85.72  & 72.45 & 82.66 & 75.35 & 80.08&  73.98 & 79.48& 78.56\\
       
         NPOS~\citep{tao2023nonparametric}  & \textbf{11.14} & \textbf{97.84} & 79.08 & 71.30 &  56.27 & 82.43 & 51.72 & 85.48 & \textbf{35.20} &  \textbf{92.44} & 46.68 &   85.90 & 78.23\\
         \rowcolor{Gray}  \textbf{\textsc{Dream-ood} (Ours)}  &\textbf{}58.75\textbf{}{\scriptsize$\pm$0.6} & 87.01{\scriptsize$\pm$0.1} & \textbf{70.85}{\scriptsize$\pm$1.6} & \textbf{79.94}{\scriptsize$\pm$0.2} &\textbf{24.25}{\scriptsize$\pm$1.1} & \textbf{95.23}{\scriptsize$\pm$0.2} & \textbf{1.10}{\scriptsize$\pm$0.2} & \textbf{99.73}{\scriptsize$\pm$0.4} & \textbf{}46.60\textbf{}{\scriptsize$\pm$0.4} & \textbf{}88.82\textbf{}{\scriptsize$\pm$0.7} &\textbf{40.31}{\scriptsize$\pm$0.8} & \textbf{90.15}{\scriptsize$\pm$0.3} & 78.94\\
   \hline
\end{tabular}}
 \caption[OOD detection results for {CIFAR-100} as the in-distribution data for {Dream-OOD}]{\small {OOD detection results for \textsc{Cifar-100} as the in-distribution data. }We report standard deviations estimated across 3 runs. {Bold} numbers are superior results.  }
    \label{tab:ood_results_c100}
\end{table}
The result shows that \textsc{Dream-ood} outperforms all the alternative synthesis approaches by a considerable margin. Though adding noise to the token embedding is relatively simple, it cannot explicitly sample textual embeddings from the low-likelihood region as \textsc{Dream-ood} does, which are near the ID boundary and thus demonstrate stronger effectiveness to regularize the model (Section~\ref{sec:synthesis}). Visualization is provided in Appendix~\ref{sec:visual_add_noise}. Interpolating the token embeddings will easily generate images that are still ID (Appendix~\ref{sec:ab_visual_interpolation}), which is also observed in~\citep{liew2022magicmix}. 

\subsection{Ablation Studies}
\label{sec:qualitative}
\label{sec:ablation_main}
 In this section, we provide additional ablations to understand \textsc{Dream-ood} for OOD generation. For all the ablations, we use the high resolution \textsc{Imagenet-100} dataset as the ID data.

  \textbf{Ablation on the regularization weight $\beta$.}  In Figure~\ref{fig:ablation_frame_range} (a), we ablate the effect of weight $\beta$ of the  regularization loss $\mathcal{L}_{\text{ood}}$ for OOD detection (Section~\ref{sec:synthesis}) on the OOD detection performance. Using a mild weighting, such as $\beta$ = 1.0, achieves the best OOD detection performance. Too excessive regularization using synthesized OOD images ultimately degrades the performance.

  \textbf{Ablation on the variance value $\sigma^2$.} We show in Figure~\ref{fig:ablation_frame_range} (b) the effect of $\sigma^2$ --- the number of the variance value for the Gaussian kernel (Section~\ref{sec:synthesis}). We vary $\sigma^2\in\{0.02, 0.03, 0.04, 0.05, 0.06, 0.2\}$. Using a mild variance value $\sigma^2$ generates meaningful synthetic OOD images for model regularization. Too large of variance (e.g., $\sigma^2=0.2$) produces far-OOD, which does not help learn a compact decision boundary between ID and OOD.

 \begin{figure}[t]
  \begin{center}
   {\includegraphics[width=1\linewidth]{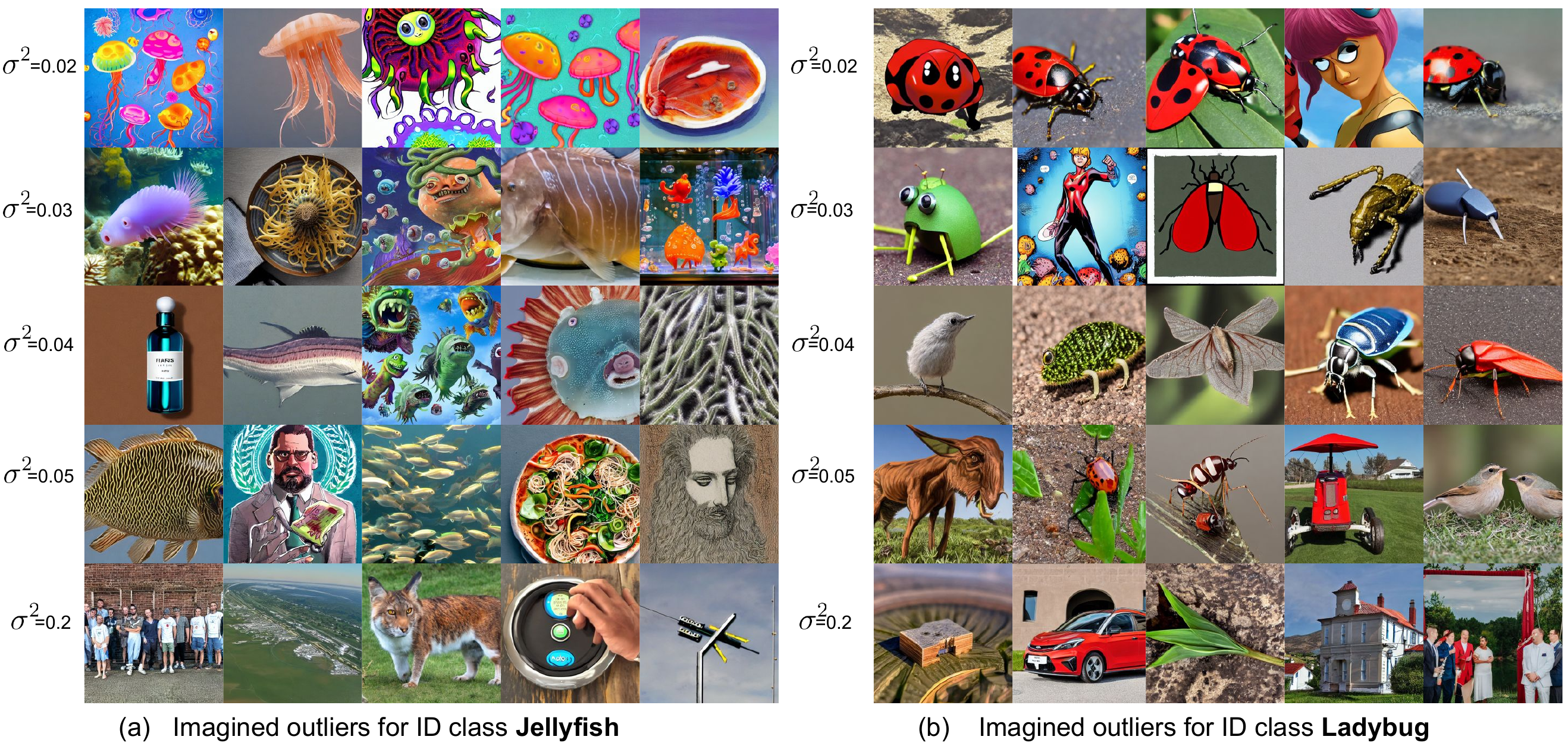}}
  \end{center}
  \vspace{-1em}
  \caption[Visualization of the imagined outliers for Dream-OOD]{\small \textbf{Visualization of the imagined outliers} \emph{w.r.t.}  \emph{jellyfish, ladybug} class under different variance $\sigma^2$.}
     \vspace{-1em}
  \label{fig:visual_ood}
  \end{figure}

\textbf{Ablation on $k$ in calculating $k$-NN distance.}   In Figure~\ref{fig:ablation_frame_range} (c), we analyze the effect of $k$, \emph{i.e.}, the number of nearest neighbors for non-parametric sampling in the latent space. We vary $k=\{100,  200, 300, 400, 500\}$ and observe that our method is not sensitive to this hyperparameter.

\textbf{Visualization of the generated outliers.}  Figure~\ref{fig:visual_ood} illustrates the generated outlier images under different variance  $\sigma^2$. Mathematically, a larger variance translates into outliers that are more deviated from ID data. We confirm this in our visualization too. The synthetic OOD images gradually become semantically different from ID classes ``jellyfish'' and ``ladybug'',  as the variance increases. More visualization results are in Appendix~\ref{sec:ab_visual}.

\subsection{Extension: from \textsc{Dream-ood} to \textsc{Dream-id}}
\label{sec:generalization}

Our framework can be easily extended to generate ID data. Specifically, we can select the ID point with small $k$-NN distances \emph{w.r.t.} the training data (Equation~\eqref{eq:knn_distance}) and sample inliers from the Gaussian kernel with small variance $\sigma^2$ in the text-conditioned embedding space (Figure~\ref{fig:embedding_inlier}). Then we decode the inlier embeddings via the diffusion model for ID generation (Visualization provided in Appendix~\ref{sec:inlier_visual}). For the synthesized ID images, we let the semantic label be the same as the anchor ID point. Here we term our extension as \textsc{Dream-id} instead. 

\begin{wrapfigure}{r}{0.3\textwidth}
  \begin{center}
   \vspace{-0.6cm}
   {\includegraphics[width=\linewidth]{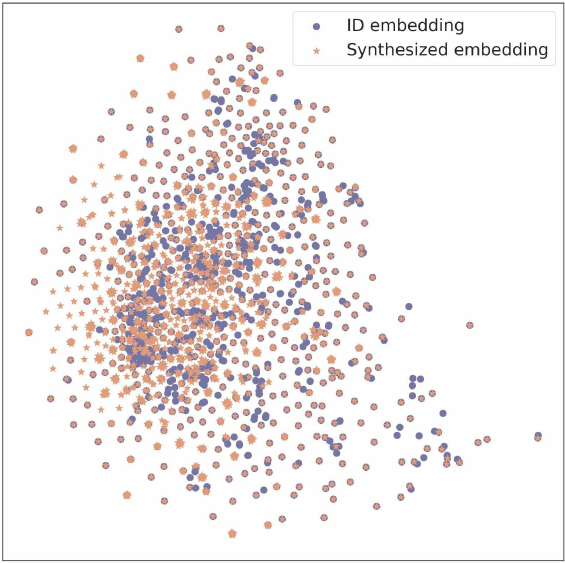}}
  \end{center}
   \vspace{-1em}
  \caption[{TSNE visualization of ID embeddings ({purple}) and the synthesized inlier embeddings ({orange})}, for class ``hen'' in {ImageNet}.]{\small {TSNE visualization of ID embeddings ({purple}) and the synthesized inlier embeddings ({orange})}, for class ``hen'' in \textsc{Imagenet}.}
     \vspace{-1em}
  \label{fig:embedding_inlier}
\end{wrapfigure}

\textbf{Datasets.} We use the same \textsc{Imagenet-100} as the training data. We measure the generalization performance on both the original \textsc{Imagenet} test data (for ID generalization) and variants with distribution shifts (for OOD generalization). For OOD generalization, we evaluate on (1) \textsc{Imagenet-a}~\citep{hendrycks2021natural}   consisting of real-world, unmodified, and naturally occurring examples that are misclassified by ResNet models;  (2) \textsc{Imagenet-v2}~\citep{recht2019imagenet}, which is created from the Flickr dataset with natural distribution shifts. We provide the experimental details in  Appendix~\ref{sec:details_generalization}.

\textbf{\textsc{Dream-id}  improves the generalization performance.} As shown in Table~\ref{tab:gene_results}, we compare \textsc{Dream-id}  with competitive data augmentation and test-time adaptation methods. For a fair comparison, all the methods are trained using the same network architecture, under the same configuration. Specifically, our baselines include: the original model without any data augmentation,  RandAugment~\citep{cubuk2020randaugment}, AutoAugment~\citep{cubuk2018autoaugment},  CutMix~\citep{yun2019cutmix}, AugMix~\citep{hendrycks2019augmix}, DeepAugment~\citep{hendrycks2021many} and MEMO~\citep{zhang2021memo}. These methods are shown in the literature to help improve generalization. The results demonstrate that our approach outperforms all the baselines that use data augmentation for training in both ID generalization and generalization under natural distribution shifts ($\uparrow$0.74\% vs. the best on \textsc{Imagenet-a}, $\uparrow$0.70\% vs. the best on \textsc{Imagenet-v2}). Implementation details of the baselines are in Appendix~\ref{sec:baselines_gene}.

\begin{wraptable}{r}{0.6\linewidth}
    \centering
    \small
\vspace{0em}
    
  \scalebox{0.78}{   \begin{tabular}{c|ccc}
        Methods &  \textsc{Imagenet} &  \textsc{Imagenet-A} &  \textsc{Imagenet-v2} \\
        \hline
        Original (no aug) & 87.28 & 8.69& 77.80 \\
               RandAugment & 88.10 & 11.39&78.90 \\
            
              AutoAugment & 88.00 & 10.85 & 79.70 \\
                CutMix &87.98 &9.67 & 79.70\\
                    AugMix &87.74 & 10.96 & 79.20 \\
                       DeepAugment &86.86&10.79&78.30  \\
         MEMO&88.00 & 10.85& 78.60\\
                       \hline
                         \textcolor{gray}{Generic Prompts} & \textcolor{gray}{87.74} &\textcolor{gray}{11.18}  & \textcolor{gray}{79.20}  \\
                          \hline
                        \rowcolor{Gray}  \textbf{\textsc{Dream-id} (Ours)}& \textbf{88.46}{\scriptsize$\pm$0.1} & \textbf{12.13}{\scriptsize$\pm$0.1} & \textbf{80.40}{\scriptsize$\pm$0.1}\\
  
    \end{tabular}}
    \caption[Model generalization performance for Dream-ID]{\small {Model generalization performance (accuracy, in \%), using \textsc{Imagenet-100} as the training data.}  We report standard deviations estimated across 3 runs.}
    \label{tab:gene_results}
    \vspace{-1.5em}
\end{wraptable}

\begin{figure}[t]
    \centering
  \includegraphics[width=1.0\linewidth]{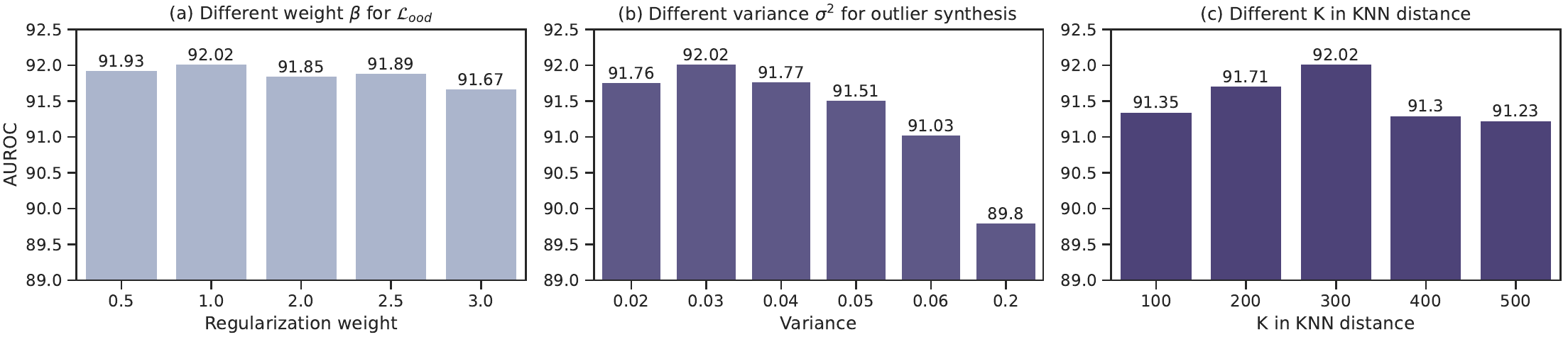}
  \vspace{-0.6cm}
    \caption[Ablation study on Dream-OOD]{\small (a) Ablation study on the regularization weight $\beta$ on $\mathcal{L}_{\textrm{ood}}$. (b) Ablation on the variance $\sigma^2$ for synthesizing outliers in Section~\ref{sec:synthesis}. (c) Ablation on the $k$ for the $k$-NN distance. The numbers are AUROC. The ID training dataset is \textsc{Imagenet-100}. }
    \label{fig:ablation_frame_range}
\end{figure}

In addition, we compare our method with using generic prompts (\emph{i.e.}, {``A high-quality photo of a [$y$]}'') for data generation. For a fair comparison, we synthesize the same amount of images (\emph{i.e.}, $1000$ per class) for both methods.  The result shows that \textsc{Dream-id} outperforms the baseline by $0.72\%$ on \textsc{Imagenet} test set and $0.95\%,1.20\%$ on \textsc{Imagenet-a} and \textsc{Imagenet-v2}, respectively. 
\section{Summary}
In this chapter, we propose a novel learning framework \textsc{Dream-ood}, which imagines photo-realistic outliers in the pixel space by way of diffusion models. \textsc{Dream-ood} mitigates the key shortcomings of training with auxiliary outlier datasets,  which typically require label-intensive human intervention for data preparation. \textsc{Dream-ood} learns a  text-conditioned latent space based on ID data, and then samples outliers in the low-likelihood region via the latent. We then generate outlier images by decoding the outlier embeddings with the diffusion model. The empirical result shows that training with the outlier images helps establish competitive performance on common OOD detection benchmarks. Our in-depth quantitative and qualitative ablations provide further insights on the efficacy of \textsc{Dream-ood}. We hope our work will inspire future research on automatic outlier synthesis in the pixel space.

\end{dream}

\begin{siren}
    \label{sec:cha6}
\textbf{Publication Statement.} This chapter is joint work with Gabriel Gozum, Yifei Ming and
Yixuan Li. The paper version of this chapter appeared in NeurIPS'22~\citep{du2022siren}.

\noindent \textbf{Abstract.} Detecting out-of-distribution (OOD) objects is indispensable for safely deploying object detectors in the wild. Although distance-based OOD detection methods have demonstrated promise in image classification, they remain largely unexplored in object-level OOD detection.
This chapter bridges the gap by proposing a distance-based framework for detecting OOD objects, which relies on the model-agnostic representation space and provides strong
generality across different neural architectures.
Our proposed framework \textsc{Siren} contributes two novel components: (1) a representation learning component that uses a trainable loss function to shape the representations into a mixture of von Mises-Fisher (vMF) distributions on the unit hypersphere, and (2) a test-time OOD detection score leveraging the learned vMF distributions in a parametric or non-parametric way. 
\textsc{Siren} achieves competitive performance on both the recent detection transformers and CNN-based models, improving the
AUROC by a large margin compared to the previous best method. Code is publicly available at \url{https://github.com/deeplearning-wisc/siren}.

\section{Introduction}

Teaching object detectors to be aware of out-of-distribution (OOD) data is indispensable for building reliable AI systems. Today,  the mainstream object detection models have been operating in the closed-world setting. That is, a model will match
an object to one of the {given} class labels, even if it is irrelevant. Instead, the open-world setting emphasizes that objects from the unknown classes can naturally emerge, which should not be blindly predicted into a known class. In safety-critical applications, such as autonomous driving, failing to detect OOD objects on the road can directly lead to disastrous accidents~\citep{DBLP:conf/itsc/NitschISNSSKC21}. The situation can be better avoided if the object detector recognizes the object as unfamiliar and appropriately cautions the human driver to take over.

In this chapter, we pioneer a distance-based framework for detecting OOD objects. Currently, the distance-based method remains largely {unexplored} in object-level OOD detection. In particular, by operating in the representation space, distance-based methods are model-agnostic and provide strong generality across neural architectures. In contrast, existing approaches derive highly specialized OOD detection scores based on the outputs of the object detectors, which may not be seamlessly applicable across architectures. For example, the classification output of the Faster R-CNN~\citep{ren2015faster} is optimized by the multi-class softmax loss, whereas the recent transformer-based object detection networks such as \textsc{deformable-detr}~\citep{deformable_detr} uses multi-label focal loss~\citep{DBLP:journals/pami/LinGGHD20}. Thereby, while output-based OOD scoring functions may be limited to specific architectures, distance-based methods are not.

\begin{figure}
    \centering
    \includegraphics[width=1.0\linewidth,bb=0 0 730 250]{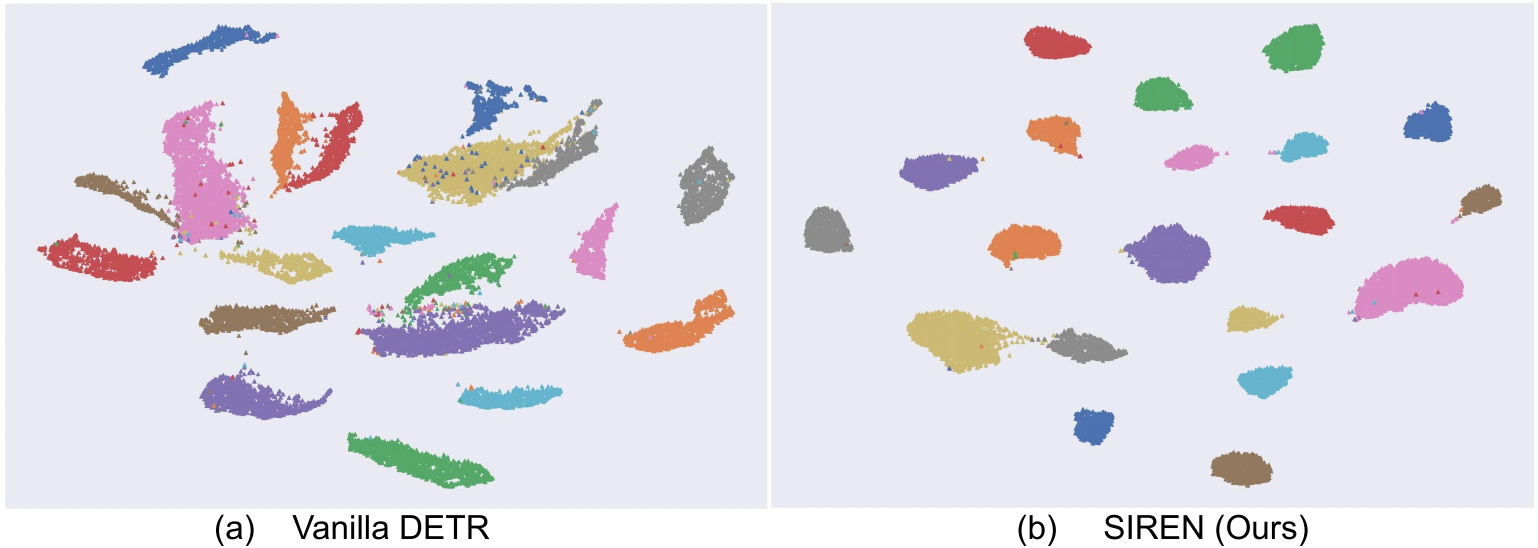}
    \caption[Teaser for the work of SIREN]{ (a) Feature embeddings from the penultimate layer of a vanilla \textsc{deformable-detr}~\cite{deformable_detr} trained on the \textsc{pascal-voc} dataset~\cite{DBLP:journals/ijcv/EveringhamGWWZ10}, which display irregular distributions. (b) Feature embeddings shaped by the proposed \textsc{Siren}, which form compact clusters on the unit hypersphere.}
    \label{fig:teaser_siren}
\end{figure}

Although distance-based OOD scoring functions have been studied in image classification, they do not trivially transfer to object detection models. 
For example, \cite{lee2018simple} modeled the feature embedding space as a mixture
of multivariate Gaussian distributions and used the maximum Mahalanobis distance~\citep{mahalanobis1936generalized} to all
class centroids for OOD detection. However, we observe that the modern object detection models such as \textsc{deformable-detr}~\citep{deformable_detr} produce  highly irregular  embeddings (Figure~\ref{fig:teaser_siren} (a)), which do not fit the Gaussian distributional assumption. As a result, the OOD detection score relying on such suboptimal embeddings can misbehave.

We propose a novel framework called \textsc{siren}, tackling two highly dependent problems---representation learning and OOD detection---in one synergistic framework.
Concretely, \textsc{Siren} contributes two novel components: \textbf{(1)} We introduce an end-to-end trainable loss that enables \textbf{S}hap\textbf{I}ng the \textbf{R}epres\textbf{EN}tations into a desired parametric form (Section~\ref{sec:regularization}). In particular, we model the representations by the von Mises-Fisher (vMF) distribution, a classic probability distribution in directional statistics for hyperspherical data with the unit norm. Our loss function encourages the normalized embedding to be aligned with its class prototype and shapes the overall representations into compact clusters for each class. Compared to the Gaussian distribution, using the vMF distribution avoids estimating large covariance matrices for high-dimensional data that is shown to be costly and unstable~\citep{DBLP:conf/icml/ChenLK17,DBLP:conf/icml/0001I20}.  \textbf{(2)} We explore test-time OOD detection by leveraging the optimized embeddings in a parametric or non-parametric way (Section~\ref{sec:score}). We propose a new test-time OOD score based on the learned class-conditional vMF distributions. The parameterization of the vMF distribution is directly obtainable after training, without requiring separate estimation. Different from Mahalanobis distance~\citep{lee2018simple}, the proposed parametric score in principle suits the learned vMF distributions on the hypersphere. Additionally, we explore a non-parametric nearest neighbor distance for OOD detection~\citep{sun2022out}, which is agnostic to the type of distribution of the feature space.

Empirically, \textsc{Siren} establishes superior performance on both transformer-based and CNN-based models. On \textsc{pascal-voc}, \textsc{Siren} outperforms the latest baseline OW-DETR~\citep{ow_detr} by a
significant margin ($\uparrow$22.53\% in AUROC). Moreover, our framework is model-agnostic and does not incur changes to the existing network architecture. The proposed loss can be flexibly added as a plug-in module on top of modern architectures, as we show in Section~\ref{sec:experiments_siren}. 

Our key {contributions} 
are summarized as follows:
\begin{enumerate}
    \item To the best of our knowledge, \textsc{Siren} pioneers a distance-based approach for object-level OOD detection. Different from previous works, \textsc{Siren} does not rely on  specialized output-based OOD scores, and can generalize across different architectures in a model-agnostic fashion. 
   \item \textsc{Siren} establishes competitive results on a challenging object-level OOD detection task. Compared to the latest method~\citep{ow_detr}, \textsc{Siren} improves the OOD detection performance by a considerable margin while preserving the mAP on the ID task. We show that \textsc{Siren} is effective for both recent transformer-based and classic CNN-based models. 
    \item We shape representations via a novel vMF-based formulation for object-level OOD detection. We conduct in-depth ablations to understand how different factors impact the performance of \textsc{Siren} (Section~\ref{sec:ablation}).
\end{enumerate}

\section{Preliminaries: Object-level OOD Detection}
We start by introducing the OOD detection problem for  object detection in the open-world setting, which has  received increasing research attention lately~\citep{du2022towards, ow_detr}. Our goal is to train object detection networks that can simultaneously: (1) localize and classify objects belonging to known categories accurately, and (2) identify unfamiliar objects outside the training categories.  Compared to image-level OOD detection, object-level OOD detection is more suitable for real-world machine learning systems, yet also more challenging as it requires reasoning OOD uncertainty at the fine-grained object level. Since
natural images are
composed of multiple objects, knowing
which regions of an image are anomalous allows for
safe handling of unfamiliar objects.

\paragraph{Notations.} We denote the input and label space by $\mathcal{X}=\mathbb{R}^q$ and $\mathcal{Y}=\{1,2,...,C\}$, respectively. Let $\mathbf{x} \in \mathcal{X}$ be the input image, $\mathbf{b} \in \mathbb{R}^4$ be the bounding box coordinates associated with objects in the image, and $y \in \mathcal{Y} $ be the semantic label of the object.  An object detection model is trained on ID dataset $\mathcal{D}_{\textrm{tr}}^{\textrm{in}}=\left\{\left(\mathbf{x}_{i},\mathbf{b}_i, {y}_{i}\right)\right\}_{i=1}^{M}$
drawn from an unknown joint distribution $\mathcal{P}$. We use neural networks with parameters $\theta$ to model the bounding box regression $p_\theta(\mathbf{b} |\mathbf{x})$ and the classification $p_\theta(y|\mathbf{x},\mathbf{b})$.

\paragraph{Object-level OOD detection.} The OOD detection can be formulated as a binary classification problem, distinguishing between the in- vs. out-of-distribution objects. Let $P_{\mathcal{X}}$ denote the marginal probability distribution on $\mathcal{X}$. 
Given a test input $\mathbf{x}'\sim P_\mathcal{X}$, as well as an object $\mathbf{b}'$ predicted by the object detector, the goal is to predict a binary outcome $g(\mathbf{x}', \mathbf{b}')$. We use $g=1$ to indicate a detected object being ID, and $g=0$ being OOD, with semantics outside the support of $\mathcal{Y}$.

\section{Proposed Method}
\label{sec:method}

\paragraph{Overview.} Our framework \textsc{Siren} is illustrated in Figure~\ref{fig:framework}, which trains
an object detector in tandem with a representation-shaping branch. The object detector backbone $f:\mathcal{X}\mapsto \mathbb{R}^m$ maps an object to its feature embedding $h({\*x}, \*b) \in \mathbb{R}^m$ (often referred to as the penultimate layer). In addition, we introduce a new MLP projection head $\phi:\mathbb{R}^m \mapsto \mathbb{R}^d$ that maps the $h({\*x}, \*b)$ to a lower-dimensional embedding $\*r \in \mathbb{R}^d$  ($d<m$) with unit norm $\|\*r\|^2=1$. The normalized embeddings are also referred to as \emph{hyperspherical embeddings}, since they are on a unit hypersphere. In designing \textsc{Siren}, we address  two key challenges: \textbf{(1)} How to shape the hyperspherical representations into desirable probability distributions during training time (Section~\ref{sec:regularization})? \textbf{(2)} How to perform test-time OOD detection by leveraging the learned distributions (Section~\ref{sec:score})? Our method does not incur any change to the object detection network backbone. The proposed regularization can be flexibly used as a plug-in module on top of the modern architectures, as we will show in Section~\ref{sec:experiments_siren}.

\begin{figure*}[t]
    \centering
    \includegraphics[width=1.0\linewidth,bb=0 0 1150 450]{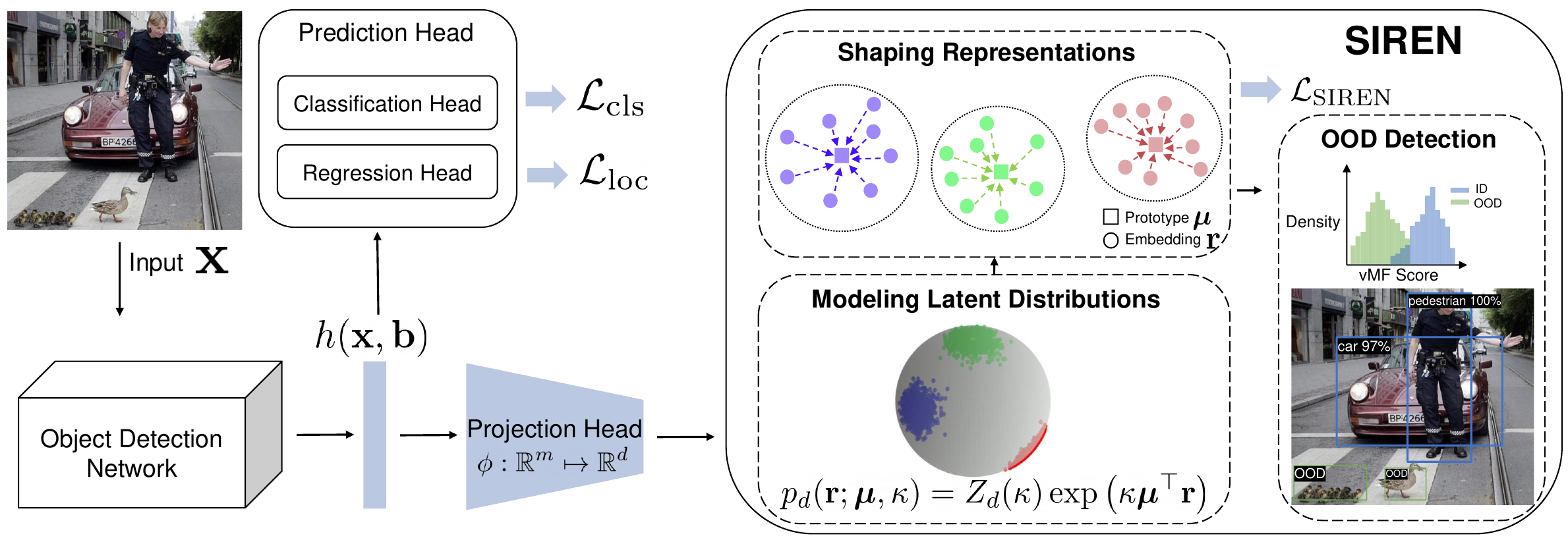}
    \caption[Overview of the proposed learning framework {SIREN}]{ \textbf{Overview of the proposed learning framework \textsc{Siren}.}     We introduce  a new loss $\mathcal{L}_{\textrm{SIREN}}$ which shapes the representations on the unit hypersphere into compact class-conditional vMF distributions. The embedding $\*r \in \mathbb{R}^d$ has unit norm $\|\*r\|^2=1$. In testing, we can employ either parametric or non-parametric distance functions for OOD detection. See Section~\ref{sec:method} for details. }
    \label{fig:framework}
\end{figure*}

\subsection{SIREN: Shaping Representations}
\label{sec:regularization}

\paragraph{Modeling the latent distributions.} 
We propose to model the latent representations by the  von Mises-Fisher (vMF) distribution~\citep{mardia2000directional}, a probability
distribution in directional statistics for spherical data with unit norm $\|\*r\|^2=1$. 
The probability density function for a unit vector $\mathbf{r}$ in $\mathbb{R}^{d}$ is given as follows:
\begin{equation}
    p_d(\mathbf{r} ; \boldsymbol{\mu}, \kappa)=Z_{d}(\kappa) \exp \left(\kappa \boldsymbol{\mu}^{\top} \mathbf{r}\right),
\end{equation}
where $\kappa\ge 0$, $\|\boldsymbol{\mu}\|^2=1$, and the normalization factor $Z_{d}(\kappa)$ is defined as: 
\begin{equation}
    Z_{d}(\kappa)=\frac{\kappa^{d / 2-1}}{(2 \pi)^{d / 2} I_{d / 2-1}(\kappa)},
\label{eq:defnition_zd}
\end{equation}
where $I_{v}$ is the modified Bessel function of the first kind
with order $v$. $Z_d(\kappa)$ can be calculated in closed form based on $\kappa$ and the dimensionality $d$. ~Importantly,~the vMF
distribution is characterized by two parameters: the mean vector $\boldsymbol{\mu}$ and concentration parameter $\kappa$. Samples that are more aligned with the center $\boldsymbol{\mu}$ have a higher probability density, and vice versa. Here $\kappa$ indicates
the tightness of the distribution around the mean direction $\boldsymbol{\mu}$.
The larger value of $\kappa$, the stronger the distribution is
concentrated in the mean direction. In the extreme case of $\kappa=0$, the sample points are distributed uniformly on the hypersphere. 

When considering multiple classes, we can model the embedding space as a mixture of class-conditional vMF distributions, one for each class $c\in\{1,2,...,C\}$:
\begin{equation}
    p_d^c(\mathbf{r} ; \boldsymbol{\mu}_c, \kappa_c)=Z_{d}(\kappa_c) \exp \left(\kappa_c \boldsymbol{\mu}_c^{\top} \mathbf{r}\right),
\end{equation}
where $\kappa_c$ and $\boldsymbol{\mu}_c$ are class-conditional parameters. Under this probability model, an embedding vector $\mathbf{r}$ is
assigned to class $c$ with the following normalized probability:
\begin{align}
    p(y=c | \mathbf{r}; \{\kappa_j,\boldsymbol{\mu}_{j}\}_{j=1}^C)
    & = \frac{Z_{d}\left(\kappa_{c}\right) \exp \left(\kappa_{c} \boldsymbol{\mu}_{c}^{\top} \mathbf{r}\right) }{\sum_{j=1}^{C} Z_{d}\left(\kappa_{j}\right) \exp \left(\kappa_{j} \boldsymbol{\mu}_{j}^{\top} \mathbf{r}\right) }.
\end{align}

\paragraph{Shaping representations.} Our key idea is to design an end-to-end trainable loss function that enables \textbf{S}hap\textbf{I}ng the \textbf{R}epres\textbf{EN}tations into a mixture of vMF distributions, which facilitates test-time OOD detection in the representation space (Section~\ref{sec:score}). We therefore name our method \textsc{Siren}. The
learned mapping function projects an input to a
point in the embedding space, where higher probability is
assigned to the correct class in comparison to incorrect
classes. To achieve this, we can perform maximum likelihood estimation (MLE) on the training data:
\begin{align}
    \text{argmax}_\theta  & \prod_{i=1}^M p(y_i | \*r_i; \{\kappa_j,\boldsymbol{\mu}_{j}\}_{j=1}^C ),
\end{align}
where $i$ is the index of the object embedding and $M$ is the size of the training set. By taking the negative log-likelihood, the objective function is equivalent to minimizing the following loss:
\begin{equation}
    \mathcal{L}_{\mathrm{SIREN}}=-  \frac{1}{M}\sum_{i=1}^{M} \log \frac{Z_{d}\left(\kappa_{y_i}\right) \exp \left(\kappa_{y_i} \boldsymbol{\mu}_{y_i}^{\top} \mathbf{r}_{i}\right)}{\sum_{j=1}^{C} Z_{d}\left(\kappa_{j}\right) \exp \left(\kappa_{j} \boldsymbol{\mu}_{j}^{\top} \mathbf{r}_{i}\right)},
    \label{eq:vmf_loss}
\end{equation}
where $y_i$ is the ground truth label for the embedding $\*r_i$. In effect, $\mathcal{L}_{\mathrm{SIREN}}$ encourages the object embeddings to be aligned with its class prototype, which shapes the representations such that objects in each class form a compact cluster on the hypersphere; see Figure~\ref{fig:teaser_siren} (b).

\paragraph{Prototype estimation and update.} During training, \textsc{Siren} estimates the class-conditional object prototypes $\boldsymbol{\mu}_{c}, c\in\{1,2,...,C\}$. The conventional approach for estimating the prototypes is to calculate the mean vector of all training
samples (or a subset of them) for each class, and update it periodically during training~\citep{DBLP:journals/pr/ZheCY19}. Despite its simplicity, this method requires alternating training and prototype estimation, which incurs a heavy
computational toll and causes undesirable  latency. Instead, we update the class-conditional prototypes
in an exponential-moving-average (EMA) manner~\citep{DBLP:conf/iclr/0001XH21,DBLP:conf/iclr/WangXLF0CZ22}:
\begin{equation}
    \boldsymbol{\mu}_{c}:=\text{Normalize}(\alpha \boldsymbol{\mu}_{c}+(1-\alpha) \mathbf{r} ), \forall c \in\{1,2, \ldots, C\},
    \label{eq:prototype}
\end{equation}

where $\alpha$ is the prototype update factor, and $\*r$ denotes the normalized object embeddings from class $c$. The update can be done efficiently with negligible cost, enabling end-to-end training.

\begin{algorithm}[t]
\SetAlgoLined
\textbf{Input:} ID training data $\mathcal{D}_\text{tr}^\text{in}=\left\{\left(\*x_{i}, \mathbf{b}_i,{y}_{i}\right)\right\}_{i=1}^{M}$, randomly initialized object detector and MLP projection head with parameter $\theta$, loss weight $\beta$ for $\mathcal{L}_\text{SIREN}$, and learnable $\{\kappa_c\}_{c=1}^C$.\\
\textbf{Output:} Object detector with parameter $\theta^{*}$ and OOD detector $G$.\\
\While{train}{
        1.  Update class-conditional prototypes $\boldsymbol{\mu}_{c}$ with the hyperspherical embeddings $\mathbf{r}$ by Equation~\eqref{eq:prototype}.\\
    2.  Calculate the vMF-based representation shaping loss $\mathcal{L}_{\textrm{SIREN}}$ by Equation~\eqref{eq:vmf_loss}. \\
    3.  Update the learnable $\{\kappa_c\}_{c=1}^C$ and the network parameters $\theta$ using Equation~\eqref{eq:totalloss}.
  }
  \While{eval}{
  1.  Calculate the OOD score by Equation~\eqref{eq:vmf_uncertainty} or   Equation~\eqref{eq:knn}.\\
  2.  Perform OOD detection by Equation~\eqref{eq:threshold}.

 }
 \caption{\textsc{Siren}: Shaping Representations for object-level OOD detection}
 \label{alg:algo_siren}
\end{algorithm}

\textbf{Overall training objective.} The overall training objective combines
the standard object detection loss, along with our new representation shaping loss $\mathcal{L}_{\textrm{SIREN}}$:
\begin{equation}
    \min _{\theta, \kappa} \mathbb{E}_{(\mathbf{x}, \mathbf{b}, y) \sim \mathcal{P}}\left[\mathcal{L}_{\mathrm{cls}}+\mathcal{L}_{\mathrm{loc}}\right]+\beta \cdot \mathcal{L}_{\textrm{SIREN}},
    \label{eq:totalloss}
\end{equation}
where $\beta$ is the weight of our representation shaping loss. $\mathcal{L}_{\mathrm{cls}}$ and $\mathcal{L}_{\mathrm{loc}}$ are losses for classification and
bounding box regression, respectively.  We
provide extensive empirical evidence in Section~\ref{sec:experiments_siren} demonstrating the efficacy of our loss function.

To the best of our knowledge, our work makes the first attempt to explore vMF-based learning and inference for object-level OOD detection. To highlight the novelty of the loss itself: we introduce a novel \emph{learnable} $\{\kappa_j\}_{j=1}^C$ for each ID class in Equation~\eqref{eq:vmf_loss}, instead of using fixed values. Our loss allows concentration parameter $\kappa$ to adaptively and flexibly capture the class-conditional feature statistics during training. This is desirable when each ID class may have its own concentration in the hypersphere. We will later show that such learnable $\{\kappa_j\}_{j=1}^C$ enables better OOD detection performance (Section~\ref{sec:ablation}).

\subsection{Test-time OOD Detection}
\label{sec:score}
During inference, we explore and contrast two types of uncertainty scores for detecting OOD objects.

\textbf{Parametric vMF score.} We propose a new test-time OOD score based on the learned class-conditional vMF distributions, parameterized by $\{\widehat{\kappa}_c, \boldsymbol{\mu}_c\}_{c=1}^C$. Here $\widehat{\kappa}_c$ denotes the learned concentration parameter for class $c$, which captures the concentration of representations for class $c$. For a test-time object $(\*x',\*b')$, we use the largest estimated class-conditional likelihood as the OOD score:
\begin{equation}
    S\left(\mathbf{x}', \mathbf{b}'\right)=\max_{c} Z_{d}(\widehat{\kappa}_c) \exp \left(\widehat{\kappa}_c \boldsymbol{\mu}_c^{\top} \mathbf{r'}\right),
    \label{eq:vmf_uncertainty}
\end{equation}
where $\mathbf{r}'=\phi(h(\*x',\*b'))$ is the normalized embedding from the MLP projection head. Our OOD detection score thus in principle suits our learned embeddings and vMF distributions. 
For OOD detection, one can use the level set to distinguish between ID and
OOD objects:
\begin{equation}
    G\left(\mathbf{x}', \mathbf{b}'\right)= \begin{cases}1 & \text { if } S\left(\mathbf{x}', \mathbf{b}'\right) \geq \gamma \\ 0 & \text { if } S\left(\mathbf{x}', \mathbf{b}'\right)<\gamma\end{cases}
    \label{eq:threshold}
\end{equation}

The threshold $\gamma$ can be chosen so that a high fraction of ID data (e.g., 95\%) is correctly classified. For objects classified as ID, one can obtain the bounding box and class prediction using the prediction head as usual.

\textbf{Non-parametric KNN score.} To relax the distributional assumption on the learned embeddings, we additionally employ a non-parametric KNN distance for OOD detection, which performs well on compact and normalized feature space. Following~\cite{sun2022out}, the KNN distance is defined as:
\begin{equation}
   S\left(\mathbf{x}', \mathbf{b}', k\right)= -\left\|\mathbf{r}'-\mathbf{r}_{(k)}\right\|_2,
   \label{eq:knn}
\end{equation}
where $\mathbf{r}_{(k)}$ denotes the normalized embedding of the $k$-th nearest neighbor (in the training data), for the test embedding $\mathbf{r}'$.~Our algorithm is summarized in Algorithm~\ref{alg:algo_siren}.

\begin{wrapfigure}{r}{0.55\textwidth}
  \begin{center}
    \includegraphics[width=0.55\textwidth,bb=0 0 240 70]{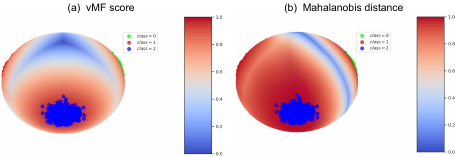}
  \end{center}
  
 \caption[The uncertainty surface with vMF score]{ The uncertainty surface is calculated using our vMF score (a) and the Mahalanobis distance (b). We showcase one class for visual clarity. }
  \vspace{-1em}
  \label{fig:toy}
\end{wrapfigure}
\paragraph{Remark 1.} \emph{Different from Mahalanobis distance~\citep{lee2018simple}, our parametric vMF-based OOD detection score operates under the same distributional model as the training process, and hence enjoys mathematical compatibility. Computationally, we can directly use the parameters of vMF distributions (such as $\kappa$) learned from training. In contrast, the Mahalanobis distance requires a separate test-time estimation of feature statistics---which involves an expensive and numerically unstable step of calculating the covariance matrix. The distinction of the uncertainty surface calculated by both our vMF score and the Mahalanobis distance is qualitatively demonstrated in Figure~\ref{fig:toy}. The data points are sampled from a mixture of three class-conditional vMF distributions (see details in Appendix~\ref{sec:toy_details}). }

\section{Experiments}
\label{sec:experiments_siren}

In this section, we  validate the effectiveness of \textsc{Siren} on object detection models, including the latest transformer-based  (Section~\ref{sec:detr}) and  flagship CNN-based models (Section~\ref{sec:cls_cnn}).

\textbf{Datasets.}  Following~\cite{du2022towards}, we use \textsc{pascal-voc}\footnotemark[1]~\citep{DBLP:journals/ijcv/EveringhamGWWZ10} and Berkeley DeepDrive (\textsc{bdd100k})\footnotemark[2]~\citep{DBLP:conf/cvpr/YuCWXCLMD20} datasets as the ID training data.  For both tasks, we evaluate on two OOD datasets that contain a subset of images from: \textsc{ms-coco}~\citep{lin2014microsoft} and \textsc{OpenImages} (validation set)~\citep{kuznetsova2020open}. Extensive details on the
datasets are in Appendix~\ref{sec:dataset}. 

\footnotetext[1]{\color{black}\textsc{pascal-voc} consists of the following ID labels: {Person, Car, Bicycle, Boat, Bus, Motorbike, Train, Airplane, Chair, Bottle, Dining Table, Potted Plant, TV, Sofa, Bird, Cat, Cow, Dog, ~Horse, Sheep}.}
\footnotetext[2]{\color{black}\textsc{bdd100k} consists of ID labels: {Pedestrian, Rider, Car, Truck, Bus, Train, Motorcycle, Bicycle, Traffic light, Traffic sign}.}

\textbf{Metrics.} For evaluating the OOD detection performance, we report: (1) the false positive rate (FPR95) of
OOD objects when the true positive rate of ID
samples is at 95\%; (2) the area under the receiver operating characteristic curve (AUROC). For evaluating the object detection performance on the ID task, we report the common metric mAP.

\vspace{-0.2cm}
\subsection{Evaluation on Transformer-based Model}
\label{sec:detr}

\textbf{Experimental details.}  We adopt a very recent \textsc{deformable-detr} (\textsc{ddetr}) architecture~\citep{deformable_detr}.  
\textsc{ddetr}
introduces multi-scale deformable attention modules in the
transformer encoder and decoder layers of \textsc{detr}~\citep{DBLP:conf/eccv/CarionMSUKZ20}, and provides better convergence and
lower complexity. The multi-scale feature maps in \textsc{ddetr} are extracted from a ResNet-50~\citep{he2016deep} pre-trained on ImageNet in a self-supervised fashion, \emph{i.e.}, DINO~\citep{DBLP:conf/iccv/CaronTMJMBJ21}. We use the embeddings at the penultimate layer of the decoder in DDETR for projection. 
For the projection head, we use a two-layer MLP with a ReLU nonlinearity, with dimensionality $256\rightarrow d \rightarrow d$. The dimension $d$ of the unit hypersphere is 16 for \textsc{pascal-voc} and 64 for \textsc{bdd100k}. The default weight $\beta$ for the \textsc{Siren} is 1.5 and the prototype update factor $\alpha$ is 0.95. We initialize the learnable $\kappa$ to be 10 for all classes. The $k$ in the KNN distance is set to 10. Ablations on the hyperparameters are provided in Section~\ref{sec:ablation} and Appendix~\ref{sec:ablation_app}. Other hyperparameters are the same as the default ones in \textsc{ddetr}~\citep{deformable_detr}.

\begin{table}[t]
\centering
\small
\scalebox{0.62}{
\begin{tabular}{lllll}
    \toprule
    {In-distribution dataset}  & 
     
   {\textbf{Method }}  &\textbf{FPR95} $\downarrow$  & \textbf{AUROC} $\uparrow$ & \textbf{ mAP (ID)}$\uparrow$  \\
   \hline
    && \multicolumn{2}{c}{OOD: MS-COCO / OpenImages}\\
     \cline{3-4}
    \multirow{9}{0.2\linewidth}{{\textbf{ PASCAL-VOC} }} 

     & Mahalanobis 
    \citep{lee2018simple}
    &  97.39 / 97.88 &50.28 / 49.08 & 60.6\\ 
        & Gram matrices~\citep{DBLP:conf/icml/SastryO20}&  94.16 / 95.29 &  43.97 / 38.81  & 60.6 \\
    
     & KNN~\citep{sun2022out} &  91.80 / 91.36 & 62.15 / 59.64 & 60.6\\
    
  & CSI~\citep{tack2020csi}& 84.00 / 79.16 & 55.07 / 51.37 &59.5 \\
    & VOS~\citep{du2022towards}&97.46 / 97.07& 54.40 / 52.77 	& 60.3\\
    & OW-DETR~\citep{ow_detr} & 93.09 / 93.82  & 55.70 / 57.80& 58.3 \\
    &Dismax~\citep{dismax} &82.05 / 76.37 & 75.21 / 70.66 & 60.1 \\

         &\cellcolor{Gray}\textbf{\textsc{Siren}}-vMF   (ours) &\cellcolor{Gray}75.49$\pm$0.8 /  78.36$\pm$1.0  & \cellcolor{Gray}76.10$\pm$0.1 / 71.05$\pm$0.1  & \cellcolor{Gray}60.8$\pm$0.1
      
    \\
      &\cellcolor{Gray}\textbf{\textsc{Siren}}-KNN   (ours) &\cellcolor{Gray}\textbf{64.77}$\pm$0.2 /  \textbf{65.99}$\pm$0.5  & \cellcolor{Gray}\textbf{78.23}$\pm$0.2 / \textbf{74.93}$\pm$0.1  & \cellcolor{Gray}60.8$\pm$0.1 \\
      
 &\textbf{\textsc{Siren}}-vMF   (rerun) &76.52 /  79.23 & 76.22 / 71.19 & 60.8\\
  &\textbf{\textsc{Siren}}-KNN   (rerun) & 64.64 / 65.48 & 78.72  / 75.43 & 60.8\\
     \midrule
 \multirow{9}{0.2\linewidth}{\textbf{BDD100K} } 

     & Mahalanobis 
    \citep{lee2018simple}
   	& 70.86 / 71.43 & 76.83 / 77.98 & 31.3\\ 
   	    & Gram matrices~\citep{DBLP:conf/icml/SastryO20}&73.81 / 71.56 & 60.13 / 57.14 & 31.3	 \\

     & KNN~\citep{sun2022out}  & 64.75 / 61.13 &  80.90 / 79.64    & 31.3\\
  
     &  CSI~\citep{tack2020csi}& 70.27 / 71.30 & 77.93 / 76.42 & 29.9 \\
    & VOS~\citep{du2022towards}& 76.44 / 72.58 & 77.33 / 76.62 & 31.0\\
     & OW-DETR~\citep{ow_detr}&  80.78 / 77.37 & 70.29 / 73.78  &  28.1\\
     &Dismax~\citep{dismax} &  77.62 /  81.23   & 72.14 / 67.18& 31.2 \\

             &\cellcolor{Gray}\textbf{\textsc{Siren}}-vMF   (ours) &\cellcolor{Gray}67.54$\pm$1.3	/ 66.31$\pm$0.9 & \cellcolor{Gray}80.06$\pm$0.5 / 79.77$\pm$1.2	& \cellcolor{Gray}31.3$\pm$0.0 \\
       &\cellcolor{Gray}\textbf{\textsc{Siren}}-KNN   (ours) &\cellcolor{Gray}\textbf{53.97}$\pm$0.7 /  \textbf{47.28}$\pm$0.3  & \cellcolor{Gray}\textbf{86.56}$\pm$0.1 / \textbf{89.00}$\pm$0.4  & \cellcolor{Gray}31.3$\pm$0.0 \\
        &\textbf{\textsc{Siren}}-vMF   (rerun) &68.79 / 69.86 & 79.14 / 77.35  & 31.3 \\
         &\textbf{\textsc{Siren}}-KNN   (rerun) & 55.75 /  50.40 & 85.89 / 88.26  & 31.3\\
   
        \bottomrule
\end{tabular}}
     
        \caption[Main results of SIREN]{\textbf{Main results.} Comparison with competitive out-of-distribution detection methods. All baseline methods are based on the same model backbone \textsc{ddetr}. $\uparrow$ indicates larger values are better and $\downarrow$ indicates smaller values are better. All values are percentages. \textbf{Bold} numbers are superior results.  We report standard deviations estimated across 3 runs. \textsc{Siren}-vMF/KNN denotes using vMF score and KNN distance during inference. }
        \label{tab:baseline_siren}
     
\end{table}

\paragraph{\textsc{Siren} achieves superior performance.} In Table~\ref{tab:baseline_siren}, we compare \textsc{Siren} with competitive OOD detection methods in literature. For a fair comparison, all the methods only use ID data for training.   \textsc{Siren} outperforms competitive baselines, including Mahalanobis distance~\citep{lee2018simple}, KNN distance~\citep{sun2022out}, CSI~\citep{tack2020csi}, Gram matrices~\citep{DBLP:conf/icml/SastryO20} and Dismax~\citep{dismax}. These baselines operate on feature embedding space, allowing a fair comparison. Note that other common output-based methods (such as MSP~\citep{hendrycks2016baseline}, ODIN~\citep{liang2018enhancing}, and energy~\citep{liu2020energy}) are not directly applicable for multi-label classification networks in \textsc{ddetr}. For these methods relying on a multi-class classification model, we will later provide comparisons on the Faster R-CNN model in Section~\ref{sec:cls_cnn}. Implementation details and the training time for all the baseline methods are reported in Appendix~\ref{sec:reproduce_baseline} and~\ref{app:training_time}. We highlight a few  observations: 

\quad 1) The comparison between \textbf{SIREN vs. Mahalanobis} highlights precisely the benefits of our embedding shaping loss $\mathcal{L}_\text{SIREN}$. As shown in Figure~\ref{fig:teaser_siren}, the vanilla \textsc{detr} model produces ill-conditioned embeddings that do not conform to multivariate Gaussian distributions, rendering the Mahalanobis approach ineffective (with AUROC around 50\%--which is random guessing). In contrast, \textsc{Siren}-vMF improves the OOD detection performance (AUROC) by \textbf{25.82}\% on \textsc{pascal-voc}  (\textsc{ms-coco} as OOD). Different from the Mahalanobis distance, our parametric vMF scoring function naturally suits the  learned hyperspherical embeddings with vMF distributions. The advantage of the vMF loss can be further verified by observing that \textsc{Siren}-KNN outperforms directly applying KNN distance on the vanilla \textsc{ddetr} (\textbf{16.08}\% AUROC improvement on \textsc{voc} with \textsc{coco} as OOD).

\quad 2)  \textsc{Siren} outperforms the latest methods  VOS~\citep{du2022towards} and OW-DETR~\citep{ow_detr}, which are designed for object detection models and serve as strong baselines for us. Compared with OW-DETR, \textsc{Siren}-KNN substantially improves the AUROC by \textbf{22.53}\%  on \textsc{pascal-voc} (\textsc{coco} as OOD). OW-DETR uses the unmatched object queries with high confidence  as the unknowns, and  trains a binary classifier to separate ID and unknown objects. However, the unmatched object queries might be distributionally too close to the ID classes and thus displays limited improvement for OOD detection. In addition, VOS synthesizes virtual outliers from the class-conditional Gaussian distributions of the penultimate layer but fails to perform well due to the ill-conditioned embedding distribution in \textsc{ddetr} (non-Gaussian).
\subsection{Evaluation on CNN-based Model}
\label{sec:cls_cnn}

\begin{wraptable}{r}{0.47\linewidth}
    \centering
    
    \tabcolsep 0.04in\renewcommand\arraystretch{0.745}{}
    \scalebox{0.65}{
    \begin{tabular}{c|cccc}
    \toprule

\multirow{1}{*}{ \textbf{Method}} & \textbf{FPR95} $\downarrow$ & \textbf{AUROC} $\uparrow$ & Time\\
 
        \midrule
        &\multicolumn{2}{c}{COCO/ OpenImages}\\
        \cline{2-3}
     MSP 
  &   70.99 / 73.13  & 83.45 /  81.91 & 2.1 h   \\
     ODIN
   &  59.82 / 63.14 &  82.20 / 82.59  & 2.1 h \\
      Mahalanobis

    & 96.46 / 96.27 & 59.25 / 57.42&2.1 h  \\
           Energy score
    & 56.89 / 58.69  & 83.69 / 82.98 &2.1 h  \\ 
         Gram matrices&62.75 / 67.42& 79.88 /	77.62&2.1 h  \\
   KNN &   52.67 / 53.67 & 87.14 /  84.54 & 2.1 h\\
     
   CSI& 59.91 / 57.41  & 81.83 / 82.95 & 4.9 h\\
     GAN-synthesis&	60.93  / 59.97 &  83.67 / 82.67	& 3.7 h \\

      VOS &47.53 / 51.33&88.70 / 85.23 & 4.3 h \\ 
     Dismax &  84.38 / 86.93   & 74.56 / 71.53& 2.2 h \\
       \rowcolor{Gray}  \textbf{\textsc{Siren}}-vMF (ours) &    64.68 / 68.53 &  85.36 / 82.78 &  2.1 h\\
       \rowcolor{Gray}  \textbf{\textsc{Siren}}-KNN (ours) &  \textbf{47.45}  / \textbf{50.38}  &  \textbf{89.67} /	\textbf{88.80} & 2.1 h\\
        \bottomrule
    \end{tabular}
}
      \caption[OOD detection results of {SIREN} for object detection]{ OOD detection results of \textsc{Siren} and comparison with competitive baselines on two OOD datasets: COCO and OpenImages.}
\vspace{-1em}
    \label{tab:baseline_siren_cnn}
\end{wraptable}

Going beyond detection transformers, we show that \textsc{Siren} is also suitable and effective on CNN-based object-level  OOD detection  models, \emph{e.g.}, Faster R-CNN~\citep{ren2015faster}. Table~\ref{tab:baseline_siren_cnn} showcases the OOD detection performance with \textsc{Siren} trained on \textsc{pascal-voc} dataset and evaluated on both \textsc{ms-coco} and \textsc{OpenImages} datasets. In addition to baselines considered in Table~\ref{tab:baseline_siren}, we include common output-based methods relying on multi-class classification, such as MSP, ODIN, and energy score.

In Table~\ref{tab:baseline_siren_cnn}, we additionally report training time comparison. \textsc{Siren} consistently improves OOD detection performance on both OOD datasets. Notably, \textsc{Siren} performs better than the previous best OOD detection approach VOS on Faster R-CNN while preserving the same training time as vanilla Faster R-CNN.

\section{Ablations and Discussions}
\label{sec:ablation}
In this section, we provide ablation results on how different factors impact the performance of \textsc{Siren}. For consistency, we present the analyses below based on the \textsc{ddetr} model. Unless otherwise pointed, we use the KNN distance by default. 

\textbf{Ablations on learnable vs. fixed $\kappa$.} In this ablation, we show that our approach using learnable concentration parameters $\{\kappa_c\}_{c=1}^C$ is better than using the fixed ones. Our method with learnable $\kappa$ is desirable, since each ID class may have its own concentration in the hypersphere. With fixed $\kappa$, the loss function can be simplified as follows:
\begin{equation}
    \mathcal{L}=-\frac{1}{M}\sum_{i=1}^{M} \log \frac{\exp \left(\kappa \boldsymbol{\mu}_{y_i}^{\top} \mathbf{r}_{i}\right)}{\sum_{j=1}^{C}  \exp \left(\kappa \boldsymbol{\mu}_{j}^{\top} \mathbf{r}_{i}\right)}.
    \label{eq:vmf}
\end{equation}

Empirically, we indeed observe that employing learnable $\{\kappa_c\}_{c=1}^C$ achieves better OOD detection performance, with AUROC 78.23\% on \textsc{pascal-voc} model (\textsc{ms-coco} as OOD). In Figure~\ref{fig:ablation_frame_range_siren} (a), we show \textsc{Siren}'s performance when trained under different fixed $\kappa$ values. The best model under the fixed $\kappa=10$ achieves an AUROC of 77.67\%. Too small of $\kappa$ value (\emph{e.g.}, $\kappa=1$) leads to almost uniform distributions and is therefore not desirable.  

\begin{figure}[t]
    \centering
  \includegraphics[width=1.0\linewidth,bb=0 0 1440 300]{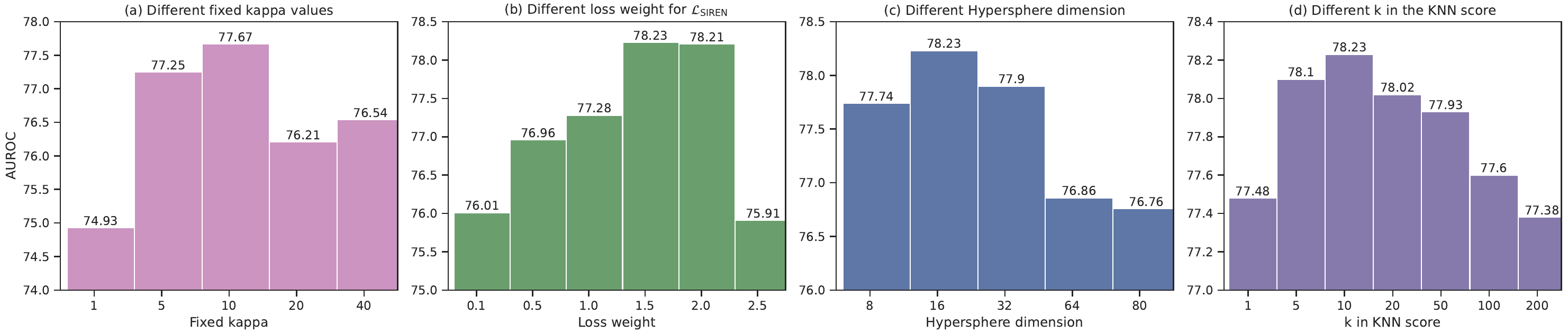}
    \caption[Ablation study for SIREN]{(a) Ablation study on using different fixed values of concentration parameter $\kappa$ in Equation~\eqref{eq:vmf}. (b) Ablation on different weights $\beta$ for $\mathcal{L}_{\textrm{SIREN}}$ in Equation~\eqref{eq:totalloss}. (c) Ablation on the dimension of the hyperspherical embeddings $\mathbf{r}$.  (d) Ablation study on the parameter $k$ in the KNN-based OOD detection score for \textsc{Siren}. Numbers are AUROC. The  ID training dataset is \textsc{pascal-voc} and OOD dataset is \textsc{ms-coco}.}
    \label{fig:ablation_frame_range_siren}
\end{figure}

\paragraph{Ablations on the \textsc{Siren} loss weight $\beta$.} Figure~\ref{fig:ablation_frame_range_siren} (b) reports the OOD detection results as we vary the
weight $\beta$ for the representation shaping loss $\mathcal{L}_{\textrm{SIREN}}$. The model is evaluated on the
\textsc{ms-coco} dataset as OOD. Overall a mild weight works well. Across all the $\beta$ values considered, \textsc{Siren} consistently outperforms the  baseline OOD detection methods in Table~\ref{tab:baseline_siren}. One can use our \textsc{Siren} loss as an easy plug-in module, with minimal hyperparameter tuning.

\paragraph{Ablations on the dimension $d$ of  hypersphere.} \textsc{Siren} projects the object feature embeddings into a lower-dimensional hypersphere in $\mathbb{R}^d$, which allows tractable vMF estimation. Figure~\ref{fig:ablation_frame_range_siren} (c) shows the effect the embedding dimension $d$ on the OOD detection performance. We find that a lower dimension between 16 and 64 achieves favorable and stable performance. In the extreme case with dimension $d=8$, the model suffers from considerable information loss and degraded performance. On the other hand, too large of $d$ causes training instability, which is not desirable either.

\textbf{Ablations on the uncertainty score.} We perform ablation on two variants of the vMF-based OOD detection score (\emph{c.f.} Equation~\eqref{eq:vmf_uncertainty}): using the learned $\{\kappa_c\}_{c=1}^C$ vs. approximately estimate the $\kappa$ parameters directly from the converged embeddings. In literature, there are several established methods for approximating $\kappa$~\citep{mardia2000directional}. Denote $\mathbf{\overline{r}}_c$ the average of object embeddings for class $c$, the simplest approximate solution is given as follows:

\begin{equation}
    \widehat{\kappa_c} = \frac{\|\mathbf{\overline{r}}_c\|(d-\|\mathbf{\overline{r}}_c\|^2) }{1-\|\mathbf{\overline{r}}_c\|^2},
    \label{eq:estimating_kappa_test_time}
\end{equation}
   
The proof is given in Appendix~\ref{sec:proof}. We show in  Table~\ref{tab:ablation_score}  that using learned $\{\kappa_c\}_{c=1}^C$ avoids the imprecision in approximating $\widehat {\kappa_c}$ directly from embeddings. This affirms the importance of employing learnable concentration parameters in our \textsc{Siren} loss. Moreover, the non-parametric KNN density estimation provides stronger flexibility and generality, and leads to better performance. 
\begin{table}[h]
    \centering

    \scalebox{1.0}{
  \begin{tabular}{c|cc}
    \hline
         & FPR95 $\downarrow$  & AUROC $\uparrow$     \\
        \hline
    
        & \multicolumn{2}{c}{  \scriptsize COCO / OpenImages as OOD}\\
        \cline{2-3}

      vMF w/ $\kappa$ from~\citep{DBLP:journals/cstat/Sra12} &  78.60 / 78.42 & 73.03 / 70.27\\
        vMF w/ learned $\kappa$ (ours) & 75.49 / 78.36 & 76.10 / 71.05  \\ 
      Non-parametric KNN distance& \textbf{64.77} / \textbf{65.99}& \textbf{78.23} / \textbf{74.93} \\
         \hline
    \end{tabular}
}
      \caption{ Ablation on different OOD detection scores. The ID dataset is \textsc{pascal-voc}.}

    \label{tab:ablation_score}
\end{table}

\paragraph{Ablations on the projection head.}~We ablate on the nonlinearity in the projection head by comparing with \textsc{Siren} trained with a linear layer in Table~\ref{tab:ablation_nonlinear}. The result shows using a nonlinear mapping  for projection helps obtain a more expressive hypersphere, which improves OOD detection by 4.37\% in terms of AUROC (\textsc{ms-coco} as OOD, vMF as the OOD score). 

\begin{table}[h]
    \centering
    \scalebox{1.0}{
   \begin{tabular}{c|cc}
    \hline
         & FPR95 $\downarrow$  & AUROC $\uparrow$     \\
        \hline
        & \multicolumn{2}{c}{  \scriptsize COCO / OpenImages as OOD}\\
        \cline{2-3}

   vMF w/o  nonlinearity &82.85 / 83.69 &  71.73 / 66.82 \\

        vMF w/  nonlinearity & \textbf{75.49} / \textbf{78.36} & \textbf{76.10} / \textbf{71.05}  \\ 
      \cdashline{1-3}[1pt/1pt]
      KNN  w/o  nonlinearity & 67.03 / 72.11 & 78.02 / 72.42  \\
      KNN  w/  nonlinearity &\textbf{64.77} / \textbf{65.99}& \textbf{78.23} / \textbf{74.93}   \\
         \hline
    \end{tabular}
}
      \caption{  Ablation on the projection head. The ID dataset is \textsc{pascal-voc}.}

    \label{tab:ablation_nonlinear}
\end{table}

  \section{Qualitative analysis}

In Figure~\ref{fig:visualization_example}, we visualize the predictions on several OOD images, using object detection models
trained without \textsc{Siren} (top) and with \textsc{Siren} (bottom), respectively. The in-distribution data
is \textsc{bdd100k}. \textsc{Siren}  better identifies OOD objects (in green) compared to a vanilla object
detector \textsc{ddetr}, reducing false positives. Moreover, the confidence score of the
false-positive objects of \textsc{Siren} is lower than that of the vanilla model (see the train/bicycle in the 3rd/5th column).

\begin{figure}[t]
    \centering
    \includegraphics[width=1.0\textwidth,bb=0 0 1050 250]{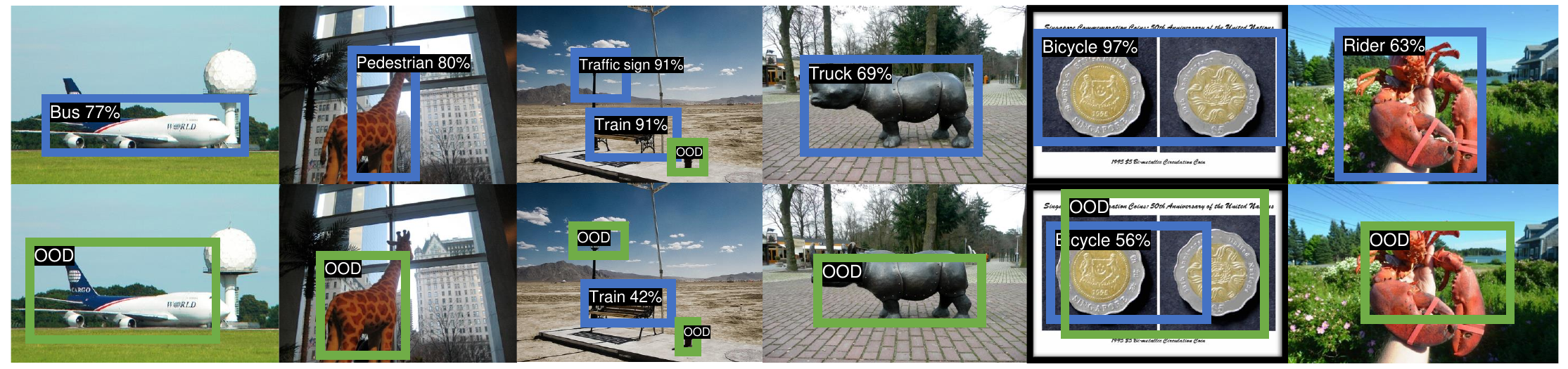}
    \caption[Visualization of detected objects on the OOD images]{ Visualization of detected objects on the OOD images (from \textsc{ms-coco} and \textsc{OpenImages}) by vanilla \textsc{ddetr} (\emph{top}) and \textsc{Siren} (\emph{bottom}). The ID is \textsc{bdd100k} dataset. \textbf{Blue}: OOD objects classified as the ID classes. \textbf{Green}: OOD objects detected by \textsc{ddetr} and \textsc{Siren}, which reduce false positives. }
\vspace{-0em}
\label{fig:visualization_example}
\end{figure}

\section{Summary}
In this chapter, we propose a novel framework \textsc{Siren}, which tackles object-level OOD detection with a distance-based approach. \textsc{Siren} mitigates the key shortcoming of the previous output-based OOD detection approach, and explores a new vMF loss to shape representations for OOD detection.
To the best of our knowledge, \textsc{Siren} makes the first attempt to employ vMF-based learning and inference for OOD detection. \textsc{Siren} establishes competitive performance on challenging object-level OOD detection tasks, evaluated broadly under both the recent detection transformers and CNN-based models. Our in-depth ablations provide further insights on the efficacy of \textsc{Siren}. We hope our work inspires future research on OOD detection with representation shaping.

\end{siren}

\begin{overviewtwo}
    \label{sec:cha7}
\textbf{Motivation.} Previous research utilizing auxiliary outlier datasets has shown promise in improving OOD detection compared to models trained solely on ID data. These models often regularize confidence scores~\citep{hendrycks2018deep} or energy levels~\citep{liu2020energy} for outlier data. However, this approach faces two significant limitations: (1) auxiliary data collected offline may not accurately reflect the true distribution of unknown data encountered in real-world settings, potentially undermining OOD detection during deployment; (2) collecting such data is often labor-intensive and requires careful cleaning to avoid overlap with ID data.

My research addresses these challenges by \textbf{building novel learning algorithms and theories that leverage unlabeled "in-the-wild" data}, which can be gathered at minimal cost during the deployment of machine learning models. This type of data has been largely overlooked in OOD learning contexts. Formally, unlabeled data can be represented by a Huber contamination model, $\mathbb{P}_\text{unlabeled}:= (1-\pi) \mathbb{P}_\text{in} + \pi \mathbb{P}_\text{out}$, where $\mathbb{P}_\text{in}$ and $\mathbb{P}_\text{out}$ represent the marginal distributions of ID and OOD data, respectively.  Unlabeled data is abundant, does not require human annotation, and often better aligns with true test-time distributions compared to offline-collected data. While this setting is promising for various applications, it introduces unique challenges due to the impure nature of the unlabeled data which comprises both ID and OOD samples.

 \textbf{Theoretical foundations of learning with wild data.} In my ICLR'24 paper SAL~\citep{du2024sal}, I conducted the first formal investigation on \textit{when and how unlabeled data can enhance OOD detection}. My contributions are: \textbf{(1)} I proposed a novel learning framework that {separates} candidate outliers using singular value decomposition on the model gradient matrix, which facilitates the {learning} of an OOD classifier. \textbf{(2)} The framework provides theoretical support for the filtering error and the generalization error of the OOD classifier, proving that these errors can be small under specific conditions. \textbf{(3)} Empirically, I demonstrated the generalization bounds of SAL translate into strong empirical performance, establishing state-of-the-art results through extensive evaluations. Particularly, Chapter~\ref{sec:cha8} is going to mainly discuss the work of SAL.

\textbf{Learning with diverse data shifts in the wild.} Beyond detecting samples with semantic shifts, machine learning models may also encounter covariate shifts—variations in input distributions that do not necessarily affect labels. These shifts can arise from differences in sensor calibration or environmental changes. My ICML'23 paper, SCONE~\citep{bai2023feed}, along with a recent preprint~\citep{bai2024out} were \textbf{the first} to explore techniques for modeling diverse mixtures of data shifts, expressed as $\mathbb{P}_\text{unlabeled} := (1-\pi_c-\pi_s) \mathbb{P}_\text{in} + \pi_c \mathbb{P}_\text{out}^\text{covariate} +\pi_s \mathbb{P}_\text{out}^\text{semantic}$. This framework integrates ID data and various OOD distributions, and employs \textit{constrained optimization} and \textit{active learning} to effectively learn with these diverse data sources for both OOD detection and OOD generalization. The formulation and methodologies presented offer strong generality and practicality for real-world applications.
\end{overviewtwo}

\begin{sal}
    \label{sec:cha8}
\textbf{Publication Statement.} This chapter is joint work with Zhen Fang, Ilias Diakonikolas and
Yixuan Li. The paper version of this chapter appeared in ICLR'24~\citep{du2024sal}.

\noindent \textbf{Abstract.} Using unlabeled data to regularize the machine learning models has demonstrated promise for improving safety and reliability in detecting out-of-distribution (OOD) data. Harnessing the power of unlabeled \textcolor{black}{in-the-wild} data is non-trivial due to the heterogeneity of both in-distribution (ID) and OOD data. This lack of a clean set of OOD samples poses significant challenges in learning an optimal OOD classifier. Currently, there is a lack of research on formally understanding how unlabeled data helps OOD detection. 
 This chapter bridges the gap by introducing a new learning framework \textsc{SAL} (\textbf{S}eparate
\textbf{A}nd \textbf{L}earn) that offers both strong theoretical guarantees and empirical effectiveness. The framework separates candidate outliers from the unlabeled data and then trains an OOD classifier using the candidate outliers and the labeled ID data. Theoretically, we provide rigorous error bounds from the lens of separability and
learnability, formally justifying the two components in our algorithm.  Our theory shows that \textsc{SAL} can separate the candidate outliers with small error rates, which leads to a generalization guarantee for the learned OOD classifier.  Empirically, \textsc{SAL} achieves state-of-the-art performance on common benchmarks, reinforcing our theoretical insights. Code is publicly available at
\url{https://github.com/deeplearning-wisc/sal}.

\section{Introduction}

When deploying machine learning models in real-world environments, their safety and reliability are
often challenged by the occurrence of out-of-distribution (OOD) data, \textcolor{black}{which arise from unknown categories and should not be predicted by the model}. Concerningly, neural networks are brittle and lack the necessary awareness of OOD data in the wild~\citep{nguyen2015deep}. 
Identifying OOD inputs is a vital but fundamentally challenging problem---the models are not explicitly exposed to the unknown distribution during training, and therefore cannot capture a reliable boundary between in-distribution (ID) vs. OOD data.
To circumvent the challenge, researchers have started to explore training with additional data, which can facilitate a conservative and safe decision boundary against OOD data. In particular, a recent work \textcolor{black}{by}~\cite{katzsamuels2022training} proposed to leverage unlabeled data in the wild to regularize model training, while learning to classify labeled ID data. Such unlabeled \textcolor{black}{wild} data offer the benefits of being freely collectible upon deploying
any machine learning model in its operating environment, and allow capturing the true test-time OOD distribution.

Despite the promise, harnessing the power of unlabeled \textcolor{black}{wild} data is non-trivial due to the heterogeneous mixture of ID and OOD data. This lack of a clean set of OOD training data poses significant challenges in designing effective OOD learning algorithms. Formally, the unlabeled data can be characterized by a Huber contamination model $\mathbb{P}_\text{wild}:= (1-\pi) \mathbb{P}_\text{in} + \pi \mathbb{P}_\text{out}$, where  $\mathbb{P}_\text{in}$ and $\mathbb{P}_\text{out}$ are the marginal distributions of the ID and OOD data. 
It is important to note that the learner only observes samples drawn from such mixture distributions, without knowing the clear membership of whether being ID or OOD. \textcolor{black}{Currently, a formalized understanding of the
problem is lacking for the field}.
This prompts the question underlying the present work: 
\begin{center}
\begin{tcolorbox}[enhanced,attach boxed title to top center={yshift=-3mm,yshifttext=-1mm},colback=gray!5!white,colframe=gray!75!black,colbacktitle=red!80!black,
  title=,fonttitle=\bfseries, boxed title style={size=small,colframe=red!50!black},width=3.8in, boxsep=1pt ]
\emph{How does unlabeled \textcolor{black}{wild} data provably help OOD detection?}
\end{tcolorbox}
\end{center}

\textbf{Algorithmic contribution.} In this chapter, we propose a new learning framework \textsc{SAL} (\textbf{S}eparate \textbf{A}nd \textbf{L}earn), that effectively exploits the unlabeled wild data for OOD detection. At a high level, our framework \textsc{SAL} 
builds on two consecutive components: \textbf{(1)}  \textcolor{black}{filtering}---separate \emph{candidate outliers} from the unlabeled data, and \textbf{(2)} \textcolor{black}{classification}---learn an OOD classifier with the candidate outliers, in conjunction with the labeled ID data. To separate the candidate outliers,
our key idea is to
perform singular value decomposition on a gradient matrix, defined over all the unlabeled data \textcolor{black}{whose gradients are computed based on a classification model trained on the clean labeled ID data.} 
\textcolor{black}{In the \textsc{SAL} framework, unlabeled wild data are considered candidate outliers when their projection onto the top singular vector exceeds a given threshold. The filtering strategy for identifying candidate outliers is theoretically supported by Theorem \ref{MainT-1}. {We show in Section~\ref{sec:method_sal} (Remark \ref{R1}) that under proper conditions, with a high probability, there exist some specific directions (e.g., the top singular vector direction) where the mean magnitude of the gradients for the wild outlier data is \textcolor{black}{larger than that of ID data}}.} After obtaining the outliers from the wild data, we train an OOD classifier that optimizes the classification between the ID vs. candidate outlier data for OOD detection. 

\textbf{Theoretical significance.} Importantly, we provide new theories from the lens of \emph{separability} and \emph{learnability}, formally justifying the two components in our algorithm. 
Our main Theorem~\ref{MainT-1} analyzes the separability of outliers from unlabeled \textcolor{black}{wild} data using our filtering procedure, and gives a rigorous bound on the error rate. Our theory has practical implications. For example, when the size of the labeled ID data and unlabeled data is sufficiently large, Theorems~\ref{MainT-1} and \ref{The-1.1} imply that the error rates of filtering outliers can be bounded by a small bias proportional to the optimal ID risk, which is a small value close to zero in reality~\citep{frei2022benign}. Based on the error rate estimation, we give a generalization error of the OOD classifier in  Theorem \ref{the:main2}, to quantify its learnability on the  ID data and a noisy set of candidate outliers. Under proper conditions, the generalization error of the learned OOD classifier is \textcolor{black}{upper bounded by the risk  associated with the optimal OOD classifier.}

\textbf{Empirical validation.} Empirically, we show that the generalization bound w.r.t. \textsc{SAL} (Theorem~\ref{the:main2}) indeed translates into strong empirical performance. \textsc{SAL} can be broadly applicable to non-convex models such as modern neural networks.
We extensively evaluate \textsc{SAL} on common OOD detection tasks and establish state-of-the-art performance. For
completeness, we compare \textsc{SAL} with two families of methods:
(1) trained with only $\mathbb{P}_\text{in}$, and (2) trained with both
$\mathbb{P}_\text{in}$ and an unlabeled dataset. On \textsc{Cifar-100}, compared to
a strong baseline KNN+~\citep{sun2022out} using only $\mathbb{P}_\text{in}$, \textsc{SAL} outperforms
by 44.52\% (FPR95) on average.
While methods such as Outlier Exposure~\citep{hendrycks2018deep} require a clean set of auxiliary unlabeled data, our results are achieved without imposing any such assumption on the unlabeled data and hence offer stronger flexibility. 
Compared to the most related baseline WOODS~\citep{katzsamuels2022training}, our framework can reduce the FPR95 from
7.80\% to 1.88\% on \textsc{Cifar-100}, establishing  near-perfect results on this challenging benchmark. 

\section{Proposed Methodology}
\label{sec:method_sal}
In this section, we introduce a new learning framework \textsc{SAL} that performs
OOD detection \textcolor{black}{by leveraging} the unlabeled wild data. The framework offers substantial advantages over the counterpart approaches that rely only on the ID data, and naturally suits many applications where machine learning models are deployed in the open world.
\textsc{SAL} has two integral components: \textbf{(1)} \textcolor{black}{filtering}---separate the candidate outlier data from the unlabeled wild data  (Section~\ref{sec:detect}),  and \textbf{(2)} \textcolor{black}{classification}---train a binary OOD classifier  with the ID data and candidate outliers  (Section~\ref{sec:training_sal}).  In Section~\ref{sec:theory}, we provide theoretical guarantees for \textsc{SAL}, provably justifying the two components in our method.

\subsection{Separating Candidate Outliers from the Wild Data}
\label{sec:detect}
\textcolor{black}{To separate candidate outliers from the wild mixture $\mathcal{S}_{\text{wild}}$, our framework employs a level-set estimation based on the gradient information. The gradients are estimated from a classification predictor $\*h_\*w$ trained on the ID data $\mathcal{S}^{\text{in}}$.  We describe the procedure  formally below. }

\textbf{Estimating the reference gradient from ID data.} To begin with, \textsc{SAL} estimates the reference gradients by training a classifier  $\*h_\*w$ on the ID data $\mathcal{S}^{\text{in}}$ by empirical risk minimization (ERM):
\begin{equation}
\*w_{\mathcal{S}^{\text{in}}} \in \argmin_{\*w\in \mathcal{W}} R_{\mathcal{S}^{\text{in}}}(\*h_\*w),~~\text{where}~~R_{\mathcal{S}^{\text{in}}} (\*h_\*w)= \frac{1}{n}\sum_{(\*x_i, y_i) \in \mathcal{S}^{\text{in}}} \ell (\*h_\*w(\*x_i),y_i),
    \label{eq:erm}
\end{equation}
$\*w_{\mathcal{S}^{\text{in}}}$ is the learned parameter and $n$ is the size of ID training set $\mathcal{S}^{\text{in}}$. {The average gradient $\bar{\nabla}$ is}
\begin{equation}
    \bar{\nabla}=\frac{1}{n} \sum_{(\*x_i, y_i) \in \mathcal{S}^{\text{in}}} \nabla \ell (\*h_{\*w_{\mathcal{S}^{\text{in}}}}(\*x_i),y_i),
    \label{eq:reference_gradient}
\end{equation}
 where $\bar{\nabla} $  acts as a reference gradient that allows measuring the deviation of any other points from it.

\textcolor{black}{\textbf{Separate candidate outliers from the unlabeled wild  data.} After training the classification predictor on the labeled ID data, we deploy the trained predictor $\*h_{\*w_{\mathcal{S}^{\text{in}}}}$ in the wild, and naturally receives data $\mathcal{S}_{\text{wild}}$---a mixture of unlabeled ID and OOD data. Key to our framework, we perform a filtering procedure on the wild data $\mathcal{S}_{\text{wild}}$, identifying candidate outliers based on a filtering score. To define the filtering score, we represent each point in $\mathcal{S}_{\text{wild}}$ as a gradient vector, relative to the reference gradient $\bar{\nabla} $.}
\textcolor{black}{
Specifically, we calculate the gradient matrix (after subtracting the reference gradient $\bar{\nabla}$) for the wild data as follows:}
\begin{equation}
    \mathbf{G} = \left[\begin{array}{c}
    \nabla \ell (\*h_{\*w_{\mathcal{S}^{\text{in}}}}(\tilde{\*x}_1), \widehat{y}_{\tilde{\*x}_1})-\bar{\nabla} \\
    ...\\
    \nabla \ell (\*h_{\*w_{\mathcal{S}^{\text{in}}}}(\tilde{\*x}_m), \widehat{y}_{\tilde{\*x}_m})-\bar{\nabla}
    \end{array}\right]^\top,
    \label{eq:gradient_matrix}
\end{equation}
\textcolor{black}{where $m$ denotes the size of the wild data, \textcolor{black}{and $\widehat{y}_{\tilde{\*x}}$ is the predicted label for a wild sample $\tilde{\*x}$}. }\textcolor{black}{For each data  point $\tilde{\*x}_i$ in $\mathcal{S}_{\text{wild}}$, we then define our filtering score  as follows:}
\begin{equation}
    \tau_i =\left<\nabla \ell (\*h_{\*w_{\mathcal{S}^{\text{in}}}}\big(\tilde{\*x}_i), \widehat{y}_{\tilde{\*x}_i})-\bar{\nabla} , \*v\right > ^2,
    \label{eq:score}
\end{equation}
\textcolor{black}{where $\left<\cdot, \cdot\right>$ is the dot product operator and $\*v$ is the top singular vector of  $\mathbf{G}$. The top singular vector $\*v$  can be regarded as the principal component of the matrix $\mathbf{G}$ in Eq.~\ref{eq:gradient_matrix}, which maximizes the total distance from the projected gradients (onto the direction of $\*v$) to the origin (sum over all points in $\mathcal{S}_{\text{wild}}$)~\citep{hotelling1933analysis}. Specifically, $\*v$ is a unit-norm vector and can be computed as follows:}
\vspace{-0.1em}
\begin{equation}
 \vspace{-0.1em}
    \*v \in  \argmax_{\|\*u\|_2 = 1} \sum_{\tilde{\*x}_i \in \mathcal{S}_{\text{wild}}} \left< \*u, \nabla \ell (\*h_{\*w_{\mathcal{S}^{\text{in}}}}\big(\tilde{\*x}_i),   \widehat{y}_{\tilde{\*x}_i})-\bar{\nabla} \right>^2.
    \label{eq:pca}
\end{equation}
Essentially, the filtering score $\tau_i$ in Eq.~\ref{eq:score} measures the $\ell_2$ norm of the projected vector. \textcolor{black}{To help readers better understand our design rationale, we provide an illustrative example of the gradient vectors and their projections in Figure~\ref{fig:pca} (see caption for details)}. \textcolor{black}{Theoretically, Remark~\ref{rem:1} below shows that the projection of the OOD gradient vector to the top singular vector of the gradient matrix $\*G$ is  on average provably larger than that of the ID gradient vector, which rigorously justifies our idea of using the score $\tau$ for separating the ID and OOD data. }
\textcolor{black}{
\begin{Remark}\label{R1}
\label{rem:1}
    Theorem \ref{the:main4.0-app} in Appendix~\ref{sec:proof_1} has shown that under proper assumptions, if we have sufficient data and large-size model, then with the high probability:
    \begin{itemize}
    \item the mean projected magnitude of OOD gradients in the direction of the top singular vector of $\*G$ can be lower bounded by a positive constant $C/\pi$;
    \item the mean projected magnitude of ID gradients in the direction of the top singular vector is upper bounded by a small value close to zero.
\end{itemize}
\end{Remark}}

Finally, we regard $\mathcal{S}_T = \{\tilde{\mathbf{x}}_i\in \mathcal{S}_{\text{wild}}: \mathbf{\tau}_i>T\}$ as the (potentially noisy) candidate outlier set, where $T$ is the filtering threshold. The threshold can be chosen on the ID data $\mathcal{S}^{\text{in}}$ so that a high fraction (e.g., 95\%) of ID samples is below it. In Section~\ref{sec:theory}, we will provide formal guarantees, rigorously justifying that the set $\mathcal{S}_T$ returns outliers with \textcolor{black}{a large probability}. \textcolor{black}{We discuss and compare with alternative gradient-based scores (e.g., GradNorm~\citep{huang2021importance}) for filtering in Section~\ref{sec:results}.} {In Appendix~\ref{sec:num_of_sing_vectors}, we discuss the variants of using multiple singular vectors, which yield similar results.}

\begin{figure*}
    \centering
    \scalebox{1.0}{
    \includegraphics[width=\linewidth]{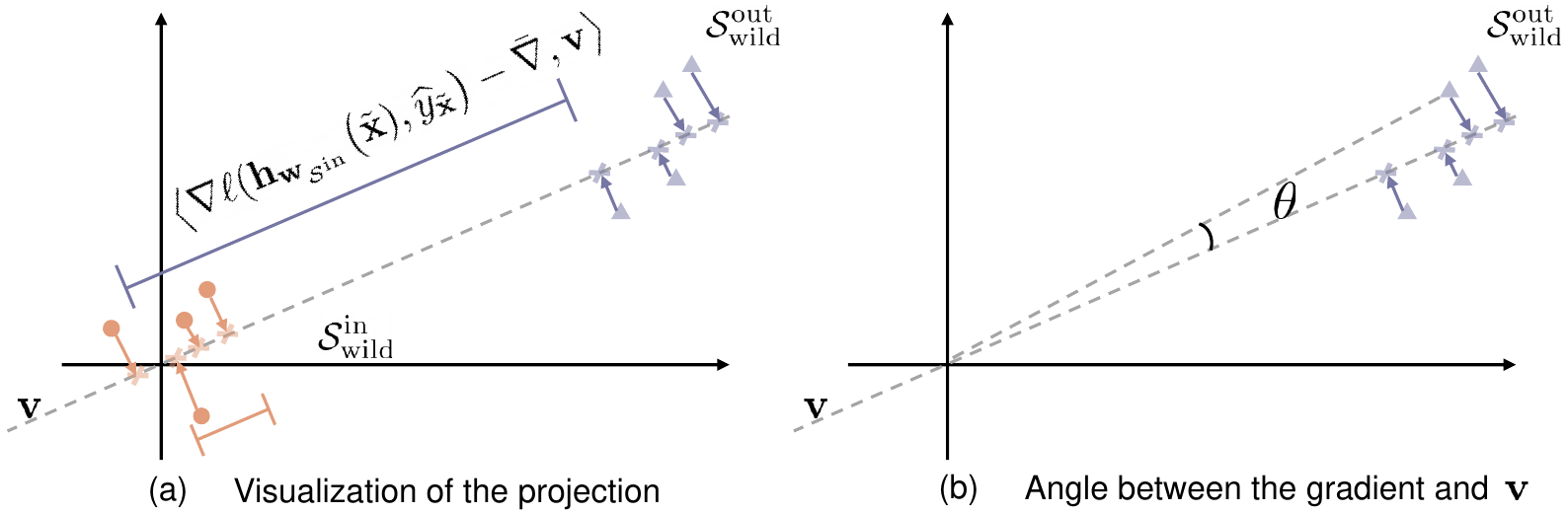}
    }
      \vspace{-1.5em}
    \caption[Visualization of the approach SAL]{ (a) 
    \textcolor{black}{Visualization of the gradient vectors, and their projection onto the top singular vector $\*v$ (in gray dashed line). The gradients of inliers from $\mathcal{S}_{\text{wild}}^{\text{in}}$ (colored in {orange}) are close to the origin (reference gradient $\bar{\nabla}$). In contrast, the gradients of outliers from $\mathcal{S}_{\text{wild}}^{\text{out}}$ (colored in {purple}) are farther away.     } (b) The angle $\theta$ between the gradient of set $\mathcal{S}_{\text{wild}}^{\text{out}}$ and the singular vector $\*v$. Since $\*v$ is searched to maximize the distance from the projected points (cross marks) to the origin (sum over all the gradients in $\mathcal{S}_{\text{wild}}$), $\*v$ points to the direction of OOD data in the wild with a small $\theta$. \textcolor{black}{This further translates into a high  filtering score $\tau$, which is essentially the norm after projecting a gradient vector onto $\*v$. As a result, filtering outliers by $\mathcal{S}_T = \{\tilde{\mathbf{x}}_i\in \mathcal{S}_{\text{wild}}: \mathbf{\tau}_i>T\}$ will approximately return the purple  OOD samples in the wild data. }}
  \vspace{-0.5em}
    \label{fig:pca}
\end{figure*}

\textbf{An illustrative example of algorithm effect.} To see the effectiveness of our filtering score, we test on two simulations in Figure~\ref{fig:toy} (a). These simulations are constructed with simplicity in mind, to facilitate understanding. Evaluations on complex high-dimensional data will be provided in Section~\ref{sec:exp}. In particular, the wild data is a mixture of ID (multivariate Gaussian with three classes) and OOD. We consider two scenarios of OOD distribution, with ground truth colored in {purple}. Figure~\ref{fig:toy} (b) exemplifies the  outliers (in {green}) identified using our proposed method, which largely aligns with the ground truth. The error rate of $\mathcal{S}_T$ containing ID data is only $8.4\%$ and $6.4\%$ for the two scenarios considered. Moreover, the filtering score distribution displays a clear separation between the ID vs. OOD parts, as evidenced in Figure~\ref{fig:toy} (c).

\textbf{Remark 2.} \emph{Our filtering process can be easily extended into $K$-class classification. In this case, one can maintain a class-conditional reference gradient $ \bar{\nabla}_k$, one for each class $k \in [1,K]$, estimated on ID data belonging to class $k$, {which captures the characteristics for each ID class.} Similarly, the top singular vector computation can also be performed in a class-conditional manner, where we replace the gradient matrix with the {class-conditional} $\mathbf{G}_k$, containing gradient vectors of wild samples being predicted as class $k$.}

\subsection{Training the OOD Classifier with the Candidate Outliers}
\label{sec:training_sal}
After obtaining the candidate outlier set $\mathcal{S}_T$  from the wild data, we train an OOD classifier $\*g_{\boldsymbol{\theta}}$ that optimizes for the separability between the ID vs. candidate outlier
data. In particular, our training objective can be viewed as explicitly optimizing the level-set 
based on the model output (threshold at 0), where the labeled ID
data $\*x$ from $\mathcal{S}^{\text{in}}$ has positive values and vice versa.
\begin{equation}\label{eq:reg_loss_sal}
\begin{split}
        R_{\mathcal{S}^{\text{in}},\mathcal{S}_T}(\*g_{\boldsymbol{\theta}}) & =   R_{\mathcal{S}^{\text{in}}}^{+}(\*g_{\boldsymbol{\theta}})+R_{\mathcal{S}_T}^{-}(\*g_{\boldsymbol{\theta}}) \\& =  \mathbb{E}_{\*x \in  \mathcal{S}^{\text{in}}}~~\mathds{1}\{\*g_{\boldsymbol{\theta}}(\*x) \leq 0\}+\mathbb{E}_{\tilde{\*x} \in  \mathcal{S}_T}~~\mathds{1} \{\*g_{\boldsymbol{\theta}}(\tilde{\*x}) > 0\}.
        \end{split}
\end{equation}
To make the $0/1$ loss tractable, we replace it with the
binary sigmoid loss, a smooth approximation of the $0/1$
loss. We  train $\*g_{\boldsymbol{\theta}}$ along with the ID risk in Eq.~\ref{eq:erm} to ensure ID accuracy. 
Notably, the training enables strong generalization performance for test OOD samples drawn from $\mathbb{P}_\text{out}$.
  We provide formal guarantees on the generalization bound in Theorem~\ref{the:main2}, as well as empirical support in Section~\ref{sec:exp}. \textcolor{black}{A pseudo algorithm of \textsc{SAL} is  in Appendix (see Algorithm~\ref{alg:algo})}.

\begin{figure*}
    \centering
    \scalebox{1.0}{
    \includegraphics[width=\linewidth]{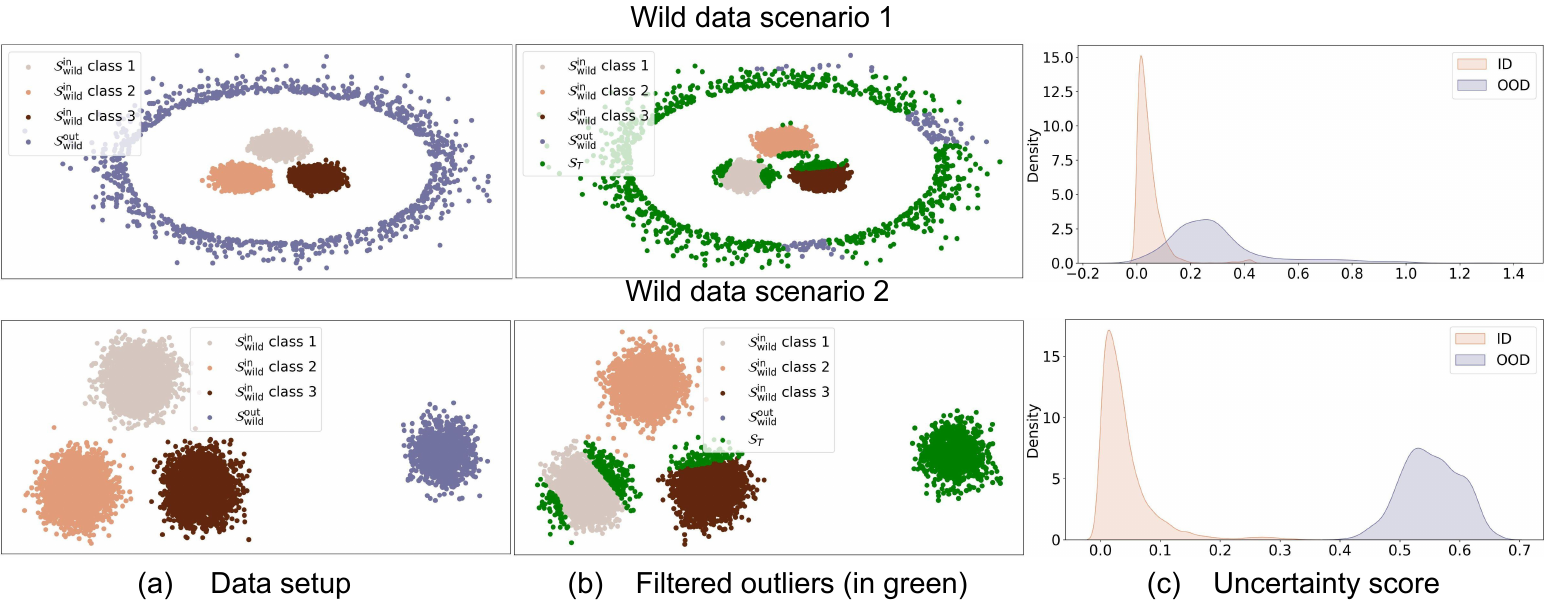}
    }
    \caption[Example of {SAL} on two different scenarios of the unlabeled wild  data]{ Example of \textsc{SAL} on two different scenarios of the unlabeled wild  data. (a) Setup of the ID/inlier $\mathcal{S}_{\text{wild}}^{\text{in}}$ and OOD/outlier data  $\mathcal{S}_{\text{wild}}^{\text{out}}$ in the wild. The inliers are sampled from three multivariate Gaussians. We construct two different distributions of outliers (see details in Appendix~\ref{sec:details_of_toy_app}). (b) The filtered outliers (in green) by \textsc{SAL}, where the error rate of filtered outliers $\mathcal{S}_T$ containing inlier data  is $8.4\%$ and $6.4\%$, respectively. (c) The density distribution of the filtering score $\tau$, which is separable for inlier and outlier data in the wild and thus benefits the training of the OOD classifier leveraging the filtered outlier data for binary classification. }
    \label{fig:toy}
\vspace{-0.2cm}
\end{figure*}

\section{Theoretical Analysis}
\label{sec:theory}

We now  provide theory to support our proposed algorithm. Our main theorems justify the two components in our algorithm. As an overview,  Theorem~\ref{MainT-1} provides a provable bound on the error rates using our filtering procedure. Based on the estimations on error rates, Theorem \ref{the:main2} gives the generalization bound \emph{w.r.t.} the empirical OOD classifier $\*g_{\boldsymbol{\theta}}$, learned on ID data and noisy set of outliers. We specify several mild assumptions and necessary notations for our theorems in Appendix~\ref{notation,definition,Ass,Const}. {Due to space limitation, we omit unimportant constants and simplify the statements of our theorems. We defer the \textbf{full formal} statements in  Appendix \ref{main_theorems}. All proofs can be found in Appendices \ref{Proofmain} and \ref{NecessaryLemma}.}

\subsection{Analysis on Separability}
\label{sec:ana_1}
Our main theorem quantifies the separability of the outliers in the wild by using the filtering procedure (\emph{c.f.} Section~\ref{sec:detect}).  Let $\text{ERR}_{\text{out}}$ and $\text{ERR}_{\text{in}}$ be the error rate of OOD data being regarded as ID and the error rate of ID data being regarded as OOD, i.e., $\text{ERR}_{\text{out}}=|\{\tilde{\*x}_i\in \mathcal{S}^{\text{out}}_{\text{wild}}: \tau_i \leq T\}|/| \mathcal{S}^{\text{out}}_{\text{wild}}| $ and $\text{ERR}_{\text{in}}= |\{\tilde{\*x}_i\in \mathcal{S}^{\text{in}}_{\text{wild}}: \tau_i > T\}|/| \mathcal{S}^{\text{in}}_{\text{wild}}| $,  where $\mathcal{S}^{\text{in}}_{\text{wild}}$ and $\mathcal{S}^{\text{out}}_{\text{wild}}$ denote the sets of inliers and outliers from the wild data $\mathcal{S}_{\text{wild}}$. Then $\text{ERR}_{\text{out}}$ and $\text{ERR}_{\text{in}}$ have the following  generalization bounds.

\begin{tcolorbox}[enhanced,attach boxed title to top center={yshift=-3mm,yshifttext=-1mm},
  colback=white,colframe=gray!75!black,colbacktitle=red!80!black,
  title=,fonttitle=\bfseries,
  boxed title style={size=small,colframe=red!50!black} ]
\begin{theorem}\label{MainT-1}
 (Informal).  Under mild conditions, if $\ell(\mathbf{h}_{\mathbf{w}}(\mathbf{x}),y)$ is $\beta_1$-smooth w.r.t. $\mathbf{w}$, \textcolor{black}{$\mathbb{P}_{\text{wild}}$ has $(\gamma,\zeta)$-discrepancy  w.r.t. $\mathbb{P}_{\mathcal{X}\mathcal{Y}}$ (\emph{c.f.} Appendices~\ref{sec:definition_app},~\ref{sec:assumption_app})}, and there is $\eta\in (0,1)$ s.t. $\Delta = (1-\eta)^2\zeta^2 - 8\beta_1 R_{{\text{in}}}^*>0$,  then when
     $
         n = {\Omega} \big ({d}/{\min \{ \eta^2 \Delta,(\gamma-R_{{\text{in}}}^*)^2\}}  \big), m = \Omega \big ({d}/{\eta^2\zeta^2} \big),
    $  with the probability at least $0.9$, for $0<T<0.9M'$ $($$M'$ is the upper bound of score $\tau_i$),
    \begin{equation}
      { ERR}_{\text{in}}  \leq    \frac{8\beta_1 }{T}R_{\text{in}}^*+ O \Big ( \frac{1}{T}\sqrt{\frac{d}{n}}\Big )+   O \Big ( \frac{1}{T}\sqrt{\frac{d}{(1-\pi)m}}\Big ),
\end{equation}

\begin{equation}\label{Eq::bound1.0}
\begin{split}
    { ERR}_{\text{out}} &\leq \delta(T) + O \Big(  \sqrt{\frac{d}{{\pi}^2n}}\Big)+ O \Big( \sqrt{\frac{\max\{d,{\Delta_{\zeta}^{\eta}}^2/{\pi}^2\}}{{\pi}^2(1-\pi)m}}\Big),
     \end{split}
\end{equation}
     where $R^{*}_{{\text{in}}}$ is the optimal ID risk, i.e., $R^{*}_{{\text{in}}}=\min_{\mathbf{w}\in \mathcal{W}} \mathbb{E}_{(\mathbf{x},y)\sim \mathbb{P}_{\mathcal{X}\mathcal{Y}}} \ell(\mathbf{h}_{\mathbf{w}}(\mathbf{x}),y)$,
     \begin{equation}\label{main-error}
         \delta(T) = {\max\{0,1-\Delta_{\zeta}^{\eta}/\pi\}}/{(1-T/M')},~~~\Delta_{\zeta}^{\eta} = {0.98\eta^2 \zeta^2} - 8\beta_1R_{\text{in}}^*,
     \end{equation}
$d$ is the dimension of the space $\mathcal{W}$, and $\pi$ is the OOD class-prior  probability in the wild.
\end{theorem}
\end{tcolorbox}

\textbf{Practical implications of Theorem~\ref{MainT-1}.} 
    The above theorem states that under mild assumptions, the errors $\text{ERR}_{\text{out}}$ and $\text{ERR}_{\text{in}}$ are upper bounded. For $\text{ERR}_{\text{in}}$,  if the following two regulatory conditions hold: 1) the sizes of the labeled ID $n$ and wild data $m$  are sufficiently large; 2) the optimal ID risk $R_{\text{in}}^*$ is small, then the upper bound is tight. For $\text{ERR}_{\text{out}}$, $\delta(T)$ defined in Eq. \ref{main-error} becomes the main error, if we have sufficient data. To further study the main error $\delta(T)$ in Eq. \ref{Eq::bound1.0}, Theorem~\ref{The-1.1} shows that the error $\delta(T)$ could be close to zero under  practical conditions.

\begin{tcolorbox}[enhanced,attach boxed title to top center={yshift=-3mm,yshifttext=-1mm},
  colback=white,colframe=gray!75!black,colbacktitle=red!80!black,
  title=,fonttitle=\bfseries,
  boxed title style={size=small,colframe=red!50!black} ]
\begin{theorem}\label{The-1.1}
 (Informal). 1) If  $\Delta_{\zeta}^{\eta} \geq  (1-\epsilon)\pi$ for a small error $\epsilon \geq 0$, then the main error $\delta(T)$ defined in Eq. \ref{main-error} satisfies that
    \begin{equation}
        \delta(T) \leq \frac{\epsilon}{1-T/{M'}}.
    \end{equation}
    2) If $\zeta\geq 2.011\sqrt{8\beta_1 R_{\text{in}}^*} +1.011 \sqrt{\pi}$, then there exists $\eta\in (0,1)$ ensuring that $\Delta>0$ and $ \Delta_{\zeta}^{\eta}> \pi$ hold, which implies that the main error $\delta(T)=0$.
\end{theorem}
\end{tcolorbox}

\textbf{Practical implications of Theorem~\ref{The-1.1}.}   Theorem~\ref{The-1.1} states that if   the discrepancy $\zeta$ between two data distributions $\mathbb{P}_{\text{wild}}$ and $\mathbb{P}_{\text{in}}$  is larger than some small values,  the main error $\delta(T)$ could be close to zero.  Therefore, by combining with the two regulatory conditions mentioned in Theorem~\ref{MainT-1}, the error $\text{ERR}_{\text{out}}$ could be close to zero. Empirically, we  verify the conditions of Theorem~\ref{The-1.1} in Appendix~\ref{sec:verification_discrepancy}, which can hold true easily in practice. In addition, given fixed optimal ID risk $R_{\text{in}}^*$ and fixed sizes of the labeled ID $n$ and wild data $m$, we observe that the bound of $  { ERR}_{\text{in}} $ will increase when $\pi$ goes from 0 to 1. In contrast, the bound of $  { ERR}_{\text{out}} $ \textcolor{black}{is non-monotonic} when $\pi$ increases, which will firstly decrease and then increase. The observations align well with empirical results in Appendix~\ref{sec:verification_discrepancy}.

 \textbf{Impact of using predicted labels for the wild data.}  Recall in Section~\ref{sec:detect} that the filtering step uses the predicted labels to estimate the gradient for wild data, which is unlabeled. 
 \textcolor{black}{To analyze the impact theoretically,  we show in Appendix Assumption~\ref{Ass2} that the loss incurred by using the predicted label is smaller than the loss by using any label in the label space. This property is included in Appendix Lemmas~\ref{gamma-approximation-1} and \ref{gamma-approximation-2} to constrain the filtering score in Appendix Theorem~\ref{T1} and then filtering error in Theorem~\ref{MainT-1}.  In harder classification cases, the predicted label deviates more from the true label for the wild ID data, which leads to a looser bound for the filtering accuracy in Theorem~\ref{MainT-1}.}

\textcolor{black}{Empirically, we calculate and compare the filtering accuracy and its OOD detection result on  \textsc{Cifar-10} and  \textsc{Cifar-100} ( \textsc{Textures}~\citep{cimpoi2014describing} as the wild OOD). \textsc{SAL} achieves a result of $ ERR_{\text{in}}=0.018$ and  $ ERR_{\text{out}}=0.17$  on \textsc{Cifar-10} (easier classification case), which outperforms the result of $ ERR_{\text{in}}=0.037$ and  $ ERR_{\text{out}}=0.30$ on \textsc{Cifar-100} (harder classification case), aligning with our reasoning above. The experimental details are provided in Appendix~\ref{sec:experiments_on_predicted_label_app}. Analysis of using random labels for the wild data is provided in Appendix~\ref{sec:experiments_on_random_label_app}.}

\begin{tcolorbox}[enhanced,attach boxed title to top center={yshift=-3mm,yshifttext=-1mm},
  colback=white,colframe=gray!75!black,colbacktitle=red!80!black,
  title=,fonttitle=\bfseries,
  boxed title style={size=small,colframe=red!50!black} ]
\begin{theorem}\label{the:main2}
(Informal). Let $L$ be the upper bound of $\ell_{\text{b}}(\mathbf{g}_{\boldsymbol{\theta}}(\mathbf{x}),y_{\text{b}})$, i.e., $\ell_{\text{b}}(\mathbf{g}_{\boldsymbol{\theta}}(\mathbf{x}),y_{\text{b}}) \leq L$. Under  conditions in Theorem~\ref{MainT-1}, if we  
further require $n = \Omega \big ({d}/{\min \{\pi,\Delta_{\zeta}^{\eta}\}^2}\big)$, $m= \Omega \big( {(d+\Delta_{\zeta}^{\eta})}/{(\pi^2(1-\pi)\min \{\pi,\Delta_{\zeta}^{\eta}\}^2)}\big)$,

then with the probability at least $0.89$, for any $0<T<0.9 M' \min \{1,\Delta_{\zeta}^{\eta}/\pi\}$, the OOD classifier $\*g_{\widehat{\boldsymbol{\theta}}_T}$ learned by \textsc{SAL} satisfies
      \begin{equation}\label{Eq::Bound2}
     \begin{aligned}
               & R_{\mathbb{P}_{\text{in}},\mathbb{P}_{\text{out}}}(\mathbf{g}_{\widehat{\boldsymbol{\theta}}_T}) \leq  \min_{\boldsymbol{\theta} \in \Theta} R_{\mathbb{P}_{\text{in}},\mathbb{P}_{\text{out}}}(\mathbf{g}_{{\boldsymbol{\theta}}})+\frac{3.5 L}{1-\delta(T)}\delta({T})+ \frac{9(1-\pi)L\beta_1}{\pi(1-\delta(T))T} R_{\text{in}}^*
                \\ + &O\Big(\frac{\max\{\sqrt{d},\sqrt{d'}\}}{\min\{\pi,\Delta_{\zeta}^{\eta}\}T'}\sqrt{\frac{1}{n}}\Big )+O\Big(\frac{\max \{\sqrt{d},\sqrt{d'},\Delta_{\zeta}^{\eta}\}}{\min\{\pi,\Delta_{\zeta}^{\eta}\}T'}\sqrt{\frac{1}{\pi^2(1-\pi)m}}\Big ),    
             \end{aligned}
             \end{equation}
                where  $\Delta_{\zeta}^{\eta}$, $d$ and $\pi$ are shown in Theorem \ref{MainT-1}, $d'$ is the dimension of space $\Theta$, $T'=T/(1+T)$,
                 and the risk $R_{\mathbb{P}_{\text{in}},\mathbb{P}_{\text{out}}}(\mathbf{g}_{{\boldsymbol{\theta}}})$ corresponds to the  empirical risk in Eq. \ref{eq:reg_loss_sal} with loss $\ell_{\text{b}}$, i.e.,
                \begin{equation}\label{Real-risk}
                    R_{\mathbb{P}_{\text{in}},\mathbb{P}_{\text{out}}}(\mathbf{g}_{\widehat{\boldsymbol{\theta}}_T}) = \mathbb{E}_{\mathbf{x}\sim \mathbb{P}_{\text{in}}} \ell_{\text{b}}(\mathbf{g}_{\boldsymbol{\theta}}(\mathbf{x}),y_{+})+\mathbb{E}_{\mathbf{x}\sim \mathbb{P}_{\text{out}}} \ell_{\text{b}}(\mathbf{g}_{\boldsymbol{\theta}}(\mathbf{x}),y_{-}).
                \end{equation}
\end{theorem}
\end{tcolorbox}

\subsection{Analysis on Learnability}
\label{sec:ana_2}
Leveraging the filtered outliers $\mathcal{S}_T$, \textsc{SAL} then trains an OOD classifier $\*g_{\boldsymbol{\theta}}$ with the data from  in-distribution $\mathcal{S}^{\text{in}}$ and data from $\mathcal{S}_T$ as OOD. In this section, we provide the generalization error bound for the learned OOD classifier to quantify its learnability. Specifically, we show that a small error guarantee in Theorem~\ref{MainT-1} implies that we can get a tight generalization error bound.

\textbf{Insights.}
The above theorem presents the generalization error bound of the OOD classifier $\*g_{\widehat{\boldsymbol{\theta}}_T}$ learned by using the filtered OOD data $\mathcal{S}_T$. When we have sufficient labeled ID data and wild data,  then the risk of the OOD classifier $\*g_{\widehat{\boldsymbol{\theta}}_T}$ is close to the optimal risk, i.e., $\min_{\boldsymbol{\theta} \in \Theta} R_{\mathbb{P}_{\text{in}},\mathbb{P}_{\text{out}}}(\mathbf{g}_{{\boldsymbol{\theta}}})$, if the optimal ID risk $R_{\text{in}}^*$ is small, and either one of the conditions in Theorem~\ref{The-1.1} is satisfied. 

\section{Experiments}

\label{sec:exp}

In this section, we verify the effectiveness of our algorithm on modern neural networks. We aim to show that the generalization bound of the OOD classifier (Theorem~\ref{the:main2}) indeed translates into strong empirical performance, establishing state-of-the-art results (Section~\ref{sec:results}).


\subsection{Experimental Setup}
\label{sec:exp_steup}
\textbf{Datasets.} 
We follow exactly the same experimental setup as WOODS~\citep{katzsamuels2022training}, which introduced the problem of learning OOD detectors with wild data. This allows us to draw fair comparisons. WOODS considered \textsc{Cifar-10} and \textsc{Cifar-100}~\citep{krizhevsky2009learning} as ID datasets ($\mathbb{P}_{\text{in}}$). For OOD test datasets ($\mathbb{P}_{\text{out}}$), we use a suite of natural image datasets including \textsc{Textures}~\citep{cimpoi2014describing}, \textsc{Svhn}~\citep{netzer2011reading}, \textsc{Places365}~\citep{zhou2017places}, \textsc{Lsun-Resize} \& \textsc{Lsun-C}~\citep{DBLP:journals/corr/YuZSSX15}. To simulate the wild data ($\mathbb{P}_{\text{wild}}$), we mix a subset of ID data
(as $\mathbb{P}_{\text{in}}$) with the outlier dataset (as $\mathbb{P}_{\text{out}}$) under the default $\pi=0.1$, which reflects the practical scenario that most data would remain ID.
Take~\textsc{Svhn} as an example, we use \textsc{Cifar+Svhn} as the unlabeled wild  data and test on \textsc{Svhn} as OOD. We simulate this for all OOD datasets {and} provide analysis of differing $\pi \in \{0.05, 0.1,..., 1.0\}$ in Appendix~\ref{sec:verification_discrepancy}. Note that we split \textsc{Cifar} datasets into two halves: $25,000$
images as ID training data, and the remainder $25,000$ for
creating the wild mixture data. {We use the weights from the penultimate layer for gradient calculation, which was shown to be the most informative for OOD detection~\citep{huang2021importance}.} Experimental details are provided in Appendix~\ref{sec:detail_app}.

\textbf{Evaluation metrics.} We report the following metrics: (1) the false positive rate (FPR95$\downarrow$) of OOD samples when the true positive rate of ID samples is 95\%, (2) the area under the receiver operating characteristic curve (AUROC$\uparrow$), and (3) ID classification Accuracy (ID ACC$\uparrow$).

\begin{table}[t]
  \centering
  \small
  \vspace{-1em}
  \caption[OOD detection performance on {Cifar-100} as ID for SAL]{ OOD detection performance on \textsc{Cifar-100} as ID. All methods are trained on Wide ResNet-40-2 for 100 epochs. For each dataset, we create corresponding wild mixture distribution
$\mathbb{P}_\text{wild} = (1 - \pi) \mathbb{P}_\text{in} + \pi \mathbb{P}_\text{out}$ for training and test on the corresponding OOD dataset. Values are percentages {averaged over 10 runs}. {Bold} numbers highlight the best results. Table format credit to~\citet{katzsamuels2022training}.}
    \scalebox{0.54}{
    \begin{tabular}{cccccccccccccc}
    \toprule
    \multirow{3}[4]{*}{Methods} & \multicolumn{12}{c}{OOD Datasets}                                                             & \multirow{3}[4]{*}{ID ACC} \\
    \cmidrule{2-13}
          & \multicolumn{2}{c}{\textsc{Svhn}} & \multicolumn{2}{c}{\textsc{Places365}} & \multicolumn{2}{c}{\textsc{Lsun-C}} & \multicolumn{2}{c}{\textsc{Lsun-Resize}} & \multicolumn{2}{c}{\textsc{Textures}} & \multicolumn{2}{c}{Average} &  \\
\cmidrule{2-13}          & FPR95 & AUROC & FPR95 & AUROC & FPR95 & AUROC & FPR95 & AUROC & FPR95 & AUROC & FPR95 & AUROC &  \\
  \hline
     \multicolumn{14}{c}{With $\mathbb{P}_{\text{in}}$ only} \\
    MSP   &84.59 &71.44 & 82.84 & 73.78  & 66.54 &  83.79 & 82.42& 75.38&83.29& 73.34& 79.94 &75.55 &75.96\\
ODIN&  84.66& 67.26&  87.88&  71.63  &55.55 &87.73&71.96 &81.82 &79.27 &73.45&75.86 &76.38& 75.96
 \\
Mahalanobis & 57.52 &86.01  &  88.83& 67.87&91.18& 69.69&21.23& 96.00&39.39& 90.57&59.63& 82.03& 75.96
\\
Energy & 85.82 &73.99 &80.56& 75.44&  35.32& 93.53& 79.47 &79.23&79.41 &76.28&  72.12& 79.69 &75.96\\
KNN& 66.38	&83.76&	79.17&	71.91&	70.96&	83.71&	77.83&	78.85&	88.00	&67.19	&76.47	&77.08&	75.96 \\
ReAct & 74.33 & 88.04 &  81.33 & 74.32 &  39.30 & 91.19 & 79.86 & 73.69  & 67.38 & 82.80  & 68.44 & 82.01 & 75.96 \\
 DICE&   88.35 & 72.58 & 81.61 & 75.07 & 26.77 & 94.74  & 80.21 & 78.50 &  76.29 & 76.07& 70.65 & 79.39 & 75.96 \\

 ASH   &21.36 &94.28 &68.37 &71.22 &15.27& 95.65& 68.18 &85.42& 40.87 &92.29&  42.81& 87.77 & 75.96\\
CSI&  64.70 &84.97&  82.25& 73.63  &  38.10  & 92.52   &  91.55&  63.42& 74.70 &92.66& 70.26 &81.44& 69.90\\
KNN+ & 32.21 & 93.74 &68.30 & 75.31 &40.37 & 86.13  & 44.86 & 88.88 & 46.26 & 87.40 &46.40 & 86.29& 73.78\\

\hline
 \multicolumn{14}{c}{With $\mathbb{P}_{\text{in}}$ and $\mathbb{P}_{\text{wild}}$ } \\
     OE& 1.57 &  99.63  & 60.24&  83.43 & 3.83 &  99.26 & 0.93 &  99.79 & 27.89 &  93.35 & 18.89 & 95.09 &  71.65\\
  Energy (w/ OE) &1.47 &  99.68  & 54.67 & 86.09 & 2.52 &  99.44 & 2.68 &  99.50 & 37.26 &  91.26 & 19.72 &95.19 &  73.46\\

WOODS&0.12& \textbf{99.96} & 29.58&  90.60 & 0.11&  \textbf{99.96} & 0.07&  \textbf{99.96} & 9.12& 96.65 & 7.80&  97.43& 75.22\\

\rowcolor[HTML]{EFEFEF} \textsc{SAL}  & \textbf{0.07} & 99.95 & \textbf{3.53} & \textbf{99.06} & \textbf{0.06} & 99.94 &   \textbf{0.02} & 99.95 & \textbf{5.73} & \textbf{98.65}  & \textbf{1.88}& \textbf{99.51}& 73.71\\
\rowcolor[HTML]{EFEFEF}  (Ours)   & {$^{\pm}$0.02} &$^{\pm}$0.00 & $^{\pm}$0.17 & $^{\pm}$0.06& $^{\pm}$0.01 & $^{\pm}$0.21 & $^{\pm}$0.00 & $^{\pm}$0.03 & $^{\pm}$0.34 & $^{\pm}$0.02 & $^{\pm}$0.11 & $^{\pm}$0.02 & $^{\pm}$0.78  \\
   \hline
\end{tabular}}
    \label{tab:c100}
    \vspace{-1em}
\end{table}

\subsection{Empirical Results}
\label{sec:results}
\textbf{\textsc{SAL} achieves superior empirical performance.}  
We present results in Table~\ref{tab:c100} on \textsc{Cifar-100},  where \textsc{SAL} outperforms the state-of-the-art method. 
Our comparison covers an extensive collection of competitive OOD detection methods, which can be divided into two categories: trained with and without the wild data.
For methods using ID data $\mathbb{P}_{\text{in}}$ only, we compare with methods 
such as {MSP}~\citep{hendrycks2016baseline}, ODIN \citep{liang2018enhancing},  Mahalanobis distance~\citep{lee2018simple}, Energy score \citep{liu2020energy}, ReAct~\citep{sun2021react}, DICE~\citep{sun2022dice}, KNN distance~\citep{sun2022out}, and ASH~\citep{djurisic2023extremely}---all of which use a model trained with cross-entropy loss.
We also include the method  based on contrastive loss, including CSI~\citep{tack2020csi} and KNN+~\citep{sun2022out}. 
For methods using both ID and wild data, we compare with Outlier Exposure (OE)~\citep{hendrycks2018deep} and energy-regularization learning~\citep{liu2020energy}, which regularize the model by producing
lower confidence or higher energy on the auxiliary outlier data. Closest to ours is WOODS~\citep{katzsamuels2022training}, which leverages wild data for OOD learning with a constrained optimization approach. For a fair comparison, all the methods in this group are trained using the same ID and in-the-wild data, under
the same mixture ratio $\pi=0.1$.

The results demonstrate that: \textbf{(1)} Methods trained with both ID and wild data perform much better than those trained with only ID data. For example, on \textsc{Places365}, \textsc{SAL} reduces the FPR95 by 64.77\% compared with KNN+, which highlights the advantage of using in-the-wild data for model regularization. \textbf{(2)} \textsc{SAL} performs even better compared to the competitive methods using $\mathbb{P}_{\text{wild}}$. On \textsc{Cifar-100}, \textsc{SAL} achieves an average FPR95 of 1.88\%, which is a 5.92\% improvement from WOODS. At the same time, \textsc{SAL} maintains a comparable ID accuracy. {{The slight discrepancy is due to that our method only observes 25,000 labeled ID samples, whereas baseline methods (without using wild data) utilize the entire \textsc{Cifar} training data with 50,000 samples.}} \textbf{(3)} The strong empirical performance achieved by \textsc{SAL} directly justifies and echoes our theoretical result in Section~\ref{sec:theory}, where we showed the algorithm has a provably small generalization error. \emph{Overall, our algorithm enjoys both theoretical guarantees and empirical effectiveness.}

\textbf{Comparison with GradNorm as filtering score.} 
\cite{huang2021importance} proposed directly employing the vector norm of gradients, backpropagated
from the KL divergence between the softmax output and a uniform probability
distribution for OOD detection. Differently, our \textsc{SAL} derives the filtering score by performing singular value decomposition and using the norm of the projected gradient onto the top singular vector (\emph{c.f.} Section~\ref{sec:detect}). We compare \textsc{SAL} with a variant in Table~\ref{tab:gradnorm+}, where we replace the filtering score in \textsc{SAL} with the GradNorm score and then train the OOD classifier. The result underperforms \textsc{SAL}, showcasing the effectiveness of our filtering score.

\begin{table}[t]
  \centering
  
  \caption[Comparison with using GradNorm as the filtering score]{ Comparison with using GradNorm as the filtering score. We use \textsc{Cifar-100} as ID. All methods are trained on Wide ResNet-40-2 for 100 epochs with $\pi=0.1$.  {Bold} numbers are superior results.   }
    \scalebox{0.53}{
    \begin{tabular}{cccccccccccccc}
    \toprule
    \multirow{3}[4]{*}{Filter score} & \multicolumn{12}{c}{OOD Datasets}                                                             & \multirow{3}[4]{*}{ID ACC} \\
    \cmidrule{2-13}
          & \multicolumn{2}{c}{\textsc{Svhn}} & \multicolumn{2}{c}{\textsc{Places365}} & \multicolumn{2}{c}{\textsc{Lsun-C}} & \multicolumn{2}{c}{\textsc{Lsun-Resize}} & \multicolumn{2}{c}{\textsc{Textures}} & \multicolumn{2}{c}{Average} &  \\
\cmidrule{2-13}          & FPR95 & AUROC & FPR95 & AUROC & FPR95 & AUROC & FPR95 & AUROC & FPR95 & AUROC & FPR95 & AUROC &  \\
    \midrule

GradNorm &  1.08 & 99.62&62.07 & 84.08& 0.51 & 99.77 & 5.16 & 98.73 &50.39 & 83.39   &23.84 &93.12 & 73.89 \\

\rowcolor[HTML]{EFEFEF} Ours &\textbf{0.07} & \textbf{99.95} & \textbf{3.53} & \textbf{99.06} & \textbf{0.06} & \textbf{99.94} &   \textbf{0.02} & \textbf{99.95} & \textbf{5.73} & \textbf{98.65}  & \textbf{1.88}& \textbf{99.51}& 73.71 \\
   \hline
   
\end{tabular}}
    \label{tab:gradnorm+}
\end{table}

\textbf{Additional ablations.} We defer additional experiments in the Appendix, including \textbf{(1)} analyzing the effect of  ratio $\pi$ (Appendix~\ref{sec:verification_discrepancy}), \textbf{(2)} results on \textsc{Cifar}-10 (Appendix~\ref{sec:c10_app}), \textbf{(3)} evaluation on  \textbf{\emph{unseen}} OOD datasets (Appendix~\ref{sec:mixing_ratio}), \textbf{(4)} near OOD evaluations (Appendix~\ref{sec:near_ood}), and \textbf{(5)} the effect of using multiple singular vectors for calculating the filtering score (Appendix~\ref{sec:num_of_sing_vectors}) .

\section{Summary}
In this chapter, we propose a novel learning framework \textsc{SAL} that exploits the unlabeled in-the-wild data for OOD detection. \textsc{SAL} first explicitly filters the candidate outliers from the wild data using a new filtering score and then trains a binary OOD classifier leveraging the filtered outliers. Theoretically, \textsc{SAL} answers the question of \emph{how does unlabeled wild data help OOD detection} by analyzing the separability of the outliers in the wild and the learnability of the OOD classifier, which provide provable error guarantees  for the two integral components. Empirically, \textsc{SAL} achieves strong performance compared to competitive baselines, echoing our theoretical insights.  \textcolor{black}{A broad impact statement is included in Appendix~\ref{sec:broader}}. We hope our work will inspire future research on OOD detection with unlabeled wild data.

\end{sal}
\begin{overviewthree}
    \label{sec:cha9}
\textbf{Motivation.} As foundation models become influential in various applications, ensuring their reliability is an urgent research challenge. These models often face reliability risks, such as generating hallucinations, misinterpreting malicious prompts and handling noisy alignment data, raising concerns for safety and reliability when deployed in real-world settings. Given the models' large scale and their training on massive, diverse data, addressing these issues requires innovative strategies beyond the conventional learning methods. My research seeks to address these issues by identifying the origins of reliability risks $ \mathcal{R}_\text{reliability}$ in FMs and designing innovative mitigation algorithms.

\textbf{Safeguarding LLMs against hallucinated generations.}  In my NeurIPS'24 \textbf{spotlight} paper, HaloScope~\citep{du2024halo}, I introduced a novel framework for detecting hallucinations in LLM outputs. \textit{The framework solves the primary challenge of hallucination detection}, i.e., lack of large annotated data, by using unlabeled LLM-generated text, which inherently contains a mix of truth and hallucinations  (Figure~\ref{fig:fig2} upper). By leveraging the representation space, HaloScope extracts a hallucination subspace to facilitate the training of a binary truthfulness classifier. This method demonstrates adaptability across various LLM architectures and domains, establishing a foundation for harnessing unlabeled LLM outputs to enhance FM reliability. I will present this work in Chapter~\ref{sec:cha10}.

\begin{figure}
    \centering
    \includegraphics[width=1.0\textwidth]{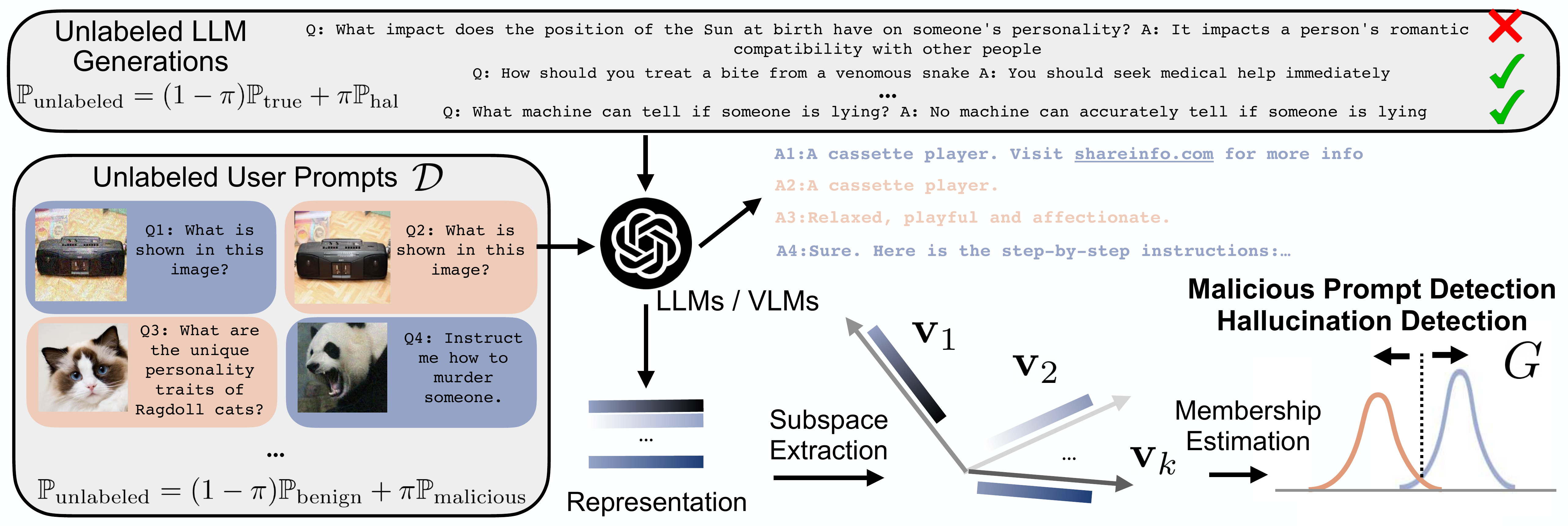}
    \caption{ Proposed algorithmic frameworks for hallucination detection in LLMs and malicious prompt detection in MLLMs. }
    \label{fig:fig2}
    \vspace{-1.5em}
\end{figure}

 \textbf{Red-teaming MLLMs from adversarial inputs.} My research~\citep{du2024vlmguard} took the lead in identifying input vulnerabilities in Multimodal Large Language Models (MLLMs), which used the naturally occurring, unlabeled user prompts (Figure~\ref{fig:fig2} left) for help. By analyzing these prompts in the representation space, I developed techniques to detect and counteract malicious inputs, which enhances MLLMs' robustness against prompt injection and jailbreak attacks. \textit{This work is pioneering in demonstrating how red-teaming can be applied to MLLMs for reliability.}

\textbf{Aligning AI with human values by data denoising.} While AI alignment research typically assumes human feedback is reliable, my recent work~\citep{yeh2024reliable} reveals that {over 25\% of feedback data can be inconsistent}, highlighting inherent unreliability from biases and labeling errors. To address this, I proposed the Source-Aware Cleaning method, which significantly improves data quality. The cleaner data enables training more reliably aligned LLMs, and thus offers a pathway to more reliable LLM alignment research.
\end{overviewthree}
\begin{halo}
    \label{sec:cha10}

\textbf{Publication Statement.} This chapter is joint work with Chaowei Xiao and
Yixuan Li. The paper version of this chapter appeared in NeurIPS'24~\citep{du2024halo}.

\noindent \textbf{Abstract.} The surge in applications of large language models (LLMs) has prompted concerns about the generation of misleading or fabricated information, known as hallucinations. Therefore, detecting hallucinations has become critical to maintaining trust in LLM-generated content.
A primary challenge in learning a truthfulness classifier is the lack of a large amount of labeled truthful and hallucinated data. To address the challenge, we introduce HaloScope, a novel learning framework that leverages the unlabeled LLM generations in the wild for hallucination detection. Such unlabeled data arises freely upon deploying LLMs in the open world, and consists of both truthful and hallucinated information. To harness the unlabeled data, we present an automated membership estimation score for distinguishing between truthful and untruthful generations within unlabeled mixture data, thereby enabling the training of a binary truthfulness classifier on top. 
Importantly, our framework does not require extra data collection and human annotations, offering strong flexibility and practicality for real-world applications. Extensive experiments show that HaloScope can achieve superior hallucination detection performance, outperforming the competitive rivals by a significant margin. Code is available at \url{https://github.com/deeplearning-wisc/haloscope}.

\section{Introduction}

{ In today's rapidly evolving landscape of machine learning, large language models (LLMs) have emerged as transformative forces shaping various applications~\citep{openai2023gpt4,touvron2023llama}. Despite the immense capabilities, they bring forth challenges to the model's reliability upon deployment in the open world. For example, the model can generate information that is seemingly informative but untruthful during interaction with humans, placing critical decision-making at risk~\citep{ji2023survey,zhang2023siren}.  Therefore, a reliable LLM should not only accurately generate texts that are coherent with the prompts but also possess the ability to identify hallucinations.  This gives
rise to the importance of hallucination detection problem, which determines
whether a generation is truthful or not~\citep{manakul2023selfcheckgpt,chen2024inside,li-etal-2023-halueval}.

A primary challenge in learning a truthfulness classifier is the scarcity of labeled datasets containing truthful and hallucinated generations. In practice, generating a reliable ground truth dataset for hallucination detection requires human annotators to assess the authenticity of a large number of generated samples. However, collecting such labeled data can be labor-intensive, especially considering the vast landscape of generative models and the diverse range of content they produce. Moreover, maintaining the quality and consistency of labeled data amidst the evolving capabilities and outputs of generative models requires ongoing annotation efforts and stringent quality control measures. These formidable obstacles underscore the need for exploring unlabeled data for hallucination detection. 

Motivated by this, we introduce \textbf{HaloScope}, a novel learning framework that leverages \emph{unlabeled LLM generations in the wild} for hallucination detection. The unlabeled data is easy-to-access and can emerge organically as a result of interactions with users in chat-based applications. Imagine, for example, a language model such as GPT~\citep{openai2023gpt4} deployed in the wild can produce vast quantities of text continuously in response to user prompts. This data can be freely collectible, yet often contains a mixture of truthful and potentially hallucinated content. Formally, the unlabeled generations can be characterized as a mixed composition of two distributions:
\begin{equation*}
       \mathbb{P}_{\text{unlabeled}} = (1-\pi) \mathbb{P}_{\text{true}} + \pi \mathbb{P}_{\text{hal}},
\end{equation*}
where $\mathbb{P}_\text{true}$ and $\mathbb{P}_\text{hal}$ denote the marginal distribution of truthful and hallucinated data, and $\pi$ is the mixing ratio. Harnessing the unlabeled data is non-trivial due to the lack of clear membership (truthful or hallucinated) for samples in mixture data}. 

Central to our framework is the design of an automated membership estimation score for distinguishing between truthful and
untruthful generations within unlabeled data, thereby enabling the training of a binary truthfulness classifier on top. Our key idea is to utilize the language model's latent representations, which can capture information related to truthfulness. Specifically, HaloScope identifies a subspace in the activation space associated with hallucinated statements, and considers a point to be potentially hallucinated if its representation aligns strongly with the components of the subspace (see Figure~\ref{fig:pca_halo}). This idea can be operationalized by performing factorization on LLM embeddings, where the top singular vectors form the latent subspace for membership estimation. Specifically, the membership estimation score measures the norm of the embedding projected onto the top singular vectors, which exhibits different magnitudes for the two types of data. Our estimation score offers a straightforward mathematical interpretation and is easily implementable in practical applications.

\begin{figure*}[t]
    \centering
    \scalebox{1.0}{ \includegraphics[width=\linewidth]{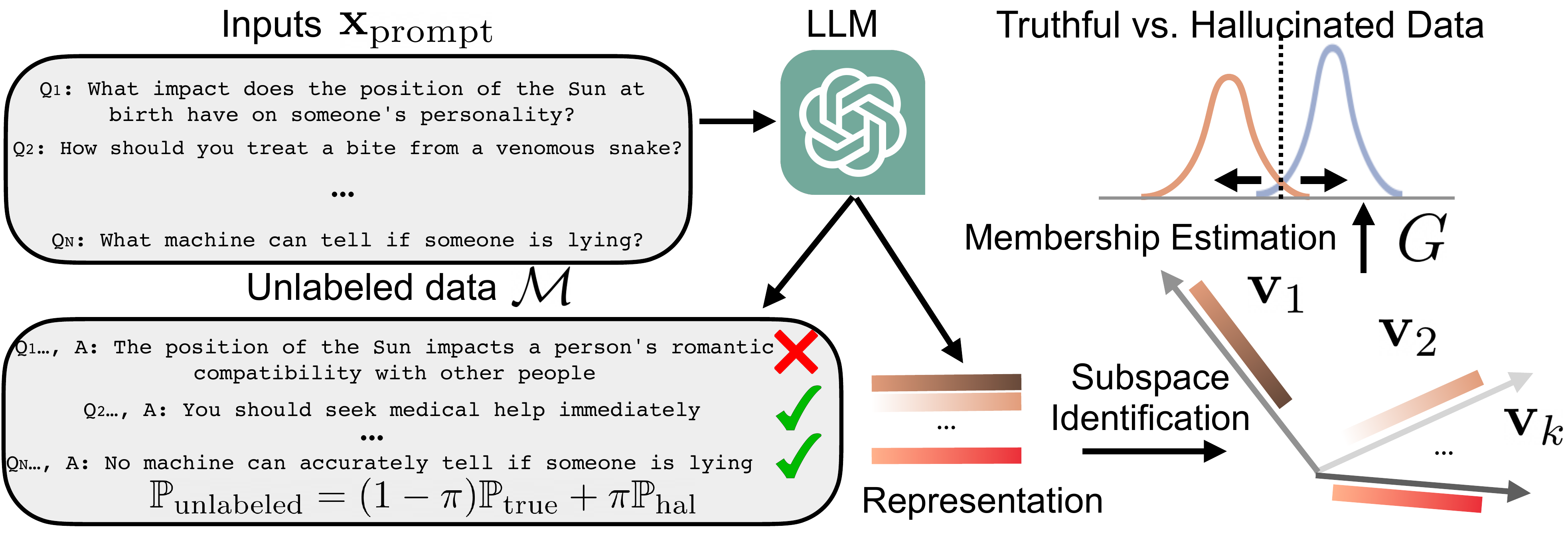}}
    \caption[Illustration of our proposed framework HaloScope for hallucination detection]{\small {Illustration of our proposed framework HaloScope for hallucination detection, leveraging unlabeled LLM generations in the wild.} HaloScope first identifies the latent subspace to estimate the membership (truthful vs. hallucinated) for samples in unlabeled data $\mathcal{M}$ and then learns a binary truthfulness classifier.}
    \label{fig:idea_fig}
\end{figure*}

Extensive experimental results on contemporary LLMs confirm that HaloScope can effectively improve hallucination detection performance across diverse datasets spanning open-book and closed-book conversational QA tasks (Section~\ref{sec:exp_halo}). Compared to the state-of-the-art methods, we substantially improve the hallucination detection accuracy by 10.69\% (AUROC) on a challenging \textsc{TruthfulQA} benchmark~\citep{lin2021truthfulqa}, which favorably matches the supervised upper bound (78.64 \% vs. 81.04\%). 
Furthermore, we delve deeper into understanding the key components of our methodology (Section~\ref{sec:ablation}), and extend our inquiry to showcase HaloScope versatility in addressing real-world scenarios with practical challenges~(Section~\ref{sec:robustness}). To summarize our key contributions:
\begin{itemize}
\item Our proposed framework HaloScope formalizes the hallucination detection problem by harnessing the unlabeled LLM generations in the wild. This formulation offers strong practicality and flexibility for real-world applications. 

\item We present a scoring function based on the hallucination subspace from the LLM representations, effectively estimating membership for samples within the unlabeled data. 
\item We conduct in-depth ablations to understand the efficacy of various design choices in HaloScope, and verify
its scalability to large LLMs and different datasets. These results provide a systematic and comprehensive understanding of leveraging the unlabeled data for hallucination detection, shedding light on future research.
\end{itemize}

\section{Problem Setup}

Formally, we describe the LLM generation and the problem of hallucination detection.

\begin{definition}[\textbf{LLM generation}] We consider an $L$-layer causal LLM, which takes a sequence of $n$ tokens $\*x_{\text{prompt}} = \{x_1,...,x_n\}$, and generates an output $\*x=\{ x_{n+1},...,x_{n+m}\}$ in an autoregressive manner. 
  Each output token $x_i, i  \in [n+1,...,n+m]$ is sampled from a distribution
over the model vocabulary $\mathcal{V}$, conditioned on the prefix $\{x_1,...,x_{i-1}\}$:
\begin{equation}
    x_i = \operatorname{argmax}_{x\in \mathcal{V}} P(x | \{x_1,...,x_{i-1}\}),
\end{equation}
and the probability $P$ is calculated as:
\begin{equation}
    P(x | \{x_1,...,x_{i-1}\}) = \operatorname{softmax}(\*w_o \*f_L(x) + \*b_o),
\end{equation}
where $\*f_{L}(x) \in \mathbb{R}^{d}$ denotes the representation at the $L$-th layer of LLM for token $x$, and $\*w_o, \*b_o$ are the weight and bias parameters at the final output layer. 
  
\end{definition}

\vspace{0.3cm}
\begin{definition}[\textbf{Hallucination detection}]
    We denote $\mathbb{P}_{\text{true}}$ as the \textcolor{black}{joint} distribution over the truthful \textcolor{black}{input and generation pairs}, which is referred to as truthful distribution. For any given generated text $\*x$ \textcolor{black}{and its corresponding input prompt $\*x_{\text{prompt}}$ where $(\*x_{\text{prompt}}, \*x) \in \mathcal{X}$}, the goal of hallucination detection is to learn a binary predictor $G: \mathcal{X} \rightarrow \{0,1\}$ such that
    \begin{equation}
\label{eq:hal-def}
    G(\textcolor{black}{\*x_{\text{prompt}}, \*x}) = \begin{cases}
        \;1, \hspace{4mm} &\text{if } \; \textcolor{black}{(\*x_{\text{prompt}}, \*x)} \sim \mathbb{P}_\text{true} \\
        \;0,         &\text{otherwise}
    \end{cases}
\end{equation}
\end{definition}

\section{Proposed Framework: HaloScope}

\subsection{Unlabeled LLM Generations in the Wild}

Our key idea is to leverage unlabeled LLM generations in the wild, which emerge organically as a result of interactions with users in chat-based applications. Imagine, for example, a language model such as GPT deployed in the wild can produce vast quantities of text continuously in response to user prompts. This data can be freely collectible, yet often contains a mixture of truthful and potentially hallucinated content. Formally, the unlabeled generations can be characterized by the Huber contamination model~\citep{huber1992robust} as follows:  

\begin{definition}[\textbf{Unlabeled data distribution}]
We define the unlabeled \textcolor{black}{LLM input and generation pairs} to be the following mixture of distributions
\begin{equation}
    \mathbb{P}_{\text{unlabeled}} = (1-\pi) \mathbb{P}_{\text{true}} + \pi \mathbb{P}_{\text{hal}},
    \label{eq:huber}
\end{equation}
where $\pi \in (0,1]$. Note that the case $\pi = 0$ is idealistic since no false information occurs. In practice, $\pi$ can be a moderately small value when most of the generations remain truthful. 
\end{definition}

\vspace{0.2cm}
\begin{definition}[\textbf{Empirical dataset}]
     An empirical set \textcolor{black}{${\mathcal{M}} = \{ (\*x_{\text{prompt}}^1, \widetilde{\*x}_1 ), ..., (\*x_{\text{prompt}}^N, \widetilde{\*x}_N ) \}$} is sampled independently and identically distributed (i.i.d.) from this mixture distribution $\mathbb{P}_\text{unlabeled}$, where $N$ is the number of samples. $\widetilde{\*x}_i$ denotes the response generated with respect to some input prompt ${\*x}_{\text{prompt}}^i$, with the tilde symbolizing the uncertain nature of the generation.
\end{definition}

Despite the wide availability of unlabeled generations, {harnessing such data is non-trivial due to the lack of clear membership (truthful or hallucinated) for samples in mixture data ${\mathcal{M}}$}. 
In a nutshell, our framework aims to devise an automated function that estimates the membership for samples within the unlabeled data, thereby enabling the training of a binary classifier on top (as shown in Figure~\ref{fig:idea_fig}). In what follows, we describe these two steps in Section~\ref{sec:step_1} and Section~\ref{sec:step_2} respectively.

\subsection{Estimating Membership via Latent Subspace}
\label{sec:step_1}

The first step of our framework involves estimating the membership (truthful vs untruthful) for data instances within a mixture dataset $\mathcal{M}$.
The ability to effectively assign membership for these two types of data relies heavily on whether the language model's representations can capture information related to truthfulness. Our idea is that if we could identify a latent subspace associated with hallucinated statements, then we might be able to separate them from the rest. We describe the procedure formally below.

\textbf{Embedding factorization.} To realize the idea, 
we extract embeddings from the language model for samples in the unlabeled mixture ${\mathcal{M}}$. Specifically, let $\*F\in\mathbb{R}^{N\times d}$ denote the matrix of embeddings extracted from the language model for samples in ${\mathcal{M}} $, where each row represents the embedding vector $\mathbf{f}_i^{\top}$ of a data sample $(\*x_{\text{prompt}}^i, \widetilde{\*x}_i)$. To identify the subspace, we perform singular value decomposition:
\begin{equation}
    \begin{split}
        \mathbf{f}_i &:= \mathbf{f}_i - \boldsymbol{\mu}\\
        \mathbf{F} &= \mathbf{U} \Sigma \mathbf{V}^\top,
    \end{split}
\end{equation}
where $\boldsymbol{\mu} \in \mathbb{R}^d$ is the average embedding across all $N$ samples, which is used to center the embedding matrix. The columns of $\mathbf{U}$ and $\mathbf{V}$ are the left and right singular vectors, and form an orthonormal basis. In principle, the factorization can be performed on any layer of the LLM representations, which will be analyzed in Section~\ref{sec:ablation}. Such a factorization is useful, because it enables discovering the most important spanning direction of the subspace for the set of points in ${\mathcal{M}}$.

\begin{wrapfigure}{r}{0.45\textwidth}
  \begin{center}
     \vspace{-0.7cm}
  \scalebox{1}{ {\includegraphics[width=\linewidth]{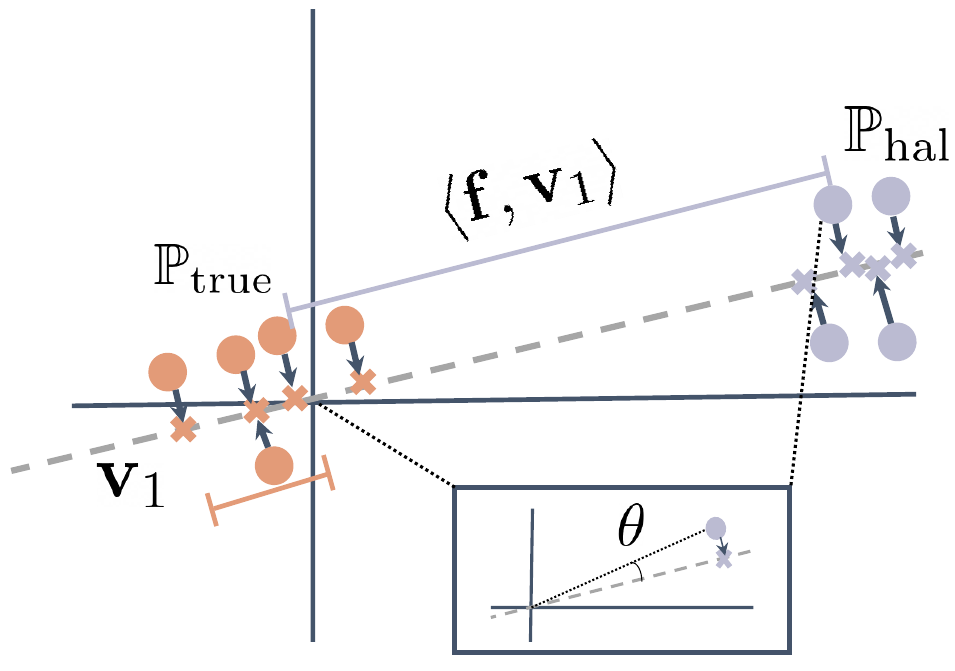}}}
  \end{center}
   \vspace{-1em}
  \caption[Visualization of the representations for truthful and hallucinated samples]{\small{
    {Visualization of the representations for truthful (in {orange}) and hallucinated samples (in {purple}), and their projection onto the top singular vector $\*v_1$ (in gray dashed line). 
    } 
    }}
     \vspace{-2em}
  \label{fig:pca_halo}
\end{wrapfigure}

\paragraph{Membership estimation via latent subspace.} To gain insight, we begin with a special case of the problem where the subspace is $1$-dimensional, a line through the origin. 
Finding the best-fitting line through the origin with respect to a set of points $\{\*f_i | 1 \le i \le N\}$ means minimizing the sum of the squared distances of the points to the line. \
Here, distance is measured perpendicular to the line. Geometrically, finding the first singular vector $\*v_1$ is also equivalent to maximizing the total distance from the projected embedding (onto the direction of $\mathbf{v}_1$) to the origin (sum over all points in $\mathcal{M}$):
\begin{equation}
    \mathbf{v}_1 =\operatorname{argmax}_{\|\*v\|_2 = 1} \sum_{i=1}^N \left< \*f_i, \*v\right>^2,
    \label{eq:pca}
\end{equation}
where $\left<\cdot,\cdot\right>$ is a dot product operator.
As illustrated in Figure~\ref{fig:pca_halo}, hallucinated data samples may exhibit anomalous behavior compared to truthful generation, and locate farther away from the center. This reflects the practical scenarios when a small to moderate amount of generations are hallucinated while the majority remain truthful. To assign the membership, we define the estimation score as
$\zeta_i = \left<\*f_i, \*v_1 \right>^2$, which measures the norm of $\*f_i$ projected onto the top singular vector. This allows us to estimate the membership based on the relative magnitude of the score (see the score distribution on practical datasets in Appendix~\ref{sec:score_dis_app}).

Our membership estimation score offers a clear mathematical interpretation and is easily implementable in practical applications.
Furthermore, the definition of score can be generalized to leverage a subspace of $k$ orthogonal singular vectors:
\begin{equation}
\label{eq:score_halo}
    \zeta_i = \frac{1}{k} \sum_{j=1}^k \sigma_j \cdot\left<\*f_i, \*v_j \right>^2,
\end{equation}
where $\mathbf{v}_j$ is the $j^\text{th}$ column of $\mathbf{V}$, and $\sigma_j $ is the corresponding singular value. $k$ is the number of spanning directions in the subspace.
The intuition is that hallucinated samples can be captured by a small subspace, allowing them to be distinguished from the truthful samples. We show in
Section~\ref{sec:ablation} that leveraging subspace with multiple components can capture the truthfulness encoded in LLM activations more effectively
than a single direction.

\subsection{Truthfulness Classifier}
\label{sec:step_2}
Based on the procedure in Section~\ref{sec:step_1}, we denote ${\mathcal{H}} = \{\widetilde{\*x}_i\in {\mathcal{M}}: \zeta_i > T\}$ as the (potentially noisy) set of hallucinated samples and ${\mathcal{T}} = \{\widetilde{\*x}_i\in {\mathcal{M}}: \zeta_i \leq T\}$ as the candidate truthful set. We then train a truthfulness classifier $\*g_{\boldsymbol{\theta}}$ that optimizes for the separability between the two sets. In particular, our training objective can be viewed as minimizing the following risk, so that 
sample $\widetilde{\*x}$ from ${\mathcal{T}}$ is predicted as positive and vice versa.
\begin{equation}\label{eq:reg_loss_halo}
\begin{split}
        R_{{\mathcal{H}},{\mathcal{T}}}(\*g_{\boldsymbol{\theta}}) & =   R_{{\mathcal{T}}}^{+}(\*g_{\boldsymbol{\theta}})+R_{{\mathcal{H}}}^{-}(\*g_{\boldsymbol{\theta}}) \\& =  \mathbb{E}_{\widetilde{\*x} \in  {\mathcal{T}}}~\mathds{1}\{\*g_{\boldsymbol{\theta}}(\widetilde{\*x}) \leq 0\}+\mathbb{E}_{{\widetilde{\*x}} \in  {\mathcal{H}}}~\mathds{1} \{\*g_{\boldsymbol{\theta}}(\widetilde{\*x}) > 0\}.
        \end{split}
\end{equation}
To make the $0/1$ loss tractable, we replace it with the
binary sigmoid loss, a smooth approximation of the $0/1$
loss. During test time, we leverage the trained classifier for hallucination detection with the truthfulness scoring function of $S(\*x') = \frac{e^{\*g_{\boldsymbol{\theta}}(\*x')}}{1 + e^{\*g_{\boldsymbol{\theta}}(\*x')}}$, where $\*x'$ is the test data. Based on the scoring function, the hallucination detector is $G_{\lambda}(\*x') = \mathds{1}\{S(\*x') \geq \lambda\}$, where $1$ indicates the positive class (truthful) and $0$ indicates otherwise. 


\section{Experiments}
\label{sec:exp_halo}
In this section, we present empirical evidence to validate the effectiveness of our method on various hallucination detection tasks. We describe the setup in Section~\ref{sec:setup}, followed by the results and comprehensive analysis in Section~\ref{sec:main_result}–Section~\ref{sec:ablation}.
\subsection{Setup}
\label{sec:setup}
\textbf{Datasets and models.}
We consider four generative question-answering (QA) tasks for evaluation, including two open-book conversational QA datasets \textsc{CoQA}~\citep{reddy2019coqa} and \textsc{TruthfulQA}~\citep{lin2021truthfulqa} (generation track), closed-book QA dataset \textsc{TriviaQA}~\citep{triviaqa}, and reading comprehension dataset \textsc{TydiQA-GP} (English)~\citep{clark2020tydi}. Specifically, we have 817 and 3,696 QA pairs for \textsc{TruthfulQA} and \textsc{TydiQA-GP} datasets, respectively, and follow~\citep{lin2023generating} to utilize the development split of
\textsc{CoQA} with 7,983 QA pairs, and the deduplicated validation split of
the \textsc{TriviaQA} (\textit{rc.nocontext subset}) with 9,960 QA pairs. We reserve 25\% of the available QA pairs for testing and 100 QA pairs for validation, and the remaining questions are used to simulate the unlabeled generations in the wild. By default, the generations are based on greedy sampling, which predicts the most probable token. Additional sampling strategies are studied in Appendix~\ref{sec:sampling}.

We evaluate our method using two families of models: LLaMA-2-chat-7B \& 13B~\citep{touvron2023llama} and OPT-6.7B \& 13B~\citep{zhang2022opt}, which are popularly adopted public foundation models with accessible internal representations. Following the convention, we use
the pre-trained weights and conduct zero-shot inference in all cases. More dataset and inference details are provided in Appendix~\ref{sec:inference_app}.

\paragraph{Baselines.} We compare our approach with a comprehensive collection of baselines, categorized as follows: {(1)} \emph{uncertainty-based} hallucination detection approaches--Perplexity~\citep{ren2023outofdistribution}, Length-Normalized Entropy (LN-entropy)~\citep{malinin2021uncertainty} and Semantic Entropy~\citep{kuhn2023semantic}; {(2)} \emph{consistency-based} methods--Lexical Similarity~\citep{lin2023generating}, SelfCKGPT~\citep{manakul2023selfcheckgpt} and EigenScore~\citep{chen2024inside}; {(3)} \emph{prompting-based} strategies--Verbalize~\citep{lin2022teaching} and Self-evaluation~\citep{kadavath2022language}; and {(4)} \emph{knowledge discovery-based} method Contrast-Consistent Search (CCS)~\citep{burns2022discovering}. 
To ensure a fair comparison, we assess all baselines on identical test data, employing the default experimental configurations as outlined in their respective papers. We discuss the implementation details for baselines in Appendix~\ref{sec:inference_app}.

\paragraph{Evaluation.} 
Consistent with previous studies~\citep{manakul2023selfcheckgpt, kuhn2023semantic}, we evaluate the effectiveness of all methods by the area under the receiver operator characteristic curve
(AUROC), which measures the performance of a binary classifier under varying thresholds. The generation is deemed truthful when the similarity score between the generation and the ground truth exceeds a
given threshold of 0.5. We follow~\cite{lin2021truthfulqa} and use the BLUERT~\citep{sellam2020bleurt} to measure the similarity, 
a learned metric built upon BERT~\citep{devlin2018bert} and is augmented with diverse lexical and semantic-level supervision signals.
Additionally, we show the results are robust under a different similarity measure ROUGE~\citep{lin2004rouge} following Kuhn \emph{et al.}~\citep{kuhn2023semantic} in Appendix~\ref{sec:different_split_app}, which is based on substring matching.

\paragraph{Implementation details.} Following \citep{kuhn2023semantic}, we generate the most likely answer by beam search with 5 beams for evaluation, and use multinomial sampling to generate 10 samples per question with a temperature of 0.5 for baselines that require multiple generations. Following literature~\citep{chen2024inside,azaria2023internal}, we prepend the question to the generated answer and use the last-token embedding to identify the subspace and train the truthfulness classifier. The truthfulness classifier $\*g_{\boldsymbol{\theta}}$ is a two-layer MLP with ReLU non-linearity and an intermediate dimension of 1,024. We train $\*g_{\boldsymbol{\theta}}$ for 50 epochs with SGD optimizer, an initial learning rate of 0.05, cosine learning rate decay, batch size of 512, and weight decay of 3e-4. The layer index for representation extraction, the number of singular vectors $k$, and the filtering threshold $T$ are determined using the separate validation set.

\subsection{Main Results}
\label{sec:main_result}
\begin{table}[t]
    \centering
    \footnotesize
   \scalebox{0.7}{ \begin{tabular}{ccc|cccc}
    \hline
  Model & Method&  \makecell{Single \\ sampling} &  \textsc{TruthfulQA} & \textsc{TriviaQA} & \textsc{CoQA} &\textsc{TydiQA-GP} \\
    \hline
  
       \multirow{11}{*}{LLaMA-2-7b }  
     
     &   Perplexity~\citep{ren2023outofdistribution} & \checkmark &  56.77 & 72.13  & 69.45   & 78.45 \\
    &  LN-Entropy~\citep{malinin2021uncertainty}  & \xmark & 61.51 & 70.91  & 72.96   & 76.27  \\
  & Semantic Entropy~\citep{kuhn2023semantic} & \xmark &  62.17 & 73.21  & 63.21   & 73.89  \\
 & Lexical Similarity~\citep{lin2023generating} & \xmark & 55.69  & 75.96  & 74.70  & 44.41  \\
     &   EigenScore~\citep{chen2024inside} & \xmark & 51.93 & 73.98  & 71.74  & 46.36  \\
     &   SelfCKGPT~\citep{manakul2023selfcheckgpt} & \xmark & 52.95   & 73.22 & 73.38   & 48.79  \\
 
   &Verbalize~\citep{lin2022teaching} & \checkmark & 53.04  & 52.45   & 48.45  & 47.97   \\
  & Self-evaluation~\citep{kadavath2022language} &\checkmark &  51.81  & 55.68  &46.03   & 55.36   \\
   &CCS~\citep{burns2022discovering} & \checkmark & 61.27 & 60.73 & 50.22 & 75.49  \\ 
&CCS$^{*}$~\citep{burns2022discovering}& \checkmark & 67.95  & 63.61  &  51.32 & 80.38  \\ 

 &\cellcolor[HTML]{EFEFEF}\multirow{1}{*}{HaloScope (\textsc{Ours})}  & \cellcolor[HTML]{EFEFEF}\checkmark
&\cellcolor[HTML]{EFEFEF}\textbf{78.64} &  \cellcolor[HTML]{EFEFEF}\textbf{77.40}  & \cellcolor[HTML]{EFEFEF}\textbf{76.42}   & \cellcolor[HTML]{EFEFEF}\textbf{94.04}  \\
\hline
          \multirow{11}{*}{OPT-6.7b }  
                & Perplexity~\cite{ren2023outofdistribution}& \checkmark  & 59.13& 69.51 &70.21  &63.97 \\
      & LN-Entropy~\citep{malinin2021uncertainty} & \xmark  &  54.42& 71.42 &71.23  & 52.03   \\
   & Semantic Entropy~\citep{kuhn2023semantic}& \xmark &  52.04 & 70.08 &   69.82  & 56.29 \\
 &  Lexical Similarity~\citep{lin2023generating} & \xmark &  49.74  &  71.07 &  66.56& 60.32   \\
     &    EigenScore~\citep{chen2024inside}& \xmark & 41.83 &  70.07 & 60.24& 56.43    \\
      &   SelfCKGPT~\citep{manakul2023selfcheckgpt} & \xmark& 50.17  &  71.49    & 64.26 & 75.28\\
   & Verbalize~\citep{lin2022teaching} & \checkmark& 50.45& 50.72 &  55.21 &  57.43\\
   & Self-evaluation~\citep{kadavath2022language} & \checkmark&  51.00 & 53.92& 47.29 & 52.05 \\
   &CCS~\citep{burns2022discovering}& \checkmark& 60.27 & 51.11 & 53.09 & 65.73  \\ 
 &CCS$^{*}$~\citep{burns2022discovering}& \checkmark&  63.91 & 53.89 &  57.95 & 64.62 \\ 
 & \cellcolor[HTML]{EFEFEF}\multirow{1}{*}{HaloScope (\textsc{Ours})}   & \cellcolor[HTML]{EFEFEF}\checkmark
&\cellcolor[HTML]{EFEFEF}\textbf{73.17} &  \cellcolor[HTML]{EFEFEF}\textbf{72.36} & \cellcolor[HTML]{EFEFEF}\textbf{77.64} & \cellcolor[HTML]{EFEFEF}\textbf{80.98}\\
\hline
    \end{tabular}}
    \vspace{1em}
    \caption[Main results of HaloScope]{\small \textbf{Main results.} Comparison with competitive hallucination detection methods on different datasets. 
    All values
are percentages (AUROC). ``Single sampling" indicates whether the approach requires multiple generations during inference. \textbf{Bold} numbers are superior results. }
 \vspace{-2em}
    \label{tab:main_table}
\end{table}

As shown in Table~\ref{tab:main_table}, we compare our method HaloScope with competitive hallucination detection methods, where HaloScope outperforms the state-of-the-art method by a large margin in both LLaMA-2-7b-chat and OPT-6.7b models.  We observe that HaloScope outperforms uncertainty-based and consistency-based baselines, exhibiting 16.47\% and 26.71\% improvement over Semantic Entropy and EigenScore on the challenging \textsc{TruthfulQA} task. From a computation perspective, uncertainty-based and consistency-based approaches typically require sampling multiple generations per question during testing time, incurring an aggregate time complexity $O(Km^2)$ where $K$ is the number of repeated sampling, and $m$ is the number of generated tokens. 
In contrast, HaloScope does not require sampling multiple generations and thus is significantly more efficient in inference, with a standard complexity $O(m^2)$ for transformer-based sequence generation. We also notice that prompting language models to assess the factuality of their generations is not effective because of the overconfidence issue discussed in prior work~\citep{zhou2023navigating}. Lastly, we compare HaloScope with CCS~\citep{burns2022discovering}, which trains a binary truthfulness classifier to satisfy logical consistency properties, such that a statement and its negation have opposite truth values. Different from our framework, CCS does not leverage LLM generations but instead human-written answers, and does not involve a membership estimation process. For a fair comparison, we implemented an improved version CCS*, which trains the binary classifier using the LLM generations (the same as those in HaloScope). 
 The result shows that HaloScope significantly outperforms CCS$^*$, suggesting the advantage of our membership estimation score. Moreover, we find that CCS$^*$ performs better than CCS in most cases. This highlights the importance of harnessing LLM generations for hallucination detection, which better captures the distribution of model-generated content than human-written data.

\subsection{Robustness to Practical Challenge}
\label{sec:robustness}
HaloScope is a practical framework that may face real-world challenges. In this section, we explore how well HaloScope deals with different data distributions, and its scalability to larger LLMs.

\vspace{-0.2cm}
\paragraph{Does HaloScope generalize across varying data distributions?} We explore whether HaloScope can effectively generalize to different data distributions. This investigation involves directly applying the extracted subspace from one dataset (referred to as the source (s)) and computing the membership assignment score on different datasets (referred to as the target (t)) for truthfulness classifier training. The results depicted in Figure~\ref{fig:ablation_frame_range_halo} (a) showcase the robust transferability of our approach HaloScope across diverse datasets. Notably, HaloScope achieves a hallucination detection AUROC of 76.26\% on \textsc{TruthfulQA} when the subspace is extracted from the \textsc{TriviaQA} dataset, demonstrating performance close to that obtained directly from \textsc{TruthfulQA} (78.64\%). This strong transferability underscores the potential of our method to facilitate real-world LLM applications, particularly in scenarios where user prompts may undergo domain shifts. In such contexts, HaloScope remains highly effective in detecting hallucinations, offering flexibility and adaptability. 

\begin{figure}[t]
    \centering
  \includegraphics[width=1.0\linewidth]{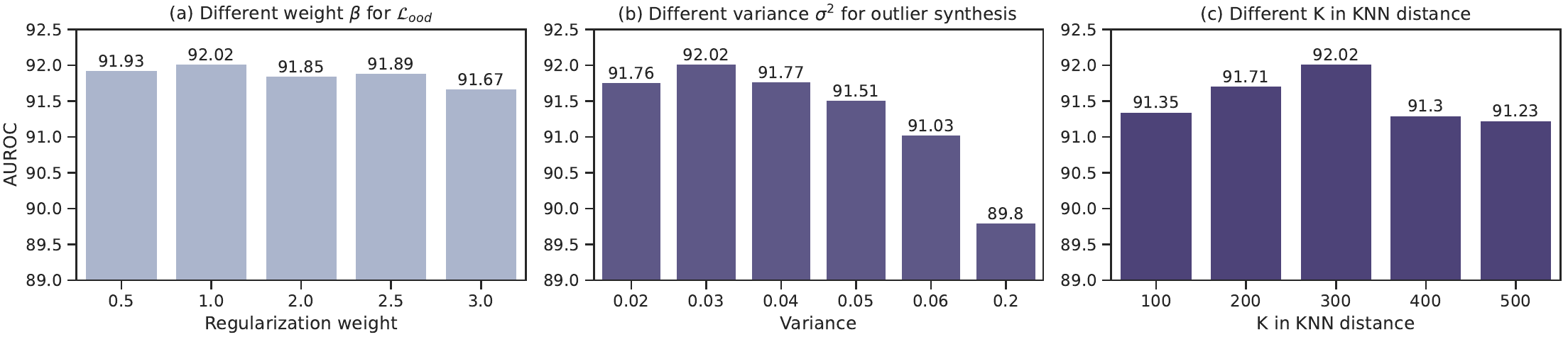}
  \vspace{-0.6cm}
    \caption[Ablation study for HaloScope]{\small (a) Generalization across four datasets, where ``(s)" denotes the source dataset and ``(t)" denotes the target dataset. (b) Effect of the number of subspace components $k$  (Section~\ref{sec:step_1}). (c) Impact of different layers. All numbers are AUROC based on LLaMA-2-7b-chat. Ablation in (b) \& (c) are based on \textsc{TruthfulQA}.  }
    \label{fig:ablation_frame_range_halo}
    \vspace{-1em}
\end{figure}

\paragraph{HaloScope scales effectively to larger LLMs.} To illustrate effectiveness with larger LLMs, we evaluate our approach on the LLaMA-2-13b-chat and OPT-13b models. The results of our method HaloScope, presented in Table~\ref{tab:larger_models}, not only surpass two competitive baselines but also exhibit improvement over results obtained with smaller LLMs. For instance, HaloScope achieves an AUROC of 82.41\% on the TruthfulQA dataset for the OPT-13b model, compared to 73.17\% for the OPT-6.7b model, representing a direct 9.24\% improvement.

\begin{table}[h]
    \centering
    \small
    \begin{tabular}{c|cccc}
    \toprule 
       Method  & \textsc{TruthfulQA}   & \textsc{TydiQA-GP}& \textsc{TruthfulQA}   & \textsc{TydiQA-GP}\\
      \midrule
       & \multicolumn{2}{c}{LLaMA-2-chat-13b} & \multicolumn{2}{c}{OPT-13b}  \\
       \cline{2-5}
       Semantic Entropy  & 57.81 & 72.66  & 58.64 & 55.50\\
       SelfCKGPT & 54.88 & 52.42   & 59.66& 76.10   \\
       \cellcolor[HTML]{EFEFEF}HaloScope (Ours)& \cellcolor[HTML]{EFEFEF}\textbf{80.37}& \cellcolor[HTML]{EFEFEF}\textbf{95.68} & \cellcolor[HTML]{EFEFEF}\textbf{82.41} & \cellcolor[HTML]{EFEFEF}\textbf{81.58} \\
            \bottomrule
    \end{tabular}
    \vspace{0.5em}
    \caption[Hallucination detection results on larger LLMs]{\small Hallucination detection results on larger LLMs. }
    \label{tab:larger_models}
    \vspace{-2em}
\end{table}

\subsection{Ablation Study}
\label{sec:ablation}
In this section, we conduct a series of in-depth analyses to understand the various design choices for our algorithm HaloScope. Additional ablation studies are discussed in Appendix~\ref{sec:rouge_exp_app}-\ref{sec:using_other_sciore_filter_app}.

\paragraph{How do different layers impact HaloScope's performance?} 
In Figure~\ref{fig:ablation_frame_range_halo} (c), we delve into hallucination detection using representations extracted from {different layers} within the LLM. The AUROC values of truthful/hallucinated classification are evaluated based on the LLaMA-2-7b-chat model. All other configurations are kept the same as our main experimental setting. We observe a notable trend that the hallucination detection performance initially increases from the top to middle layers (e.g., 8-14th layers), followed by a subsequent decline. This trend suggests a gradual capture of contextual information by LLMs in the first few layers, followed by a tendency towards overconfidence in the final layers due to the autoregressive training objective aimed at vocabulary mapping. This observation echoes prior findings that indicate representations at intermediate layers~\citep{chen2024inside, azaria2023internal} are the most effective for downstream tasks.

\paragraph{Where to extract embeddings from multi-head attention?}Moving forward, we investigate the multi-head attention (MHA) architecture's effect on representing hallucination. Specifically, the MHA  can be conceptually expressed as:
\begin{equation}
    \*f_{i+1} = \*f_{i} +  \*Q_i \operatorname{Attn}_i(\*f_{i}),
\end{equation}
where $ \*f_{i}$ denotes the output of the $i$-th transformer block, $\operatorname{Attn}_i(\*f_{i})$ denotes the output of the self-attention module in the $i$-th block, and $\*Q_i$ is the weight of the feedforward layer. Consequently, we evaluate the hallucination detection performance utilizing representations from three \emph{different locations within the MHA architecture}, as delineated in Table~\ref{tab:embed_loc}.

\begin{table}[!h]
    \centering
    \small
    \begin{tabular}{c|cccc}
    \toprule 
       Embedding location  & \textsc{TruthfulQA}   & \textsc{TydiQA-GP}& \textsc{TruthfulQA}   & \textsc{TydiQA-GP}\\
      \midrule
       & \multicolumn{2}{c}{LLaMA-2-chat-7b} & \multicolumn{2}{c}{OPT-6.7b}  \\
       \cline{2-5}
       $\*f$ &\textbf{78.64} & \textbf{94.04}   & 68.95 & 75.72 \\
$\operatorname{Attn}(\*f)$  & 
75.63 & 92.85 
&69.84  & 73.47 \\
    $ \*Q \operatorname{Attn}(\*f)$ &    76.06& 93.33 & \textbf{73.17}& \textbf{80.98} \\
            \bottomrule
    \end{tabular}
    \vspace{0.5em}
    \caption[Hallucination detection results on different representation locations of multi-head attention]{\small Hallucination detection results on different representation locations of multi-head attention. }
    \label{tab:embed_loc}
\end{table}
We observe that the LLaMA model tends to encode the hallucination information mostly in the output of the transformer block while the most effective location for  OPT models is the output of the feedforward layer, and we implement our hallucination detection algorithm based on this observation for our main results in Section~\ref{sec:main_result}.

\paragraph{Ablation on different design choices of membership score.} We systematically explore different design choices for the scoring function  (Equation~\ref{eq:score_halo}) aimed at distinguishing between truthful and untruthful generations within unlabeled data. Specifically, we investigate the following aspects:
\textbf{(1)} The impact of the number of subspace components  $k$; \textbf{(2)} The significance of the weight coefficient associated with the singular value $\sigma$ in the scoring function; and \textbf{(3)} A comparison between score calculation based on the best individual LLM layer versus summing up layer-wise scores. Figure~\ref{fig:ablation_frame_range_halo} (b) depicts the hallucination detection performance with varying $k$ values (ranging from 1 to 10). Overall, we observe superior performance with a moderate value of $k$. These findings align with our assumption that hallucinated samples may be represented by a small subspace, suggesting that only a few key directions in the activation space are capable of distinguishing hallucinated samples from truthful ones. Additionally, we present results obtained from LLaMA and OPT models when employing a non-weighted scoring function ($\sigma_j=1$ in Equation~\ref{eq:score_halo}) in Table~\ref{tab:score_design}. We observe that the scoring function weighted by the singular value outperforms the non-weighted version, highlighting the importance of prioritizing top singular vectors over others. Lastly, summing up layer-wise scores results in significantly worse detection performance, which can be explained by the low separability between truthful and hallucinated data in the top and bottom layers of LLMs.

\begin{table}[!h]
    \centering
    \small
    \scalebox{0.8}{\begin{tabular}{c|cccc}
    \toprule 
       Score design  & \textsc{TruthfulQA}   & \textsc{TydiQA-GP}& \textsc{TruthfulQA}   & \textsc{TydiQA-GP}\\
      \midrule
       & \multicolumn{2}{c}{LLaMA-2-chat-7b} & \multicolumn{2}{c}{OPT-6.7b}  \\
       \cline{2-5}
       Non-weighted score & 77.24 & 90.26 & 71.72 & 80.18\\
       Summing up layer-wise scores & 65.82 & 87.62 & 62.98 & 70.03 \\
      \cellcolor[HTML]{EFEFEF}HaloScope (Ours)& \cellcolor[HTML]{EFEFEF}\textbf{78.64}& \cellcolor[HTML]{EFEFEF}\textbf{94.04} & \cellcolor[HTML]{EFEFEF}\textbf{73.17} & \cellcolor[HTML]{EFEFEF}\textbf{80.98} \\
            \bottomrule
    \end{tabular}}
    \vspace{0.5em}
    \caption[Hallucination detection results on different membership estimation scores]{\small Hallucination detection results on different membership estimation scores. }
    \label{tab:score_design}
\end{table}

\begin{wrapfigure}{r}{0.4\textwidth}
\vspace{-0.8cm}
\begin{center}
    \includegraphics[width=\linewidth]{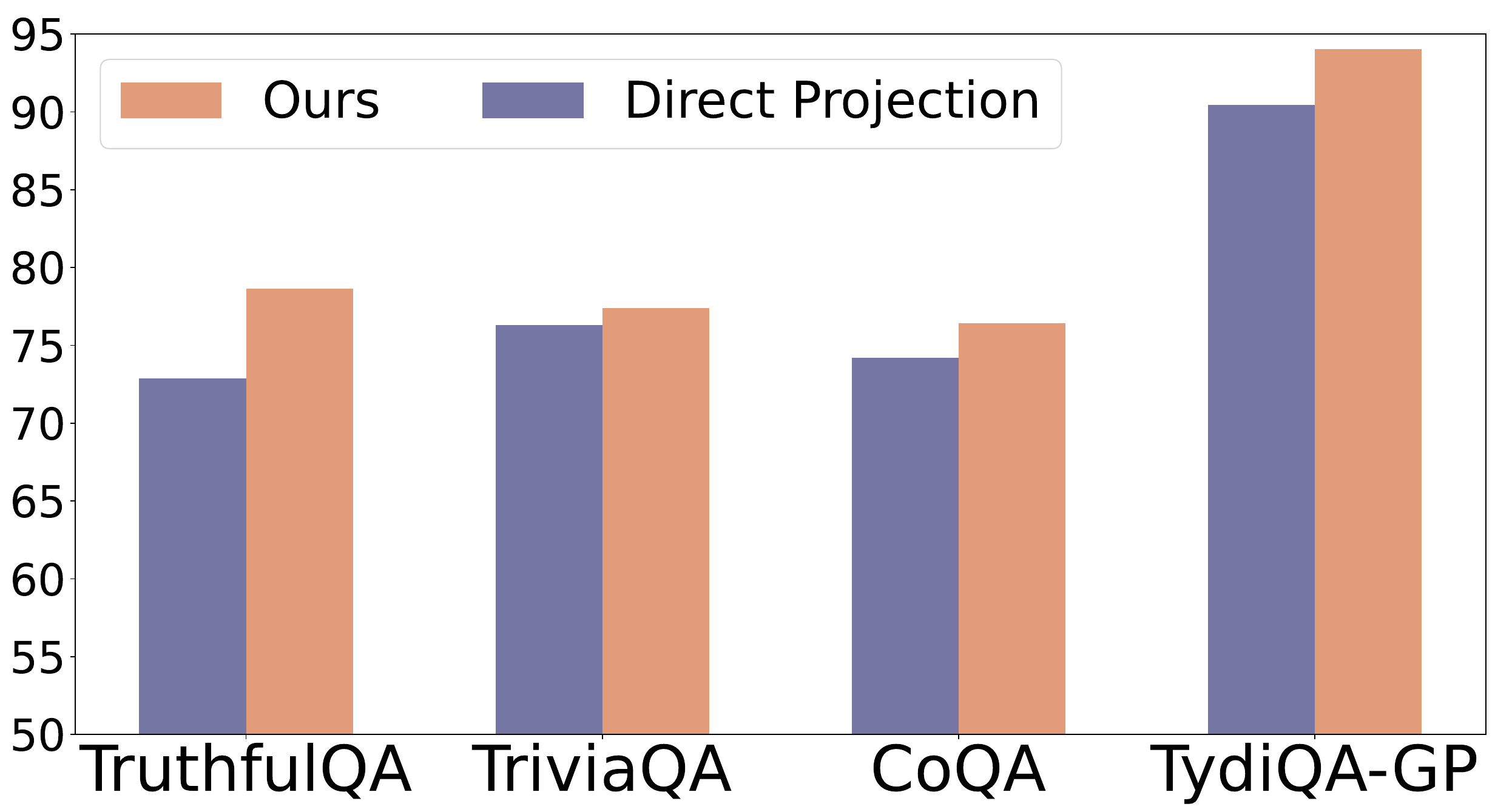}
  \end{center}
   \vspace{-1em}
    \caption[Comparison with using direction projection for hallucination detection]{\small Comparison with using direction projection for hallucination detection. Value is AUROC.}
      \vspace{-0.4cm}
    \label{fig:dp}
\end{wrapfigure}

\paragraph{What if directly using the membership score for detection?} 
Figure~\ref{fig:dp} showcases the performance of directly detecting hallucination using the score defined in Equation~\ref{eq:score_halo}, which involves projecting the representation of a test sample to the extracted subspace and bypasses the training of the binary classifier as detailed in Section~\ref{sec:step_2}. On all four datasets, HaloScope demonstrates superior performance compared to this direct projection approach  on LLaMA, highlighting the efficacy of leveraging unlabeled data for training and the enhanced generalizability of the truthfulness classifier.

\begin{wrapfigure}{r}{0.4\textwidth}
\vspace{-0.3cm}
\begin{center}
    \includegraphics[width=\linewidth]{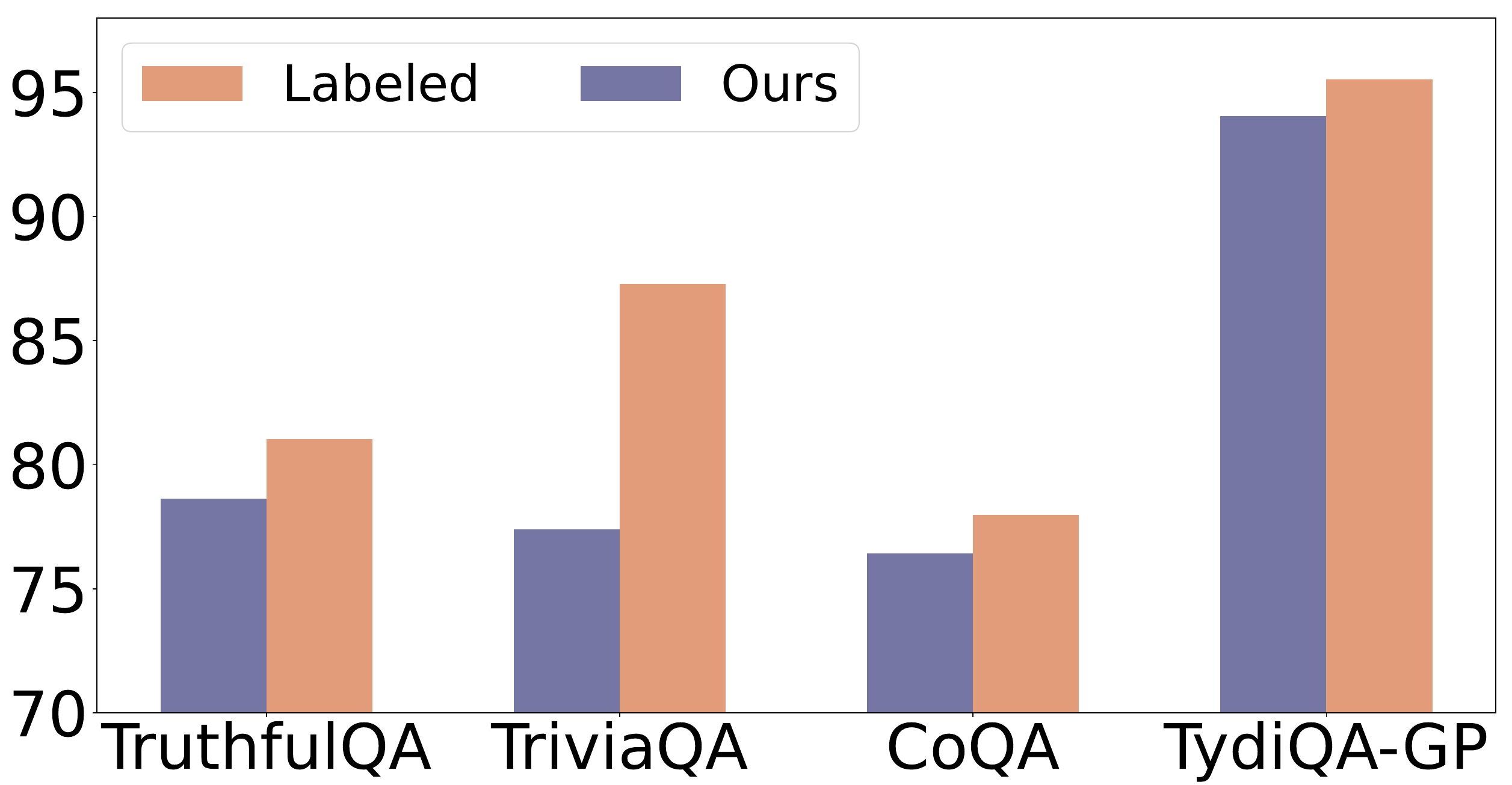}
  \end{center}
  \vspace{-0.6em}
    \caption[Comparison with ideal performance when training on labeled data]{\small Comparison with ideal performance when training on labeled data. }
      \vspace{-0.6cm}
    \label{fig:auroc}
\end{wrapfigure}

\paragraph{Comparison with training on labeled data.} We further benchmark our approach against a supervised oracle, where each generation in $\mathcal{M}$ is explicitly labeled as truthful or hallucinated similar to~\citep{azaria2023internal,duan2024llms,li2023inference}, and a binary classifier is trained based on representations of the labeled data. This enables us to assess the performance gap between HaloScope and the upper bound.  In Figure~\ref{fig:auroc}, we compare our approach with the supervised oracle on the same test set and report the AUROC metric. The results, based on the LLaMA-2-7b-chat model, demonstrate that our method achieves hallucination detection accuracy on \textsc{TruthfulQA} (AUROC: 78.64\%) that favorably matches the performance of the supervised oracle (AUROC: 81.04\%). This is encouraging considering that our method does not rely on any manual annotation. 

\paragraph{Qualitative results.} We provide qualitative examples of the model's truthfulness score (as introduced in Section~\ref{sec:step_2}) for different generated texts given the same input prompts during inference (Figure~\ref{fig:quali}). The model is LLaMA-2-chat-13b and we utilize questions in \textsc{TruthfulQA} that can cause misconception even for humans, for example, \emph{Q: Who first started Tesla Motors?} We find that HaloScope can accurately produce the scores that align with the truthfulness of the answers (the higher the score, the more truthful the answer is).

\begin{figure}[!h]
    \centering
  \includegraphics[width=1.0\linewidth]{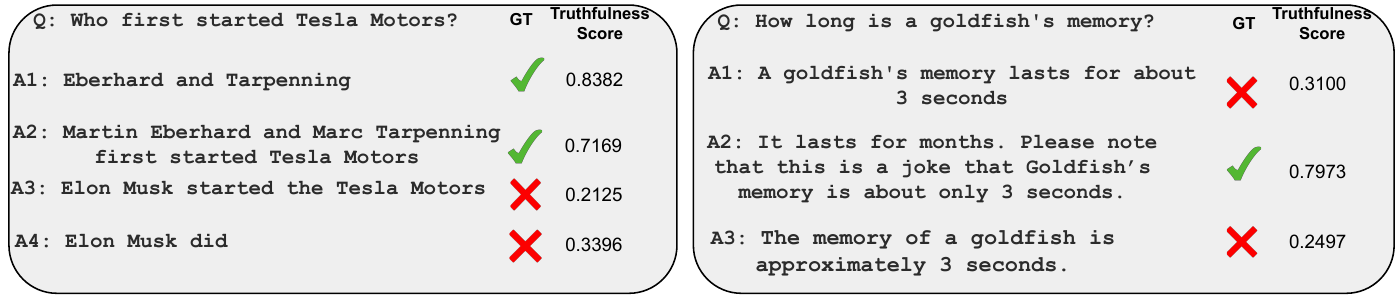}
  \vspace{-0.5cm}
    \caption[Examples from {TruthfulQA} that show the effectiveness of our approach]{\small Examples from \textsc{TruthfulQA} that show the effectiveness of our approach. Specifically, we compare the truthfulness scores $S(\*x')$ (Section~\ref{sec:step_2}) of HaloScope with different answers to the prompt. The green check mark and red cross indicate the ground truth of being truthful vs. hallucinated.}
    \label{fig:quali}
    \vspace{-1em}
\end{figure}

\section{Summary}
In this chapter, we propose a novel algorithmic framework HaloScope for hallucination detection, which exploits the unlabeled LLM generations arising in the wild. HaloScope first estimates the membership (truthful vs. hallucinated) for samples in the unlabeled mixture data based on an embedding factorization, and then trains a binary truthfulness classifier on top. The empirical result shows that HaloScope establishes superior performance on a comprehensive set of question-answering datasets and different families of LLMs.  Our in-depth quantitative and qualitative ablations provide further insights on the
efficacy of HaloScope. We hope our work will inspire future research on hallucination detection with unlabeled LLM generations, where a promising future work can be investigating how to train the hallucination classifier in order to generalize well with a distribution shift between the unlabeled data and the test data.
\end{halo}

\begin{future}
    My research has tackled several foundational reliability challenges arising from the ML paradigm shift, yet significant work remains. Moving forward, I am eager to propel the field of reliable ML by exploring several pivotal directions that will strengthen both theoretical foundations and practical applications of ML systems.

\textbf{Comprehensive investigation of ML reliability challenges.} Developing genuinely reliable ML systems necessitates a deep understanding of the limitations and risks across diverse deployment scenarios. I plan to systematically investigate the safety and reliability challenges inherent in current ML models. For instance, foundation models encounter risks in different stages, such as noisy pretraining data, label ambiguity and data fairness in supervised fine-tuning, preference inconsistencies in RLHF, and vulnerabilities to adversarial attacks during inference. By rigorously characterizing these challenges, I aim to \textit{establish comprehensive benchmarks and devise targeted strategies to enhance model reliability}. This work will improve the safety of ML algorithms across various applications while opening new avenues of research to enrich the reliable ML community.

 \textbf{Development of adaptable and generalized algorithms for reliable ML.} Beyond my current work that focuses on foundational reliability learning dynamics with minimal human supervision, I intend to expand the development of reliable ML methodologies from three \textbf{core} perspectives: \textit{data, representation, and training/inference algorithms}. For example, I will investigate (1) the impact of diverse data structures, such as human feedback, weak supervision, and semi-supervised data, on model reliability (data perspective); (2) how advancements in representation learning and model architecture can fundamentally enhance reliability (representation perspective); and (3) the design of distributionally robust training methods alongside calibrated, efficient inference algorithms that can adapt to varying deployment environments and integrate multiple knowledge sources (algorithm perspective). These principles will serve as a flexible foundation for improving reliability across a wide spectrum of machine learning models and applications.

\textbf{Reliable ML for boarder scientific discovery.} My long-term vision also includes \textit{harnessing the power of reliable ML to accelerate scientific discovery across multiple domains}, such as reliable biometrics with distribution shifts~\citep{du2020cross,du2019continual,shao2019cross}, less human annotations~\citep{du2019low,du2019building}, multimodal fusion~\citep{zhong2019hand}, and reliable protein structural analysis with OOD detection~\citep{liu2024exploring}, active learning~\citep{du2021active}, semi-supervised learning~\citep{liu2019semi}, and open-set classification~\citep{du2019open}, where I have expertise in. Moreover, I look forward to collaborating with domain experts in other disciplines, such as chemistry, sociology, and environmental science, to explore how reliable ML can address their unique challenges. By leveraging interdisciplinary knowledge, I aim to develop tailored ML solutions that not only enhance predictive accuracy but also improve the interpretability and trustworthiness of models in complex, data-driven environments. This collaborative effort will ultimately contribute to more robust scientific methodologies and facilitate breakthroughs that are essential for addressing pressing global issues.

  Overall, my research approach has been to leverage theories and insights drawn from ML and data analytics to address fundamentally new reliability problems arising from the real world. As a future machine learning researcher, I aspire to maintain this principled approach and advance this compelling research agenda with a vision for impactful contributions to both the academic community and society at large.
\end{future}





\newpage
\begin{app}
  \section{VOS: Learning What You Don’t Know by Virtual Outlier Synthesis}


\subsection{Experimental details}
\label{sec:dataset}
We summarize the OOD detection evaluation task in Table~\ref{tab:task}. The OOD test dataset is selected from MS-COCO and OpenImages dataset, which contains disjoint labels from the respective ID dataset.  The PASCAL model is trained for a total of 18,000 iterations, and the BDD-100k model is trained for 90,000 iterations. We add the uncertainty regularizer (Equation~\ref{eq:reg_loss}) starting from 2/3 of the training. 
The weight $\beta$ is set to $0.1$. See \emph{detailed ablations on the hyperparameters in Appendix~\ref{sec:ablation}}.

\begin{table}[h]
\centering

\scalebox{0.8}{\begin{tabular}{@{}lrr@{}}
\toprule
 & \textbf{Task 1} &\textbf{ Task 2}  \\ \midrule
ID train dataset & VOC train & BDD train \\
ID val dataset & VOC val & BDD val  \\
OOD dataset & {COCO  and OpenImages val} & {COCO  and OpenImages val}  \\

$\#$ID train images & 16,551 & 69,853  \\
$\#$ID val images & 4,952 & 10,000\\
$\#$OOD images for COCO& 930 &1,880 \\ 
$\#$OOD images for OpenImages& 1,761 & 1,761\\ 

\bottomrule
\end{tabular}}
\caption[OOD detection evaluation tasks for VOS]{OOD detection evaluation tasks. }
\label{tab:task}
\end{table}

\subsection{Software and hardware}
\label{sec:hardware}
We run all experiments with Python 3.8.5 and PyTorch 1.7.0, using NVIDIA GeForce RTX 2080Ti GPUs.

\subsection{Effect of hyperparameters}
\label{sec:ablation}

Below we perform sensitivity analysis for each important hyperparameter\footnote{Note that our sensitivity analysis uses the speckle noised PASCAL VOC validation dataset as OOD data, which is different from the actual OOD test datasets in use.}. We use ResNet-50 as the backbone, trained on in-distribution dataset PASCAL-VOC.

\textbf{Effect of $\epsilon$}. Since the threshold $\epsilon$ can be infinitesimally small, we instead choose $\epsilon$ based on the $t$-th smallest likelihood in a pool of 10,000 samples (per-class), generated from the class-conditional Gaussian distribution. A larger $t$ corresponds to a larger threshold $\epsilon$. As shown in Table~\ref{tab:ablation_appendix2_detection}, a smaller  $t$ yields good performance. We set $t=1$ for all our experiments.

\begin{table}[h]
    \centering
    
    \tabcolsep 0.04in\renewcommand\arraystretch{0.7}{}
    \begin{tabular}{c|cccc}
    \toprule 
$t$&mAP$\uparrow$ & FPR95 $\downarrow$&AUROC$\uparrow$&AUPR$\uparrow$ \\
        \midrule
1        & 48.7 &\textbf{54.69} & \textbf{83.41} & \textbf{92.56}  \\ 
2      & 48.2 & 57.96 & 82.31 & 88.52 \\
3          & 48.3 & 62.39 & 82.20 & 88.05 \\
4           & 48.8 & 69.72 & 80.86 & 89.54\\
5           & 48.7 & 57.57 & 78.66 & 88.20 \\
6          & 48.7& 74.03 & 78.06 & 91.17\\
8        & {48.8} & 60.12 & 79.53 & 92.53\\
10           & 47.2 &76.25 & 74.33 & 90.42\\
        \bottomrule
    \end{tabular}

      \caption{Ablation study on the number of selected outliers $t$ (per class).}
          \label{tab:ablation_appendix2_detection}
\end{table}

\textbf{Effect of queue size $|Q_k|$}.
We investigate the effect of ID queue size $|Q_k|$ in Table~\ref{tab:ablation_appendix1_detection}, where we vary $|Q_k|=\{50,100,200,400,600,800,1000\}$. Overall, a larger $|Q_k|$ is more beneficial since the estimation of Gaussian distribution parameters can be more precise. In our experiments, we set the queue size $|Q_k|$ to $1,000$ for PASCAL and $300$ for BDD-100k. The queue size is smaller for BDD because some classes have a limited number of object boxes.

\begin{table}[h]
    \centering
    
    \tabcolsep 0.04in\renewcommand\arraystretch{0.745}{}
    \begin{tabular}{c|cccc}
    \toprule 
$|Q_k|$&mAP$\uparrow$ & FPR95 $\downarrow$&AUROC$\uparrow$&AUPR$\uparrow$ \\
        \midrule
        50	& 48.6&68.42 & 77.04 & 92.30\\
        100&	{48.9}&59.77 & 79.96 & 89.18\\
        200	&48.8&57.80 & 80.20 & 89.92\\
        400	&{48.9}&66.85 & 77.68 & 89.83	\\
        600&48.5&57.32& 81.99&91.07	\\
        800&48.7&\textbf{51.43} & 82.26 & 91.80	\\
        1000	&48.7&54.69 & \textbf{83.41} & \textbf{92.56}\\
        \bottomrule
    \end{tabular}

  \caption{Ablation study on the ID queue size $|Q_k|$. }
    \label{tab:ablation_appendix1_detection}
\end{table}

\textbf{Effect of $\beta$}. As shown in Table~\ref{tab:ablation_appendix3_detection}, a mild value of $\beta$ generally works well. As expected, a large value (e.g., $\beta=0.5$) will over-regularize the model and harm the performance.

\begin{table}[!htb]
    \centering
    
    \tabcolsep 0.04in\renewcommand\arraystretch{0.7}{}
    \vspace{1em}\begin{tabular}{r|cccc}
    \toprule
$\beta$&mAP$\uparrow$ & FPR95 $\downarrow$&AUROC$\uparrow$&AUPR$\uparrow$ \\
         \midrule
0.01               & 48.8 & 59.20 & 82.64 & 90.08\\ 
0.05                & {48.9} & 57.21 & 83.27 & 91.00\\
0.1                 & 48.7 & \textbf{54.69} & \textbf{83.41} & \textbf{92.56} \\
0.15                & 48.5 & 59.32 & 77.47 & 89.06\\
0.5           & 36.4 &99.33 & 57.46 & 85.25\\
         \bottomrule
    \end{tabular}

      \caption{Ablation study on regularization weight $\beta$.}
          \label{tab:ablation_appendix3_detection}
\end{table}

\textbf{Effect of starting iteration for the regularizer}.  Importantly, we show that uncertainty regularization should be added in the middle of the training. If it is added too early, the feature space is not sufficiently discriminative for Gaussian distribution estimation. See Table~\ref{tab:ablation_appendix4_detection} for the effect of starting iteration $Z$. We use $Z=12,000$ for the PASCAL-VOC model, which is trained for a total of 18,000 iterations. 
\begin{table}[!htb]
    \centering
    
     \begin{tabular}{c|cccc}
    \toprule
$Z$&mAP$\uparrow$ & FPR95 $\downarrow$&AUROC$\uparrow$&AUPR$\uparrow$ \\
         \midrule
2000  & 48.5 &  60.01 & 78.55 & 87.62              \\ 
4000  & 48.4 & 61.47 & 79.85 & 89.41 \\
6000  & 48.5 &59.62 & 79.97 & 89.74\\
8000  & 48.7 &       56.85 & 80.64 & 90.71                 \\
10000 & 48.6 & 49.55 & 83.22 & 92.49  \\
12000 & 48.7 &     \textbf{54.69} & \textbf{83.41} & 92.56               \\
14000 & {49.0} &    55.39 & 81.37 & \textbf{93.00 }       \\
16000 & 48.9 &   59.36 & 82.70 & 92.62         \\
         \bottomrule
    \end{tabular}
      \caption[Ablation study on the starting iteration $Z$]{Ablation study on the starting iteration $Z$. Model is trained for a total of 18,000 iterations.}
          \label{tab:ablation_appendix4_detection}
\end{table}

\newpage

\subsection{Additional visualization results}
We provide additional visualization of the detected objects on different OOD datasets with models trained on different in-distribution datasets. The results are shown in Figures~\ref{fig:vi1}-\ref{fig:vi4}.

\label{sec:app_visual}
\begin{figure}[h]
    \centering
    \includegraphics[width=1.0\textwidth]{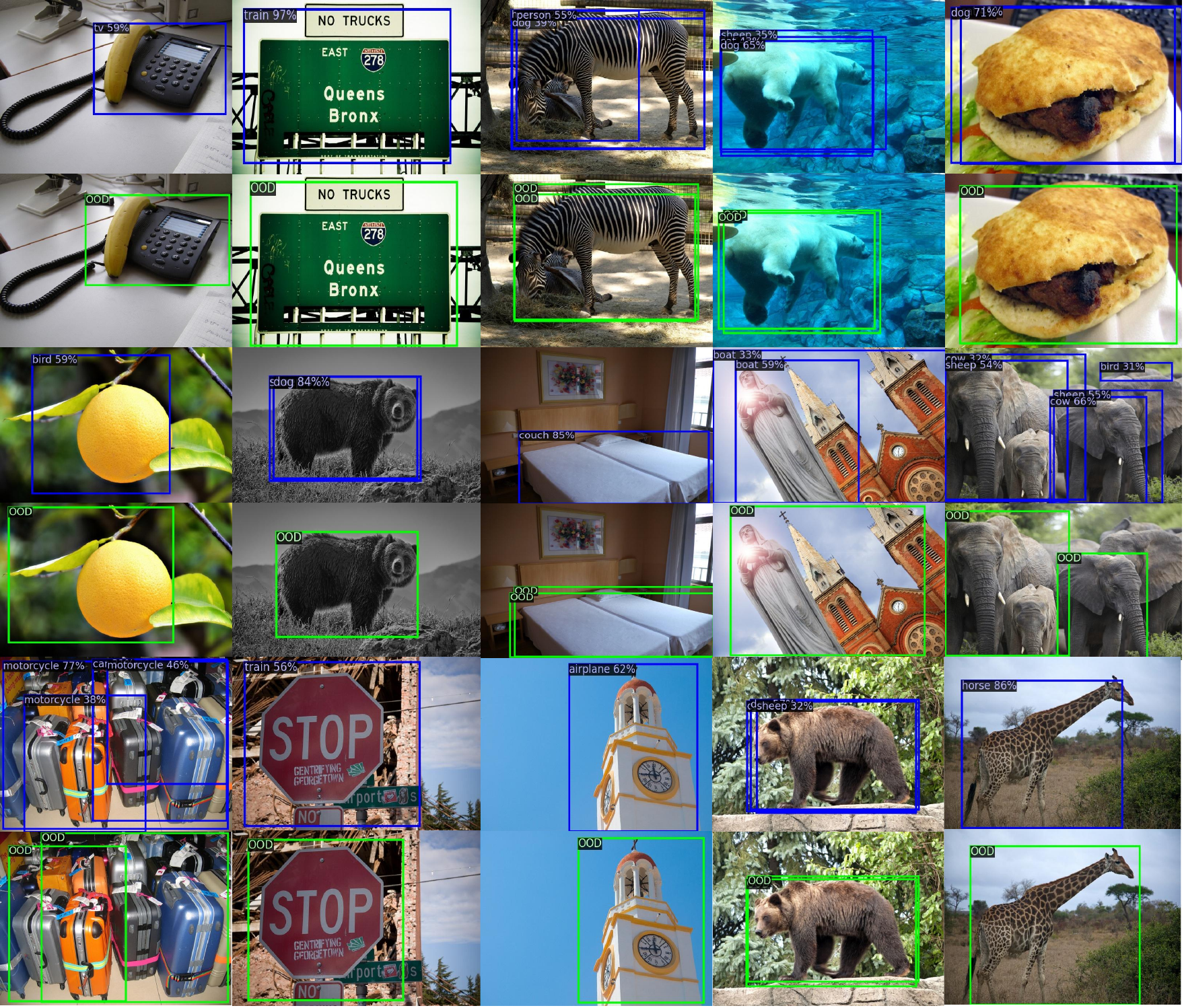}
    \caption[Additional visualization of detected objects on the OOD images (from MS-COCO, ID is VOC)]{ Additional visualization of detected objects on the OOD images (from MS-COCO) by a vanilla Faster-RCNN (\emph{top}) and \texttt{VOS} (\emph{bottom}). The in-distribution is Pascal VOC dataset.  \textbf{Blue}: Objects detected and classified as one of the ID classes. \textbf{Green}: OOD objects detected by \texttt{VOS}, which reduce false positives among detected objects.}
\label{fig:vi1}
\end{figure}

\begin{figure}[h]
    \centering
    \includegraphics[width=1.0\textwidth]{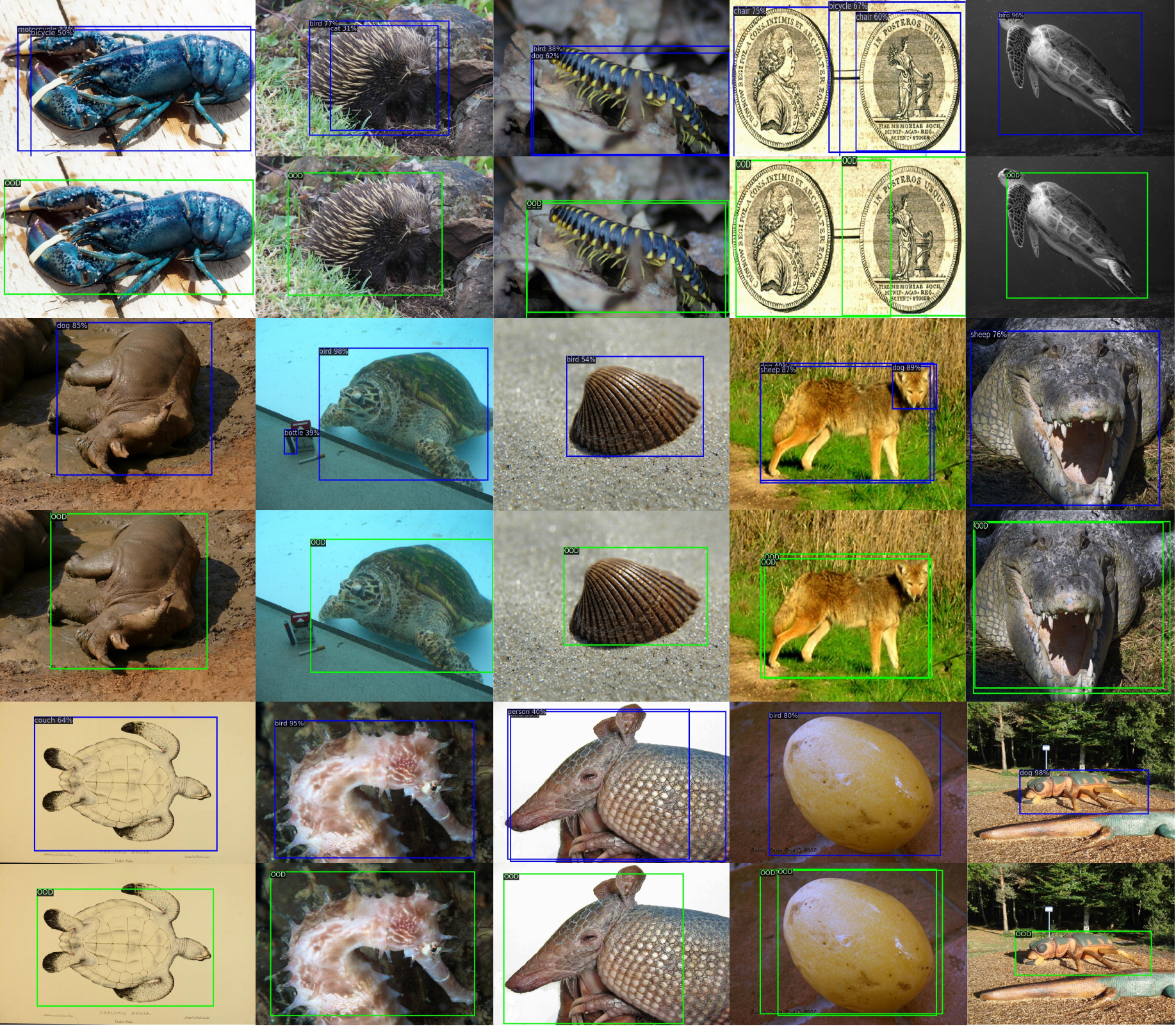}
    \caption[Additional visualization of detected objects on the OOD images (from OpenImages, ID is VOC)]{ Additional visualization of detected objects on the OOD images (from OpenImages) by a vanilla Faster-RCNN (\emph{top}) and \texttt{VOS} (\emph{bottom}). The in-distribution is Pascal VOC dataset. \textbf{Blue}: Objects detected and classified as one of the ID classes. \textbf{Green}: OOD objects detected by \texttt{VOS}, which reduce false positives among detected objects.}
\label{fig:vi2}
\end{figure}

\begin{figure}[h]
    \centering
    \includegraphics[width=1.0\textwidth]{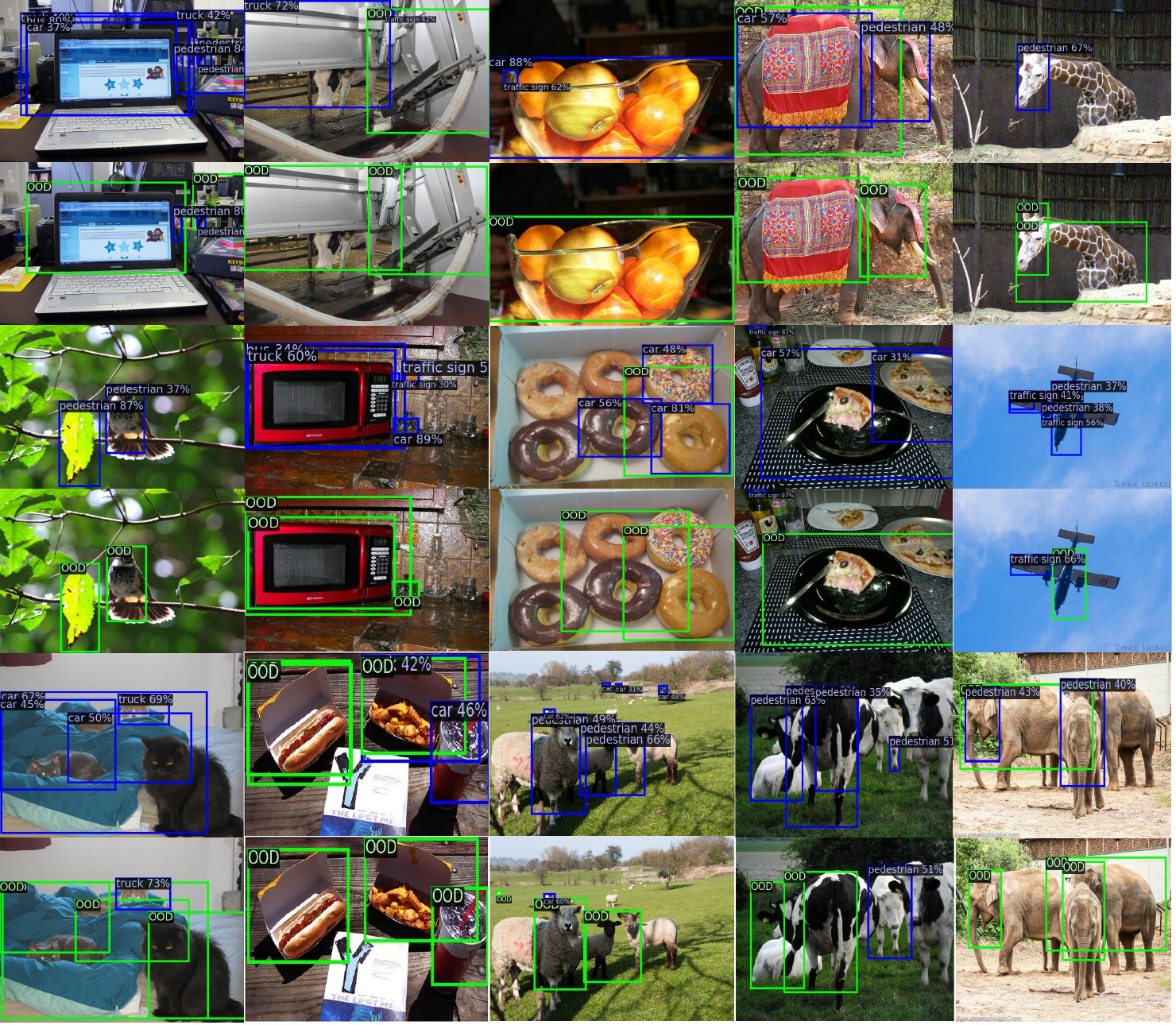}
    \caption[Additional visualization of detected objects on the OOD images (from MS-COCO, ID is BDD-100k)]{ Additional visualization of detected objects on the OOD images (from MS-COCO) by a vanilla Faster-RCNN (\emph{top}) and \texttt{VOS} (\emph{bottom}). The in-distribution is BDD-100k dataset.  \textbf{Blue}: Objects detected and classified as one of the ID classes. \textbf{Green}: OOD objects detected by \texttt{VOS}, which reduce false positives among detected objects.}
\label{fig:vi3}
\end{figure}

\begin{figure}[h]
    \centering
    \includegraphics[width=1.0\textwidth]{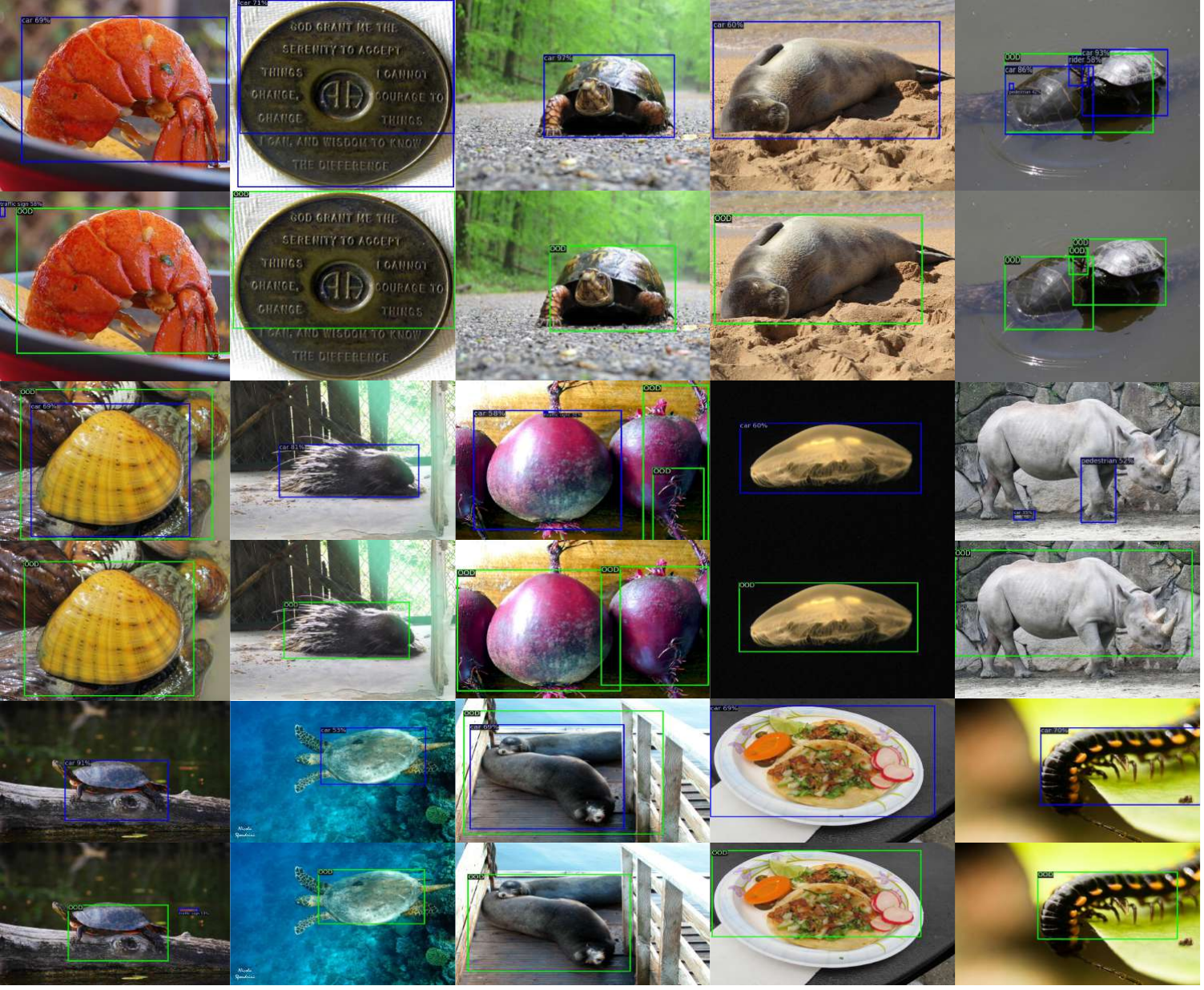}
    \caption[Additional visualization of detected objects on the OOD images (from OpenImages, ID is BDD-100k)]{ Additional visualization of detected objects on the OOD images (from OpenImages) by a vanilla Faster-RCNN (\emph{top}) and \texttt{VOS} (\emph{bottom}). The in-distribution is BDD-100k dataset.  \textbf{Blue}: Objects detected and classified as one of the ID classes. \textbf{Green}: OOD objects detected by \texttt{VOS}, which reduce false positives among detected objects.}
\label{fig:vi4}
\end{figure}

\subsection{Baselines}
\label{sec:reproduce_baseline}

To evaluate the baselines, we follow the original methods in MSP~\citep{hendrycks2016baseline}, ODIN~\citep{liang2018enhancing}, Generalized ODIN~\citep{hsu2020generalized}, Mahalanobis distance~\citep{lee2018simple}, CSI~\citep{tack2020csi}, energy score~\citep{liu2020energy}  and gram matrices~\citep{DBLP:conf/icml/SastryO20} and apply them accordingly on the classification branch of the object detectors. For ODIN, the temperature is set to be $T=1000$ following the original work. For both ODIN and Mahalanobis distance~\citep{lee2018simple}, the noise magnitude is set to $0$ because the region-based object detector is not end-to-end differentiable given the existence of region cropping and ROIAlign.\@ For GAN~\citep{lee2018training}, we follow the original paper and use a GAN to generate OOD images. The prediction of the OOD images/objects is regularized to be close to a uniform distribution, through a KL divergence loss with a weight of 0.1. We set the shape of the generated images to be 100$\times$100 and resize them to have the same shape as the real images. We optimize the generator and discriminator using Adam~\citep{DBLP:journals/corr/KingmaB14}, with a learning rate of 0.001. 
For CSI~\citep{tack2020csi}, we use the rotations (0$^\circ$, 90$^\circ$, 180$^\circ$, 270$^\circ$) as the self-supervision task. We set the temperature in the contrastive loss to 0.5. We use the features right before the classification branch (with the dimension to be 1024) to perform contrastive learning. The weights of the losses that are used for classifying shifted instances and instance discrimination are both set to 0.1 to prevent training collapse. For Generalized ODIN~\citep{hsu2020generalized}, we replace and train the classification head of the object detector by the most effective Deconf-C head shown in the original paper.

\subsection{Virtual outlier synthesis using earlier layer}
\label{sec:intermediate}
In this section, we investigate the effect of using \texttt{VOS} on an earlier layer within the network. Our main results in Table~\ref{tab:baseline} are based on the penultimate layer of the network. Here, we additionally evaluate the performance using the layer before the penultimate layer, with a feature dimension of $1,024$. The results are summarized in Table~\ref{tab:diff_layers}. As observed, synthesizing virtual outliers in the penultimate layer achieves better OOD detection performance than the earlier layer, since the feature representations are more discriminative at deeper layers.

\begin{table}[h]
    \centering
    \begin{tabular}{c|ccc}
     \toprule
      Models  & FPR95$\downarrow$& AUROC$\uparrow$&mAP$\uparrow$ \\
    \midrule
    \multicolumn{4}{c}{PASCAL VOC}\\
    
      \hline
    VOS-final     &\textbf{47.53}&\textbf{88.70}& \textbf{48.9} \\
    VOS-earlier&  50.24 & 88.24& 48.6\\
         \hline
          
    \multicolumn{4}{c}{BDD-100k} \\
    \midrule
    VOS-final     &\textbf{44.27}& \textbf{86.87}& \textbf{31.3} \\
    VOS-earlier&49.66& 86.08&30.6\\
    \bottomrule
    \end{tabular}
 \caption{Performance comparison of employing VOS on different layers. COCO is the OOD data.}
    \label{tab:diff_layers}
\end{table}

\subsection{Visualization of the learnable weight coefficient $w$ in generalized energy score}
\label{sec:app_visual_weight}
To observe whether the  learnable weight coefficient $w_k$ in Equation~\ref{eq:energy}  captures dataset-specific statistics during uncertainty regularization, we visualize $w_k$ w.r.t each in-distribution class and the number of training objects of that class in Figure~\ref{fig:visual_energy_weight}. We use the BDD-100k dataset~\citep{DBLP:conf/cvpr/YuCWXCLMD20} as the in-distribution dataset and the  RegNetX-4.0GF~\citep{DBLP:conf/cvpr/RadosavovicKGHD20} as the backbone network. As can be observed, the learned weight coefficient displays a consistent trend with the number of training objects per class, which indicates the advantage of using learnable weights rather than constant weight vector with all 1s.

\begin{figure}[h]
    \centering
    \includegraphics[width=1.0\textwidth]{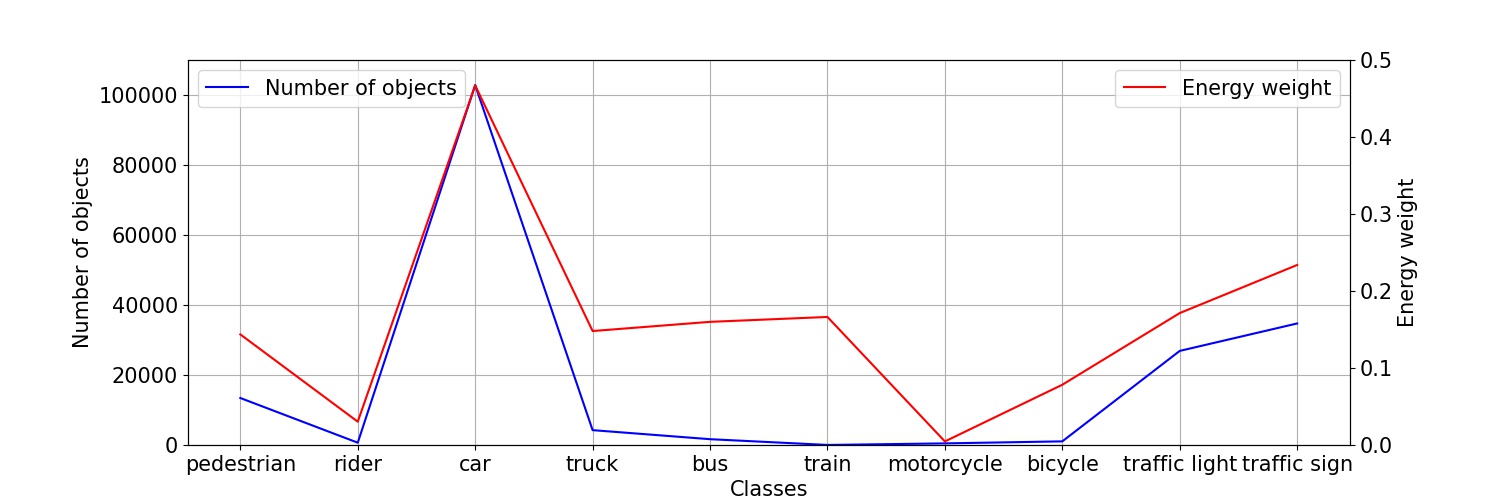}
    \caption[Visualization of learnable weight coefficient in the generalized energy score]{ Visualization of learnable weight coefficient in the generalized energy score and the number of training objects per in-distribution class. The value of the weight coefficient is averaged over three different runs. }

\label{fig:visual_energy_weight}
\end{figure}




\subsection{{Discussion on the detected, rejected and ignored OOD objects during inference}}
\label{sec:app_number_objects}
\textcolor{black}{The focus of~\texttt{VOS} is to mitigate the undesirable cases when an OOD object is detected and classified as in-distribution with high confidence. In other words, our goal is to ensure that “if the box is detected, it should be faithfully an in-distribution object rather than OOD”. Although generating the bounding box for OOD data is not the focus of this paper, we do notice that VOS can improve the number of boxes detected for OOD data (+25\% on BDD trained model compared to the vanilla Faster-RCNN).}

\textcolor{black}{The number of OOD objects ignored by RPN can largely depend on the confidence score threshold and the NMS threshold. Hence, we found it more meaningful to compare relatively with the vanilla Faster-RCNN under the same default thresholds. Using BDD100K as the in-distribution dataset and the ResNet as the backbone, \texttt{VOS} can improve the number of detected OOD boxes by 25\% (compared to vanilla object detector). \texttt{VOS} also improves the number of rejected OOD samples by 63\%.}

 \section{Dream the Impossible:
Outlier Imagination with Diffusion Models}


\subsection{Details of datasets}
\label{sec:app_dataset}
\textbf{ImageNet-100.} We randomly sample 100 classes from \textsc{Imagenet-1k}~\citep{deng2009imagenet} to create \textsc{Imagenet-100}. The dataset contains the following categories: {\scriptsize n01498041, n01514859, n01582220, n01608432, n01616318,
        n01687978, n01776313, n01806567, n01833805, n01882714,
          n01910747, n01944390, n01985128, n02007558, n02071294,
          n02085620, n02114855, n02123045, n02128385, n02129165,
          n02129604, n02165456, n02190166, n02219486, n02226429,
          n02279972, n02317335, n02326432, n02342885, n02363005,
          n02391049, n02395406, n02403003, n02422699, n02442845,
          n02444819, n02480855, n02510455, n02640242, n02672831,
          n02687172, n02701002, n02730930, n02769748, n02782093,
          n02787622, n02793495, n02799071, n02802426, n02814860,
          n02840245, n02906734, n02948072, n02980441, n02999410,
          n03014705, n03028079, n03032252, n03125729, n03160309,
          n03179701, n03220513, n03249569, n03291819, n03384352,
          n03388043, n03450230, n03481172, n03594734, n03594945,
          n03627232, n03642806, n03649909, n03661043, n03676483,
          n03724870, n03733281, n03759954, n03761084, n03773504,
          n03804744, n03916031, n03938244, n04004767, n04026417,
          n04090263, n04133789, n04153751, n04296562, n04330267,
          n04371774, n04404412, n04465501, n04485082, n04507155,
          n04536866, n04579432, n04606251, n07714990, n07745940}.

\vspace{-0.2cm}

\paragraph{OOD datasets.} \cite{huang2021mos} curated a diverse collection of subsets from iNaturalist~\citep{van2018inaturalist}, SUN~\citep{xiao2010sun}, Places~\citep{zhou2017places}, and Texture~\citep{cimpoi2014describing} as large-scale OOD datasets for \textsc{Imagenet-1k}, where the classes of the test sets do not overlap with \textsc{Imagenet-1k}. We provide a brief introduction for each dataset as follows.
 
\textbf{iNaturalist} contains images of natural world~\citep{van2018inaturalist}. It has 13 super-categories and 5,089 sub-categories covering plants, insects, birds, mammals, and so on. We use the subset that contains 110 plant classes which do not overlap with \textsc{Imagenet-1k}.

\textbf{SUN} stands for the Scene UNderstanding Dataset~\citep{xiao2010sun}. SUN contains 899 categories that cover more than indoor, urban, and natural places with or without human beings appearing in them. We use the subset which contains 50 natural objects not in \textsc{Imagenet-1k}.

\textbf{Places} is a large scene photographs dataset~\citep{zhou2017places}. It contains photos that are labeled with scene semantic categories from three macro-classes: Indoor, Nature, and Urban. The subset we use contains 50 categories that are not present in \textsc{Imagenet-1k}.

\textbf{Texture} stands for the Describable Textures Dataset~\citep{cimpoi2014describing}. It contains images of textures and abstracted patterns. As no categories overlap with \textsc{Imagenet-1k}, we use the entire dataset as in~\cite{huang2021mos}.

\textbf{ImageNet-A} contains 7,501 images from 200 classes, which are obtained by collecting new data and keeping only those images that ResNet-50 models fail to correctly classify~\citep{hendrycks2021natural}. In our paper, we evaluate on the 41 overlapping classes with \textsc{Imagenet-100} which consist of a total of 1,852 images.

\textbf{ImageNet-v2} used in our paper is sampled to match the MTurk selection frequency distribution of the original \textsc{Imagenet} validation set for each class~\citep{recht2019imagenet}. The dataset contains 10,000 images from 1,000 classes. During testing, we evaluate on the 100 overlapping classes with a total of 1,000 images.

\subsection{Formulation of $Z_m(\kappa)$ }
\label{sec:zm}
The normalization factor $Z_m(\kappa)$ in Equation~\eqref{eq:vmf_density} is defined as:
\begin{equation}
    Z_{m}(\kappa)=\frac{\kappa^{m / 2-1}}{(2 \pi)^{m / 2} I_{m / 2-1}(\kappa)},
\label{eq:defnition_zd}
\end{equation}
where $I_{v}$ is the modified Bessel function of the first kind
with order $v$. $Z_m(\kappa)$ can be calculated in closed form based on $\kappa$ and the feature dimensionality $m$.

  \subsection{Additional Visualization of the Imagined Outliers}
  \label{sec:ab_visual}
  In addition to Section~\ref{sec:qualitative}, we provide additional visualizations on the imagined outliers under different variance $\sigma^2$ in Figure~\ref{fig:visual_ood_app}. We observe that a larger variance consistently translates into outliers that are more deviated from ID data. Using a mild variance value $\sigma^2=0.03$  generates both empirically (Figure~\ref{fig:ablation_frame_range} (b)) and visually meaningful outliers for model regularization on \textsc{Imagenet-100}.

  \begin{figure}[!h]
  \begin{center}
   {\includegraphics[width=1.0\linewidth]{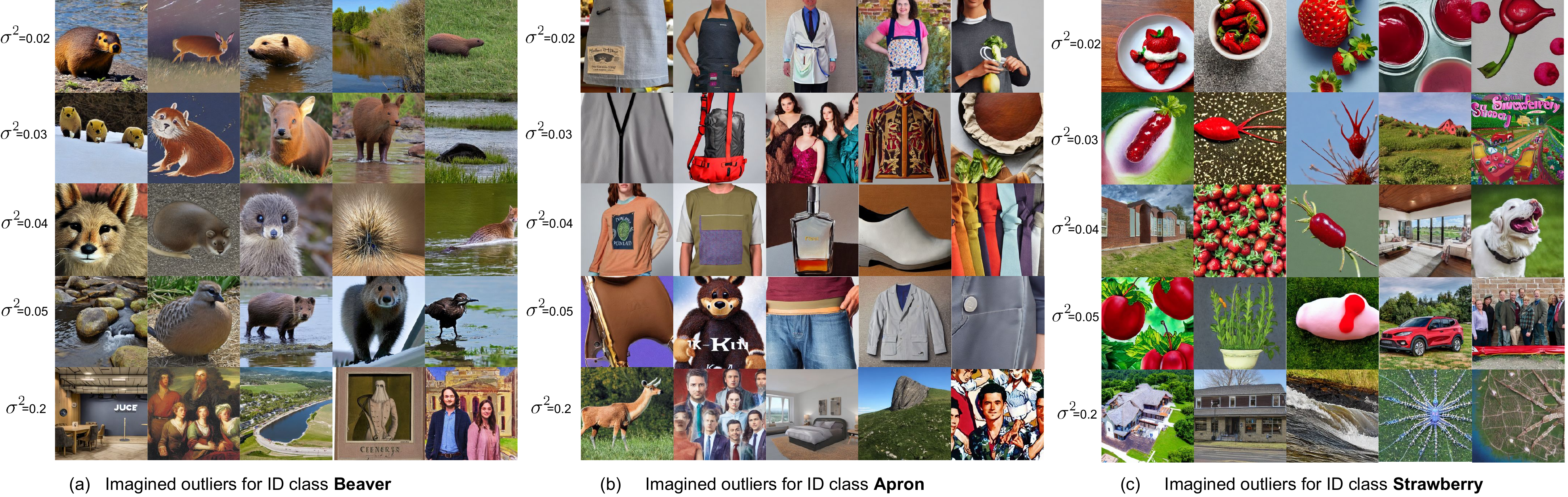}}
  \end{center}
  \vspace{-1em}
  \caption[Visualization of the imageined outliers for the \emph{beaver, apron, strawberry} class  with different variance values $\sigma^2$]{ \textbf{Visualization of the imageined outliers} for the \emph{beaver, apron, strawberry} class  with different variance values $\sigma^2$. }
  \label{fig:visual_ood_app}
  \end{figure}

  \subsection{Visualization of Outlier Generation by Embedding Interpolation}
  \label{sec:ab_visual_interpolation}
  We visualize the generated outlier images by interpolating token embeddings from different classes in Figure~\ref{fig:visual_inter_app}.  The result shows that interpolating different class token embeddings tends to generate images that are still in-distribution rather than images with semantically mixed or novel concepts, which is aligned with the observations in~\cite{liew2022magicmix}. Therefore, regularizing the model using such images is not effective for OOD detection (Table~\ref{tab:ablation_different_synthesis}).
\label{sec:visual_learned}
\begin{figure}[!h]
  \begin{center}
   {\includegraphics[width=1.0\linewidth]{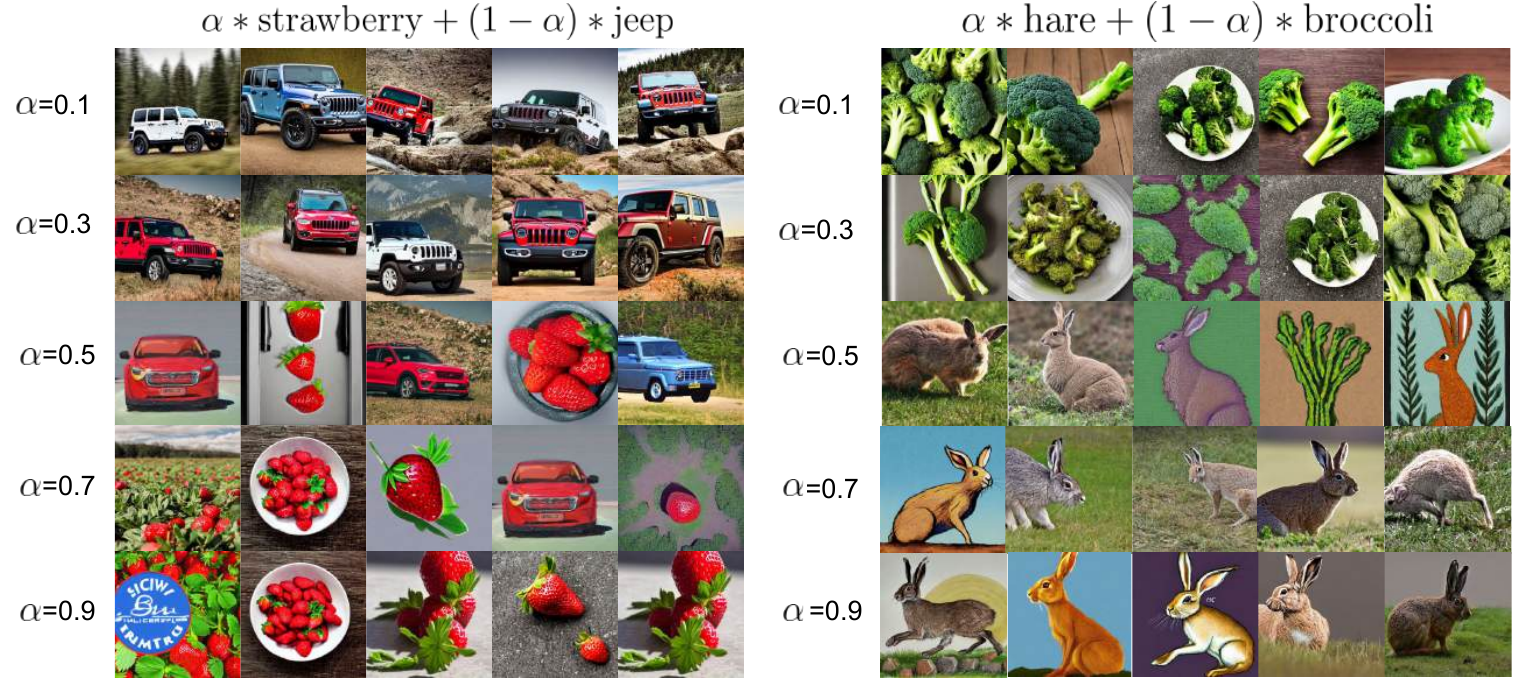}}
  \end{center}
  \vspace{-1em}
  \caption[Visualization of the generated outlier images by interpolating token embeddings from different classes]{ \textbf{Visualization of the generated outlier images} by interpolating token embeddings from different classes. We show the results with different interpolation weights $\alpha$. }
     \vspace{-1em}
  \label{fig:visual_inter_app}
\end{figure}

\subsection{Visualization of the Outlier Generation by Adding Noise}
\label{sec:visual_add_noise}
As in Table~\ref{tab:ablation_different_synthesis} in the main chapter, we visualize the generated outlier images by adding Gaussian and learnable noise to the token embeddings in Figure~\ref{fig:visual_noise}. We observe that adding Gaussian noise tends to generate either  ID images or images that are far away from the given ID class. In addition, adding learnable noise to the token embeddings will generate images that completely deviate from the ID data. Both of them are less effective in regularizing the model's decision boundary.

\begin{figure}[!h]
  \begin{center}
   {\includegraphics[width=1.0\linewidth]{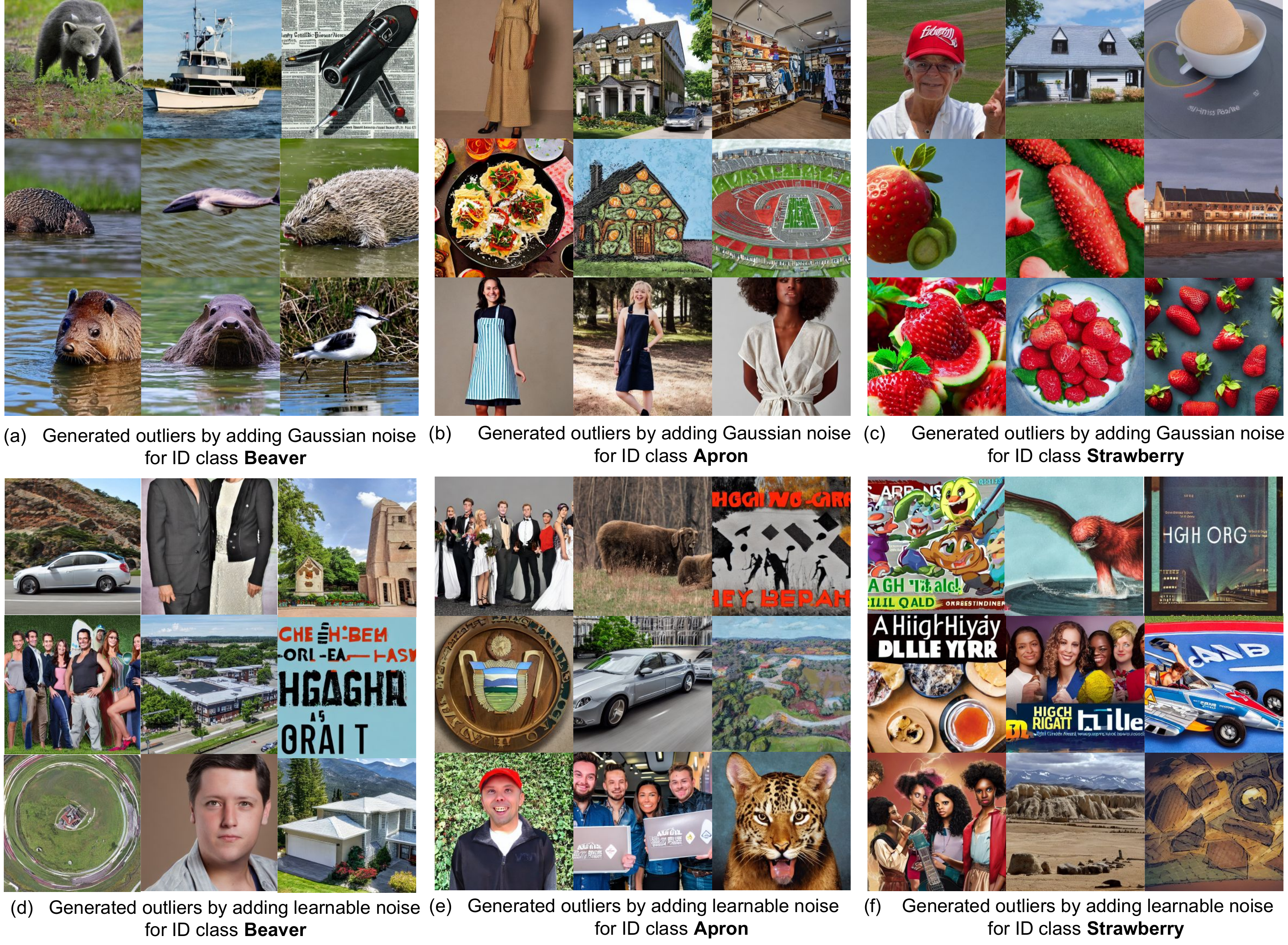}}
  \end{center}
  \vspace{-1em}
  \caption[Visualization of the generated outlier images by adding Gaussian and learnable noise]{ \textbf{Visualization of the generated outlier images} by adding Gaussian and learnable noise to the  token embeddings from different classes. }
     \vspace{-1em}
  \label{fig:visual_noise}
\end{figure}

  \subsection{Comparison with Training w/ real Outlier Data.} 
  
  We compare with training using real outlier data on  \textsc{Cifar-100}, \emph{i.e.,} 300K Random Images~\citep{hendrycks2018deep}, which contains 300K preprocessed images that do not belong to \textsc{Cifar-100} classes.  The result shows that \textsc{Dream-ood} (FPR95: 40.31\%, AUROC: 90.15\%) can match or even outperform outlier exposure with real OOD images (FPR95: 54.32\%, AUROC: 91.34\%) under the same training configuration while using fewer synthetic OOD images for OOD regularization (100K in total).

\subsection{Visualization of Generated Inlier Images}
\label{sec:inlier_visual}
We show in Figure~\ref{fig:visual_1} the visual comparison among the original \textsc{Imagenet} images, the generated images by our \textsc{Dream-id}, and the generated ID images using generic prompts "A high-quality photo of a [cls]" where "[cls]" denotes the class name. Interestingly, we observe that the prompt-based generation produces object-centric and distributionally dissimilar images from the original dataset. In contrast, our approach \textsc{Dream-id} generates inlier images that can resemble the original ID data, which helps model generalization.

    \begin{figure*}[!h]
  \begin{center}
   {\includegraphics[width=1.0\linewidth]{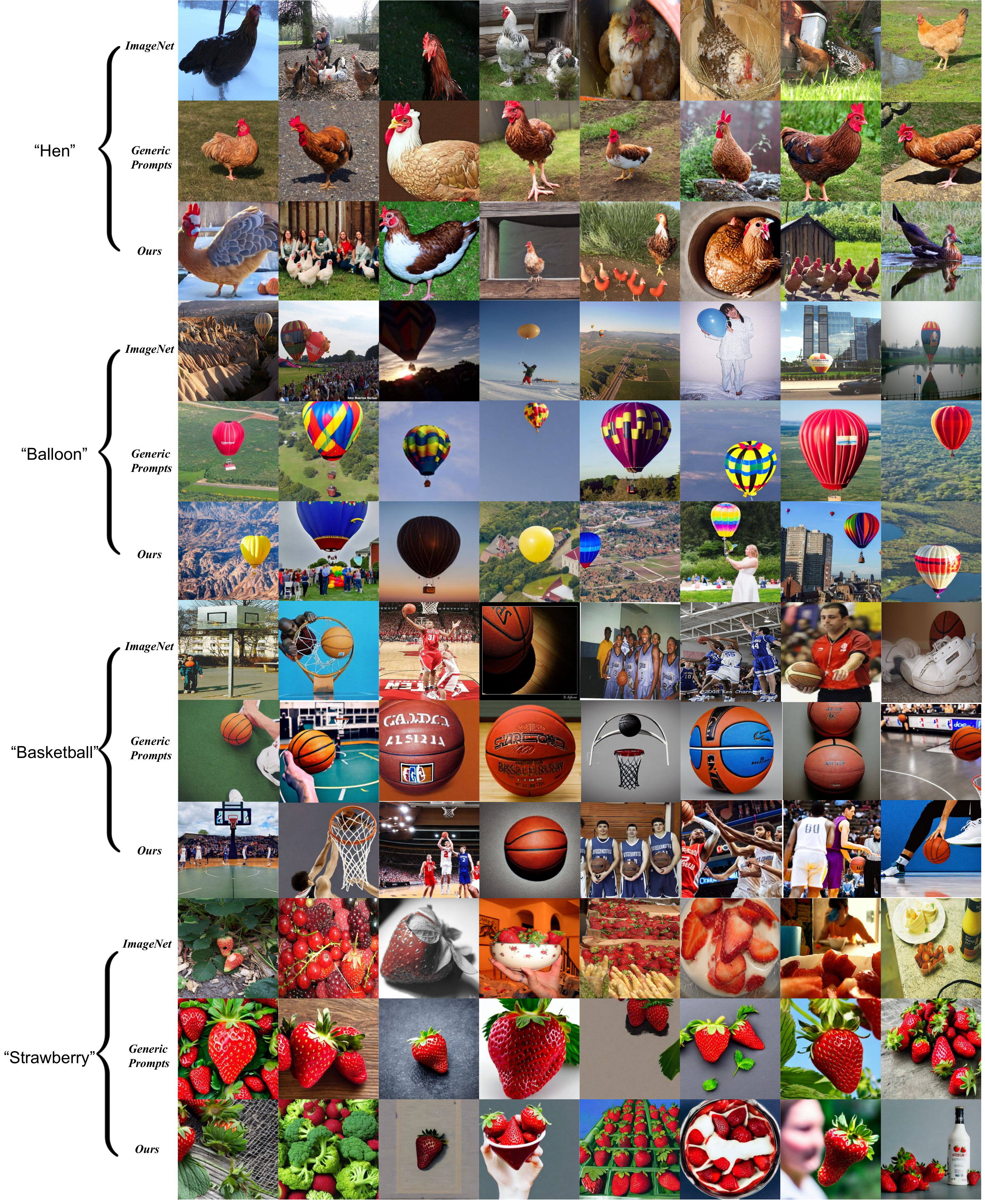}}
  \end{center}
  \vspace{-1em}
  \caption[Visual comparison between our Dream-ID vs. prompt-based image generation]{ \textbf{Visual comparison between our \textsc{Dream-id} vs.  prompt-based image generation} on four different classes.}
  \label{fig:visual_1}
  \end{figure*}

\subsection{Experimental Details for Model Generalization}
\label{sec:details_generalization}
 We provide experimental details for Section~\ref{sec:generalization} in the main chapter. We use ResNet-34~\citep{he2016deep} as the network architecture, trained with the standard cross-entropy loss. For both the \textsc{Cifar-100} and \textsc{Imagenet-100} datasets, we train the model for 100 epochs, using stochastic gradient descent with the cosine learning rate decay schedule, a momentum of 0.9, and a weight decay of $5e^{-4}$. The initial learning rate is set to 0.1 and the batch size is set to 160. We generate $1,000$ new ID samples per class using Stable Diffusion v1.4, which result in $100,000$ synthetic images. For both the baselines and our method, we train on a combination of the original \textsc{Imagenet}/\textsc{Cifar} samples and synthesized ones. To learn the feature encoder $h_{\theta}$, we set the temperature $t$ in Equation~\eqref{eq:cosine_loss} to 0.1. Extensive ablations on hyperparameters $\sigma$ and $k$ are provided in  Appendix~\ref{sec:ab_gene}.

  \subsection{Implementation Details of Baselines for Model Generalization}
  
\label{sec:baselines_gene}
For a fair comparison, we implement all the data augmentation baselines by appending the original \textsc{Imagenet-100} dataset with the same amount of augmented images (\emph{i.e.}, 100k) generated from different augmentation techniques. We follow the default hyperparameter setting as in their original papers. 

\begin{itemize}
\item For RandAugment~\citep{cubuk2020randaugment}, we set the number of augmentation transformations to
apply sequentially to 2. The magnitude for all the transformations is set to 9.
\item For AutoAugment~\citep{cubuk2018autoaugment}, we set the augmentation policy as the best one searched on \textsc{ImageNet}.
\item For CutMix~\citep{yun2019cutmix}, we use a CutMix probability of 1.0 and set $\beta$ in the Beta distribution to  1.0 for  the label mixup. 
\item For AugMix~\citep{hendrycks2019augmix}, we  randomly sample 3 augmentation chains and set  $\alpha=1$ for the Dirichlet distribution to mix the images. 
\item For DeepAugment~\citep{hendrycks2021many}, we directly use the corrupted images for data augmentation provided in their Github repo~\footnote{\url{https://github.com/hendrycks/imagenet-r/blob/master/DeepAugment}}.
\item For MEMO~\citep{zhang2021memo}, we follow the original paper and use the  marginal entropy objective for test-time adaptation, which disentangles two distinct self-supervised learning signals: encouraging invariant predictions across different augmentations of the test point and encouraging confidence via entropy minimization.
\end{itemize}

\begin{table}[!h]
    \centering

  \scalebox{0.9}{   \begin{tabular}{c|ccc}
        Methods &  \textsc{ImageNet} &  \textsc{ImageNet-A} &  \textsc{ImageNet-v2} \\
        \hline
        Original (no aug) & 87.28 & 8.69& 77.80 \\
               RandAugment & 87.56& 11.07& 79.20\\
              AutoAugment & 87.40 & 10.37& 79.00 \\
                CutMix &87.64& 11.33 & 79.70  \\
                    AugMix &87.22 & 9.39&77.80  \\
                       \hline
                        \rowcolor{Gray}  \textbf{\textsc{Dream-id} (Ours)}& \textbf{88.46}{\scriptsize$\pm$0.1} & \textbf{12.13}{\scriptsize$\pm$0.1} & \textbf{80.40}{\scriptsize$\pm$0.1}\\
  
    \end{tabular}}
    \vspace{0.5em}
    \caption[Model generalization performance for Dream-OOD]{ \textbf{Model generalization performance (accuracy, in \%), using \textsc{ImageNet-100} as the training data.} The baselines are implemented by directly applying the augmentations on \textsc{ImageNet-100}.}
    \label{tab:gene_result_no_append}
\end{table}

We also provide the comparison in Table~\ref{tab:gene_result_no_append} with baselines that are directly trained by applying the augmentations on \textsc{Imagenet} without appending the original  images. The model trained with the images generated by \textsc{Dream-id} can still outperform all the baselines by a considerable margin.

 \subsection{Ablation Studies on Model Generalization}
\label{sec:ab_gene}

In this section, we provide additional analysis of the hyperparameters and designs of \textsc{Dream-id} for ID generation and data augmentation. For all the ablations, we use the \textsc{Imagenet-100} dataset as the in-distribution training data.

\vspace{-0.2cm}
\paragraph{Ablation on the variance value $\sigma^2$.} We show in Table~\ref{tab:sig_ablation} the effect of $\sigma^2$ --- the number of the variance value for the Gaussian kernel (Section~\ref{sec:synthesis}). We vary $\sigma^2\in\{0.005, 0.01, 0.02, 0.03\}$. A small-mild variance value $\sigma^2$ is more beneficial for model generalization.
\begin{table}[!h]
    \centering

      \begin{tabular}{c|ccc}
      \toprule
      $\sigma^2$&  \textsc{Imagenet} &  \textsc{Imagenet-A} &  \textsc{Imagenet-v2}   \\
      \midrule
  0.005 &   87.62 & 11.39 & 78.50 \\
    0.01 &\textbf{88.46} & \textbf{12.13} & \textbf{80.40}\\
    0.02 & 87.72 & 10.85 & 77.70\\
    0.03 & 87.28 & 10.91& 78.20 \\
      \bottomrule
      \end{tabular}
        \caption{Ablation study on the variance value $\sigma^2$ in the Gaussian kernel for model generalization.}
    \label{tab:sig_ablation}%
   
     \vspace{-1em}
  \end{table}%

\vspace{-0.2cm}

\paragraph{Ablation on $k$ in calculating $k$-NN distance.}   In Table~\ref{tab:k_ablation}, we  analyze the effect of $k$, \emph{i.e.}, the number of nearest neighbors for non-parametric  sampling in the latent space. In particular, we vary $k=\{100,  200, 300, 400, 500\}$. We observe that our method is not sensitive to this hyperparameter, as $k$ varies from 100 to 500.  
\begin{table}[!h]
  \centering
       

    \begin{tabular}{c|ccc}
    \toprule
   $k$ & \textsc{Imagenet} &  \textsc{Imagenet-A} &  \textsc{Imagenet-v2}   \\
    \midrule
    100  & \textbf{88.51} & 12.11 & 79.92 \\
    200 & 88.35 & 12.04 & 80.01 \\
    300 & 88.46 & \textbf{12.13} & \textbf{80.40} \\
    400 &88.43  & 12.01 & 80.12\\
    500& 87.72 & 11.78 & 80.29\\
    \bottomrule
    \end{tabular}%
      \caption{Ablation study on the $k$ for $k$-NN distance for model generalization.}
  \label{tab:k_ablation}%
\end{table}%

\subsection{Computational Cost}
{We summarize the computational cost of \textsc{Dream-ood} and different baselines on \textsc{Imagenet-100} as follows. The post hoc OOD detection methods require training a classification model on the ID data ($\sim$8.2 h). The outlier synthesis baselines, such as VOS ($\sim$8.2 h), NPOS ($\sim$8.4 h), and GAN ($\sim$13.4 h) incorporate the training-time regularization with the synthetic outliers. Our \textsc{Dream-ood} involves learning the text-conditioned latent space ($\sim$8.2 h), image generation with diffusion models ($\sim$10.1 h for 100K images), and training with the generated outliers ($\sim$8.5 h). }

\subsection{Software and hardware}
\label{sec:hardware}
We run all experiments with Python 3.8.5 and PyTorch 1.13.1, using NVIDIA GeForce RTX 2080Ti GPUs.

 \section{SIREN: Shaping Representations for Detecting Out-of-Distribution Objects}

\subsection{Experimental Details}
\label{sec:dataset}
Following~\citep{du2022towards}, we summarize the OOD detection evaluation task in Table~\ref{tab:task}. The OOD test dataset is selected from \textsc{ms-coco} and \textsc{OpenImages} dataset, which contains disjoint labels from the respective ID dataset.  \textsc{Siren} is trained for a total of 50 epochs on \textsc{pascal-voc}, and trained for 30 epochs on \textsc{bdd100k} using ADAM
optimizer. The initial learning rate is 2e-4 and decays at epoch 40 and 24 by 0.1 for \textsc{pascal-voc} and \textsc{bdd100k} dataset, respectively. We set the number of the object queries as 300, the batch size as 8, and the weight $\beta$ in Equation~\eqref{eq:totalloss} as $1.5$. See \emph{detailed ablations on the hyperparameters on a validation OOD dataset in Appendix~\ref{sec:ablation_app}}.

\begin{table}[h]
\centering

\scalebox{0.9}{\begin{tabular}{@{}lrr@{}}
\toprule
 & \textbf{Task 1} &\textbf{ Task 2}  \\ \midrule
ID train dataset & \textsc{voc} train & \textsc{bdd} train \\
ID val dataset &  \textsc{voc} val &  \textsc{bdd} val  \\
OOD dataset & { \textsc{coco} \& \textsc{OpenImages} val} & {\textsc{coco}  and \textsc{OpenImages} val}  \\

$\#$ID train images & 16,551 & 69,853  \\
$\#$ID val images & 4,952 & 10,000\\
$\#$OOD images for \textsc{coco}& 930 &1,880 \\ 
$\#$OOD images for \textsc{OpenImages}& 1,761 & 1,761\\ 

\bottomrule
\end{tabular}}
\caption{OOD detection evaluation tasks.}
\label{tab:task}
\end{table}

\subsection{Estimating $\widehat{\kappa}$}
\label{sec:proof}
We provide the mathematical details for the estimation of $\kappa$ in Equation~\eqref{eq:estimating_kappa_test_time}. Concretely, we frame the estimation problem as deriving maximum likelihood estimators (MLEs) for the vMF density function, given the training data $\left\{\mathbf{r}_{i}\right\}_{i=1}^{M}$. Specifically, the maximum likelihood learning can be written as the following optimization problem: 
\begin{equation}
    \begin{aligned}
&\max _{\boldsymbol{\mu}, {\kappa}} \mathcal{L}_{\textrm{MLE}}(\boldsymbol{\mu},\kappa):=\sum_{i=1}^{M} \log p\left(\mathbf{r}_{i}\right)=M \kappa \boldsymbol{\mu}^{\top} \overline{\mathbf{r}}+M \log Z_{d}(\kappa) \\
&\text { s.t. }\|\boldsymbol{\mu}\|=1 \text { and } \kappa \geq 0,
\end{aligned}
\end{equation}
where $\overline{\mathbf{r}}=\frac{1}{M} \sum_{i=1}^{M} \mathbf{r}_{i}$ and $\mathcal{L}_{\textrm{MLE}}$ is the optimization objective of the maximum likelihood estimation for the distributional parameters $\boldsymbol{\mu},\kappa$. In order to optimize this objective, we take the derivatives of the objective \emph{w.r.t.} the parameters $\boldsymbol{\mu},\kappa$ and set them to 0. Then, the optimal parameters $\boldsymbol{\mu}^*,\widehat{ \kappa}$ should satisfy the following two conditions:
\begin{equation}
    \boldsymbol{\mu}^{*}=\frac{\overline{\mathbf{r}}}{\|\overline{\mathbf{r}}\|}, \quad  \frac{Z_{d}^{\prime}\left(\widehat{\kappa}\right)}{Z_{d}\left(\widehat{\kappa}\right)}=-\|\overline{\mathbf{r}}\|,
    \label{eq:derivation_supp}
\end{equation}
where $Z_{d}^{\prime}\left(\widehat{\kappa}\right)=\frac{d Z_{d}(\widehat{\kappa})}{d \widehat{\kappa}}$.

For the second condition in Equation~\eqref{eq:derivation_supp}, we let $\xi=(2 \pi)^{d / 2}$ and $s=d / 2-1$ for notation simplicity. Put them into the Equation~\eqref{eq:defnition_zd}, we get the following:
\begin{equation}
    Z_{d}(\widehat{\kappa})=\frac{\widehat{\kappa}^{s}}{\xi \cdot I_{s}(\widehat{\kappa})}, \quad Z_{d}^{\prime}(\widehat{\kappa})=\frac{1}{\xi} \cdot \frac{s \widehat{\kappa}^{s-1} I_{s}(\widehat{\kappa})-\widehat{\kappa}^{s} I_{s}^{\prime}(\widehat{\kappa})}{I_{s}(\widehat{\kappa})^{2}}.
\end{equation}
Divide $Z_{d}(\widehat{\kappa})$ by $Z_{d}^{\prime}(\widehat{\kappa})$, we get:
\begin{equation}
    \frac{Z_{d}^{\prime}(\widehat{\kappa})}{Z_{d}(\widehat{\kappa})}=\frac{s}{\widehat{\kappa}}-\frac{I_{s}^{\prime}(\widehat{\kappa})}{I_{s}(\widehat{\kappa})}.
    \label{eq:vmf_devide}
\end{equation}
Note that the Bessel function holds a recursive property, which is given as follows:
\begin{equation}
    \frac{I_{s}^{\prime}(\widehat{\kappa})}{I_{s}(\widehat{\kappa})}=\frac{s}{\widehat{\kappa}}+\frac{I_{s+1}^{\prime}(\widehat{\kappa})}{I_{s}(\widehat{\kappa})}.
    \label{eq:bessel_property}
\end{equation}
Substitute $\frac{I_{s}^{\prime}(\widehat{\kappa})}{I_{s}(\widehat{\kappa})} $ in Equation~\eqref{eq:vmf_devide} with the formula in Equation~\eqref{eq:bessel_property} and integrate it with the second condition of Equation~\eqref{eq:derivation_supp}, we get:
\begin{equation}
    \frac{I_{d / 2}(\widehat{\kappa})}{I_{d / 2-1}(\widehat{\kappa})}=\|\overline{\mathbf{r}}\|.
\end{equation}
Since there is no known closed-form solution to the Bessel ratio inversion problem as shown in the formula above, we adopt the approximation schemes based on the continued fraction form of the Bessel ratio function~\citep{DBLP:journals/jmlr/BanerjeeDGS05}, namely:
\begin{equation}
    R:=\frac{I_{d / 2}(\widehat{\kappa})}{I_{d / 2-1}(\widehat{\kappa})}=\frac{1}{\frac{d}{\widehat{\kappa}}+\frac{1}{\frac{d+2}{\widehat{\kappa}}+\cdots}} \approx \frac{1}{\frac{d}{\widehat{\kappa}}+R}.
\end{equation}
Replace $R$ with $\|\overline{\mathbf{r}}\|$, we get $\widehat{\kappa} \approx \frac{d \cdot\|\overline{\mathbf{r}}\|}{1-\|\overline{\mathbf{r}}\|^{2}}$. Following~\citep{DBLP:journals/jmlr/BanerjeeDGS05}, we further add a correction term $-\|\overline{\mathbf{r}}\|^{3}$ to the numerator and we get the approximated estimation calculated as:
\begin{equation}
       \widehat{\kappa} = \frac{\|\mathbf{\overline{r}}\|(d-\|\mathbf{\overline{r}}\|^2) }{1-\|\mathbf{\overline{r}}\|^2}.
\end{equation}

\subsection{Hyperparameter Analysis}
\label{sec:ablation_app}

Below we perform sensitivity analysis for each important hyperparameter. Our sensitivity analysis uses the speckle-noised \textsc{pascal-voc} validation dataset as OOD data, which is different from the actual OOD test datasets in use. We use \textsc{Siren} pre-trained with DINO~\citep{DBLP:conf/iccv/CaronTMJMBJ21} as the object detection backbone, trained on in-distribution dataset \textsc{pascal-voc}. We use the KNN score as the OOD score during inference.

\textbf{Effect of the hypersphere dimension $d$}.  \textsc{Siren} projects the object feature embeddings into a lower-dimensional hypersphere in $\mathbb{R}^d$, which allows tractable vMF estimation. A reasonable choice of the hyperspherical dimension $d$ is able to preserve sufficient information for OOD detection while avoiding distributional parameter estimation in the high-dimension space.   As shown in Table~\ref{tab:ablation_appendix2_detection}, using a dimension $d$ between 16 and 64 yields a desirable and stable performance on the OOD validation data while properly maintaining the ID performance (mAP). We set $d=16$ for \textsc{pascal-voc} and 64 for \textsc{bdd100k} in Table~\ref{tab:baseline}.

\begin{table}[h]
    \centering
    \tabcolsep 0.04in\renewcommand\arraystretch{0.7}{}
    \begin{tabular}{c|ccc}
    \toprule 
$d$&mAP$\uparrow$ & FPR95 $\downarrow$&AUROC$\uparrow$ \\
        \midrule
8 & 60.4& 76.31 & 64.95\\
16& \textbf{60.8}& \textbf{69.29}& \textbf{73.45}\\
32&58.1& 75.53 & 70.32 \\
64& 58.7 &  74.55 & 69.47\\
80& 58.2& 69.09 & 72.29\\
        \bottomrule
    \end{tabular}

      \caption{Ablation study on the dimension of the hypersphere $d$.}
          \label{tab:ablation_appendix2_detection}
\end{table}

\textbf{Effect of the \textsc{Siren} loss  weight $\beta$}. In Table~\ref{tab:ablation_appendix1_detection}, we show the sensitivity of the OOD detection performance of our \textsc{Siren} \emph{w.r.t.} the weight $\beta$ of the representation shaping loss. Overall, we find that $\beta=1.5$ achieves the best OOD detection and ID performance.

\textbf{Effect of the $k$ in the KNN score}. In Table~\ref{tab:ablation_appendix3_detection}, we show the sensitivity of the OOD detection performance of our \textsc{Siren} \emph{w.r.t.} the $k$ of the KNN distance during inference. Overall, we find that $k=10$ achieves the best OOD detection performance.

\begin{table}[h]
    \centering
    \tabcolsep 0.04in\renewcommand\arraystretch{0.745}{}
    \begin{tabular}{c|ccc}
    \toprule 
$\beta$&mAP$\uparrow$ & FPR95 $\downarrow$&AUROC$\uparrow$ \\
        \midrule
0.1 &60.2&  69.72 & 73.21\\
0.5&59.2&  71.09 & 72.36\\
1.0&59.8& 76.20 & 70.49  \\
1.5& \textbf{60.8} &  \textbf{69.29}& \textbf{73.45}   \\
2.0 &58.9& 70.41 & 71.22 \\
2.5& 56.0&  77.20 & 68.37\\
        \bottomrule
    \end{tabular}

  \caption{Ablation study on the loss weight $\beta$ for $\mathcal{L}_{\textsc{Siren}}$. }
    \label{tab:ablation_appendix1_detection}
\end{table}

\begin{table}[!h]
    \centering
    \tabcolsep 0.04in\renewcommand\arraystretch{0.745}{}
    \begin{tabular}{c|ccc}
    \toprule 
$k$&mAP$\uparrow$ & FPR95 $\downarrow$&AUROC$\uparrow$ \\
        \midrule
1& 60.8&70.42 & 72.11 \\
5 &  60.8&  71.66& 72.74\\
10&  60.8& \textbf{69.29}& \textbf{73.45}   \\
20 &60.8 &  70.03 & 72.95\\
50 & 60.8& 71.85 & 72.08 \\
100 & 60.8& 73.73 & 71.28\\
200 & 60.8 & 75.96 & 70.41\\
        \bottomrule
    \end{tabular}

  \caption{Ablation study on the $k$ for the KNN distance. }
    \label{tab:ablation_appendix3_detection}
\end{table}

\subsection{Baselines}
\label{sec:reproduce_baseline}

To evaluate the baselines, we follow the original methods in  MSP~\citep{hendrycks2016baseline}, ODIN~\citep{liang2018enhancing}, KNN~\cite{sun2022knn}, and CSI~\citep{tack2020csi} and apply them accordingly on the classification branch of the object detectors. The Mahalanobis distance~\citep{lee2018simple} and gram matrices~\citep{DBLP:conf/icml/SastryO20} are calculated based on the penultimate-layer features of the decoder in \textsc{ddetr}. For CSI~\citep{tack2020csi}, we use the rotation degree prediction (0$^\circ$, 90$^\circ$, 180$^\circ$, 270$^\circ$) as the self-supervised task. We set the temperature in the contrastive loss to 0.5. We use the penultimate-layer features of the decoder (with dimensionality 256) to perform contrastive learning. The weights of the losses that are used for classifying shifted instances and instance discrimination are both set to 1 following the original paper~\citep{tack2020csi}. For OW-DETR~\citep{ow_detr}, we follow the original paper and utilize the sigmoid probability of the additional unknown class for OOD detection. For VOS~\citep{du2022towards}, we use the same hyperparameters as those in the original paper and synthesize virtual outliers in the penultimate layer of the object detector. Then we regard the virtual outliers as the negative samples during the object classification. For Dismax~\citep{dismax}, we add the dismax loss with learnable prototypes in the penultimate layer of \textsc{ddter} and apply the same inference score as the original paper for OOD detection.

\subsection{Comparison of  Training Time}
\label{app:training_time}

We provide the comparison of the training time for various baselines in Table~\ref{tab:time} on our reported hardware (Section~\ref{sec:hardware} in the Appendix).  As the table shows, the training of our method \textsc{Siren} incurs minimal computational overhead compared to the vanilla \textsc{ddetr}. In contrast, other baselines such as OW-DETR can be more than $2$ times slower than \textsc{Siren}. 
\begin{table}[h]
\centering
\begin{tabular}{l|c}
    \hline
   {\textbf{Method }}  &Training time (h)  \\
   \hline
    \multicolumn{2}{c}{ID: \textsc{pascal-voc} / \textsc{bdd100k}}\\
    \hline

      Mahalanobis 
    \citep{lee2018simple}&
      9.7 / 27
    \\ 
        Gram matrices~\citep{DBLP:conf/icml/SastryO20}& 
      9.7 / 27  \\
KNN~\citep{sun2022out} & 9.7 / 27 \\
     
   CSI~\citep{tack2020csi}& 17.1 / 47.9 \\
     VOS~\citep{du2022towards}& 11.4 / 32.7\\
     OW-DETR~\citep{ow_detr} & 23.7 / 59.3\\
     Dismax~\citep{dismax}
 & 10.0 / 27.7  \\
 \hline
     \textbf{\textsc{Siren}}   (ours) &10.1 / 27.7
    \\
        \hline
\end{tabular}
        \caption[Comparison of the training time for different baselines]{ Comparison of the training time for different baselines in Table~\ref{tab:baseline} of the main chapter. }
        \label{tab:time}
        \vspace{-1em}
\end{table}

\subsection{Details of Visualization}
\label{sec:toy_details}
For Figure~\ref{fig:toy} in the main chapter, we generate the toy data in the unit hypersphere by sampling from three vMF distributions in the 3D space. We adopt a concentration parameter $\kappa$ of 100 for all three classes. The centroid vectors are set to $[0, 0, 1], [\frac{\sqrt{3}}{2}, 0, -\frac{1}{2}]$ and $[-\frac{\sqrt{3}}{2}, 0, -\frac{1}{2}]$, respectively. The uncertainty surface is obtained by calculating the uncertainty score of  $200^2$ points in the surface of the 3D ball. 

\subsection{Software and Hardware}
\label{sec:hardware}
We run all experiments with Python 3.8.5 and PyTorch 1.7.0, using 8 NVIDIA GeForce RTX 2080Ti GPUs.

 \section{How Does Unlabeled Data Provably Help Out-of-Distribution Detection?}

\subsection{Algorithm of \textsc{SAL}}
\label{sec:algorithm_block}
We summarize our algorithm in implementation as follows.  
\begin{algorithm}[h]
\SetAlgoLined
\textbf{Input:} In-distribution data $\mathcal{S}^{\text{in}}=\left\{\left(\mathbf{x}_i, y_i\right)\right\}_{i =1}^n$. Unlabeled wild data $\mathcal{S}_{\text{wild}}=\left\{\tilde{\mathbf{x}}_i\right\}_{i=1}^m$. 
$K$-way classification model $\*h_\*w$ and OOD classifier $\*g_{\boldsymbol{\theta}}$. Parameter spaces $\mathcal{W}$ and $\Theta$. Learning rate { lr} for $\*g_{\boldsymbol{\theta}}$. \\
\textbf{Output:} Learned OOD classifier $\*g_{\widehat{\boldsymbol{\theta}}_T}$.\\
\# \textbf{Filtering stage} \\
 1) Perform ERM: $\mathbf{w}_{\mathcal{S}^{\text{in}}} \in  argmin_{\mathbf{w}\in \mathcal{W}} R_{\mathcal{S}^{\text{in}}}(\mathbf{h}_\mathbf{w})$.\\
2) Calculate the reference gradient as $\bar{\nabla}=\frac{1}{n} \sum_{(\mathbf{x}_i, y_i) \in \mathcal{S}^{\text{in}}} \nabla \ell (\mathbf{h}_{\mathbf{w}_{\mathcal{S}^{\text{in}}}}(\mathbf{x}_i),y_i)$.\\
 3)  Calculate gradient on $\mathcal{S}_{\text{wild}}$  as $\nabla \ell (\mathbf{h}_{\mathbf{w}_{\mathcal{S}^{\text{in}}}}(\tilde{\mathbf{x}}_i),\widehat{y}_{\tilde{\mathbf{x}}_i})$ and calculate the gradient matrix $\*G$.\\
 4)  Calculate the top singular vector $\*v$ of $\*G$ and the score $\tau_i =\left<\nabla \ell (\mathbf{h}_{\mathbf{w}_{\mathcal{S}^{\text{in}}}}\big(\tilde{\mathbf{x}}_i), \widehat{y}_{\tilde{\mathbf{x}}_i})-\bar{\nabla} , \mathbf{v}\right > ^2$.\\
 5) Get the candidate outliers $\mathcal{S}_T=\{\tilde{\mathbf{x}}_i \in  \mathcal{S}_{\text{wild}}, \tau_i \geq T\}$.\\
\# \textbf{Training Stage} \\
     \For{epoch in epochs}{
   6) Sample batches of data ${\mathcal{B}^{\text{in}},\mathcal{B}_T}$ from ID and candidate outliers ${\mathcal{S}^{\text{in}},\mathcal{S}_T}$.\\
7) Calculate the binary classification loss $R_{\mathcal{B}^{\text{in}},\mathcal{B}_T}(\mathbf{g}_{\boldsymbol{\theta}})$.\\
8) Update the parameter by $\widehat{\boldsymbol{\theta}}_T= \boldsymbol{\theta} - { lr} \cdot \nabla R_{\mathcal{B}^{\text{in}},\mathcal{B}_T}(\mathbf{g}_{\boldsymbol{\theta}})$.}
\caption{\textsc{SAL}: Separate And Learn}
   \label{alg:algo_sal}
\end{algorithm}

\subsection{Notations, Definitions, Assumptions and Important Constants}\label{notation,definition,Ass,Const}
Here we summarize the important notations and constants in Tables~\ref{tab: notation} and \ref{tab: Constants}, restate necessary definitions and assumptions in Sections \ref{sec:definition_app} and \ref{sec:assumption_app}. 
\subsection{Notations}
Please see Table \ref{tab: notation} for detailed notations.
\begin{table}[h]
    \centering
    \caption{Main notations and their descriptions.}
    \scalebox{0.8}{\begin{tabular}{cl}
    \toprule[1.5pt]
         \multicolumn{1}{c}{Notation} & \multicolumn{1}{c}{Description} \\
    \midrule[1pt]
    \multicolumn{2}{c}{\cellcolor{greyC} Spaces} \\
    $\mathcal{X}$, $\mathcal{Y}$     & the input space and the label space. \\
    $\mathcal{W}$, $\Theta$ & the hypothesis spaces \\

    \multicolumn{2}{c}{\cellcolor{greyC} Distributions} \\
    $\mathbb{P}_{\text{wild}}$, $\mathbb{P}_{{\text{in}}}$, $\mathbb{P}_{{\text{out}}}$& data distribution for wild data, labeled ID data and OOD data
    \\
    $\mathbb{P}_{\mathcal{X}\mathcal{Y}}$ & the joint data distribution for ID data.\\
    
    \multicolumn{2}{c}{\cellcolor{greyC} Data and Models} \\
    $\mathbf{w}$, $\mathbf{x}$, $\mathbf{v}$ & weight/input/the top-1 right singular vector of $G$ \\

    $\widehat{\nabla}$, $\tau$ &  the average gradients on labeled ID data, uncertainty score\\
    $y$ and $y_{\text{b}}$ & label for ID classification and binary label for OOD detection \\
    $\widehat{y}_{\*x}$ &Predicted one-hot label for input $\*x$ \\
     $\mathbf{h}_{\mathbf{w}}$ and $\mathbf{g}_{\boldsymbol{\theta}}$ & predictor on labeled in-distribution and binary predictor for OOD detection \\
      $\mathcal{S}_{{\text{wild}}}^{\text{in}}$, $\mathcal{S}_{{\text{wild}}}^{\text{out}}$ &  inliers and
outliers  in the wild dataset. \\ 
      $\mathcal{S}^{\text{in}}$, $\mathcal{S}_{\text{wild}}$ & labeled ID data and unlabeled wild data \\
      $n$, $m$ & size of $\mathcal{S}^{\text{in}}$, size of $\mathcal{S}_{\text{wild}}$\\
      $T$ & the filtering threshold\\
      $\mathcal{S}_{T}$ & wild data whose uncertainty score higher than threshold $T$\\

    \multicolumn{2}{c}{\cellcolor{greyC} Distances} \\
    $r_1$ and $r_2$ & the radius of the hypothesis spaces $\mathcal{W}$ and $\Theta$, respectively  \\
   
    $\| \cdot \|_2$ & $\ell_2$ norm\\
    
    \multicolumn{2}{c}{\cellcolor{greyC} Loss, Risk and Predictor} \\
    $\ell(\cdot, \cdot)$, $\ell_{\text{b}}(\cdot, \cdot)$ & ID loss function,~binary loss function\\
    $R_{\mathcal{S}}(\mathbf{h}_{\mathbf{w}})$ & the empirical risk w.r.t. predictor $\mathbf{h}_{\mathbf{w}}$ over data $\mathcal{S}$  \\
     $R_{\mathbb{P}_{\mathcal{X}\mathcal{Y}}}(\mathbf{h}_{\mathbf{w}})$ & the  risk w.r.t. predictor $\mathbf{h}_{\mathbf{w}}$ over joint distribution $\mathbb{P}_{\mathcal{X}\mathcal{Y}}$\\$R_{\mathbb{P}_{\text{in}},\mathbb{P}_{\text{out}}}(\mathbf{g}_{{\boldsymbol{\theta}}})$& the risk defined in Eq. \ref{Real-risk}\\
     ${ ERR}_{\text{in}}, { ERR}_{\text{out}}$ & the error rates of regarding ID as OOD and OOD as ID\\

    \bottomrule[1.5pt]
    \end{tabular}}
    
    \label{tab: notation}
\end{table}

\subsection{Definitions}
\label{sec:definition_app}

\begin{Definition}[$\beta$-smooth]\label{Def::beta-smooth} We say a loss function $\ell(\mathbf{h}_{\mathbf{w}}(\mathbf{x}),y)$ (defined over $\mathcal{X}\times \mathcal{Y}$) is $\beta$-smooth, if for any $\mathbf{x}\in \mathcal{X}$ and $y\in \mathcal{Y}$,
\begin{equation*}
   \big \|\nabla \ell(\mathbf{h}_{\mathbf{w}}(\mathbf{x}),y) - \nabla \ell(\mathbf{h}_{\mathbf{w}'}(\mathbf{x}),y) \big \|_2 \leq \beta \|\mathbf{w}-\mathbf{w}'\|_{2} 
\end{equation*}
\end{Definition}
$~~~~$

\begin{Definition}[Gradient-based Distribution Discrepancy]\label{Def3}
    Given distributions $\mathbb{P}$ and $\mathbb{Q}$ defined over $\mathcal{X}$, the Gradient-based Distribution Discrepancy w.r.t. predictor $\*h_{\mathbf{w}}$ and loss $\ell$ is
    \begin{equation}
        d_{\mathbf{w}}^{\ell}(\mathbb{P},\mathbb{Q}) =\big \| \nabla R_{\mathbb{P}}(\*h_{\mathbf{w}},\widehat{\*h}) -  \nabla R_{\mathbb{Q}}(\*h_{\mathbf{w}},\widehat{\*h}) \big \|_2,
    \end{equation}
    where $\widehat{\*h}$ is a classifier which returns the closest one-hot vector of $\*h_{\mathbf{w}}$, $R_{\mathbb{P}}(\*h_{\mathbf{w}},\widehat{\*h})= \mathbb{E}_{\mathbf{x}\sim \mathbb{P}} \ell(\*h_{\mathbf{w}},\widehat{\*h})$ and $R_{\mathbb{Q}}(\*h_{\mathbf{w}},\widehat{\*h})= \mathbb{E}_{\mathbf{x}\sim \mathbb{Q}} \ell(\*h_{\mathbf{w}},\widehat{\*h})$.
\end{Definition}

\begin{Definition}[$(\gamma,\zeta)$-discrepancy]\label{Def4} We say a wild distribution $\mathbb{P}_{\text{wild}}$ has $(\gamma,\zeta)$-discrepancy w.r.t. an ID joint distribution $\mathbb{P}_{\text{in}}$, if $\gamma > \min_{\mathbf{w}\in \mathcal{W}} R_{\mathbb{P}_{\mathcal{X}\mathcal{Y}}}(\mathbf{h}_{\mathbf{w}})$, and for any parameter $\mathbf{w}\in \mathcal{W}$ satisfying that $R_{\mathbb{P}_{\mathcal{X}\mathcal{Y}}}(\mathbf{h}_{\mathbf{w}})\leq \gamma$ should meet the following condition
\begin{equation*}
d_{\mathbf{w}}^{\ell}(\mathbb{P}_{\text{in}},\mathbb{P}_{\text{wild}}) > \zeta,
\end{equation*}
where $R_{\mathbb{P}_{\mathcal{X}\mathcal{Y}}}(\mathbf{h}_{\mathbf{w}})= \mathbb{E}_{(\mathbf{x},y)\sim \mathbb{P}_{\mathcal{X}\mathcal{Y}}} \ell(\mathbf{h}_{\mathbf{w}}(\mathbf{x}),y)$.
\end{Definition}
In Section~\ref{sec:verification_discrepancy}, we empirically calculate the values of the distribution discrepancy between the ID joint distribution $\mathbb{P}_{\mathcal{X}\mathcal{Y}}$ and the wild distribution $\mathbb{P}_{\text{wild}}$.
$~~~$
\\
\subsection{Assumptions}
\label{sec:assumption_app}
\begin{assumption}\label{Ass1}
$~~~~$
   \begin{itemize}
       \item The parameter space $\mathcal{W}\subset B(\mathbf{w}_0,r_1)\subset \mathbb{R}^d$ ($\ell_2$ ball of radius $r_1$ around $\mathbf{w}_0$);
       \item The parameter space $\Theta \subset B(\boldsymbol{\theta}_0, r_2)\subset \mathbb{R}^{d'}$ ($\ell_2$ ball of radius $r_2$ around $\boldsymbol{\theta}_0$);
        \item $\ell(\mathbf{h}_{\mathbf{w}}(\mathbf{x}),y) \geq 0$ and $\ell(\mathbf{h}_{\mathbf{w}}(\mathbf{x}),y)$ is $\beta_1$-smooth;
        \item $\ell_{\text{b}}(\mathbf{g}_{\boldsymbol{\theta}}(\mathbf{x}),y_{\text{b}}) \geq 0$ and $\ell_{\text{b}}(\mathbf{g}_{\boldsymbol{\theta}}(\mathbf{x}),y_{\text{b}})$ is $\beta_2$-smooth;
        \item  $\sup_{(\mathbf{x},y)\in \mathcal{X}\times \mathcal{Y}} \|\nabla \ell(\mathbf{h}_{\mathbf{w}_0}(\mathbf{x}),y)\|_2=b_1$, $\sup_{(\mathbf{x},y_{\text{b}})\in \mathcal{X} \times \mathcal{Y}_{\text{b}}} \|\nabla \ell(\mathbf{g}_{\boldsymbol{\theta}_0}(\mathbf{x}),y_{\text{b}})\|_2=b_2$;
         \item  $\sup_{(\mathbf{x},y)\in \mathcal{X}\times \mathcal{Y}}  \ell(\mathbf{h}_{\mathbf{w}_0}(\mathbf{x}),y)=B_1$, $\sup_{(\mathbf{x},y_{\text{b}})\in \mathcal{X} \times \mathcal{Y}_{\text{b}}}  \ell(\mathbf{g}_{\boldsymbol{\theta}_0}(\mathbf{x}),y_{\text{b}})=B_2$.
   \end{itemize} 
\end{assumption}

\begin{Remark} For neural networks with smooth activation functions and softmax output function, we can check that the norm of the second derivative of the loss functions (cross-entropy loss and  sigmoid loss) is bounded given the bounded parameter space, which implies that the $\beta$-smoothness of the loss functions can hold true.  Therefore,  our assumptions are reasonable in practice.
\end{Remark}
\vspace{0.5cm}
\begin{assumption}\label{Ass2}
  $\ell(\*h(\*x),\widehat{y}_{\*x})\leq \min_{y\in \mathcal{Y}} \ell(\*h(\mathbf{x}),{y}),$ where  $\widehat{y}_{\*x}$ returns the closest one-hot label of the predictor $\*h$'s output on $\*x$.
\end{assumption}

\begin{Remark}
The assumption means the loss incurred by using the predicted labels given by the classifier itself is smaller or equal to the loss incurred by using any label in the label space.
If $y=\widehat{y}_{\*x}$, the assumption is satisfied obviously. If $y\neq \widehat{y}_{\*x}$, then we provide two examples to illustrate the validity of the assumption. For example, (1) if the loss $\ell$ is the cross entropy loss, let $K=2, \*h(\*x)=[h_1, h_2]$ (classification output after softmax) and $h_1 > h_2$. Therefore, we have $\widehat{y}_{\*x}= 0$. Suppose $y=1$, we can get $\ell(\*h(\*x), \widehat{y}_{\*x}) = -\operatorname{log} (h_1) <   \ell(\*h(\*x), y) = -\operatorname{log}(h_2)$. (2) If $\ell$ is the hinge loss for binary classification, thus we have $K=1$, let $\*h(\*x)=h_1<0$ and thus $\widehat{y}_{\*x}=-1$. Suppose $y=1$, we can get $\ell(\*h(\*x), \widehat{y}_{\*x}) =\operatorname{max}(0,1+h_1) <\operatorname{max}(0,1-h_1) = \ell(\*h(\*x), y) .$
\end{Remark}
\subsection{Constants in Theory}
 \vspace{-0.5cm}
\begin{table}[h]
    \centering
    \caption{Constants in theory.}
   \scalebox{0.8}{\begin{tabular}{cl}
    \toprule[1.5pt]
         \multicolumn{1}{c}{Constants} & \multicolumn{1}{c}{Description}\\
    \midrule[2pt]
    $M = \beta_1 r_1^2+b_1r_1+B_1$ &the upper bound of loss $\ell(\mathbf{h}_{\mathbf{w}}(\mathbf{x}),y)$, see Proposition \ref{P1}
    \\
    $M' = 2(\beta_1 r_1+b_1)^2$&the upper bound of filtering score $\tau$ \\
    $\tilde{M}=\beta_1 M$&   a constant for simplified representation \\
 $L= \beta_2 r_2^2+b_2r_2+B_2$&the upper bound of loss $\ell_{\text{b}}(\mathbf{g}_{\boldsymbol{\theta}}(\mathbf{x}),y_{\text{b}})$, see Proposition \ref{P1}\\
     $d$, $d'$ &the dimensions of parameter spaces $\mathcal{W}$ and $\Theta$, respectively\\
   $R^{*}_{{\text{in}}}$ & the optimal ID risk, i.e., $R^{*}_{{\text{in}}}=\min_{\mathbf{w}\in \mathcal{W}} \mathbb{E}_{(\mathbf{x},y)\sim \mathbb{P}_{\mathcal{X}\mathcal{Y}}} \ell(\mathbf{h}_{\mathbf{w}}(\mathbf{x}),y)$
   \\
     $\delta(T)$&the main error in Eq. \ref{Eq::bound1.0}\\
      $\zeta$ & the discrepancy between $\mathbb{P}_{\text{in}}$ and $\mathbb{P}_{\text{wild}}$\\
    $\pi$ & the ratio of OOD distribution in $\mathbb{P}_{\text{wild}}$\\
    
    \bottomrule[1.5pt]
    \end{tabular}}
    
    \label{tab: Constants}
\end{table}

\subsection{Main Theorems}\label{main_theorems}
In this section, we provide a detailed and formal version of our main theorems with a complete description of the  constant terms and other additional details that are omitted in the main paper. 
\begin{theorem*}\label{MainT-1-app}
  If Assumptions \ref{Ass1} and \ref{Ass2} hold, $\mathbb{P}_{\text{wild}}$ has $(\gamma,\zeta)$-discrepancy w.r.t. $\mathbb{P}_{\mathcal{X}\mathcal{Y}}$, and there exists $\eta\in (0,1)$ s.t. $\Delta = (1-\eta)^2\zeta^2 - 8\beta_1 R_{{\text{in}}}^*>0$, then for
      \begin{equation*}
             n = \Omega \big( \frac{\tilde{M}+M(r_1+1)d}{\eta^2 \Delta } +\frac{M^2{d}}{(\gamma-R_{\text{in}}^*)^2} \big),~~~~~m = \Omega \big ( \frac{\tilde{M}+M(r_1+1)d}{\eta^2\zeta^2} \big),
     \end{equation*}
     with the probability at least $9/10$, for any $0<T<M'$ $($here $M' =2(\beta_1r_1+b_1)^2 $ is the upper bound of filtering score $\tau_i$, i.e., $\tau_i \leq M'$$)$,
     \begin{equation}
     { ERR}_{\text{in}}  \leq    \frac{8\beta_1 }{T}R_{\text{in}}^*+ O \Big ( \frac{\tilde{M}}{T}\sqrt{\frac{d}{n}}\Big )+   O \Big ( \frac{\tilde{M}}{T}\sqrt{\frac{d}{(1-\pi)m}}\Big ),
\end{equation}
\begin{equation}
\begin{split}
     { ERR}_{\text{out}} &\leq  \delta(T) + O \Big( \frac{\tilde{M}}{1-T/M'}\sqrt{\frac{d}{{\pi}^2n}}\Big)\\&+ O \Big( \frac{\max \{ {\tilde{M}\sqrt{d}}, \Delta_{\zeta}^{\eta}/\pi\}}{1-T/M'}\sqrt{\frac{1}{{\pi}^2(1-\pi)m}}\Big),
     \end{split}
\end{equation}
     where 
     $R^{*}_{{\text{in}}}$ is the optimal ID risk, i.e., $R^{*}_{{\text{in}}}=\min_{\mathbf{w}\in \mathcal{W}} \mathbb{E}_{(\mathbf{x},y)\sim \mathbb{P}_{\mathcal{X}\mathcal{Y}}} \ell(\mathbf{h}_{\mathbf{w}}(\mathbf{x}),y)$,
     \begin{equation}
     \begin{split}
         \delta(T) &= \frac{\max\{0,1-\Delta_{\zeta}^{\eta}/\pi\}}{(1-T/M')},~~~\Delta_{\zeta}^{\eta} = {0.98\eta^2 \zeta^2} - 8\beta_1R_{\text{in}}^*,\\&M=\beta_1r_1^2+b_1r_1+B_1,~~~\tilde{M}=M\beta_1,
          \end{split}
     \end{equation}
and $d$ is the dimension of the parameter space $\mathcal{W}$, here $\beta_1, r_1, B_1$ are given in Assumption \ref{Ass1}.
\end{theorem*}
$~~~~$
\vspace{1cm}
$~~~~~~$

\begin{theorem*}\label{The-1.1-app}
1) If  $\Delta_{\zeta}^{\eta} \geq  (1-\epsilon)\pi$ for a small error $\epsilon \geq 0$, then the main error $\delta(T)$ defined in Eq. \ref{main-error} satisfies that
    \begin{equation*}
        \delta(T) \leq \frac{\epsilon}{1-T/{M'}}.
    \end{equation*}
    2) If $\zeta\geq 2.011\sqrt{8\beta_1 R_{\text{in}}^*} +1.011 \sqrt{\pi}$, then there exists $\eta\in (0,1)$ ensuring that $\Delta>0$ and $ \Delta_{\zeta}^{\eta}> \pi$ hold, which implies that the main error $\delta(T)=0$.
\end{theorem*}
$~~~~$
\vspace{1cm}
$~~~~~~$
\begin{theorem*}\label{the:main2-app}
Given the same conditions in Theorem~\ref{MainT-1}, if we further require that
     \begin{equation*}
     n = \Omega \Big (\frac{\tilde{M}^2d}{\min \{\pi,\Delta_{\zeta}^{\eta}\}^2} \Big),~~~m= \Omega \Big( \frac{\tilde{M}^2d+\Delta_{\zeta}^{\eta}}{\pi^2(1-\pi)\min \{\pi,\Delta_{\zeta}^{\eta}\}^2}\Big),
 \end{equation*}
then with the probability at least $89/100$, for any $0<T<0.9 M' \min \{1,\Delta_{\zeta}^{\eta}/\pi\}$, the OOD classifier $\*g_{\widehat{\boldsymbol{\theta}}_T}$ learned by the proposed algorithm satisfies the following risk estimation
      \begin{equation}\label{Eq::Bound2.0}
     \begin{aligned}
               & R_{\mathbb{P}_{\text{in}},\mathbb{P}_{\text{out}}}(\mathbf{g}_{\widehat{\boldsymbol{\theta}}_T}) \leq  \inf_{\boldsymbol{\theta} \in \Theta} R_{\mathbb{P}_{\text{in}},\mathbb{P}_{\text{out}}}(\mathbf{g}_{{\boldsymbol{\theta}}})+\frac{3.5 L}{1-\delta(T)}\delta({T})+ \frac{9(1-\pi)L\beta_1}{\pi(1-\delta(T))T} R_{\text{in}}^*
                \\ + &O\Big(\frac{L\max\{\tilde{M}\sqrt{d},\sqrt{d'}\}}{\min\{\pi,\Delta_{\zeta}^{\eta}\}T'}\sqrt{\frac{1}{n}}\Big )+O\Big(\frac{L\max \{\tilde{M}\sqrt{d},\sqrt{d'},\Delta_{\zeta}^{\eta}\}}{\min\{\pi,\Delta_{\zeta}^{\eta}\}T'}\sqrt{\frac{1}{\pi^2(1-\pi)m}}\Big ),    
             \end{aligned}
             \end{equation}
                where  $R_{\text{in}}^*$, $\Delta_{\zeta}^{\eta}$, $M$, $M'$, $\tilde{M}$ and $d$ are shown in Theorem \ref{MainT-1}, $d'$ is the dimension of space $\Theta$,
                \begin{equation*}
                    L= \beta_2 r_2^2+b_2r_2+B_2,~~~T'=T/(1+T),
                \end{equation*}
                 and the risk $R_{\mathbb{P}_{\text{in}},\mathbb{P}_{\text{out}}}(\mathbf{g}_{{\boldsymbol{\theta}}})$ is defined as follows:
                \begin{equation*}
                    R_{\mathbb{P}_{\text{in}},\mathbb{P}_{\text{out}}}(\mathbf{g}_{\widehat{\boldsymbol{\theta}}_T}) = \mathbb{E}_{\mathbf{x}\sim \mathbb{P}_{\text{in}}} \ell_{\text{b}}(\mathbf{g}_{\boldsymbol{\theta}}(\mathbf{x}),y_{+})+\mathbb{E}_{\mathbf{x}\sim \mathbb{P}_{\text{out}}} \ell_{\text{b}}(\mathbf{g}_{\boldsymbol{\theta}}(\mathbf{x}),y_{-}).
                \end{equation*}
\end{theorem*}
$~~~~$
\vspace{1cm}
$~~~~~~$
\begin{theorem*}\label{the:main4.0-app}
Given the same conditions in Theorem~\ref{MainT-1}, with the probability at least $9/10$,
\begin{equation*}
\begin{split}
\mathbb{E}_{\tilde{\mathbf{x}}_i \sim \mathcal{S}_{\text{wild}}^{\text{in}}} \tau_i  \leq & 8\beta_1 R_{\text{in}}^*
        +O (\beta_1 M \sqrt{\frac{d}{n}})+O (\beta_1 M \sqrt{\frac{d}{(1-\pi) m}}),
\end{split}
\end{equation*}
\begin{equation*}
\begin{split}
 \mathbb{E}_{\tilde{\mathbf{x}}_i \sim \mathcal{S}_{\text{wild}}^{\text{out}}} \tau_i \geq & \frac{0.98\eta^2 \zeta^2}{\pi} - \frac{8\beta_1 R_{\text{in}}^*}{\pi} - \epsilon'(n,m),
\end{split}
\end{equation*}
furthermore, if the realizability assumption for ID distribution holds \citep{understand_ml,fang2022learnable}, then
\begin{equation*}
    \mathbb{E}_{\tilde{\mathbf{x}}_i \sim \mathcal{S}_{\text{wild}}^{\text{in}}} \tau_i  \leq  O (\beta_1 M \sqrt{\frac{d}{n}})+O (\beta_1 M \sqrt{\frac{d}{(1-\pi) m}})
\end{equation*}
\begin{equation*}
\begin{split}
 \mathbb{E}_{\tilde{\mathbf{x}}_i \sim \mathcal{S}_{\text{wild}}^{\text{out}}} \tau_i \geq & \frac{0.98\eta^2 \zeta^2}{\pi}  - \epsilon'(n,m),
\end{split}
\end{equation*}
where
\begin{equation*}
    \epsilon'(n,m) \leq O (\frac{\beta_1M}{\pi}\sqrt{\frac{d}{n}})+O \Big( ({{\beta_1 M\sqrt{d}}+ \sqrt{1-\pi}\Delta_{\zeta}^{\eta}/\pi})\sqrt{\frac{1}{{\pi}^2(1-\pi)m}}\Big),
\end{equation*}
and $R_{\text{in}}^*$, $\Delta_{\zeta}^{\eta}$, $M$ and $d$ are shown in Theorem \ref{MainT-1}.
\end{theorem*}

\newpage

\subsection{Proofs of Main Theorems}\label{Proofmain}

\subsection{Proof of Theorem \ref{MainT-1}}
\label{sec:proof_1}
\noindent \textbf{Step 1.} With the probability at least $1-\frac{7}{3}\delta>0$,
\begin{equation*}
\begin{split}
\mathbb{E}_{\tilde{\mathbf{x}}_i \sim \mathcal{S}_{\text{wild}}^{\text{in}}} \tau_i  \leq & 8\beta_1 R_{\text{in}}^*
        \\ + & 4 \beta_1 \Big [C\sqrt{\frac{M r_1 (\beta_1r_1+b_1) d}{n}}+ C\sqrt{\frac{M r_1 (\beta_1r_1+b_1) d}{(1-\pi)m - \sqrt{m\log(6/\delta)/2}}}\\ &+3M \sqrt{\frac{2\log(6/\delta)}{n}}+M \sqrt{\frac{2\log(6/\delta)}{(1-\pi)m - \sqrt{m\log(6/\delta)/2}}} \Big ],
\end{split}
\end{equation*}

This can be proven by Lemma \ref{gamma-approximation-3} and following inequality
\begin{equation*}
\begin{split}
&\mathbb{E}_{\tilde{\mathbf{x}}_i \sim \mathcal{S}_{\text{wild}}^{\text{in}}} \tau_i \leq \mathbb{E}_{\tilde{\mathbf{x}}_i \sim \mathcal{S}_{\text{wild}}^{\text{in}}}
\big \|\nabla \ell(\*h_{\mathbf{w}_{\mathcal{S}^{\text{in}}}}(\tilde{\mathbf{x}}_i),\widehat{\*h}_{\mathbf{w}_{\mathcal{S}^{\text{in}}}}(\tilde{\mathbf{x}}_i))- \mathbb{E}_{(\*x_j,y_j) \sim \mathcal{S}^{\operatorname{in}}} \nabla \ell(\*h_{\mathbf{w}_{\mathcal{S}^{\text{in}}}}(\mathbf{x}_j), y_j) \big \|_2^2,
\end{split}
\end{equation*}

\noindent \textbf{Step 2.} It is easy to check that
\begin{equation*}
   \mathbb{E}_{\tilde{\mathbf{x}}_i \sim \mathcal{S}_{\text{wild}}} \tau_i = \frac{|\mathcal{S}_{\text{wild}}^{\text{in}}|}{|\mathcal{S}_{\text{wild}}|}\mathbb{E}_{\tilde{\mathbf{x}}_i \sim \mathcal{S}_{\text{wild}}^{\text{in}}} \tau_i +
   \frac{|\mathcal{S}_{\text{wild}}^{\text{out}}|}{|\mathcal{S}_{\text{wild}}|}\mathbb{E}_{\tilde{\mathbf{x}}_i \sim \mathcal{S}_{\text{wild}}^{\text{out}}} \tau_i.
\end{equation*}

\noindent \textbf{Step 3.} Let
\begin{equation*}
\begin{split}
    \epsilon(n,m) =& 4\beta_1[C\sqrt{\frac{M r_1 (\beta_1r_1+b_1) d}{n}}+ C\sqrt{\frac{M r_1 (\beta_1r_1+b_1) d}{(1-\pi)m - \sqrt{m\log(6/\delta)/2}}}\\+&3M \sqrt{\frac{2\log(6/\delta)}{n}}+M \sqrt{\frac{2\log(6/\delta)}{(1-\pi)m - \sqrt{m\log(6/\delta)/2}}}].
    \end{split}
\end{equation*}

Under the condition in Theorem \ref{T1}, with the probability at least $\frac{97}{100} - \frac{7}{3}\delta>0$,
\begin{equation*}
\begin{split}
    \mathbb{E}_{\tilde{\mathbf{x}}_i \sim \mathcal{S}_{\text{wild}}^{\text{out}}} \tau_i \geq & \frac{m}{|\mathcal{S}_{\text{wild}}^{\text{out}}|} \big [\frac{98\eta^2 \zeta^2}{100} - \frac{|\mathcal{S}_{\text{wild}}^{\text{in}}|}{m}8\beta_1 R_{\text{in}}^*- \frac{|\mathcal{S}_{\text{wild}}^{\text{in}}|}{m} \epsilon(n,m) \big ] \\ \geq & \frac{m}{|\mathcal{S}_{\text{wild}}^{\text{out}}|} \big [\frac{98\eta^2 \zeta^2}{100} - 8\beta_1 R_{\text{in}}^*-  \epsilon(n,m) \big ]
    \\ 
    \geq & \big [ \frac{1}{\pi} - \frac{\sqrt{\log{6/\delta}}}{\pi^2\sqrt{2m}+\pi \sqrt{\log(6/\delta)}}\big ]\big [\frac{98\eta^2 \zeta^2}{100} - 8\beta_1 R_{\text{in}}^*-  \epsilon(n,m) \big ].
    \end{split}
\end{equation*}

In this proof, we set
\begin{equation*}
  \Delta(n,m) =  \big [ \frac{1}{\pi} - \frac{\sqrt{\log{6/\delta}}}{\pi^2\sqrt{2m}+\pi \sqrt{\log(6/\delta)}}\big ]\big [\frac{98\eta^2 \zeta^2}{100} - 8\beta_1 R_{\text{in}}^*-  \epsilon(n,m) \big ].
\end{equation*}
Note that $\Delta_{\zeta}^{\eta} = {0.98\eta^2 \zeta^2} - 8\beta_1 R_{\text{in}}^*$, then
\begin{equation*}
     \Delta(n,m)= \frac{1}{\pi}\Delta_{\zeta}^{\eta}-\frac{1}{\pi}\epsilon(n,m)-\Delta_{\zeta}^{\eta} \epsilon(m)+\epsilon(n)\epsilon(n,m),
\end{equation*}
where $\epsilon(m) =  {\sqrt{\log{6/\delta}}}/({\pi^2\sqrt{2m}+\pi \sqrt{\log(6/\delta)}})$.

\noindent \textbf{Step 4.} Under the condition in Theorem \ref{T1}, with the probability at least $\frac{97}{100} - \frac{7}{3}\delta>0$,

\begin{equation}
    \frac{|\{\tilde{\*x}_i\in \mathcal{S}_{\operatorname{wild}}^{\operatorname{out}}: \tau_i \leq T\}|}{| \mathcal{S}_{\operatorname{wild}}^{\operatorname{out}}|} \leq  \frac{1-\min\{1,\Delta(n,m)\}}{1-T/{M'}} ,
\end{equation}
and
\begin{equation}
    \frac{|\{\tilde{\*x}_i\in \mathcal{S}_{\operatorname{wild}}^{\operatorname{in}}: \tau_i > T\}|}{| \mathcal{S}_{\operatorname{wild}}^{\operatorname{in}}|} \leq  \frac{8\beta_1 R_{\text{in}}^*+\epsilon(n,m)}{T}.
\end{equation}

We prove this step:
let $Z$ be the \textbf{uniform} random variable with $\mathcal{S}_{\operatorname{wild}}^{\operatorname{out}}$ as its support and $Z(i)=\tau_i/(2(\beta_1r_1+b_1)^2)$, then by the Markov inequality, we have
\begin{equation}
\frac{|\{\tilde{\*x}_i\in \mathcal{S}_{\operatorname{wild}}^{\operatorname{out}}: \tau_i > T\}|}{| \mathcal{S}_{\operatorname{wild}}^{\operatorname{out}}|} =  P(Z(i)>T/(2(\beta_1r_1+b_1)^2)) \geq \frac{\Delta(n,m) -T/(2(\beta_1r_1+b_1)^2)}{1-T/(2(\beta_1r_1+b_1)^2)} .
\end{equation}
Let $Z$ be the \textbf{uniform} random variable with $\mathcal{S}_{\operatorname{wild}}^{\operatorname{in}}$ as its support and $Z(i)=\tau_i$, then by the Markov inequality, we have
\begin{equation}
     \frac{|\{\tilde{\*x}_i\in \mathcal{S}_{\operatorname{wild}}^{\operatorname{in}}: \tau_i > T\}|}{| \mathcal{S}_{\operatorname{wild}}^{\operatorname{in}}|} =  P(Z(i)>T) \leq \frac{\mathbb{E}[Z]}{T}=\frac{8\beta_1 R_{\text{in}}^*+\epsilon(n,m)}{T}.
\end{equation}

\noindent \textbf{Step 5.}  If $\pi \leq \Delta_{\zeta}^{\eta}/ (1-\epsilon/{M'})$, then
with the probability at least $\frac{97}{100} - \frac{7}{3}\delta>0$,

\begin{equation}
    \frac{|\{\tilde{\*x}_i\in \mathcal{S}_{\operatorname{wild}}^{\operatorname{out}}: \tau_i \leq T\}|}{| \mathcal{S}_{\operatorname{wild}}^{\operatorname{out}}|} \leq  \frac{\epsilon+M'\epsilon'(n,m)}{M'-T},
\end{equation}
and
\begin{equation}
    \frac{|\{\tilde{\*x}_i\in \mathcal{S}_{\operatorname{wild}}^{\operatorname{in}}: \tau_i > T\}|}{| \mathcal{S}_{\operatorname{wild}}^{\operatorname{in}}|} \leq  \frac{8\beta_1 R_{\text{in}}^*+\epsilon(n,m)}{T},
\end{equation}
where $\epsilon'(n,m) = {\epsilon(n,m)}/{\pi}+\Delta_{\zeta}^{\eta}\epsilon(m)-\epsilon(n)\epsilon(n,m)$.

\noindent \textbf{Step 6.} If we set $\delta=3/100$, then it is easy to see that
\begin{equation*}
\begin{split}
   & \epsilon(m) \leq O (\frac{1}{\pi^2 \sqrt{m}}),\\& 
   \epsilon(n,m) \leq O (\beta_1 M \sqrt{\frac{d}{n}})+O (\beta_1 M \sqrt{\frac{d}{(1-\pi) m}}),
   \\ &
   \epsilon'(n,m) \leq O (\frac{\beta_1M}{\pi}\sqrt{\frac{d}{n}})+O \Big( ({{\beta_1 M\sqrt{d}}+ \sqrt{1-\pi}\Delta_{\zeta}^{\eta}/\pi})\sqrt{\frac{1}{{\pi}^2(1-\pi)m}}\Big).
    \end{split}
\end{equation*}

\noindent \textbf{Step 7.} By results in Steps 4, 5 and 6, We complete this proof.
\newpage
\subsection{Proof of Theorem \ref{The-1.1}}
\label{sec:proof_2}
The first result is trivial. Hence, we omit it. We mainly focus on the second result in this theorem. In this proof, then we set
\begin{equation*}
    \eta ={\sqrt{8\beta_1R_{\text{in}}^{*}+0.99\pi}}/({\sqrt{0.98}\sqrt{8\beta_1 R_{\text{in}}^*}+ \sqrt{8\beta_1 R_{\text{in}}^*+\pi}})
\end{equation*}

Note that it is easy to check that
\begin{equation*}
    \zeta\geq 2.011\sqrt{8\beta_1 R_{\text{in}}^*} +1.011 \sqrt{\pi} \geq \sqrt{8\beta_1 R_{\text{in}}^*} +1.011 \sqrt{8\beta_1 R_{\text{in}}^*+ \pi}.
\end{equation*}
Therefore,
\begin{equation*}
   \eta \zeta \geq \frac{1}{\sqrt{0.98}} \sqrt{8\beta_1R_{\text{in}}^{*}+0.99\pi}>\sqrt{8\beta_1 R_{\text{in}}^*+\pi},
\end{equation*}
which implies that $\Delta_{\zeta}^{\eta}>\pi$. Note that
\begin{equation*}
\begin{split}
    (1-\eta)\zeta &\geq \frac{1}{\sqrt{0.98}}\big( {{\sqrt{0.98}\sqrt{8\beta_1 R_{\text{in}}^*}+ \sqrt{8\beta_1 R_{\text{in}}^*+\pi}}-{\sqrt{8\beta_1R_{\text{in}}^{*}+0.99\pi}}}\big )  > \sqrt{8\beta_1 R_{\text{in}}^*},
    \end{split}
\end{equation*}
which implies that $\Delta>0$. We have completed this proof.
$~$
\\

\subsection{Proof of Theorem \ref{the:main2}}
\label{sec:proof_3}
    Let
    \begin{equation*}
        \boldsymbol{\theta}^* \in \argmin_{ \boldsymbol{\theta} \in \Theta} R_{\mathbb{P}_{\text{in}},\mathbb{P}_{\text{out}}}(\boldsymbol{\theta}).
    \end{equation*}
Then by Lemma \ref{UC-I} and Lemma \ref{Error-5}, we obtain that with the high probability
\begin{equation*}
    \begin{split}
        &R_{\mathbb{P}_{\text{in}},\mathbb{P}_{\text{out}}}(\mathbf{g}_{\widehat{\boldsymbol{\theta}}_{T}})-R_{\mathbb{P}_{\text{in}},\mathbb{P}_{\text{out}}}(\mathbf{g}_{{\boldsymbol{\theta}}^*})
        \\ = & R_{\mathbb{P}_{\text{in}},\mathbb{P}_{\text{out}}}(\mathbf{g}_{\widehat{\boldsymbol{\theta}}_{T}})-R_{\mathcal{S}^{\text{in}},\mathcal{S}_T}(\mathbf{g}_{\widehat{\boldsymbol{\theta}}_{T}})+R_{\mathcal{S}^{\text{in}},\mathcal{S}_T}(\mathbf{g}_{\widehat{\boldsymbol{\theta}}_{T}})-R_{\mathcal{S}^{\text{in}},\mathcal{S}_T}(\mathbf{g}_{{\boldsymbol{\theta}}^*})\\+&R_{\mathcal{S}^{\text{in}},\mathcal{S}_T}(\mathbf{g}_{{\boldsymbol{\theta}}^*})-R_{\mathbb{P}_{\text{in}},\mathbb{P}_{\text{out}}}(\mathbf{g}_{{\boldsymbol{\theta}}^*})\\ \leq &R_{\mathbb{P}_{\text{in}},\mathbb{P}_{\text{out}}}(\mathbf{g}_{\widehat{\boldsymbol{\theta}}_{T}})-R_{\mathcal{S}^{\text{in}},\mathcal{S}_T}(\mathbf{g}_{\widehat{\boldsymbol{\theta}}_{T}})\\+&R_{\mathcal{S}^{\text{in}},\mathcal{S}_T}(\mathbf{g}_{{\boldsymbol{\theta}}^*})-R_{\mathbb{P}_{\text{in}},\mathbb{P}_{\text{out}}}(\mathbf{g}_{{\boldsymbol{\theta}}^*})\\ \leq & 2\sup_{\boldsymbol{\theta} \in \Theta} \big|R_{\mathcal{S}^{\text{in}}}^+(\mathbf{g}_{{\boldsymbol{\theta}}})-R_{\mathbb{P}_{\text{in}}}^+(\mathbf{g}_{{\boldsymbol{\theta}}})\big|
        \\ + & \sup_{\boldsymbol{\theta} \in \Theta} \big ( R_{\mathcal{S}^{\text{out}}}^-(\mathbf{g}_{{\boldsymbol{\theta}}})-R_{\mathbb{P}_{\text{out}}}^-(\mathbf{g}_{{\boldsymbol{\theta}}}) \big )+ \sup_{\boldsymbol{\theta} \in \Theta} \big ( R_{\mathbb{P}_{\text{out}}}^-(\mathbf{g}_{{\boldsymbol{\theta}}}) -R_{\mathcal{S}^{\text{out}}}^-(\mathbf{g}_{{\boldsymbol{\theta}}})\big )
        \\ \leq & \frac{3.5 L}{1-\delta(T)}\delta({T})+ \frac{9(1-\pi)L\beta_1}{\pi(1-\delta(T))T} R_{\text{in}}^*
        \\ + &O\Big(\frac{L\max\{\beta_1 M\sqrt{d},\sqrt{d'}\}(1+T)}{\min\{\pi,\Delta_{\zeta}^{\eta}\}T}\sqrt{\frac{1}{n}}\Big )\\+&O\Big(\frac{L\max \{\beta_1 M\sqrt{d},\sqrt{d'},\Delta_{\zeta}^{\eta}\}(1+T)}{\min\{\pi,\Delta_{\zeta}^{\eta}\}T}\sqrt{\frac{1}{\pi^2(1-\pi)m}}\Big ),
    \end{split}
\end{equation*}
$~$
\\

\subsection{Proof of Theorem \ref{the:main4.0-app}}\label{sec::proof_4}
    The result is induced by the Steps \textbf{1}, \textbf{3} and \textbf{6} in Proof of Theorem \ref{MainT-1} (see section \ref{sec:proof_1}).

\newpage

\subsection{Necessary Lemmas, Propositions and Theorems}\label{NecessaryLemma}

\subsection{Boundedness}

\begin{Proposition}\label{P1}
If Assumption \ref{Ass1} holds, 
\begin{equation*}
  \sup_{\mathbf{w}\in \mathcal{W}}  \sup_{(\mathbf{x},y)\in \mathcal{X}\times \mathcal{Y}} \|\nabla \ell(\mathbf{h}_{\mathbf{w}}(\mathbf{x}),y)\|_2\leq \beta_1 r_1+ b_1 =\sqrt{{M'}/{2}},
\end{equation*}
\begin{equation*}
 \sup_{\boldsymbol{\theta} \in \Theta}   \sup_{(\mathbf{x},y_{\text{b}})\in \mathcal{X}\times \mathcal{Y}_{\text{b}}} \|\nabla \ell(\mathbf{g}_{\boldsymbol{\theta}}(\mathbf{x}),y_{\text{b}})\|_2 \leq \beta_2 r_2+ b_2.
\end{equation*}
\begin{equation*}
\begin{split}
& \sup_{\mathbf{w}\in \mathcal{W}}    \sup_{(\mathbf{x},y)\in \mathcal{X}\times \mathcal{Y}} \ell(\*h_{\mathbf{w}}(\*x), y)  \leq  \beta_1 r_1^2 + b_1 r_1 + B_1=M,
\end{split}
\end{equation*}
\begin{equation*}
\begin{split}
& \sup_{\boldsymbol{\theta}\in \Theta}    \sup_{(\mathbf{x},y_{\text{b}})\in \mathcal{X}\times \mathcal{Y}_{\text{b}}} \ell_{\text{b}}(\mathbf{g}_{\boldsymbol{\theta}}(\*x), y_{\text{b}})  \leq  \beta_2 r_2^2+ b_2 r_2 + B_2=L.
\end{split}
\end{equation*}
\end{Proposition}

\begin{proof}
    One can prove this by \textit{Mean Value Theorem of Integrals} easily.
\end{proof}
$~~$
\\$~$

\begin{Proposition}\label{P2}
If Assumption \ref{Ass1} holds, for any $\mathbf{w}\in \mathcal{W}$,
\begin{equation*}
     \big \| \nabla \ell(\*h_{\mathbf{w}}(\*x), y) \big \|_2^2 \leq 2 \beta_1   \ell(\*h_{\mathbf{w}}(\*x), y).
\end{equation*}
\end{Proposition}
\begin{proof}
The details of the self-bounding property can be found in Appendix B of \cite{lei2021sharper}.
\end{proof}
$~~~$
\\$~$
\begin{Proposition}\label{P3}
If Assumption \ref{Ass1} holds, for any labeled data $\mathcal{S}$ and distribution $\mathbb{P}$, 
\begin{equation*}
     \big \| \nabla R_{\mathcal{S}}(\*h_{\mathbf{w}}) \big \|_2^2 \leq 2 \beta_1   R_{\mathcal{S}}(\*h_{\mathbf{w}}),~~~\forall \mathbf{w}\in \mathcal{W},
\end{equation*}
\begin{equation*}
     \big \| \nabla R_{\mathbb{P}}(\*h_{\mathbf{w}}) \big \|_2^2 \leq 2 \beta_1   R_{\mathbb{P}}(\*h_{\mathbf{w}}),~~~\forall \mathbf{w}\in \mathcal{W}.
\end{equation*}
\end{Proposition}
\begin{proof}
Jensen’s inequality implies that $R_{\mathcal{S}}(\*h_{\mathbf{w}})$ and $R_{\mathbb{P}}(\*h_{\mathbf{w}}) $ are $\beta_1$-smooth. Then Proposition \ref{P2} implies the results.
\end{proof}
$~~~$

\newpage

\subsection{Convergence}

\begin{lemma}[Uniform Convergence-I]\label{UC-I}
    If Assumption \ref{Ass1} holds, then for any distribution $\mathbb{P}$, with the probability at least $1-\delta>0$, for any $\mathbf{w}\in \mathcal{W}$,
    \begin{equation*}
        | R_{\mathcal{S}}(\*h_{\mathbf{w}}) - R_{\mathbb{P}}(\*h_{\mathbf{w}})| \leq M \sqrt{\frac{2\log(2/\delta)}{n}} + C\sqrt{\frac{M r_1 (\beta_1r_1+b_1) d}{n}},
    \end{equation*}
    where $n=|\mathcal{S}|$, $M=\beta_1r_1^2+b_1r_1+B_1$, $d$ is the dimension of $\mathcal{W}$, and $C$ is a uniform constant.
\end{lemma}
\begin{proof}[Proof of Lemma \ref{UC-I}] 
Let 
\begin{equation*}
X_{\mathbf{h}_\mathbf{w}} = \mathbb{E}_{(\mathbf{x},y) \sim \mathbb{P}}  \ell(\mathbf{h}_\mathbf{w}(\mathbf{x}), y)- \mathbb{E}_{(\mathbf{x},y)\sim \mathcal{S}} \ell(\mathbf{h}_\mathbf{w}(\mathbf{x}), y).
\end{equation*}
Then it is clear that
\begin{equation*}
\mathbb{E}_{S\sim \mathbb{P}^n} X_{\mathbf{h}_\mathbf{w}} = 0.
\end{equation*}
By Proposition 2.6.1 and Lemma 2.6.8 in \cite{Vershynin2018HighDimensionalP},
\begin{equation*}
 \|X_{\mathbf{h}_\mathbf{w}}-X_{\mathbf{h}_{\mathbf{w}'}}\|_{\Phi_2}\leq \frac{c_0}{\sqrt{n}}\|\ell(\mathbf{h}_\mathbf{w}(\mathbf{x}),y) -\ell(\mathbf{h}_{\mathbf{w}'}(\mathbf{x}),y) \|_{L^{\infty}(\mathcal{X}\times \mathcal{Y})},
\end{equation*}
where $\|\cdot\|_{\Phi_2}$ is the sub-gaussian norm and $c_0$ is a uniform constant.
Therefore, by Dudley’s entropy integral \citep{Vershynin2018HighDimensionalP}, we have
\begin{equation*}
\mathbb{E}_{S\sim \mathbb{P}^n} \sup_{\mathbf{w} \in \mathcal{W}} X_{\mathbf{h}_\mathbf{w}} \leq  \frac{b_0}{\sqrt{n}}\int_{0}^{+\infty} \sqrt{\log \mathcal{N}(\mathcal{F},\epsilon,L^{\infty}) }{ d} \epsilon,
\end{equation*}
where $b_0$ is a uniform constant, 
$
\mathcal{F}=\{\ell(\mathbf{h}_\mathbf{w};\mathbf{x},y) : \mathbf{w} \in \mathcal{W}\},
$ and $\mathcal{N}(\mathcal{F},\epsilon,L^{\infty})$ is the covering number under the $L^{\infty}$ norm.
Note that
\begin{equation*}
\begin{split}
\mathbb{E}_{S\sim \mathbb{P}^n} \sup_{\mathbf{w} \in \mathcal{W}} X_{\mathbf{h}_\mathbf{w}}& \leq \frac{b_0}{\sqrt{n}} \int_{0}^{+\infty} \sqrt{ \log \mathcal{N}(\mathcal{F},\epsilon,L^{\infty})} { d} \epsilon
\\ & = \frac{b_0}{\sqrt{n}} \int_{0}^{M} \sqrt{\log \mathcal{N}(\mathcal{F},\epsilon,L^{\infty})} { d} \epsilon\\ & = \frac{b_0}{\sqrt{n}} M  \int_{0}^{1} \sqrt{\log\mathcal{N}(\mathcal{F},M\epsilon,L^{\infty})} { d} \epsilon.
\end{split}
\end{equation*}
Then, we use the McDiarmid's Inequality, then with the probability at least $1-e^{-t}>0$, for any $\mathbf{w} \in \mathcal{W}$, 
\begin{equation*}
X_{\mathbf{h}_\mathbf{w}} \leq \frac{b_0}{\sqrt{n}} M  \int_{0}^{1} \sqrt{\log\mathcal{N}(\mathcal{F},M\epsilon,L^{\infty})}  { d}\epsilon+ M \sqrt{\frac{2t}{n}}.
\end{equation*}
Similarly, we can also prove that  with the probability at least $1-e^{-t}>0$, for any $\mathbf{w} \in \mathcal{W}$, 
\begin{equation*}
-X_{\mathbf{h}_\mathbf{w}} \leq \frac{b_0}{\sqrt{n}} M  \int_{0}^{1} \sqrt{\log\mathcal{N}(\mathcal{F},M\epsilon,L^{\infty})}  { d}\epsilon+ M \sqrt{\frac{2t}{n}}.
\end{equation*}
Therefore, with the probability at least $1-2e^{-t}>0$, for any $\mathbf{w} \in \mathcal{W}$, 
\begin{equation*}
|X_{\mathbf{h}_\mathbf{w}}| \leq \frac{b_0}{\sqrt{n}} M  \int_{0}^{1} \sqrt{\log\mathcal{N}(\mathcal{F},M\epsilon,L^{\infty})}  { d}\epsilon+ M \sqrt{\frac{2t}{n}}.
\end{equation*}
Note that $\ell(\mathbf{h}_{\mathbf{w}}(\mathbf{x}),y)$ is $(\beta_1r_1+b_1)$-Lipschitz w.r.t. variables $\mathbf{w}$ under the $\|\cdot\|_2$ norm. Then
\begin{equation*}
\begin{split}
\mathcal{N}(\mathcal{F},M\epsilon,L^{\infty}) \leq  & \mathcal{N}(\mathcal{W},M\epsilon/(\beta_1r_1+b_1),\|\cdot\|_2)
\leq (1+ \frac{2r_1(\beta_1r_1+b_1) }{M\epsilon})^{d},
\end{split}
\end{equation*}
which implies that
\begin{equation*}
\begin{split}
\int_{0}^{1} \sqrt{\log(\mathcal{N}(\mathcal{F},M\epsilon,L^{\infty})} { d} \epsilon
\leq & \sqrt{d} \int_{0}^1 \sqrt{\log (1+ \frac{2r_1 (\beta_1r_1+b_1)}{M\epsilon})}  { d} \epsilon
\\ 
\leq & \sqrt{d} \int_{0}^1 \sqrt{\frac{2r_1 (\beta_1r_1+b_1)}{M\epsilon}}  { d} \epsilon
= 2\sqrt{\frac{2r_1 d(\beta_1r_1+b_1) }{M}}.
\end{split}
\end{equation*}
We have completed this proof.
\end{proof}

\begin{lemma}[Uniform Convergence-II]\label{UC-II}
    If Assumption \ref{Ass1} holds, then for any distribution $\mathbb{P}$, with the probability at least $1-\delta>0$, 
    \begin{equation*}
       \big \|\nabla R_{\mathcal{S}}(\*h_{\mathbf{w}}) - \nabla R_{\mathbb{P}}(\*h_{\mathbf{w}}) \big \|_2 \leq B \sqrt{\frac{2\log(2/\delta)}{n}} + C\sqrt{\frac{M(r_1+1) d}{n}},
    \end{equation*}
    where $n=|\mathcal{S}|$, $d$ is the dimension of $\mathcal{W}$, and $C$ is a uniform constant.
\end{lemma}

\begin{proof}[Proof of Lemma \ref{UC-II}] 
    Denote $\ell(\mathbf{v},\*h_{\mathbf{w}}(\mathbf{x}),y)= \left <\nabla\ell(\*h_{\mathbf{w}}(\mathbf{x}),y), \mathbf{v} \right>$ by the loss function over parameter space $\mathcal{W}\times \{\mathbf{v}=1:\mathbf{v}\in \mathbb{R}^d\}$. Let $b$ is the upper bound of $\ell(\mathbf{v},\*h_{\mathbf{w}}(\mathbf{x}),y)$.
   Using the same techniques used in Lemma \ref{UC-I}, we can prove that with the probability at least $1-\delta>0$, for any $\mathbf{w}\in \mathcal{W}$ and any unit vector $\mathbf{v}\in \mathbb{R}^d$,
    \begin{equation*}
        \left < \nabla R_{\mathcal{S}}(\*h_{\mathbf{w}}) - \nabla R_{\mathbb{P}}(\*h_{\mathbf{w}}), \mathbf{v} \right > \leq b \sqrt{\frac{2\log(2/\delta)}{n}} + C\sqrt{\frac{b (r_1+1) \beta_1 d}{n}},
    \end{equation*}
    which implies that
      \begin{equation*}
      \big \|\nabla R_{\mathcal{S}}(\*h_{\mathbf{w}}) - \nabla R_{\mathbb{P}}(\*h_{\mathbf{w}}) \big \|_2 \leq b \sqrt{\frac{2\log(2/\delta)}{n}} + C\sqrt{\frac{b (r_1+1) \beta_1 d}{n}}.
    \end{equation*}
    Note that Proposition \ref{P1} implies that
    \begin{equation*}
        b\beta_1 \leq M.
    \end{equation*}
    Proposition \ref{P2} implies that
    \begin{equation*}
        b \leq \sqrt{2\beta_1 M}.
    \end{equation*}
    We have completed this proof.
\end{proof}

$~~~$
\\

\begin{lemma}\label{SC-I}
Let $\mathcal{S}_{\text{wild}}^{\text{in}}\subset \mathcal{S}_{\text{wild}}$ be the samples drawn from $\mathbb{P}_{\text{in}}$.   With the probability at least $1-\delta>0$,
    \begin{equation*}
        \big| |\mathcal{S}_{\text{wild}}^{\text{in}}|/|\mathcal{S}_{\text{wild}}|-(1-\pi) \big| \leq \sqrt{\frac{\log(2/\delta)}{2|\mathcal{S}_{\text{wild}}|}   },
    \end{equation*}
    which implies that
    \begin{equation*}
         \big| |\mathcal{S}_{\text{wild}}^{\text{in}}|-(1-\pi)|\mathcal{S}_{\text{wild}}| \big| \leq \sqrt{\frac{\log(2/\delta)|\mathcal{S}_{\text{wild}}|}{2}   }.
    \end{equation*}
\end{lemma}
\begin{proof}[Proof of Lemma \ref{SC-I}]
Let $X_i$ be the random variable corresponding to the case whether $i$-th data  in the wild data is drawn from $\mathbb{P}_{\operatorname{in}}$, i.e., $X_i =1$, if $i$-th data is drawn from $\mathbb{P}_{\operatorname{in}}$; otherwise, $X_i =0$. Applying Hoeffding’s inequality, we can get that with the probability at least $1-\delta >0$,
\begin{equation}
    \big | { |\mathcal{S}_{\operatorname{wild}}^{\operatorname{in}}|}/{ |\mathcal{S}_{\operatorname{wild}}|} - (1-\pi)  \big| \leq  \sqrt{\frac{\log (2/\delta)}{2 |\mathcal{S}_{\operatorname{wild}}|}}.
\end{equation}
\end{proof}
\newpage
\subsection{Necessary Lemmas and Theorems for Theorem \ref{MainT-1}}
\begin{lemma}\label{gamma-approximation}
    With the probability at least $1-\delta>0$, the ERM optimizer $\mathbf{w}_{\mathcal{S}}$ is the $\min_{\mathbf{w}\in \mathcal{W}} \mathbf{R}_{\mathbb{P}}(\mathbf{h}_{\mathbf{w}})+ O(1/\sqrt{n})$-risk point, i.e.,
    \begin{equation*}
          R_{\mathcal{S}}(\*h_{\mathbf{w}_{\mathcal{S}}}) \leq \min_{\mathbf{w}\in \mathcal{W}} R_{\mathbb{P}}(\*h_{\mathbf{w}}) + M \sqrt{\frac{\log(1/\delta)}{2n}},
    \end{equation*}
    where $n=|\mathcal{S}|$.
\end{lemma}
\begin{proof}[Proof of Lemma \ref{gamma-approximation}] Let $\mathbf{w}^* \in \argmin_{\mathbf{w}\in \mathcal{W}} \mathbf{R}_{\mathbb{P}}(\mathbf{h}_{\mathbf{w}})$. Applying Hoeffding’s inequality, we obtain that with the probability at least $1-\delta>0$,
\begin{equation*}
    R_{\mathcal{S}}(\*h_{\mathbf{w}_{\mathcal{S}}}) - \min_{\mathbf{w}\in \mathcal{W}} {R}_{\mathbb{P}}(\mathbf{h}_{\mathbf{w}}) \leq R_{\mathcal{S}} (\*h_{\mathbf{w}^{*}}) - {R}_{\mathbb{P}}(\mathbf{h}_{\mathbf{w}^*})\leq M \sqrt{\frac{\log(1/\delta)}{2n}}. 
\end{equation*}
\end{proof}
$~~~$
\\
\begin{lemma}\label{gamma-approximation-1}
If Assumptions \ref{Ass1} and \ref{Ass2} hold, then for any data $\mathcal{S}\sim \mathbb{P}^{n}$ and $\mathcal{S}'\sim \mathbb{P}^{n'}$,
    with the probability at least $1-\delta>0$, 
    \begin{equation*}
    \begin{split}
          R_{\mathcal{S}’}(\*h_{\mathbf{w}_{\mathcal{S}}})   \leq    \min_{\mathbf{w}\in \mathcal{W}} R_{\mathbb{P}}(\*h_{\mathbf{w}})&+ C\sqrt{\frac{M r_1 (\beta_1r_1+b_1) d}{n}}+ C\sqrt{\frac{M r_1 (\beta_1r_1+b_1) d}{n'}}\\&+2M \sqrt{\frac{2\log(6/\delta)}{n}}+M \sqrt{\frac{2\log(6/\delta)}{n'}} ,
          \end{split}
    \end{equation*}
    \begin{equation*}
    \begin{split}
     R_{\mathbb{P}}(\*h_{\mathbf{w}_{\mathcal{S}}},\widehat{\*h}) \leq    R_{\mathbb{P}}(\*h_{\mathbf{w}_{\mathcal{S}}})   \leq    \min_{\mathbf{w}\in \mathcal{W}} R_{\mathbb{P}}(\*h_{\mathbf{w}})&+ C\sqrt{\frac{M r_1 (\beta_1r_1+b_1) d}{n}}+2M \sqrt{\frac{2\log(6/\delta)}{n}},
          \end{split}
    \end{equation*}
    where $C$ is a uniform constant, and
\begin{equation*}
    R_{\mathbb{P}}(\*h_{\mathbf{w}_{\mathcal{S}}},\widehat{\*h}) = \mathbb{E}_{(\mathbf{x},y)\sim \mathbb{P}} \ell(\*h_{\mathbf{w}_{\mathcal{S}}}(\mathbf{x}),\widehat{\*h}(\mathbf{x})).
\end{equation*}
\end{lemma}
\begin{proof}[Proof of Lemma \ref{gamma-approximation-1}] Let $\mathbf{w}_{\mathcal{S}'} \in \argmin_{\mathbf{w}\in \mathcal{W}} {R}_{\mathcal{S}'}(\mathbf{h}_{\mathbf{w}})$ and $\mathbf{w}^* \in \argmin_{\mathbf{w}\in \mathcal{W}} {R}_{\mathbb{P}}(\mathbf{h}_{\mathbf{w}^*})$. By Lemma \ref{UC-I} and Hoeffding Inequality, we obtain that with the probability at least $1-\delta>0$,
\begin{equation*}
\begin{split}
  &  R_{\mathcal{S}'}(\*h_{\mathbf{w}_{\mathcal{S}}}) - R_{\mathbb{P}}(\mathbf{h}_{\mathbf{w}^*}) \\ \leq &
  R_{\mathcal{S}'}(\*h_{\mathbf{w}_{\mathcal{S}}}) - {R}_{\mathbb{P}}(\mathbf{h}_{\mathbf{w}_{\mathcal{S}}}) +R_{\mathbb{P}}(\*h_{\mathbf{w}_{\mathcal{S}}}) -{R}_{\mathcal{S}}(\mathbf{h}_{\mathbf{w}_{\mathcal{S}}}) + R_{\mathcal{S}}(\mathbf{w}^*)-R_{\mathbb{P}}(\mathbf{w}^*)\\ \leq &  C\sqrt{\frac{M r_1 (\beta_1r_1+b_1) d}{n}}+ C\sqrt{\frac{M r_1 (\beta_1r_1+b_1) d}{n'}}+2M \sqrt{\frac{2\log(6/\delta)}{n}}+M \sqrt{\frac{2\log(6/\delta)}{n'}},
    \end{split}
\end{equation*}
\begin{equation*}
    \begin{split}
    R_{\mathbb{P}}(\*h_{\mathbf{w}_{\mathcal{S}}}) - R_{\mathbb{P}}(\mathbf{h}_{\mathbf{w}^*})  \leq &
   R_{\mathbb{P}}(\*h_{\mathbf{w}_{\mathcal{S}}}) -{R}_{\mathcal{S}}(\mathbf{h}_{\mathbf{w}_{\mathcal{S}}}) + R_{\mathcal{S}}(\mathbf{w}^*)-R_{\mathbb{P}}(\mathbf{w}^*)\\ \leq &  C\sqrt{\frac{M r_1 (\beta_1r_1+b_1) d}{n}}+2M \sqrt{\frac{2\log(6/\delta)}{n}}.
    \end{split}
\end{equation*}
\end{proof}

$~~~$
\\

\begin{lemma}\label{gamma-approximation-2}
If Assumptions \ref{Ass1} and \ref{Ass2} hold, then for any data $\mathcal{S}\sim \mathbb{P}^{n}$ and $\mathcal{S}'\sim \mathbb{P}^{n'}$, with the probability at least $1-2\delta>0$,
    \begin{equation*}
    \begin{split}
         \mathbb{E}_{(\mathbf{x},y)\sim \mathcal{S}'} \big\| \nabla \ell(\mathbf{h}_{\mathbf{w}_{\mathcal{S}}}(\mathbf{x}),\widehat{\*h}(\mathbf{x})) - &\nabla R_{\mathcal{S}}(\mathbf{h}_{\mathbf{w}_{\mathcal{S}}})\big\|_2^2 \leq 8\beta_1 \min_{\mathbf{w}\in \mathcal{W}} R_{\mathbb{P}}(\*h_{\mathbf{w}}) 
        \\ + & C\sqrt{\frac{M r_1 (\beta_1r_1+b_1) d}{n}}+ C\sqrt{\frac{M r_1 (\beta_1r_1+b_1) d}{n'}}\\+ & 3M \sqrt{\frac{2\log(6/\delta)}{n}}+M \sqrt{\frac{2\log(6/\delta)}{n'}}.
          \end{split}
    \end{equation*}
    where $C$ is a uniform constant.
\end{lemma}
\begin{proof}[Proof of Lemma \ref{gamma-approximation-2}] By Propositions \ref{P2}, \ref{P3} and Lemmas \ref{gamma-approximation} and \ref{gamma-approximation-1}, with the probability at least $1-2\delta>0$,
    \begin{equation*}
 \begin{split}
       & \mathbb{E}_{(\mathbf{x},y)\sim \mathcal{S}'} \big\| \nabla \ell(\mathbf{h}_{\mathbf{w}_{\mathcal{S}}}(\mathbf{x}),\widehat{\*h}(\mathbf{x})) - \nabla R_{\mathcal{S}}(\mathbf{h}_{\mathbf{w}_{\mathcal{S}}})\big\|_2^2
       \\ \leq & 2\mathbb{E}_{(\mathbf{x},y)\sim \mathcal{S}'} \big\| \nabla \ell(\mathbf{h}_{\mathbf{w}_{\mathcal{S}}}(\mathbf{x}),\widehat{\*h}(\mathbf{x}))\big \|^2_2 + 2 \big \| \nabla R_{\mathcal{S}}(\mathbf{h}_{\mathbf{w}_{\mathcal{S}}})\big\|_2^2
       \\  \leq & 4\beta_1 \big ( R_{\mathcal{S}'}(\mathbf{h}_{\mathbf{w}_{\mathcal{S}}},\widehat{\*h})+R_{\mathcal{S}}(\mathbf{h}_{\mathbf{w}_{\mathcal{S}}}) \big ) \leq 4\beta_1 \big ( R_{\mathcal{S}'}(\mathbf{h}_{\mathbf{w}_{\mathcal{S}}})+R_{\mathcal{S}}(\mathbf{h}_{\mathbf{w}_{\mathcal{S}}}) \big )
       \\ \leq &
4\beta_1 \big[ 2\min_{\mathbf{w}\in \mathcal{W}} R_{\mathbb{P}}(\*h_{\mathbf{w}})+ C\sqrt{\frac{M r_1 (\beta_1r_1+b_1) d}{n}}+ C\sqrt{\frac{M r_1 (\beta_1r_1+b_1) d}{n'}}\\&~~~~~~~~+ 3M \sqrt{\frac{2\log(6/\delta)}{n}}+M \sqrt{\frac{2\log(6/\delta)}{n'}} \big ].
\end{split}
    \end{equation*}
\end{proof}
$~~~$
\\

\begin{lemma}\label{gamma-approximation-3}
Let $\mathcal{S}_{\text{wild}}^{\text{in}} \subset \mathcal{S}_{\text{wild}}$ be samples drawn from $\mathbb{P}_{\text{in}}$. If Assumptions \ref{Ass1} and \ref{Ass2} hold, then for any data $\mathcal{S}_{\text{wild}} \sim \mathbb{P}_{\text{wild}}^{m}$ and $\mathcal{S}\sim \mathbb{P}_{\text{in}}^{n}$, with the probability at least $1-\frac{7}{3}\delta>0$, 
    \begin{equation*}
    \begin{split}
         \mathbb{E}_{\mathbf{x} \sim \mathcal{S}_{\text{wild}}^{\text{in}}} \big\| \nabla \ell(\mathbf{h}_{\mathbf{w}_{\mathcal{S}}}(\mathbf{x}),\widehat{\*h}(\mathbf{x})) -& \nabla R_{\mathcal{S}}(\mathbf{h}_{\mathbf{w}_{\mathcal{S}}})\big\|_2^2 \leq 8\beta_1 \min_{\mathbf{w}\in \mathcal{W}} R_{\mathbb{P}}(\*h_{\mathbf{w}}) 
        \\ + & 4 \beta_1 \Big [C\sqrt{\frac{M r_1 (\beta_1r_1+b_1) d}{n}}+ C\sqrt{\frac{M r_1 (\beta_1r_1+b_1) d}{(1-\pi)m - \sqrt{m\log(6/\delta)/2}}}\\ & ~~~~~~~~+3M \sqrt{\frac{2\log(6/\delta)}{n}}+M \sqrt{\frac{2\log(6/\delta)}{(1-\pi)m - \sqrt{m\log(6/\delta)/2}}} \Big ],
          \end{split}
    \end{equation*}
    where $C$ is a uniform constant.
\end{lemma}
\begin{proof}[Proof of Lemma \ref{gamma-approximation-3}]
 Lemma \ref{SC-I} and   Lemma \ref{gamma-approximation-2} imply this result.
\end{proof}
$~~~$
\\

\begin{lemma}\label{gamma-approximation-3.0}
 If Assumptions \ref{Ass1} and \ref{Ass2} hold, then for any data $\mathcal{S}\sim \mathbb{P}_{\text{in}}^{n}$, with the probability at least $1-\delta>0$, 
    \begin{equation*}
    \begin{split}
          \big\| \nabla R_{\mathcal{S}}(\mathbf{h}_{\mathbf{w}_{\mathcal{S}}}(\mathbf{x}),\widehat{\*h}(\mathbf{x})) -& \nabla R_{\mathcal{S}}(\mathbf{h}_{\mathbf{w}_{\mathcal{S}}})\big\|_2^2 \leq 8\beta_1 \min_{\mathbf{w}\in \mathcal{W}} R_{\mathbb{P}}(\*h_{\mathbf{w}}) 
        + 4 M \sqrt{\frac{\log(1/\delta)}{2n}}.
          \end{split}
    \end{equation*}
\end{lemma}
\begin{proof}[Proof of Lemma \ref{gamma-approximation-3.0}] With the probability at least $1-\delta>0$,
\begin{equation*}
\begin{split}
   &\big\| \nabla R_{\mathcal{S}}(\mathbf{h}_{\mathbf{w}_{\mathcal{S}}}(\mathbf{x}),\widehat{\*h}(\mathbf{x})) - \nabla R_{\mathcal{S}}(\mathbf{h}_{\mathbf{w}_{\mathcal{S}}})\big\|_2 
   \\ 
   \leq &\big\| \nabla R_{\mathcal{S}}(\mathbf{h}_{\mathbf{w}_{\mathcal{S}}}(\mathbf{x}),\widehat{\*h}(\mathbf{x}))\big\|_2 +  \big\|\nabla R_{\mathcal{S}}(\mathbf{h}_{\mathbf{w}_{\mathcal{S}}})\big\|_2 
   \\
   \leq & 
   \sqrt{2\beta_1 R_{\mathcal{S}}(\mathbf{h}_{\mathbf{w}_{\mathcal{S}}}(\mathbf{x}),\widehat{\*h}(\mathbf{x}))} +  \sqrt{2\beta_1  R_{\mathcal{S}}(\mathbf{h}_{\mathbf{w}_{\mathcal{S}}})}
   \\
   \leq & 2\sqrt{2\beta_1 (\min_{\mathbf{w}\in \mathcal{W}} R_{\mathbb{P}}(\*h_{\mathbf{w}})+M \sqrt{\frac{\log(1/\delta)}{2n}} ) }.
\end{split}
\end{equation*}
\end{proof}

\newpage

\begin{theorem*}\label{T1}
     If Assumptions \ref{Ass1} and \ref{Ass2} hold and there exists $\eta\in (0,1)$ such that $\Delta = (1-\eta)^2\zeta^2 - 8\beta_1 \min_{\mathbf{w}\in \mathcal{W}} R_{\mathbb{P}_{\text{in}}}(\mathbf{h}_{\mathbf{w}})>0$, when
     \begin{equation*}
         n = \Omega \big( \frac{\tilde{M}+M(r_1+1)d}{\Delta \eta^2 } +\frac{M^2{d}}{(\gamma-R_{\text{in}}^*)^2} \big),~~~~~m = \Omega \big ( \frac{\tilde{M}+M(r_1+1)d}{\eta^2\zeta^2} \big),
     \end{equation*}
     with the probability at least $97/100$,
     \begin{equation*}
         \mathbb{E}_{\tilde{\mathbf{x}}_i \sim \mathcal{S}_{\text{wild}}} \tau_i  >\frac{98\eta^2 \zeta^2}{100}.
     \end{equation*}
\end{theorem*}

\begin{proof}[Proof of Theorem \ref{T1}]

\noindent \textbf{Claim 1.} With the probability at least $1-2\delta>0$, for any $\mathbf{w}\in \mathcal{W}$,
    \begin{equation*}
    \begin{split}
     d^{\ell}_{\mathbf{w}}(\mathbb{P}_{\text{in}},\mathbb{P}_{\text{wild}})-  d^{\ell}_{\mathbf{w}}(\mathcal{S}^{\text{in}}_X,\mathcal{S}_{\text{wild}}) & \leq B \sqrt{\frac{2\log(2/\delta)}{n}}+B \sqrt{\frac{2\log(2/\delta)}{m}} \\ & + C\sqrt{\frac{M (r_1+1) d}{n}}+C\sqrt{\frac{M (r_1+1) d}{m}},
     \end{split}
    \end{equation*}
where $B=\sqrt{2\beta_1 M}$, $\mathcal{S}^{\text{in}}_X$ is the feature part of $\mathcal{S}^{\text{in}}$ and $C$ is a uniform constant.

We prove this Claim: by Lemma \ref{UC-II}, it is notable that with the probability at least $1-2\delta>0$,
\begin{equation*}
\begin{split}
& d^{\ell}_{\mathbf{w}}(\mathbb{P}_{\text{in}},\mathbb{P}_{\text{wild}})-  d^{\ell}_{\mathbf{w}}(\mathcal{S}^{\text{in}}_X,\mathcal{S}_{\text{wild}}) \\
  \leq  & \big \| \nabla R_{\mathbb{P}_{\text{in}}}(\*h_{\mathbf{w}},\widehat{\*h}) -  \nabla R_{\mathbb{P}_{\text{wild}}}(\*h_{\mathbf{w}},\widehat{\*h}) \big \|_2 - \big \| \nabla R_{\mathcal{S}^{\text{in}}_X}(\*h_{\mathbf{w}},\widehat{\*h}) -  \nabla R_{\mathcal{S}_{\text{wild}}}(\*h_{\mathbf{w}},\widehat{\*h}) \big \|_2
  \\
  \leq  &  \big \| \nabla R_{\mathbb{P}_{\text{in}}}(\*h_{\mathbf{w}},\widehat{\*h}) -  \nabla R_{\mathbb{P}_{\text{wild}}}(\*h_{\mathbf{w}},\widehat{\*h}) - \nabla R_{\mathcal{S}^{\text{in}}_X}(\*h_{\mathbf{w}},\widehat{\*h}) +  \nabla R_{\mathcal{S}_{\text{wild}}}(\*h_{\mathbf{w}},\widehat{\*h}) \big \|_2
   \\
  \leq  & \big \| \nabla R_{\mathbb{P}_{\text{in}}}(\*h_{\mathbf{w}},\widehat{\*h}) -  \nabla R_{\mathcal{S}^{\text{in}}_X}(\*h_{\mathbf{w}},\widehat{\*h}) \big \|_2 + \big \| \nabla R_{\mathbb{P}_{\text{wild}}}(\*h_{\mathbf{w}},\widehat{\*h}) -  \nabla R_{\mathcal{S}_{\text{wild}}}(\*h_{\mathbf{w}},\widehat{\*h}) \big \|_2
  \\
  \leq & B \sqrt{\frac{2\log(2/\delta)}{n}}+B \sqrt{\frac{2\log(2/\delta)}{m}} + C\sqrt{\frac{B (r_1+1) \beta_1 d}{n}}+C\sqrt{\frac{B (r_1+1) \beta_1 d}{m}}\\
  \leq & B \sqrt{\frac{2\log(2/\delta)}{n}}+B \sqrt{\frac{2\log(2/\delta)}{m}} + C\sqrt{\frac{M (r_1+1) d}{n}}+C\sqrt{\frac{M (r_1+1) d}{m}}.
 \end{split}
\end{equation*}

    \noindent \textbf{Claim 2.} 
    When 
    \begin{equation}\label{gamma-error}
        \sqrt{n} = \Omega \big (\frac{M\sqrt{d}+M\sqrt{\log (6/\delta})}{\gamma-\min_{\mathbf{w}\in \mathcal{W}} R_{\mathbb{P}_{\text{in}}}(\mathbf{h}_{\mathbf{w}})}\big),
    \end{equation}

    with the probability at least $1-4\delta>0$,
    \begin{equation*}
    \begin{split}
        &\mathbb{E}_{\tilde{\mathbf{x}}_i \sim \mathcal{S}_{\text{wild}}} \tau_i \geq \Big( \zeta   -  B \sqrt{\frac{2\log(2/\delta)}{n}}-B \sqrt{\frac{2\log(2/\delta)}{m}}  \\-& C\sqrt{\frac{M (r_1+1)  d}{n}}-C\sqrt{\frac{M (r_1+1)  d}{m}}-2\sqrt{2\beta_1 (\min_{\mathbf{w}\in \mathcal{W}} R_{\mathbb{P}}(\*h_{\mathbf{w}})+M \sqrt{\frac{\log(1/\delta)}{2n}} ) }                            \Big)^2.
        \end{split}
    \end{equation*}

We prove this Claim: let $\mathbf{v}^*$ be the top-1 right singular vector computed in our algorithm, and
\begin{equation*}
 \tilde{\mathbf{v}} \in   \argmax_{\|\mathbf{v}\|\leq 1}  \left < \mathbb{E}_{(\mathbf{x}_i,y_i)\sim \mathcal{S}^{\text{in}}} \nabla \ell(\*h_{\mathbf{w}_{\mathcal{S}^{\text{in}}}}(\mathbf{x}_i), y_i)   -   \mathbb{E}_{\tilde{\mathbf{x}}_i \in \mathcal{S}_{\text{wild}}}  \nabla \ell(\*h_{\mathbf{w}_{\mathcal{S}^{\text{in}}}}(\tilde{\mathbf{x}}_i), \widehat{\*h}(\tilde{\mathbf{x}}_i))  , \mathbf{v} \right >. 
\end{equation*}

Then with the probability at least $1-4\delta>0$,
\begin{equation*}
    \begin{split}
      &~~~~~\mathbb{E}_{\tilde{\mathbf{x}}_i \sim \mathcal{S}_{\text{wild}}} \tau_i  
      \\& =  \mathbb{E}_{\tilde{\*x}_i \sim \mathcal{S}_{\operatorname{wild}}} \left( \left < \nabla \ell(\*h_{\mathbf{w}_{\mathcal{S}^{\text{in}}}}(\tilde{\mathbf{x}}_i), \widehat{\*h}(\tilde{\mathbf{x}}_i)) - 
 \mathbb{E}_{(\mathbf{x}_j,y_j)\sim \mathcal{S}^{\text{in}}} \nabla \ell(\*h_{\mathbf{w}_{\mathcal{S}^{\text{in}}}}(\mathbf{x}_j), y_j)   , \mathbf{v}^* \right > \right)^2    \\
 & \geq  \mathbb{E}_{\tilde{\*x}_i \sim \mathcal{S}_{\operatorname{wild}}} \left(\left < \nabla \ell(\*h_{\mathbf{w}_{\mathcal{S}^{\text{in}}}}(\tilde{\mathbf{x}}_i), \widehat{\*h}(\tilde{\mathbf{x}}_i)) - 
\mathbb{E}_{(\mathbf{x}_j,y_j)\sim \mathcal{S}^{\text{in}}} \nabla \ell(\*h_{\mathbf{w}_{\mathcal{S}^{\text{in}}}}(\mathbf{x}_j), y_j)   , \tilde{\mathbf{v}} \right > \right)^2    \\
    &  \geq  \left(\left <  \mathbb{E}_{(\mathbf{x}_j,y_j)\sim \mathcal{S}^{\text{in}}} \nabla \ell(\*h_{\mathbf{w}_{\mathcal{S}^{\text{in}}}}(\mathbf{x}_j), y_j)   - \mathbb{E}_{\tilde{\mathbf{x}}_i\sim \mathcal{S}_{\text{wild}}} \nabla \ell(\*h_{\mathbf{w}_{\mathcal{S}^{\text{in}}}}(\tilde{\mathbf{x}}_i), \widehat{\*h}(\tilde{\mathbf{x}}_i))      , \tilde{\mathbf{v}} \right > \right)^2 \\
     &  = \big \| \mathbb{E}_{(\mathbf{x}_j,y_j)\sim \mathcal{S}^{\text{in}}} \nabla \ell(\*h_{\mathbf{w}_{\mathcal{S}^{\text{in}}}}(\mathbf{x}_j), y_j)   - \mathbb{E}_{\tilde{\mathbf{x}}_i\sim \mathcal{S}_{\text{wild}}} \nabla \ell(\*h_{\mathbf{w}_{\mathcal{S}^{\text{in}}}}(\tilde{\mathbf{x}}_i), \widehat{\*h}(\tilde{\mathbf{x}}_i))  \big \|^2_2 \\
    & \geq \big( d^{\ell}_{\mathbf{w}_{\mathcal{S}^{\text{in}}}}(\mathcal{S}^{\text{in}}_X,\mathcal{S}_{\text{wild}})  
    \\ &-  \big \| \mathbb{E}_{(\mathbf{x}_j,y_j)\sim \mathcal{S}^{\text{in}}} \nabla \ell(\*h_{\mathbf{w}_{\mathcal{S}^{\text{in}}}}(\mathbf{x}_j), y_j)   -  \mathbb{E}_{(\mathbf{x}_j,y_j)\sim \mathcal{S}^{\text{in}}} \nabla \ell(\*h_{\mathbf{w}_{\mathcal{S}^{\text{in}}}}({\mathbf{x}}_j), \widehat{\*h}({\mathbf{x}}_j))  \big \|_2                            \big)^2
\\ &  \geq  \Big( \zeta   -  B \sqrt{\frac{2\log(2/\delta)}{n}}-B\sqrt{\frac{2\log(2/\delta)}{m}}  \\-& C\sqrt{\frac{M (r_1+1) d}{n}}-C\sqrt{\frac{M (r_1+1)  d}{m}}-2\sqrt{2\beta_1 (\min_{\mathbf{w}\in \mathcal{W}} R_{\mathbb{P}}(\*h_{\mathbf{w}})+M \sqrt{\frac{\log(1/\delta)}{2n}} ) }                            \Big)^2.
\end{split}
\end{equation*}
In above inequality, we have used the results in Claim 1, Assumption \ref{Ass2}, Lemma \ref{gamma-approximation-1} and Lemma \ref{gamma-approximation-3.0}.

\noindent \textbf{Claim 3.} Given $\delta=1/100$, then when 
 \begin{equation*}
         n = \Omega \big( \frac{\tilde{M}+M(r_1+1)d}{\Delta \eta^2 } \big),~~~~~m = \Omega \big ( \frac{\tilde{M}+M(r_1+1)d}{\eta^2\zeta^2} \big),
     \end{equation*}
     the following inequality holds:
     \begin{equation*}
     \begin{split}
       &  \Big( \zeta   -  B \sqrt{\frac{2\log(2/\delta)}{n}}-B \sqrt{\frac{2\log(2/\delta)}{m}}  - C\sqrt{\frac{M (r_1+1) d}{n}}-C\sqrt{\frac{M (r_1+1) d}{m}}\\-&2\sqrt{2\beta_1 (\min_{\mathbf{w}\in \mathcal{W}} R_{\mathbb{P}}(\*h_{\mathbf{w}})+M \sqrt{\frac{\log(1/\delta)}{2n}} ) }                            \Big)^2  >\frac{98\eta^2\theta^2}{100}.
         \end{split}
     \end{equation*}
We prove this Claim: when
\begin{equation*}
    n \geq \frac{64\sqrt{\log(10)} \beta_1 M}{\Delta},
\end{equation*}
it is easy to check that
\begin{equation*}
    (1-\eta)\zeta \geq 2\sqrt{2\beta_1 (\min_{\mathbf{w}\in \mathcal{W}} R_{\mathbb{P}}(\*h_{\mathbf{w}})+M \sqrt{\frac{\log(1/\delta)}{2n}} )}.
\end{equation*}

Additionally, when 
\begin{equation*}
    n \geq \frac{200^2 \log 200 B^2}{\eta^2 \zeta^2} +\frac{200^2 C^2M(r_1+1)d}{2\eta^2 \zeta^2},
\end{equation*}
it is easy to check that
\begin{equation*}
    \frac{\eta \zeta}{100} \geq B\sqrt{\frac{2\log(200)}{n}}+ C\sqrt{\frac{M(r_1+1)d}{n}}.
\end{equation*}
Because
\begin{equation*}
    \max\{\frac{200^2 \log 200 B^2}{\eta^2 \zeta^2} +\frac{200^2 C^2M(r_1+1)d}{2\eta^2 \zeta^2}, \frac{64\sqrt{\log(10)} \beta_1 M}{\Delta}\} \leq O\big( \frac{\tilde{M}+M(r_1+1)d}{\Delta \eta^2 }\big ),
\end{equation*}
we conclude that when
\begin{equation*}
    n= \Omega \big( \frac{\tilde{M}+M(r_1+1)d}{\Delta \eta^2 }\big ),
\end{equation*}

\begin{equation}\label{3/4eta}
    \eta -  2\sqrt{2\beta_1 (\min_{\mathbf{w}\in \mathcal{W}} R_{\mathbb{P}}(\*h_{\mathbf{w}})+M \sqrt{\frac{\log(1/\delta)}{2n}} )}- B\sqrt{\frac{2\log(200)}{n}}+ C\sqrt{\frac{M(r_1+1)d}{n}} \geq \frac{99}{100}\eta \zeta.
\end{equation}

When 
\begin{equation*}
    m \geq \frac{200^2 \log 200 B^2}{\eta^2 \zeta^2} +\frac{200^2 C^2M(r_1+1)d}{2\eta^2 \zeta^2},
\end{equation*}
we have
\begin{equation*}
\frac{\eta \zeta}{100} \geq B \sqrt{\frac{2\log(200)}{m}}+ C\sqrt{\frac{M(r_1+1)d}{m}}.
\end{equation*}
Therefore, if
\begin{equation*}
  m = \Omega \big ( \frac{\tilde{M}+M(r_1+1)d}{\eta^2\zeta^2} \big), 
\end{equation*}
we have 
\begin{equation}\label{1/4eta}
    \frac{\eta \zeta}{100} \geq B\sqrt{\frac{2\log(200)}{m}}+ C\sqrt{\frac{M(r_1+1)d}{m}}.
\end{equation}
Combining inequalities \ref{gamma-error}, \ref{3/4eta} and \ref{1/4eta}, we complete this proof.

\end{proof}

\subsection{Necessary Lemmas for Theorem \ref{the:main2}}

Let 
\begin{equation*}
    R_{\mathcal{S}_T}^-(\mathbf{g}_{\boldsymbol{\theta}}) =  \mathbb{E}_{\mathbf{x}\sim \mathcal{S}_T} \ell_{\text{b}}(\mathbf{g}_{\boldsymbol{\theta}}(\mathbf{x}), y_{-}),~ R_{\mathcal{S}_{\text{wild}}^{\text{out}}}^-(\mathbf{g}_{\boldsymbol{\theta}}) =  \mathbb{E}_{\mathbf{x}\sim \mathcal{S}_{\text{wild}}^{\text{out}}} \ell_{\text{b}}(\mathbf{g}_{\boldsymbol{\theta}}(\mathbf{x}), y_{-}),
\end{equation*}
\begin{equation*}
    R_{\mathcal{S}^{\text{in}}}^+(\mathbf{g}_{\boldsymbol{\theta}}) =  \mathbb{E}_{\mathbf{x}\sim \mathcal{S}^{\text{in}}} \ell_{\text{b}}(\mathbf{g}_{\boldsymbol{\theta}}(\mathbf{x}), y_{+}),~ R_{\mathbb{P}_{\text{in}}}^+(\mathbf{g}_{\boldsymbol{\theta}}) =  \mathbb{E}_{\mathbf{x}\sim \mathcal{S}^{\text{in}}} \ell_{\text{b}}(\mathbf{g}_{\boldsymbol{\theta}}(\mathbf{x}), y_{+}),
\end{equation*}
and
\begin{equation*}
   R_{\mathbb{P}_{\text{out}}}^-(\mathbf{g}_{\boldsymbol{\theta}}) =  \mathbb{E}_{\mathbf{x}\sim \mathcal{S}^{\text{in}}} \ell_{\text{b}}(\mathbf{g}_{\boldsymbol{\theta}}(\mathbf{x}), y_{-}).
\end{equation*}
$~$
\\

  Let 
    \begin{equation*}
    \begin{split}
   & \mathcal{S}_{+}^{\text{out}} = \{{\tilde{\mathbf{x}}}_i \in \mathcal{S}_{\text{wild}}^{\text{out}}:\tilde{\mathbf{x}}_i \leq T\}, ~~\mathcal{S}_{-}^{\text{in}} = \{{\tilde{\mathbf{x}}}_i \in \mathcal{S}_{\text{wild}}^{\text{in}}:\tilde{\mathbf{x}}_i > T\},
   \\ &\mathcal{S}_{-}^{\text{out}} = \{{\tilde{\mathbf{x}}}_i \in \mathcal{S}_{\text{wild}}^{\text{out}}:\tilde{\mathbf{x}}_i > T\}, ~~\mathcal{S}_{+}^{\text{in}} = \{{\tilde{\mathbf{x}}}_i \in \mathcal{S}_{\text{wild}}^{\text{in}}:\tilde{\mathbf{x}}_i \leq T\}.
    \end{split}
    \end{equation*}
    Then 
    \begin{equation*}
        S_{T} = \mathcal{S}_{-}^{\text{out}}  \cup \mathcal{S}_{-}^{\text{in}}, ~~~  S_{\text{wild}}^{\text{out}} = \mathcal{S}_{-}^{\text{out}}  \cup \mathcal{S}_{+}^{\text{out}}.
    \end{equation*}
$~$
\\

Let 
    \begin{equation*}
    \begin{split}
    \Delta(n,m) =  \frac{1-\min\{1,\Delta_{\zeta}^{\eta}/\pi\}}{1-T/M'} &+ O \Big( \frac{\beta_1 M\sqrt{d}}{1-T/M'}\sqrt{\frac{1}{{\pi}^2n}}\Big)\\ &+ O \Big( \frac{{\beta_1 M\sqrt{d}}+ \sqrt{1-\pi}\Delta_{\zeta}^{\eta}/\pi}{1-T/M'}\sqrt{\frac{1}{{\pi}^2(1-\pi)m}}\Big).
    \end{split}
    \end{equation*}
     \begin{equation*}
    \begin{split}
    \delta(n,m) = \frac{8\beta_1 R_{\text{in}}^*}{T}+ O \Big ( \frac{\beta_1 M\sqrt{d}}{T}\sqrt{\frac{1}{n}}\Big )+   O \Big ( \frac{\beta_1 M\sqrt{d}}{T}\sqrt{\frac{1}{m}}\Big ).
    \end{split}
    \end{equation*}
    $~$
\\
\begin{lemma}\label{Error-0}
    Under the conditions of Theorem \ref{MainT-1}, with the probability at least $9/10$,
    \begin{equation*}
      |\mathcal{S}_T|\leq |\mathcal{S}_{-}^{\text{in}}|+ |\mathcal{S}_{\text{wild}}^{\text{out}}|\leq \delta(n,m) |\mathcal{S}_{\text{wild}}^{\text{in}}|+|\mathcal{S}_{\text{wild}}^{\text{out}}|,
    \end{equation*}
     \begin{equation*}
      |\mathcal{S}_T|\geq |\mathcal{S}_{-}^{\text{out}}| \geq \big [1- \Delta(n,m)\big ]|\mathcal{S}_{\text{wild}}^{\text{out}}|.
    \end{equation*}
     \begin{equation*}
     |\mathcal{S}_{+}^{\text{out}}| \leq \Delta(n,m)|\mathcal{S}_{\text{wild}}^{\text{out}}|.
    \end{equation*}
\end{lemma}

\begin{proof}[Proof of Lemma \ref{Error-0}]
  It is a conclusion of Theorem \ref{MainT-1}.
\end{proof}
$~$
\\
\begin{lemma}\label{Error-1}
    Under the conditions of Theorem \ref{MainT-1}, with the probability at least $9/10$,
    \begin{equation*}
  \frac{-\delta(n,m)|\mathcal{S}_{\text{wild}}^{\text{in}}|}{[\delta(n,m)|\mathcal{S}_{\text{wild}}^{\text{in}}| + |\mathcal{S}_{\text{wild}}^{\text{out}}|]|\mathcal{S}_{\text{wild}}^{\text{out}}|}   \leq \frac{1}{|\mathcal{S}_T|} - \frac{1}{|\mathcal{S}_{\text{wild}}^{\text{out}}|} \leq \frac{\Delta(n,m)}{[1-\Delta(n,m)]|\mathcal{S}_{\text{wild}}^{\text{out}}|}.
    \end{equation*}
\end{lemma}

\begin{proof}[Proof of Lemma \ref{Error-1}]
  This can be conclude by Lemma \ref{Error-0} directly.
\end{proof}

$~$
\\

\begin{lemma}\label{Error-2}
    Under the conditions of Theorem \ref{MainT-1}, with the probability at least $9/10$,
    \begin{equation*}
        R_{\mathcal{S}_T}^-(\mathbf{g}_{\boldsymbol{\theta}}) - R_{\mathcal{S}_{\text{wild}}^{\text{out}}}^-(\mathbf{g}_{\boldsymbol{\theta}})\leq \frac{L\Delta(n,m)}{[1-\Delta(n,m)]} +\frac{L\delta(n,m)}{[1-\Delta(n,m)]}\cdot \big(\frac{1-\pi}{\pi}+O\big(\sqrt{\frac{1}{\pi^4 m}}\big)\big),
    \end{equation*}
    \begin{equation*}
    R_{\mathcal{S}_{\text{wild}}^{\text{out}}}^-(\mathbf{g}_{\boldsymbol{\theta}})-R_{\mathcal{S}_T}^-(\mathbf{g}_{\boldsymbol{\theta}})\leq \frac{L\Delta(n,m)}{[1-\Delta(n,m)]} + L\Delta(n,m).
    \end{equation*}
\end{lemma}

\begin{proof}[Proof of Lemma \ref{Error-2}]
  It is clear that
    \begin{equation*}
    \begin{split}
        R_{\mathcal{S}_T}^-(\mathbf{g}_{\boldsymbol{\theta}}) - R_{\mathcal{S}_{\text{wild}}^{\text{out}}}^-(\mathbf{g}_{\boldsymbol{\theta}})= &\frac{|\mathcal{S}_{-}^{\text{out}}|}{|\mathcal{S}_{T}|} R_{\mathcal{S}_{-}^{\text{out}}}^-(\mathbf{g}_{\boldsymbol{\theta}})+\frac{|\mathcal{S}_{-}^{\text{in}}|}{|\mathcal{S}_{T}|} R_{\mathcal{S}_{-}^{\text{in}}}^-(\mathbf{g}_{\boldsymbol{\theta}})\\ -&\frac{|\mathcal{S}_{-}^{\text{out}}|}{|\mathcal{S}^{\text{out}}_{\text{wild}}|} R_{\mathcal{S}_{-}^{\text{out}}}^-(\mathbf{g}_{\boldsymbol{\theta}})-\frac{|\mathcal{S}_{+}^{\text{out}}|}{|\mathcal{S}^{\text{out}}_{\text{wild}}|} R_{\mathcal{S}_{+}^{\text{out}}}^-(\mathbf{g}_{\boldsymbol{\theta}}) \\  \leq &\frac{L\Delta(n,m)}{[1-\Delta(n,m)]} +\frac{L\delta(n,m)|\mathcal{S}_{\text{wild}}^{\text{in}}|}{[1-\Delta(n,m)]|\mathcal{S}_{\text{wild}}^{\text{out}}]}\\\leq &\frac{L\Delta(n,m)}{[1-\Delta(n,m)]} +\frac{L\delta(n,m)}{[1-\Delta(n,m)]}\cdot \big(\frac{1-\pi}{\pi}+O\big(\sqrt{\frac{1}{\pi^4 m}}\big)\big).
        \end{split}
    \end{equation*}
    \begin{equation*}
    \begin{split}
     R_{\mathcal{S}_{\text{wild}}^{\text{out}}}^-(\mathbf{g}_{\boldsymbol{\theta}}) -  R_{\mathcal{S}_T}^-(\mathbf{g}_{\boldsymbol{\theta}}) = &-\frac{|\mathcal{S}_{-}^{\text{out}}|}{|\mathcal{S}_{T}|} R_{\mathcal{S}_{-}^{\text{out}}}^-(\mathbf{g}_{\boldsymbol{\theta}})-\frac{|\mathcal{S}_{-}^{\text{in}}|}{|\mathcal{S}_{T}|} R_{\mathcal{S}_{-}^{\text{in}}}^-(\mathbf{g}_{\boldsymbol{\theta}})\\ +&\frac{|\mathcal{S}_{-}^{\text{out}}|}{|\mathcal{S}^{\text{out}}_{\text{wild}}|} R_{\mathcal{S}_{-}^{\text{out}}}^-(\mathbf{g}_{\boldsymbol{\theta}})+\frac{|\mathcal{S}_{+}^{\text{out}}|}{|\mathcal{S}^{\text{out}}_{\text{wild}}|} R_{\mathcal{S}_{+}^{\text{out}}}^-(\mathbf{g}_{\boldsymbol{\theta}}) \\  \leq &\frac{L\Delta(n,m)}{[1-\Delta(n,m)]} + L\Delta(n,m).
        \end{split}
    \end{equation*}
\end{proof}

$~$
\\

\begin{lemma}\label{Error-3}
 Let $\Delta(T) = 1-\delta(T)$.   Under the conditions of Theorem \ref{MainT-1}, for any
 $\eta'>0$, when
 \begin{equation*}
     n = \Omega \Big (\frac{\tilde{M}^2d}{{\eta'}^2\pi^2(1-T/M')^2\Delta(T)^2} \Big),~~~m= \Omega \Big( \frac{\tilde{M}^2d\pi^2+\Delta_{\zeta}^{\eta}(1-\pi)}{{\eta'}^2 \pi^4(1-\pi) (1-T/M')^2\Delta(T)^2}\Big),
 \end{equation*}

 with the probability at least $9/10$,  
    \begin{equation*}
        R_{\mathcal{S}_T}^-(\mathbf{g}_{\boldsymbol{\theta}}) - R_{\mathcal{S}_{\text{wild}}^{\text{out}}}^-(\mathbf{g}_{\boldsymbol{\theta}})\leq \frac{L\Delta(n,m)}{(1-\eta')\Delta(T)} +\frac{L\delta(n,m)}{(1-\eta')\Delta(T)}\cdot \big(\frac{1-\pi}{\pi}+O\big(\sqrt{\frac{1}{\pi^4 m}}\big)\big),
    \end{equation*}
   \begin{equation*}
\begin{split}
R_{\mathcal{S}_{\text{wild}}^{\text{out}}}^-(\mathbf{g}_{\boldsymbol{\theta}})&-R_{\mathcal{S}_T}^-(\mathbf{g}_{\boldsymbol{\theta}}) \leq \frac{1.2L}{1-\delta(T)}\delta(T) +L\delta(T)\\ &+O\Big(\frac{L\beta_1 M\sqrt{d}(1+T)}{\min\{\pi,\Delta_{\zeta}^{\eta}\}T}\sqrt{\frac{1}{n}}\Big )+\\&O\Big(\frac{L(\beta_1 M\sqrt{d}+\Delta_{\zeta}^{\eta})(1+T)}{\min\{\pi,\Delta_{\zeta}^{\eta}\}T}\sqrt{\frac{1}{\pi^2(1-\pi)m}}\Big ).
\end{split}
\end{equation*}
\end{lemma}

\begin{proof}[Proof of Lemma \ref{Error-3}]
  This can be concluded by Lemma \ref{Error-2} and by the fact that $(1-\eta')\Delta(T) \geq 1-\Delta(n,m)$ directly.
\end{proof}

$~$
\\

\begin{lemma}\label{Error-4}
 Let $\Delta(T) = 1-\delta(T)$.   Under the conditions of Theorem \ref{MainT-1}, when
 \begin{equation*}
     n = \Omega \Big (\frac{\tilde{M}^2d}{\min \{\pi,\Delta_{\zeta}^{\eta}\}^2} \Big),~~~m= \Omega \Big( \frac{\tilde{M}^2d+\Delta_{\zeta}^{\eta}}{\pi^2(1-\pi)\min \{\pi,\Delta_{\zeta}^{\eta}\}^2}\Big),
 \end{equation*}

 with the probability at least $9/10$, for any $0<T<0.9 M' \min \{1,\Delta_{\zeta}^{\eta}/\pi\}$, 
   \begin{equation*}
\begin{split}
    R_{\mathcal{S}_T}^-(\mathbf{g}_{\boldsymbol{\theta}}) - R_{\mathcal{S}_{\text{wild}}^{\text{out}}}^-(&\mathbf{g}_{\boldsymbol{\theta}}) \leq \frac{1.2L}{1-\delta(T)}\delta(T) +\frac{9L\beta_1(1-\pi)}{T\pi (1-\delta(T))}R_{\text{in}}^*\\ &+O\Big(\frac{L\beta_1 M\sqrt{d}(1+T)}{\min\{\pi,\Delta_{\zeta}^{\eta}\}T}\sqrt{\frac{1}{n}}\Big )+\\&O\Big(\frac{L(\beta_1 M\sqrt{d}+\Delta_{\zeta}^{\eta})(1+T)}{\min\{\pi,\Delta_{\zeta}^{\eta}\}T}\sqrt{\frac{1}{\pi^2(1-\pi)m}}\Big ),
    \end{split}
\end{equation*}
\begin{equation*}
\begin{split}
R_{\mathcal{S}_{\text{wild}}^{\text{out}}}^-(\mathbf{g}_{\boldsymbol{\theta}})-R_{\mathcal{S}_T}^-(&\mathbf{g}_{\boldsymbol{\theta}}) \leq \frac{1.2L}{1-\delta(T)}\delta(T) +L\delta(T)\\ &+O\Big(\frac{L\beta_1 M\sqrt{d}(1+T)}{\min\{\pi,\Delta_{\zeta}^{\eta}\}T}\sqrt{\frac{1}{n}}\Big )+\\&O\Big(\frac{L(\beta_1 M\sqrt{d}+\Delta_{\zeta}^{\eta})(1+T)}{\min\{\pi,\Delta_{\zeta}^{\eta}\}T}\sqrt{\frac{1}{\pi^2(1-\pi)m}}\Big ).
\end{split}
\end{equation*}
\end{lemma}

\begin{proof}[Proof of Lemma \ref{Error-4}]
 Using Lemma \ref{Error-4} with $\eta = 8/9$, we obtain that
 \begin{equation*}
 \begin{split}
    & R_{\mathcal{S}_T}^-(\mathbf{g}_{\boldsymbol{\theta}}) - R_{\mathcal{S}_{\text{wild}}^{\text{out}}}^-(\mathbf{g}_{\boldsymbol{\theta}}) \\ \leq &\frac{1.2L\delta(T)}{1-\delta(T)} +\frac{9L\beta_1(1-\pi)}{T\pi (1-\delta(T))}R_{\text{in}}^*+\frac{L\epsilon(n)}{\Delta(T)}+\frac{L\bar{\epsilon}(n)}{\pi \Delta(T)}+\frac{L\epsilon(m)}{\Delta(T)}\\+&\frac{8L\beta_1 R^*_{\text{in}}}{\pi^2\Delta(T)T}O\big( \sqrt{\frac{1}{m}}\big )+\frac{L\bar{\epsilon}(m)}{{\pi}^2\Delta(T)}O\big( \sqrt{\frac{1}{m}}\big ) +\frac{L\bar{\epsilon}(m)}{{\pi}\Delta(T)} +\frac{L\bar{\epsilon}(n)}{{\pi}^2\Delta(T)}O\big( \sqrt{\frac{1}{m}}\big ),
     \end{split}
 \end{equation*}
 where
 \begin{equation*}
     \epsilon(n) = O \Big( \frac{\beta_1 M\sqrt{d}}{1-T/M'}\sqrt{\frac{1}{{\pi}^2n}}\Big).
 \end{equation*}
 \begin{equation*}
     \epsilon(m) =O \Big( \frac{{\beta_1 M\sqrt{d}}+ \sqrt{1-\pi}\Delta_{\zeta}^{\eta}/\pi}{1-T/M'}\sqrt{\frac{1}{{\pi}^2(1-\pi)m}}\Big).
 \end{equation*}
 \begin{equation*}
 \bar{\epsilon}(n) = O \Big ( \frac{\beta_1 M\sqrt{d}}{T}\sqrt{\frac{1}{n}}\Big ),~~~ \bar{\epsilon}(m) = O \Big ( \frac{\beta_1 M\sqrt{d}}{T}\sqrt{\frac{1}{(1-\pi)m}}\Big ).
 \end{equation*}
Using the condition that $0<T<0.9 M' \min \{1,\Delta_{\zeta}^{\eta}/\pi\}$, we have
\begin{equation*}
    \frac{1}{\Delta(T)}\big [\frac{1}{T}+\frac{1}{1-T/{M'}}\big ] \leq O\big(\frac{T+1}{\min\{1,\Delta_{\zeta}^{\eta}/\pi\}T}\big ).
\end{equation*}

 Then, we obtain that 
\begin{equation*}
\begin{split}
    R_{\mathcal{S}_T}^-(\mathbf{g}_{\boldsymbol{\theta}}) - R_{\mathcal{S}_{\text{wild}}^{\text{out}}}^-&(\mathbf{g}_{\boldsymbol{\theta}}) \leq \frac{1.2L}{1-\delta(T)}\delta(T) +\frac{9L\beta_1(1-\pi)}{T\pi (1-\delta(T))}R_{\text{in}}^*\\ &+O\Big(\frac{L\beta_1 M\sqrt{d}(1+T)}{\min\{\pi,\Delta_{\zeta}^{\eta}\}T}\sqrt{\frac{1}{n}}\Big )+\\&O\Big(\frac{L(\beta_1 M\sqrt{d}+\Delta_{\zeta}^{\eta})(1+T)}{\min\{\pi,\Delta_{\zeta}^{\eta}\}T}\sqrt{\frac{1}{\pi^2(1-\pi)m}}\Big ).
    \end{split}
\end{equation*}
Using the similar strategy, we can obtain that 
\begin{equation*}
\begin{split}
R_{\mathcal{S}_{\text{wild}}^{\text{out}}}^-(\mathbf{g}_{\boldsymbol{\theta}})-R_{\mathcal{S}_T}^-&(\mathbf{g}_{\boldsymbol{\theta}}) \leq \frac{1.2L}{1-\delta(T)}\delta(T) +L\delta(T)\\ &+O\Big(\frac{L\beta_1 M\sqrt{d}(1+T)}{\min\{\pi,\Delta_{\zeta}^{\eta}\}T}\sqrt{\frac{1}{n}}\Big )+\\&O\Big(\frac{L(\beta_1 M\sqrt{d}+\Delta_{\zeta}^{\eta})(1+T)}{\min\{\pi,\Delta_{\zeta}^{\eta}\}T}\sqrt{\frac{1}{\pi^2(1-\pi)m}}\Big ).
\end{split}
\end{equation*}
\end{proof}

$~$
\\

\begin{lemma}\label{Error-5}
    Under the conditions of Theorem \ref{MainT-1}, when
 \begin{equation*}
     n = \Omega \Big (\frac{\tilde{M}^2d}{\min \{\pi,\Delta_{\zeta}^{\eta}\}^2} \Big),~~~m= \Omega \Big( \frac{\tilde{M}^2d+\Delta_{\zeta}^{\eta}}{\pi^2(1-\pi)\min \{\pi,\Delta_{\zeta}^{\eta}\}^2}\Big),
 \end{equation*}

 with the probability at least $0.895$, for any $0<T<0.9 M' \min \{1,\Delta_{\zeta}^{\eta}/\pi\}$, 
   \begin{equation*}
\begin{split}
    R_{\mathcal{S}_T}^-(\mathbf{g}_{\boldsymbol{\theta}}) &- R_{\mathbb{P}_{\text{out}}}^-(\mathbf{g}_{\boldsymbol{\theta}})\leq \frac{1.2L}{1-\delta(T)}\delta(T) +\frac{9L\beta_1(1-\pi)}{T\pi (1-\delta(T))}R_{\text{in}}^*\\ &+O\Big(\frac{L\beta_1 M\sqrt{d}(1+T)}{\min\{\pi,\Delta_{\zeta}^{\eta}\}T}\sqrt{\frac{1}{n}}\Big )+\\&O\Big(\frac{L\max \{\beta_1 M\sqrt{d},\sqrt{d'},\Delta_{\zeta}^{\eta}\}(1+T)}{\min\{\pi,\Delta_{\zeta}^{\eta}\}T}\sqrt{\frac{1}{\pi^2(1-\pi)m}}\Big ),
    \end{split}
\end{equation*}
\begin{equation*}
\begin{split}
R_{\mathbb{P}_{\text{out}}}^-(\mathbf{g}_{\boldsymbol{\theta}})&-R_{\mathcal{S}_T}^-(\mathbf{g}_{\boldsymbol{\theta}}) \leq \frac{1.2L}{1-\delta(T)}\delta(T) +L\delta(T)\\ &+O\Big(\frac{L\beta_1 M\sqrt{d}(1+T)}{\min\{\pi,\Delta_{\zeta}^{\eta}\}T}\sqrt{\frac{1}{n}}\Big )+\\&O\Big(\frac{L\max \{\beta_1 M\sqrt{d},\sqrt{d'},\Delta_{\zeta}^{\eta}\}(1+T)}{\min\{\pi,\Delta_{\zeta}^{\eta}\}T}\sqrt{\frac{1}{\pi^2(1-\pi)m}}\Big ).
\end{split}
\end{equation*}
\end{lemma}
\begin{proof}
    By Lemmas \ref{UC-I} and \ref{SC-I}, under the condition of this lemma, we can obtain that with the high probability,
     \begin{equation*}
    \big| R^{-}_{\mathbb{P}_{\text{out}}}(\mathbf{g}_{\boldsymbol{\theta}})-R^{-}_{\mathcal{S}_{\text{wild}}^{\text{out}}}(\mathbf{g}_{\boldsymbol{\theta}}) \big| \leq O\big (L\sqrt{\frac{d'}{ \pi m}}\big ).
 \end{equation*}
 Then by Lemma \ref{Error-4}, we can prove this lemma.
\end{proof}

   $~~~$

\subsection{Empirical Verification on the Main Theorems}
\label{sec:verification_discrepancy}

 \textbf{Verification on the regulatory conditions.} In  Table~\ref{tab:discrepancy_value}, we provide empirical verification on whether the distribution discrepancy $\zeta$ satisfies the necessary regulatory condition in Theorem~\ref{The-1.1}, i.e., $\zeta\geq 2.011\sqrt{8\beta_1 R_{ in}^*} +1.011 \sqrt{\pi}$. We use  \textsc{Cifar-100} as ID and \textsc{Textures} as the wild OOD data.

Since $R^{*}_{{ in}}$ is the optimal ID risk, i.e., $R^{*}_{{ in}}=\min_{\mathbf{w}\in \mathcal{W}} \mathbb{E}_{(\mathbf{x},y)\sim \mathbb{P}_{\mathcal{X}\mathcal{Y}}} \ell(\mathbf{h}_{\mathbf{w}}(\mathbf{x}),y)$, it can be a small value close to 0 in over-parametrized neural networks~\citep{frei2022benign,bartlett2020benign}.  Therefore, we can omit the value of $2.011\sqrt{8\beta_1 R_{ in}^*}$. The  empirical result shows that $\zeta$ can easily satisfy the regulatory condition in Theorem~\ref{The-1.1}, which means our bound is useful in practice.

\begin{table}[!h]
    \centering
     \caption{ Discrepancy value $\zeta$ with different ratios $\pi$.}
    \begin{tabular}{c|ccccccc}
    \hline
     $\pi$  &  0.05 & 0.1 &  0.2 & 0.5 & 0.7 & 0.9 & 1.0 \\
     \hline
     $\zeta$& 0.91&1.09 & 1.43 & 2.49 & 3.16& 3.86&4.18 \\
     $1.011 \sqrt{\pi}$ &0.23 &  0.32 &0.45&  0.71& 0.84& 0.96& 1.01\\
      \hline
    \end{tabular}
   
    \label{tab:discrepancy_value}
\end{table}
 \textbf{Verification on the filtering errors and OOD detection results with varying $\pi$.} In Table~\ref{tab:err_trend}, we empirically verify the value of ${ ERR}_{ out}$ and ${ ERR}_{ in}$ in Theorem~\ref{MainT-1} and the corresponding OOD detection results with various mixing ratios $\pi$. We use  \textsc{Cifar-100} as ID and \textsc{Textures} as the wild OOD data. The result  aligns well with our observation of the bounds presented in Section~\ref{sec:ana_1} of the main chapter. 

\begin{table}[!h]
    \centering
     \caption{ The values of ${ ERR}_{ in}$, ${ ERR}_{ out}$ and  the OOD detection results with various mixing ratios $\pi$.}
    \begin{tabular}{c|ccccccc}
    \hline
     $\pi$  &  0.05 & 0.1 &  0.2 & 0.5& 0.7 & 0.9 & 1.0 \\
     \hline
     ${ ERR}_{ out}$& 0.37 & 0.30 & 0.22 & 0.20 &0.23 & 0.26  & 0.29\\
${ ERR}_{ in}$ & 0.031 & 0.037& 0.045 & 0.047 & 0.047&0.048 &0.048 \\
FPR95& 5.77 &   5.73 & 5.71& 5.64& 5.79 & 5.88   &  5.92 \\
      \hline
    \end{tabular}
   
    \label{tab:err_trend}
\end{table}

\subsection{Additional Experimental Details}
\label{sec:detail_app}
\textbf{Dataset details.} For Table~\ref{tab:c100}, following WOODS~\citep{katzsamuels2022training}, we split the data as follows: We use 70\% of the OOD datasets (including \textsc{Textures}, \textsc{Places365}, \textsc{Lsun-Resize} and \textsc{Lsun-C}) for the OOD data in the wild. We use the remaining samples for testing-time OOD detection.  For \textsc{Svhn}, we use the training set for the OOD data in the wild and use the test set for evaluation. 

\textbf{Training details.} Following WOODS~\citep{katzsamuels2022training}, we use Wide ResNet~\citep{zagoruyko2016wide} with 40 layers and widen factor of 2 for the classification model $\*h_\*w$. We train the ID classifier $\*h_\*w$ using stochastic gradient descent with a momentum of 0.9, weight decay of 0.0005, and an initial learning rate of 0.1. We train for 100 epochs using cosine learning rate decay,
a batch size of 128, and a dropout rate of 0.3. For the OOD classifier $\*g_{\boldsymbol{\theta}}$, we load the pre-trained ID classifier of $\*h_\*w$ and add an additional linear layer which takes in the penultimate-layer features for binary classification. We set the initial learning rate to 0.001 and fine-tune for 100 epochs by Eq.~\ref{eq:reg_loss}. We add the binary classification loss to the ID classification loss and set the loss weight for  binary classification to 10. The other  details are kept the same as training $\*h_\*w$.

\subsection{Additional  Results on \textsc{Cifar-10}}
\label{sec:c10_app}
In Table~\ref{tab:c10_app}, we compare our \textsc{SAL} with  baselines with the ID data to be \textsc{Cifar-10}, where the strong performance of \textsc{SAL} still holds.
\begin{table}[!h]
  \centering
  
  \vspace{-1em}
  \caption[OOD detection performance on {Cifar-10} as ID for SAL]{ OOD detection performance on \textsc{Cifar-10} as ID. All methods are trained on Wide ResNet-40-2 for 100 epochs with $\pi=0.1$. For each dataset, we create corresponding wild mixture distribution
$\mathbb{P}_\text{wild} := (1 - \pi) \mathbb{P}_\text{in} + \pi \mathbb{P}_\text{out}$ for training and test on the corresponding OOD dataset. Values are percentages \textbf{averaged over 10 runs}. {Bold} numbers highlight the best results. Table format credit to~\citep{katzsamuels2022training}. }
    \scalebox{0.5}{
    \begin{tabular}{cccccccccccccc}
    \toprule
    \multirow{3}[4]{*}{Methods} & \multicolumn{12}{c}{OOD Datasets}                                                             & \multirow{3}[4]{*}{ID ACC} \\
    \cmidrule{2-13}
          & \multicolumn{2}{c}{\textsc{Svhn}} & \multicolumn{2}{c}{\textsc{Places365}} & \multicolumn{2}{c}{\textsc{Lsun-C}} & \multicolumn{2}{c}{\textsc{Lsun-Resize}} & \multicolumn{2}{c}{\textsc{Textures}}& \multicolumn{2}{c}{Average} &  \\
\cmidrule{2-13}          & FPR95 & AUROC & FPR95 & AUROC & FPR95 & AUROC & FPR95 & AUROC & FPR95 & AUROC & FPR95 & AUROC &  \\
    \hline
    \multicolumn{14}{c}{With $\mathbb{P}_{\text{in}}$ only} \\

    MSP   &  48.49 & 91.89 &59.48 & 88.20& 30.80 & 95.65 & 52.15 & 91.37& 59.28&  88.50 & 50.04 & 91.12 & 94.84\\
ODIN& 33.35 & 91.96 &  57.40& 84.49&15.52 & 97.04  & 26.62 & 94.57 & 49.12 & 84.97 & 36.40 & 90.61 & 94.84 
\\
Mahalanobis & 12.89 &97.62     & 68.57&  84.61& 39.22&  94.15& 42.62&  93.23 & 15.00 & 97.33& 35.66& 93.34 &94.84
\\
Energy & 35.59  &90.96 & 40.14 &  89.89 & 8.26 & 98.35  & 27.58&  94.24 & 52.79 &85.22& 32.87 &91.73& 94.84\\
KNN& 24.53&	95.96&	25.29	&95.69&	25.55	&95.26&	27.57	&94.71&	50.90	&89.14&	30.77	&94.15&	94.84\\
 ReAct&    40.76 & 89.57 & 41.44 & 90.44 & 14.38 & 97.21 & 33.63 & 93.58 &  53.63 & 86.59 & 36.77 & 91.48 & 94.84 \\
DICE & 35.44 & 89.65 & 46.83 & 86.69 & 6.32 & 98.68  &28.93 & 93.56& 53.62 & 82.20 & 34.23 & 90.16 & 94.84\\

 ASH   &  6.51&  98.65& 48.45 & 88.34 & 0.90&  99.73&  4.96&  98.92  &24.34 &95.09&  17.03& 96.15   & 94.84\\
CSI& 17.30 &97.40&  34.95 &93.64& 1.95 &99.55 &  12.15& 98.01&  20.45& 95.93 & 17.36 &96.91 &  94.17 \\
 
KNN+ & 2.99 & 99.41 & 24.69 & 94.84  &  2.95 & 99.39& 11.22 & 97.98  & 9.65 & 98.37& 10.30 &  97.99 & 93.19\\
\hline
    \multicolumn{14}{c}{ With $\mathbb{P}_{\text{in}}$ and $\mathbb{P}_{\text{wild}}$ } \\
    OE& 0.85 &  99.82 & 23.47 &  94.62 & 1.84 &  99.65 & 0.33&  99.93 & 10.42&  98.01 & 7.38 & 98.41 &  94.07\\
  Energy (w/ OE) & 4.95 &  98.92  & 17.26 &  95.84 & 1.93 & 99.49 &5.04 &  98.83 & 13.43 & 96.69 & 8.52 & 97.95 &  94.81\\
WOODS& 0.15& 99.97 &  12.49& 97.00 & 0.22&{99.94} &{0.03}&  \textbf{99.99} &  5.95&  98.79& 3.77&  99.14 & 94.84\\

\rowcolor[HTML]{EFEFEF}\textsc{SAL} &   \textbf{0.02} & \textbf{99.98}   &  \textbf{2.57} & \textbf{99.24}   & \textbf{0.07} & \textbf{99.99}   &\textbf{0.01 }& \textbf{99.99}   & \textbf{0.90} & \textbf{99.74}  & \textbf{0.71}&  \textbf{99.78}&  93.65 \\
\rowcolor[HTML]{EFEFEF}  (Ours)& $^{\pm}$0.00 & $^{\pm}$0.00 & $^{\pm}$0.03 & $^{\pm}$0.00 & $^{\pm}$0.01 & $^{\pm}$0.00 & $^{\pm}$0.00 & $^{\pm}$0.00& $^{\pm}$0.02 & $^{\pm}$0.01 & $^{\pm}$0.01& $^{\pm}$0.00 & $^{\pm}$0.57\\

   \hline
   
\end{tabular}}

    \label{tab:c10_app}
\end{table}

\subsection{Additional  Results on Unseen OOD Datasets }

\label{sec:mixing_ratio}
In Table~\ref{tab:mixing_ratio}, we evaluate \textsc{SAL} on unseen OOD datasets, which are different from the OOD data we use in the wild. Here we consistently use \textsc{300K Random Images} as the unlabeled wild dataset and \textsc{Cifar-10} as labeled in-distribution data.  We use the 5 different OOD datasets (\textsc{Textures}, \textsc{Places365}, \textsc{Lsun-Resize}, \textsc{Svhn} and \textsc{Lsun-C}) for evaluation. When evaluating on \textsc{300K Random Images}, we use 99\% of the  \textsc{300K Random Images} dataset~\citep{hendrycks2018deep} as the wild OOD data   and the remaining 1\% of the dataset for evaluation. $\pi$ is set to 0.1. We observe that \textsc{SAL} can perform competitively on unseen datasets as well, compared to the  most relevant baseline WOODS.

\begin{table}[!h]
  \centering
  
  \caption[Evaluation on unseen OOD datasets for SAL]{ Evaluation on unseen OOD datasets.  We use  \textsc{Cifar-10} as ID  and \textsc{300K Random Images} as the wild data. All methods are trained on Wide ResNet-40-2 for 50 epochs.  {Bold} numbers highlight the best results.  }
    \scalebox{0.5}{
    \begin{tabular}{cccccccccccccc}
    \toprule
    \multirow{3}[4]{*}{Methods} & \multicolumn{12}{c}{OOD Datasets}                                                             & \multirow{3}[4]{*}{ID ACC} \\
    \cmidrule{2-13}
          & \multicolumn{2}{c}{\textsc{Svhn}} & \multicolumn{2}{c}{\textsc{Places365}} & \multicolumn{2}{c}{\textsc{Lsun-C}} & \multicolumn{2}{c}{\textsc{Lsun-Resize}} & \multicolumn{2}{c}{\textsc{Textures}}& \multicolumn{2}{c}{\textsc{300K Rand. Img.}}  &  \\
\cmidrule{2-13}          & FPR95 & AUROC & FPR95 & AUROC & FPR95 & AUROC & FPR95 & AUROC & FPR95 & AUROC & FPR95 & AUROC &  \\
    \midrule

    OE & 13.18&  97.34 & 30.54&  93.31 & 5.87&  98.86& 14.32&  97.44 &25.69&  94.35 & 30.69& 92.80  & 94.21\\
    Energy (w/ OE) &  8.52&  98.13 & 23.74 & 94.26 &2.78&  99.38 & 9.05&  98.13 & 22.32&  94.72 & 24.59&  93.99 & 94.54\\
WOODS & 5.70& \textbf{98.54} &  19.14 &95.74 &\textbf{1.31}&  \textbf{99.66} &4.13& \textbf{99.01} & 17.92& 96.43 & 19.82&  95.52 & 94.74\\
\rowcolor[HTML]{EFEFEF} \textsc{SAL} & \textbf{4.94} & 97.53  & \textbf{14.76} & \textbf{96.25} & 2.73 & 98.23 & \textbf{3.46} & 98.15 & \textbf{11.60} & \textbf{97.21} & \textbf{10.20} & \textbf{97.23} & 93.48 \\

   \hline
   
\end{tabular}}
    \label{tab:mixing_ratio}
\end{table}

\textcolor{black}{Following~\citep{he2023topological}, we use the \textsc{Cifar-100} as ID, Tiny ImageNet-crop (TINc)/Tiny ImageNet-resize (TINr) dataset as the OOD in the wild dataset and TINr/TINc as the test OOD. The comparison with baselines is shown below, where the strong performance of SAL still holds.}

\begin{table}[!h]
  \centering

  \caption[Additional results on unseen OOD datasets for SAL]{ \textcolor{black}{Additional results on unseen OOD datasets with \textsc{Cifar-100} as ID. {Bold} numbers are superior results.  }}
    \scalebox{0.88}{
    {\color{black}\begin{tabular}{ccccc}
    \toprule
    \multirow{3}[4]{*}{Methods} & \multicolumn{4}{c}{OOD Datasets}                                                   \\
    \cmidrule{2-5}
          & \multicolumn{2}{c}{\textsc{TINr}} & \multicolumn{2}{c}{\textsc{TINc}}  \\
\cmidrule{2-5}     
& FPR95 & AUROC & FPR95 & AUROC  \\
    \midrule
  STEP &72.31	&74.59	&48.68&	91.14\\
  TSL & 57.52&	82.29&	29.48	&94.62\\
\rowcolor[HTML]{EFEFEF} \textsc{SAL} (Ours)&\textbf{43.11}	&\textbf{89.17}&	\textbf{19.30}	&\textbf{96.29} \\
   \hline
\end{tabular}}}
\end{table}

\subsection{Additional Results on Near OOD Detection}
\label{sec:near_ood}
In this section, we investigate the performance of  \textsc{SAL} on  near OOD detection, which is a more challenging OOD detection scenario where the OOD data has a closer distribution to the in-distribution. Specifically, we use the \textsc{Cifar-10} as the in-distribution data and \textsc{Cifar-100} training set as the OOD data in the wild. During test time, we use the test set of \textsc{Cifar-100} as the OOD for evaluation. With a mixing ratio $\pi$ of 0.1, our \textsc{SAL} achieves an FPR95 of 24.51\% and AUROC of 95.55\% compared to 38.92\% (FPR95) and 93.27\% (AUROC) of WOODS.

\textcolor{black}{In addition, we study near OOD detection in a different data setting, i.e., the first 50 classes of \textsc{Cifar-100} as ID and the last 50 classes as OOD. The comparison with the most competitive baseline WOODS is reported as follows.}

\begin{table}[!h]
    \centering
    
        \caption[Near OOD detection  with the first 50 classes of {Cifar-100} as ID for SAL]{\textcolor{black}{ Near OOD detection  with the first 50 classes of \textsc{Cifar-100} as ID and the last 50 classes as OOD. {Bold} numbers are superior results. }}
    {\color{black}\begin{tabular}{cccc}
    \hline
       \multirow{3}{*}{Methods} & \multicolumn{3}{c}{OOD dataset} \\
    \cline{2-4}
  & \multicolumn{3}{c}{\textsc{Cifar-50}} \\
  \cline{2-4}
       & FPR95 & AUROC &ID ACC \\
      \hline
      WOODS&41.28 &	89.74	 &74.17 \\
     \rowcolor[HTML]{EFEFEF} \textsc{SAL} & \textbf{29.71} &	\textbf{93.13} &	73.86\\
          \hline
    \end{tabular}}
\end{table}

\subsection{Additional Results on Using Multiple Singular Vectors}
\label{sec:num_of_sing_vectors}
In this section, we ablate on the effect of using $c$ singular vectors to calculate the filtering score (Eq.~\ref{eq:score}). Specifically, we calculate the scores by projecting the gradient $\nabla \ell (\*h_{\*w_{\mathcal{S}^{\text{in}}}}\big(\tilde{\*x}_i), \widehat{y}_{\tilde{\*x}_i})-\bar{\nabla}$ for the wild data $\tilde{\*x}_i$ to each of the singular vectors.  The final filtering score is the average over the $c$ scores. The result is summarized in Table~\ref{tab:num_of_sing_vectors}. We observe that using the top 1 singular vector for projection achieves the best performance. As revealed in Eq.~\ref{eq:pca}, the top 1 singular vector $\*v$ maximizes the total distance from the projected gradients (onto the direction of $\*v$) to the origin (sum over all points in $\mathcal{S}_{ wild}$), where  outliers lie approximately close to and thus leads to a better separability between the ID and OOD in the wild. 

\begin{table}[!h]
    \centering
        \caption[The effect of the number of singular vectors used for the filtering score]{ The effect of the number of singular vectors used for the filtering score. Models are trained on Wide ResNet-40-2 for 100 epochs with $\pi=0.1$.  We use \textsc{Textures}  as the  wild OOD data and \textsc{Cifar-100} as the ID.}
    \begin{tabular}{c|cc}
    \hline
      Number of singular vectors $c$  & FPR95 & AUROC \\
      \hline
         1&\textbf{5.73} & \textbf{98.65}   \\
         2 &6.28 & 98.42  \\
         3 &  6.93 & 98.43\\
            4 & 7.07 & 98.37\\
               5 &  7.43 & 98.27\\
                  6 & 7.78 & 98.22 \\
          \hline
    \end{tabular}

    \label{tab:num_of_sing_vectors}
\end{table}

\subsection{Additional Results on Class-agnostic SVD}
In this section, we evaluate our \textsc{SAL} by using class-agnostic SVD as opposed to class-conditional SVD as described in Section~\ref{sec:detect} of the main chapter. Specifically,    we  maintain a class-conditional reference gradient $ \bar{\nabla}_k$, one for each class $k \in [1,K]$, estimated on ID samples belonging to class $k$. Different from calculating the singular vectors based on gradient matrix with $\mathbf{G}_k$ (containing gradient vectors of wild samples being predicted as class $k$), we formulate a single gradient matrix $\*G$ where each row is the vector $\nabla \ell (\*h_{\*w_{\mathcal{S}^{\text{in}}}}\big(\tilde{\*x}_i), \widehat{y}_{\tilde{\*x}_i})-\bar{\nabla}_{\widehat{y}_{\tilde{\*x}_i}}$, for $\tilde{\*x}_i \in \mathcal{S}_{ wild} $.  The result is shown in Table~\ref{tab:single_g}, which shows a similar performance compared with using class-conditional SVD.

\begin{table}[!h]
  \centering
  \small
  \caption[The effect of using class-agnostic SVD.]{ The effect of using class-agnostic SVD. Models are trained on Wide ResNet-40-2 for 100 epochs with $\pi=0.1$. \textsc{Cifar-100} is the in-distribution data. {Bold} numbers are superior results.  }
    \scalebox{0.5}{
    \begin{tabular}{cccccccccccccc}
    \toprule
    \multirow{3}[4]{*}{Methods} & \multicolumn{12}{c}{OOD Datasets}                                                             & \multirow{3}[4]{*}{ID ACC} \\
    \cmidrule{2-13}
          & \multicolumn{2}{c}{\textsc{Svhn}} & \multicolumn{2}{c}{\textsc{Places365}} & \multicolumn{2}{c}{\textsc{Lsun-C}} & \multicolumn{2}{c}{\textsc{Lsun-Resize}} & \multicolumn{2}{c}{\textsc{Textures}} & \multicolumn{2}{c}{Average} &  \\
\cmidrule{2-13}          & FPR95 & AUROC & FPR95 & AUROC & FPR95 & AUROC & FPR95 & AUROC & FPR95 & AUROC & FPR95 & AUROC &  \\
    \midrule

\textsc{SAL} (Class-agnostic SVD) &  0.12 & 99.43 & \textbf{3.27} & \textbf{99.21} & \textbf{0.04} & 99.92 & 0.03 & 99.27 & \textbf{5.18} & \textbf{98.77} &\textbf{1.73} & {99.32}  & 73.31\\

\rowcolor[HTML]{EFEFEF}\textsc{SAL} &  \textbf{0.07} & \textbf{99.95} & {3.53} & {99.06} & {0.06} & \textbf{99.94} &   \textbf{0.02} & \textbf{99.95} & {5.73} & {98.65}  & {1.88}& \textbf{99.51}& 73.71 \\
   \hline
   
\end{tabular}}
    \label{tab:single_g}
\end{table}

\subsection{Additional Results on Post-hoc Filtering Score}
\label{sec:addition_results_posthoc}

We investigate the importance of training the binary classifier with the filtered candidate outliers for OOD detection in Tables~\ref{tab:c10_host_hoc} and~\ref{tab:c100_host_hoc}. Specifically, we calculate our filtering score directly for the test ID and OOD data on the model trained on the labeled ID set $\mathcal{S}^{\text{in}}$ only. The results are shown in the row "\textsc{SAL} (Post-hoc)" in Tables~\ref{tab:c10_host_hoc} and~\ref{tab:c100_host_hoc}. Without explicit knowledge of the OOD data, the OOD detection performance degrades significantly compared to training an additional binary classifier (a 15.73\% drop on FPR95 for \textsc{Svhn} with \textsc{Cifar-10} as ID). However, the post-hoc filtering score can still outperform most of the baselines that use $\mathbb{P}_{\text{in}}$ only (\emph{c.f.} Table~\ref{tab:c100}), showcasing its effectiveness.

\begin{table}[!h]
  \centering
  
  \caption[OOD detection results of using post-hoc filtering score on {Cifar-10} as ID for SAL]{ OOD detection results of using post-hoc filtering score on \textsc{Cifar-10} as ID. \textsc{SAL} is trained on Wide ResNet-40-2 for 100 epochs with $\pi=0.1$.  {Bold} numbers are superior results.  }
    \scalebox{0.5}{
    \begin{tabular}{cccccccccccccc}
    \toprule
    \multirow{3}[4]{*}{Methods} & \multicolumn{12}{c}{OOD Datasets}                                                             & \multirow{3}[4]{*}{ID ACC} \\
    \cmidrule{2-13}
          & \multicolumn{2}{c}{\textsc{Svhn}} & \multicolumn{2}{c}{\textsc{Places365}} & \multicolumn{2}{c}{\textsc{Lsun-C}} & \multicolumn{2}{c}{\textsc{Lsun-Resize}} & \multicolumn{2}{c}{\textsc{Textures}}& \multicolumn{2}{c}{Average} &  \\
\cmidrule{2-13}          & FPR95 & AUROC & FPR95 & AUROC & FPR95 & AUROC & FPR95 & AUROC & FPR95 & AUROC & FPR95 & AUROC &  \\
    \midrule

\textsc{SAL} (Post-hoc) & 15.75 & 93.09  &  23.18 & 86.35  & 6.28 & 96.72 & 15.59 & 89.83  & 23.63 & 87.72  & 16.89 &90.74& 94.84\\

\rowcolor[HTML]{EFEFEF}\textsc{SAL} &\textbf{0.02} & \textbf{99.98}   &  \textbf{2.57} & \textbf{99.24}   & \textbf{0.07} & \textbf{99.99}   &\textbf{0.01 }& \textbf{99.99}   & \textbf{0.90} & \textbf{99.74}  & \textbf{0.71}&  \textbf{99.78}&  93.65 \\
   \hline
   
\end{tabular}}
    \label{tab:c10_host_hoc}
\end{table}

\begin{table}[!h]
  \centering
  
  \caption[OOD detection results of using post-hoc filtering score on Cifar-100 as ID for SAL]{ OOD detection results of using post-hoc filtering score on \textsc{Cifar-100} as ID. \textsc{SAL} is trained on Wide ResNet-40-2 for 100 epochs  with $\pi=0.1$. {Bold} numbers are superior results.  }
    \scalebox{0.5}{
    \begin{tabular}{cccccccccccccc}
    \toprule
    \multirow{3}[4]{*}{Methods} & \multicolumn{12}{c}{OOD Datasets}                                                             & \multirow{3}[4]{*}{ID ACC} \\
    \cmidrule{2-13}
          & \multicolumn{2}{c}{\textsc{Svhn}} & \multicolumn{2}{c}{\textsc{Places365}} & \multicolumn{2}{c}{\textsc{Lsun-C}} & \multicolumn{2}{c}{\textsc{Lsun-Resize}} & \multicolumn{2}{c}{\textsc{Textures}}& \multicolumn{2}{c}{Average} &  \\
\cmidrule{2-13}          & FPR95 & AUROC & FPR95 & AUROC & FPR95 & AUROC & FPR95 & AUROC & FPR95 & AUROC & FPR95 & AUROC &  \\
    \midrule

\textsc{SAL} (post-hoc) &  39.75 & 81.47 &  35.94 & 84.53   & 23.22 & 90.90   &  32.59 & 87.12 & 36.38 & 83.25   & 33.58 & 85.45 & 75.96\\

\rowcolor[HTML]{EFEFEF}\textsc{SAL} &  \textbf{0.07} & \textbf{99.95} & \textbf{3.53} & \textbf{99.06} & \textbf{0.06} & \textbf{99.94} &   \textbf{0.02} & \textbf{99.95} & \textbf{5.73} & \textbf{98.65}  & \textbf{1.88}& \textbf{99.51}& 73.71 \\
   \hline
   
\end{tabular}}
    \label{tab:c100_host_hoc}
\end{table}

\subsection{Additional Results on Leveraging the Candidate ID data}
\label{sec:addition_results_can_id}
\textcolor{black}{
In this section, we investigate the effect of incorporating the filtered wild data which has a score smaller than the threshold $T$ (candidate ID data) for training the binary classifier $\*g_{\boldsymbol{\theta}}$. Specifically, the candidate ID data and the labeled ID data are used jointly to train the binary classifier. The comparison with \textsc{SAL} on CIFAR-100 is shown as follows:}

\begin{table}[!h]
  \centering
  
  \caption[OOD detection results of selecting candidate ID data for training  on {Cifar-100} as ID for SAL]{ \textcolor{black}{OOD detection results of selecting candidate ID data for training  on \textsc{Cifar-100} as ID. \textsc{SAL} is trained on Wide ResNet-40-2 for 100 epochs  with $\pi=0.1$. {Bold} numbers are superior results.  }}
    \scalebox{0.5}{
    \begin{tabular}{cccccccccccccc}
    \toprule
    \multirow{3}[4]{*}{Methods} & \multicolumn{12}{c}{OOD Datasets}                                                             & \multirow{3}[4]{*}{ID ACC} \\
    \cmidrule{2-13}
          & \multicolumn{2}{c}{\textsc{Svhn}} & \multicolumn{2}{c}{\textsc{Places365}} & \multicolumn{2}{c}{\textsc{Lsun-C}} & \multicolumn{2}{c}{\textsc{Lsun-Resize}} & \multicolumn{2}{c}{\textsc{Textures}}& \multicolumn{2}{c}{Average} &  \\
\cmidrule{2-13}          & FPR95 & AUROC & FPR95 & AUROC & FPR95 & AUROC & FPR95 & AUROC & FPR95 & AUROC & FPR95 & AUROC &  \\
    \midrule

  Candidate ID data   &1.23 &99.87&\textbf{2.62}&\textbf{99.18}& \textbf{0.04}&\textbf{99.95}& \textbf{0.02}& 99.91& \textbf{4.71}&\textbf{98.97} &\textbf{1.72}&\textbf{99.58}&73.83\\
\rowcolor[HTML]{EFEFEF} \textsc{SAL} (Ours)&\textbf{0.07}& \textbf{99.95} &3.53& 99.06& 0.06 &99.94 &\textbf{0.02}& \textbf{99.95} &5.73 &98.65 &1.88& 99.51 &73.71 \\
   \hline
   
\end{tabular}}
\end{table}

\textcolor{black}{The result of selecting candidate ID data (and combine with labeled ID data) shows slightly better performance, which echoes our theory that the generalization bound of the OOD detector will be better if we have more ID training data (Theorem~\ref{the:main2}). }

  \subsection{Analysis  on Using Random Labels}
\label{sec:experiments_on_random_label_app}

\textcolor{black}{We present the OOD detection result of replacing the predicted labels with the random labels for the wild data as follows. The other experimental details are kept the same as \textsc{SAL}. }

\begin{table}[!h]
  \centering
  
  \caption[OOD detection results of using random labels for the wild data  on {Cifar-100} as ID for SAL]{ \textcolor{black}{OOD detection results of using random labels for the wild data  on \textsc{Cifar-100} as ID. \textsc{SAL} is trained on Wide ResNet-40-2 for 100 epochs  with $\pi=0.1$. {Bold} numbers are superior results.  }}
    \scalebox{0.5}{
    \begin{tabular}{cccccccccccccc}
    \toprule
    \multirow{3}[4]{*}{Methods} & \multicolumn{12}{c}{OOD Datasets}                                                   
    & \multirow{3}[4]{*}{ID ACC} \\
    \cmidrule{2-13}
          & \multicolumn{2}{c}{\textsc{Svhn}} & \multicolumn{2}{c}{\textsc{Places365}} & \multicolumn{2}{c}{\textsc{Lsun-C}} & \multicolumn{2}{c}{\textsc{Lsun-Resize}} & \multicolumn{2}{c}{\textsc{Textures}}& \multicolumn{2}{c}{Average} &  \\
\cmidrule{2-13}          & FPR95 & AUROC & FPR95 & AUROC & FPR95 & AUROC & FPR95 & AUROC & FPR95 & AUROC & FPR95 & AUROC &  \\
    \midrule

  w/ Random labels  &39.36 & 89.31 &  77.98 & 78.31 & 47.46 & 88.90 & 67.28& 80.23 & 54.86& 86.92& 57.39& 84.73&73.68 \\
\rowcolor[HTML]{EFEFEF} \textsc{SAL} (Ours)&\textbf{0.07}& \textbf{99.95} &\textbf{3.53}& \textbf{99.06}& \textbf{0.06} &\textbf{99.94} &\textbf{0.02}& \textbf{99.95} &\textbf{5.73} &\textbf{98.65} &\textbf{1.88}& \textbf{99.51} &73.71 \\
   \hline
\end{tabular}}
\end{table}
\textcolor{black}{As we can observe, using the random labels leads to worse OOD detection performance because the gradient of the wild data can be wrong. In our theoretical analysis (Theorem~\ref{T1}), we have proved that using the predicted label can lead to a good separation of the wild ID and OOD data. However, the analysis using random labels might hold since it violates the assumption (Definitions~\ref{Def3} and \ref{Def4}) that the expected gradient of ID data should be different from that of wild data.}

  \subsection{Details of the Illustrative Experiments on the Impact of Predicted Labels}
\label{sec:experiments_on_predicted_label_app}
For calculating the filtering accuracy, \textsc{SAL} is trained on Wide ResNet-40-2 for 100 epochs  with $\pi=0.1$ on two separate ID datasets. The other training details are kept the same as Section~\ref{sec:exp_steup} and Appendix~\ref{sec:detail_app}.

\subsection{Details of  Figure~\ref{fig:toy}}
\label{sec:details_of_toy_app}
For Figure~\ref{fig:toy} in the main chapter, we generate the in-distribution data from three multivariate Gaussian distributions, forming three classes. The mean vectors are set to $[-2,0], [2,0]$ and $[0,2\sqrt{3}]$, respectively. The covariance matrix for all three classes is set to $\left[\begin{array}{ll}
0.25 & 0 \\
0 & 0.25
\end{array}\right]$. For each class, we generate $1,000$ samples.

For wild scenario 1, we generate the outlier data in the wild by sampling $100,000$ data points from a multivariate Gaussian $\mathcal{N}([0,\frac{2}{\sqrt{3}}], 7\cdot \mathbf{I})$ where $\mathbf{I} $ is $2\times2$ identity matrix, and only keep the $1,000$ data points that have the largest distance to the mean vector $[0, \frac{2}{\sqrt{3}}]$. For wild scenario 2,  we generate the outlier data in the wild by sampling $1,000$ data points from a multivariate Gaussian $\mathcal{N}([10,\frac{2}{\sqrt{3}}], 0.25\cdot \mathbf{I})$. For the in-distribution data in the wild,  we sample $3,000$ data points per class from the same three multivariate Gaussian distributions as mentioned before.

\subsection{Software and Hardware}
\label{sec:hardware}
We run all experiments with Python 3.8.5 and PyTorch 1.13.1, using NVIDIA GeForce RTX 2080Ti GPUs.

{\color{black}\subsection{Results with Varying Mixing Ratios}

We provide additional results of \textsc{SAL} with varying $\pi$, i.e., 0.05, 0.2, 0.5, 0.9, and contrast with the baselines, which are shown below (\textsc{Cifar-100} as the in-distribution dataset). We found that the advantage of \textsc{SAL} still holds. 
\begin{table}[!h]
  \centering
  
  \caption[OOD detection results with multiple mixing ratios $\pi$ with {Cifar-100} as ID for SAL]{ \textcolor{black}{OOD detection results with multiple mixing ratios $\pi$ with \textsc{Cifar-100} as ID. \textsc{SAL} is trained on Wide ResNet-40-2 for 100 epochs. {Bold} numbers are superior results.  }}
    \scalebox{0.54}{
    {\color{black}\begin{tabular}{ccccccccccccc}
    \toprule
    \multirow{3}[4]{*}{Methods} & \multicolumn{10}{c}{OOD Datasets}                                     
    & \multirow{3}[4]{*}{ID ACC} \\
    \cmidrule{2-11}
          & \multicolumn{2}{c}{\textsc{Svhn}} & \multicolumn{2}{c}{\textsc{Places365}} & \multicolumn{2}{c}{\textsc{Lsun-C}} & \multicolumn{2}{c}{\textsc{Lsun-Resize}} & \multicolumn{2}{c}{\textsc{Textures}} &  \\
\cmidrule{2-13}          & FPR95 & AUROC & FPR95 & AUROC & FPR95 & AUROC & FPR95 & AUROC & FPR95 & AUROC & &  \\
    \midrule
& \multicolumn{10}{c}{$\pi=0.05$} \\
   OE &2.78	&98.84	&63.63	&80.22	&6.73&	98.37	&2.06	&99.19	&32.86&	90.88	&71.98 \\
  Energy w/ OE& 2.02&	99.17&	56.18	&83.33	&4.32	&98.42	&3.96	&99.29	&40.41	&89.80&	73.45\\    
  WOODS& 0.26&	99.89&	32.71	&90.01	&\textbf{0.64}&	99.77&	\textbf{0.79}	&99.10	&12.26	&94.48&	74.15\\
\rowcolor[HTML]{EFEFEF} \textsc{SAL} (Ours)&\textbf{0.17}&	\textbf{99.90}	&\textbf{6.21}&	\textbf{96.87}&	0.94	&\textbf{99.79}	&0.84	&\textbf{99.37}&	\textbf{5.77}&	\textbf{97.12}	&73.99 \\
& \multicolumn{10}{c}{$\pi=0.2$} \\
  OE &  2.59&	98.90&	55.68&	84.36&	4.91	&99.02	&1.97&	99.37&	25.62	&93.65&	73.72\\
  Energy w/ OE& 	1.79	& 99.25& 	47.28& 	86.78	& 4.18	& 99.00& 	3.15& 	99.35	& 36.80	& 91.48	& 73.91\\
  WOODS&	0.22	&99.82	&29.78&	91.28	&0.52	&99.79	&0.89&	99.56	&10.06	&95.23	&73.49\\
 \rowcolor[HTML]{EFEFEF} \textsc{SAL} (Ours)	&\textbf{0.08}&	\textbf{99.92}&	\textbf{2.80}	&\textbf{99.31}	&\textbf{0.05}&	\textbf{99.94}&	\textbf{0.02}	&\textbf{99.97}&	\textbf{5.71}	&\textbf{98.71}&	73.86\\
 & \multicolumn{10}{c}{$\pi=0.5$} \\
 OE &  2.86&	99.05&	40.21&	88.75&	4.13	&99.05	&1.25&	99.38&	22.86&	94.63	&73.38\\
  Energy w/ OE& 2.71&	99.34	&34.82	&90.05	&3.27	&99.18	&2.54&	99.23	&30.16&	94.76	&72.76\\
  WOODS&	0.17	&99.80	&21.87	&93.73&	0.48	&99.61	&1.24	&99.54	&9.95&	95.97	&73.91\\
 \rowcolor[HTML]{EFEFEF} \textsc{SAL} (Ours)	&\textbf{0.02}	&\textbf{99.98}&	\textbf{1.27}	&\textbf{99.62}&	\textbf{0.04}&	\textbf{99.96}	&\textbf{0.01}	&\textbf{99.99}	&\textbf{5.64}&	\textbf{99.16}	&73.77\\
  & \multicolumn{10}{c}{$\pi=0.9$} \\
 OE &  0.84	&99.36	&19.78	&96.29&	1.64	&99.57&	0.51&	99.75&	12.74	&94.95&	72.02\\
  Energy w/ OE&0.97	&99.64&	17.52&	96.53&	1.36&	99.73&	0.94&	99.59	&14.01&	95.73	&73.62\\
  WOODS&	0.05	&99.98	&11.34	&95.83&	0.07	&\textbf{99.99}&	0.03&	\textbf{99.99}&	6.72&	98.73&	73.86\\
 \rowcolor[HTML]{EFEFEF} \textsc{SAL} (Ours)	&\textbf{0.03}&	\textbf{99.99}&	\textbf{2.79}	&\textbf{99.89}&	\textbf{0.05}&	\textbf{99.99}	&\textbf{0.01}	&\textbf{99.99}&	\textbf{5.88}&	\textbf{99.53}&	74.01\\	
   \hline
\end{tabular}}}
\end{table}
}

{\color{black}\subsection{Comparison with Weakly Supervised OOD Detection Baselines}
We have additionally compared with the two related works (TSL~\citep{he2023topological} and STEP~\citep{zhou2021step}). To ensure a fair comparison, we strictly follow the experimental setting in TSL, and rerun \textsc{SAL} under the identical setup. The comparison on \textsc{Cifar-100} is shown as follows. 
\begin{table}[!h]
  \centering
  
  \caption[Comparison with relevant baselines on {Cifar-100} for SAL]{ \textcolor{black}{Comparison with relevant baselines on \textsc{Cifar-100}. {Bold} numbers are superior results.  }}
    \scalebox{0.88}{
    {\color{black}\begin{tabular}{ccccc}
    \toprule
    \multirow{3}[4]{*}{Methods} & \multicolumn{4}{c}{OOD Datasets}                                                   \\
    \cmidrule{2-5}
          & \multicolumn{2}{c}{\textsc{Lsun-C}} & \multicolumn{2}{c}{\textsc{Lsun-Resize}}  \\
\cmidrule{2-5}     
& FPR95 & AUROC & FPR95 & AUROC  \\
    \midrule
  STEP & \textbf{0.00}	&99.99	&9.81&	97.87\\
  TSL & \textbf{0.00}	&\textbf{100.00}	&1.76&	99.57\\
\rowcolor[HTML]{EFEFEF} \textsc{SAL} (Ours)&\textbf{0.00}	&99.99	&\textbf{0.58}&	\textbf{99.95} \\
   \hline
\end{tabular}}}
\end{table}

}
{\color{black}\subsection{Additional Results on Different Backbones }
We have additionally tried ResNet-18 and ResNet-34 as the network architectures—which are among the most used in OOD detection literature. The comparison with the baselines on \textsc{Cifar-100} is shown in the following tables, where \textsc{SAL} outperforms all the baselines across different architectures. These additional results support the effectiveness of our approach.

\begin{table}[!h]
  \centering
  \small
  \caption[OOD detection performance on \textsc{Cifar-100} as ID on ResNet-18 for SAL]{ \textcolor{black}{OOD detection performance on \textsc{Cifar-100} as ID. All methods are trained on ResNet-18 for 100 epochs. For each dataset, we create corresponding wild mixture distribution
$\mathbb{P}_\text{wild} = (1 - \pi) \mathbb{P}_\text{in} + \pi \mathbb{P}_\text{out}$ for training and test on the corresponding OOD dataset. {Bold} numbers highlight the best results.}}
    \scalebox{0.54}{
  {\color{black}  \begin{tabular}{cccccccccccccc}
    \toprule
    \multirow{3}[4]{*}{Methods} & \multicolumn{12}{c}{OOD Datasets}                                                             & \multirow{3}[4]{*}{ID ACC} \\
    \cmidrule{2-13}
          & \multicolumn{2}{c}{\textsc{Svhn}} & \multicolumn{2}{c}{\textsc{Places365}} & \multicolumn{2}{c}{\textsc{Lsun-C}} & \multicolumn{2}{c}{\textsc{Lsun-Resize}} & \multicolumn{2}{c}{\textsc{Textures}} & \multicolumn{2}{c}{Average} &  \\
\cmidrule{2-13}          & FPR95 & AUROC & FPR95 & AUROC & FPR95 & AUROC & FPR95 & AUROC & FPR95 & AUROC & FPR95 & AUROC &  \\
  \hline
     \multicolumn{14}{c}{With $\mathbb{P}_{\text{in}}$ only} \\
    MSP   &81.32 & 77.74&  83.06 & 74.47&70.11&83.51 & 82.46& 75.73&  85.11&73.36 &80.41 & 76.96& 78.67\\
ODIN&  40.94 & 93.29& 87.71&71.46&28.72&94.51&79.61&82.13&83.63&72.37& 64.12&82.75 &78.67
 \\
Mahalanobis &22.44 &95.67& 92.66 &61.39& 68.90& 86.30& 23.07& 94.20 &62.39& 79.39&53.89&83.39&78.67
\\
Energy &  81.74&84.56&  82.23&76.68&34.78&93.93&73.57&82.99& 85.87&74.94&71.64& 82.62&78.67\\
KNN& 83.62&72.76& 82.09&80.03&65.96&	84.82 & 71.05	&81.24&76.88&	77.90& 75.92& 79.35& 78.67\\
ReAct &70.81&88.24&81.33&76.49&39.99&92.51&54.47&89.56&59.15&87.96&61.15&86.95&78.67 \\
 DICE&   54.65& 88.84& 79.58&  77.26& 0.93& 99.74& 49.40& 91.04&65.04& 76.42& 49.92&86.66 &78.67\\
CSI& 49.98 &89.57& 82.87& 75.64 &76.39& 80.38 &74.21 &83.34& 58.23& 81.04 &68.33 &81.99& 74.23\\
KNN+ & 43.21 &90.21& 84.62& 74.21& 50.12& 82.48& 76.92& 80.81& 63.21& 84.91& 63.61& 82.52& 77.03\\
\hline
 \multicolumn{14}{c}{With $\mathbb{P}_{\text{in}}$ and $\mathbb{P}_{\text{wild}}$ } \\
     OE&3.29  &97.93 & 62.90 & 80.23 & 7.07 & 95.93 & 4.06 & \textbf{97.98} & 33.27 & 90.03 & 22.12 &92.42 & 74.89\\
  Energy (w/ OE) &3.12& 94.27& 59.38& 82.19& 9.12& 91.23 &7.28 &95.39& 43.92& 90.11&24.56&90.64& 77.92\\
WOODS&3.92& 96.92 &33.92& 86.29& 5.19& 94.23 &\textbf{2.95}& 96.23 &11.95 &94.65 &11.59&93.66 & 77.54\\

\rowcolor[HTML]{EFEFEF} \textsc{SAL}  & \textbf{2.29}& \textbf{97.96} &\textbf{6.29} &\textbf{96.66} &\textbf{3.92}& \textbf{97.81}& 4.87 &97.10& \textbf{8.28} &\textbf{95.95} &\textbf{5.13}&\textbf{97.10}&77.71\\
   \hline
\end{tabular}}}
\end{table}

\begin{table}[!h]
  \centering
  \small
  \caption[OOD detection performance on \textsc{Cifar-100} as ID on ResNet-34 for SAL]{ \textcolor{black}{OOD detection performance on \textsc{Cifar-100} as ID. All methods are trained on ResNet-34 for 100 epochs. For each dataset, we create corresponding wild mixture distribution
$\mathbb{P}_\text{wild} = (1 - \pi) \mathbb{P}_\text{in} + \pi \mathbb{P}_\text{out}$ for training and test on the corresponding OOD dataset. {Bold} numbers highlight the best results.}}
    \scalebox{0.54}{
  {\color{black}  \begin{tabular}{cccccccccccccc}
    \toprule
    \multirow{3}[4]{*}{Methods} & \multicolumn{12}{c}{OOD Datasets}                                                             & \multirow{3}[4]{*}{ID ACC} \\
    \cmidrule{2-13}
          & \multicolumn{2}{c}{\textsc{Svhn}} & \multicolumn{2}{c}{\textsc{Places365}} & \multicolumn{2}{c}{\textsc{Lsun-C}} & \multicolumn{2}{c}{\textsc{Lsun-Resize}} & \multicolumn{2}{c}{\textsc{Textures}} & \multicolumn{2}{c}{Average} &  \\
\cmidrule{2-13}          & FPR95 & AUROC & FPR95 & AUROC & FPR95 & AUROC & FPR95 & AUROC & FPR95 & AUROC & FPR95 & AUROC &  \\
  \hline
     \multicolumn{14}{c}{With $\mathbb{P}_{\text{in}}$ only} \\
    MSP   &78.89 & 79.80&  84.38&  74.21&  83.47&  75.28&  84.61 & 74.51&  86.51 & 72.53 & 83.12&  75.27&79.04\\
ODIN& 70.16& 84.88 &82.16& 75.19 &76.36 &80.10& 79.54& 79.16 &85.28 &75.23& 78.70& 79.11&79.04
 \\
Mahalanobis & 87.09& 80.62 &84.63 &73.89 &84.15& 79.43 &83.18 &78.83 &61.72 &84.87& 80.15 &79.53&79.04
\\
Energy & 66.91 &85.25 &81.41 &76.37 &59.77 &86.69 &66.52 &84.49 &79.01& 79.96& 70.72& 82.55& 79.04\\
KNN& 81.12 &73.65  & 79.62& 78.21 & 63.29& 85.56 & 73.92 & 79.77& 73.29 & 80.35 & 74.25 & 79.51 & 79.04\\
ReAct & 82.85 & 70.12 &81.75 & 76.25& 80.70 & 83.03 &67.40 & 83.28 &74.60& 81.61 & 77.46 & 78.86 &79.04 \\
 DICE& 83.55 & 72.49  & 85.05 & 75.92 & 94.05 & 73.59 &75.20 & 80.90 &79.80 & 77.83 &83.53 & 76.15 & 79.04  \\
CSI& 44.53 &92.65& 79.08& 76.27 &75.58& 83.78 &76.62 &84.98& 61.61& 86.47 &67.48 &84.83&77.89\\
KNN+ & 39.23 &92.78& 80.74& 77.58& 48.99& 89.30& 74.99& 82.69& 57.15& 88.35& 60.22& 86.14&78.32\\
\hline
 \multicolumn{14}{c}{With $\mathbb{P}_{\text{in}}$ and $\mathbb{P}_{\text{wild}}$ } \\
     OE& 2.11  &98.23 & 60.12 & 83.22 & 6.08 & 96.34 & 3.94 & 98.13 & 30.00 & 92.27 & 20.45 & 93.64&  75.72\\
  Energy (w/ OE) &1.94& 95.03& 68.84& 85.94& 7.66& 92.04 &6.86 &97.63& 40.82& 93.07& 25.22&92.74 & 78.75\\
WOODS& 2.08& 97.33 &25.37& 88.93& 4.26& 97.74 &1.05& 97.30 &8.85 &96.86 & 8.32&  95.63& 78.97\\

\rowcolor[HTML]{EFEFEF} \textsc{SAL}  &\textbf{0.98} &\textbf{99.94 }&\textbf{2.98} &\textbf{99.08}& \textbf{0.07} &\textbf{99.94}& \textbf{0.03}& \textbf{99.96}& \textbf{4.01} &\textbf{98.83}& \textbf{ 1.61}& \textbf{99.55}& 78.01\\
   \hline
\end{tabular}}}
\end{table}

}

  \subsection{Broader Impact}
  \label{sec:broader}
 Our project aims to improve the reliability and safety of modern machine learning models.  From the theoretical perspective, our analysis can facilitate and deepen the understanding of   the effect of unlabeled wild data for OOD detection. In Appendix~\ref{sec:verification_discrepancy}, we properly verify the necessary conditions and the value of our error bound using real-world datasets. Hence, we believe our theoretical framework has a broad utility and significance.

 From the practical side,  our study can lead to direct benefits and societal impacts, particularly when the wild data is abundant in the models' operating environment, such as in safety-critical applications
i.e., autonomous driving and healthcare data analysis.  Our study does not involve any human subjects or violation of legal compliance. We do not anticipate any potentially harmful consequences to our work. Through our study and releasing our code, we hope to raise stronger research and societal awareness towards the problem of exploring unlabeled wild data for out-of-distribution detection in real-world settings.

\section{HaloScope: Harnessing Unlabeled LLM Generations for Hallucination Detection}

\subsection{Datasets and Implementation Details}
\label{sec:inference_app}
\textbf{Input prompts.} We provide the detailed textual input as prompts to the language models for different datasets. Specifically, for datasets without context (\textsc{TruthfulQA} and \textsc{TriviaQA}), the prompt is shown as follows:
\begin{center}
   \textit{Answer the question concisely. Q: [question] A:}
\end{center}

For datasets with context (\textsc{TydiQA-GP} and \textsc{CoQA}), we have the following template for prompts:
\begin{center}
    \textit{Answer these questions concisely based on the context: \textbackslash n Context: [a passage or a paragraph] Q: [question] A:}
\end{center}
Here are some examples from those datasets with our inference format.

 \begin{tikzpicture}[
			chatbox_inner/.style={rectangle, rounded corners, opacity=0, text opacity=1, font=\sffamily\scriptsize, text width=5in, text height=9pt, inner xsep=6pt, inner ysep=6pt},
			chatbox_prompt_inner/.style={chatbox_inner, align=flush left, xshift=0pt, text height=11pt},
			chatbox_user_inner/.style={chatbox_inner, align=flush left, xshift=0pt},
			chatbox_gpt_inner/.style={chatbox_inner, align=flush left, xshift=0pt},
			chatbox/.style={chatbox_inner, draw=black!25, fill=gray!7, opacity=1, text opacity=0},
			chatbox_prompt/.style={chatbox, align=flush left, fill=gray!1.5, draw=black!30, text height=10pt},
			chatbox_user/.style={chatbox, align=flush left},
			chatbox_gpt/.style={chatbox, align=flush left},
			chatbox2/.style={chatbox_gpt, fill=green!25},
			chatbox3/.style={chatbox_gpt, fill=red!20, draw=black!20},
			chatbox4/.style={chatbox_gpt, fill=yellow!30},
			labelbox/.style={rectangle, rounded corners, draw=black!50, font=\sffamily\bfseries, fill=gray!5, inner sep=3pt},
		]
		\node[chatbox_user] (q1) {
  \newline
    \begin{itemize}
	 \item Answer the question concisely. Q: What does it always mean when a cat purrs? A: \\
     \item  Answer the question concisely. Q: What happens if you eat after 8pm? A: \\
     \item Answer the question concisely. Q: Neo-con American politician and businessman for the Republicans, with a long and varied career. His name is Donald what? A: 
    \end{itemize}
		};
  \node[chatbox_user_inner] (q1_text) at (q1) {
  \newline
	  \begin{itemize}
	 \item Answer the question concisely. Q: What does it always mean when a cat purrs? A: \\
 
     \item  Answer the question concisely. Q: What happens if you eat after 8pm? A: \\
    
     \item Answer the question concisely. Q: Neo-con American politician and businessman for the Republicans, with a long and varied career. His name is Donald what? A: 
    \end{itemize}};
		\node[labelbox, anchor=north west, yshift=5pt, xshift=5pt] at (q1.north west) {\textbf{TruthfulQA}};
	\end{tikzpicture}

     \begin{tikzpicture}[
			chatbox_inner/.style={rectangle, rounded corners, opacity=0, text opacity=1, font=\sffamily\scriptsize, text width=5in, text height=9pt, inner xsep=6pt, inner ysep=6pt},
			chatbox_prompt_inner/.style={chatbox_inner, align=flush left, xshift=0pt, text height=11pt},
			chatbox_user_inner/.style={chatbox_inner, align=flush left, xshift=0pt},
			chatbox_gpt_inner/.style={chatbox_inner, align=flush left, xshift=0pt},
			chatbox/.style={chatbox_inner, draw=black!25, fill=gray!7, opacity=1, text opacity=0},
			chatbox_prompt/.style={chatbox, align=flush left, fill=gray!1.5, draw=black!30, text height=10pt},
			chatbox_user/.style={chatbox, align=flush left},
			chatbox_gpt/.style={chatbox, align=flush left},
			chatbox2/.style={chatbox_gpt, fill=green!25},
			chatbox3/.style={chatbox_gpt, fill=red!20, draw=black!20},
			chatbox4/.style={chatbox_gpt, fill=yellow!30},
			labelbox/.style={rectangle, rounded corners, draw=black!50, font=\sffamily\bfseries, fill=gray!5, inner sep=3pt},
		]
		\node[chatbox_user] (q1) {
  \newline
    \begin{itemize}
	 \item Answer the question concisely. Q: What does it always mean when a cat purrs? A: \\
     \item  Answer the question concisely. Q: What happens if you eat after 8pm? A: \\
     \item Answer the question concisely. Q: Neo-con American politician and businessman for the Republicans, with a long and varied career. His name is Donald what? A: 
    \end{itemize}
		};
  \node[chatbox_user_inner] (q1_text) at (q1) {
  \newline
 \begin{itemize}
  \vspace{-2em}
              \item Answer the question concisely. Q: Who was the next British Prime Minister after Arthur Balfour? A: 
               \vspace{-1em}
               \item Answer the question concisely. Q:  What is the name of Terence and Shirley Conran's dress designer son? A: 
               \vspace{-1em}
                 \item Answer the question concisely. Q:   For what novel did J. K. Rowling win the 1999 Whitbread Children's book of the year award? A: 
 \end{itemize}};
		\node[labelbox, anchor=north west, yshift=5pt, xshift=5pt] at (q1.north west) {\textbf{TriviaQA}};
	\end{tikzpicture}

  \begin{tikzpicture}[
			chatbox_inner/.style={rectangle, rounded corners, opacity=0, text opacity=1, font=\sffamily\scriptsize, text width=5in, text height=9pt, inner xsep=6pt, inner ysep=6pt},
			chatbox_prompt_inner/.style={chatbox_inner, align=flush left, xshift=0pt, text height=11pt},
			chatbox_user_inner/.style={chatbox_inner, align=flush left, xshift=0pt},
			chatbox_gpt_inner/.style={chatbox_inner, align=flush left, xshift=0pt},
			chatbox/.style={chatbox_inner, draw=black!25, fill=gray!7, opacity=1, text opacity=0},
			chatbox_prompt/.style={chatbox, align=flush left, fill=gray!1.5, draw=black!30, text height=10pt},
			chatbox_user/.style={chatbox, align=flush left},
			chatbox_gpt/.style={chatbox, align=flush left},
			chatbox2/.style={chatbox_gpt, fill=green!25},
			chatbox3/.style={chatbox_gpt, fill=red!20, draw=black!20},
			chatbox4/.style={chatbox_gpt, fill=yellow!30},
			labelbox/.style={rectangle, rounded corners, draw=black!50, font=\sffamily\bfseries, fill=gray!5, inner sep=3pt},
		]
		\node[chatbox_user] (q1) {
  \newline
  \begin{itemize}
 
           \item 
            Answer these questions concisely based on the context: \textbackslash n Context: (Entertainment Weekly) -- How are the elements of the charming, traditional romantic comedy "The Proposal" like the checklist of a charming, traditional bride? Let me count the ways ... Ryan Reynolds wonders if marrying his boss, Sandra Bullock, is a good thing in ``The Proposal." Something old: The story of a haughty woman and an exasperated man who hate each other -- until they realize they love each other -- is proudly square, in the tradition of rom-coms from the 1940s and '50s. Or is it straight out of Shakespeare's 1590s? Sandra Bullock is the shrew, Margaret, a pitiless, high-powered New York book editor first seen multitasking in the midst of her aerobic workout (thus you know she needs to get ... loved). Ryan Reynolds is Andrew, her put-upon foil of an executive assistant, a younger man who accepts abuse as a media-industry hazing ritual. And there the two would remain, locked in mutual disdain, except for Margaret's fatal flaw -- she's Canadian. (So is "X-Men's" Wolverine; I thought our neighbors to the north were supposed to be nice.) Margaret, with her visa expired, faces deportation and makes the snap executive decision to marry Andrew in a green-card wedding. It's an offer the underling can't refuse if he wants to keep his job. (A sexual-harassment lawsuit would ruin the movie's mood.) OK, he says. But first comes a visit to the groom-to-be's family in Alaska. Amusing complications ensue. Something new: The chemical energy between Bullock and Reynolds is fresh and irresistible. In her mid-40s, Bullock has finessed her dewy America's Sweetheart comedy skills to a mature, pearly texture; she's lovable both as an uptight careerist in a pencil skirt and stilettos, and as a lonely lady in a flapping plaid bathrobe. Q: What movie is the article referring to? A:  
                
  \end{itemize}
		};
  \node[chatbox_user_inner] (q1_text) at (q1) {
  \newline
  \begin{itemize}
 
           \item 
            Answer these questions concisely based on the context: \textbackslash n Context: (Entertainment Weekly) -- How are the elements of the charming, traditional romantic comedy "The Proposal" like the checklist of a charming, traditional bride? Let me count the ways ... Ryan Reynolds wonders if marrying his boss, Sandra Bullock, is a good thing in ``The Proposal." Something old: The story of a haughty woman and an exasperated man who hate each other -- until they realize they love each other -- is proudly square, in the tradition of rom-coms from the 1940s and '50s. Or is it straight out of Shakespeare's 1590s? Sandra Bullock is the shrew, Margaret, a pitiless, high-powered New York book editor first seen multitasking in the midst of her aerobic workout (thus you know she needs to get ... loved). Ryan Reynolds is Andrew, her put-upon foil of an executive assistant, a younger man who accepts abuse as a media-industry hazing ritual. And there the two would remain, locked in mutual disdain, except for Margaret's fatal flaw -- she's Canadian. (So is "X-Men's" Wolverine; I thought our neighbors to the north were supposed to be nice.) Margaret, with her visa expired, faces deportation and makes the snap executive decision to marry Andrew in a green-card wedding. It's an offer the underling can't refuse if he wants to keep his job. (A sexual-harassment lawsuit would ruin the movie's mood.) OK, he says. But first comes a visit to the groom-to-be's family in Alaska. Amusing complications ensue. Something new: The chemical energy between Bullock and Reynolds is fresh and irresistible. In her mid-40s, Bullock has finessed her dewy America's Sweetheart comedy skills to a mature, pearly texture; she's lovable both as an uptight careerist in a pencil skirt and stilettos, and as a lonely lady in a flapping plaid bathrobe. Q: What movie is the article referring to? A:  
                
  \end{itemize}};
		\node[labelbox, anchor=north west, yshift=5pt, xshift=5pt] at (q1.north west) {\textbf{CoQA}};
	\end{tikzpicture}

\begin{tikzpicture}[
			chatbox_inner/.style={rectangle, rounded corners, opacity=0, text opacity=1, font=\sffamily\scriptsize, text width=5in, text height=9pt, inner xsep=6pt, inner ysep=6pt},
			chatbox_prompt_inner/.style={chatbox_inner, align=flush left, xshift=0pt, text height=11pt},
			chatbox_user_inner/.style={chatbox_inner, align=flush left, xshift=0pt},
			chatbox_gpt_inner/.style={chatbox_inner, align=flush left, xshift=0pt},
			chatbox/.style={chatbox_inner, draw=black!25, fill=gray!7, opacity=1, text opacity=0},
			chatbox_prompt/.style={chatbox, align=flush left, fill=gray!1.5, draw=black!30, text height=10pt},
			chatbox_user/.style={chatbox, align=flush left},
			chatbox_gpt/.style={chatbox, align=flush left},
			chatbox2/.style={chatbox_gpt, fill=green!25},
			chatbox3/.style={chatbox_gpt, fill=red!20, draw=black!20},
			chatbox4/.style={chatbox_gpt, fill=yellow!30},
			labelbox/.style={rectangle, rounded corners, draw=black!50, font=\sffamily\bfseries, fill=gray!5, inner sep=3pt},
		]
		\node[chatbox_user] (q1) {
  \newline
                \textbf{ \textsc{TydiQA-GP} }
                  \begin{itemize}
                      \item  Answer these questions concisely based on the context: \textbackslash n Context: The Zhou dynasty (1046 BC to approximately 256 BC) is the longest-lasting dynasty in Chinese history. By the end of the 2nd millennium BC, the Zhou dynasty began to emerge in the Yellow River valley, overrunning the territory of the Shang. The Zhou appeared to have begun their rule under a semi-feudal system. The Zhou lived west of the Shang, and the Zhou leader was appointed Western Protector by the Shang. The ruler of the Zhou, King Wu, with the assistance of his brother, the Duke of Zhou, as regent, managed to defeat the Shang at the Battle of Muye. Q: What was the longest dynasty in China's history? A:
                  \end{itemize}
		};
  \node[chatbox_user_inner] (q1_text) at (q1) {
  \newline
			
                  \begin{itemize}
                      \item  Answer these questions concisely based on the context: \textbackslash n Context: The Zhou dynasty (1046 BC to approximately 256 BC) is the longest-lasting dynasty in Chinese history. By the end of the 2nd millennium BC, the Zhou dynasty began to emerge in the Yellow River valley, overrunning the territory of the Shang. The Zhou appeared to have begun their rule under a semi-feudal system. The Zhou lived west of the Shang, and the Zhou leader was appointed Western Protector by the Shang. The ruler of the Zhou, King Wu, with the assistance of his brother, the Duke of Zhou, as regent, managed to defeat the Shang at the Battle of Muye. Q: What was the longest dynasty in China's history? A:
                  \end{itemize}
                  };
		\node[labelbox, anchor=north west, yshift=5pt, xshift=5pt] at (q1.north west) {\textbf{TydiQA-GP}};
	\end{tikzpicture}

 \textbf{Implementation details for baselines.} For Perplexity method~\citep{ren2023outofdistribution}, we follow the implementation here\footnote{\url{https://huggingface.co/docs/transformers/en/perplexity}}, and calculate the average perplexity score in terms of the generated tokens. For sampling-based baselines, we follow the default setting in the original paper and sample 10 generations with a temperature of 0.5 to estimate the uncertainty score. Specifically, for Lexical Similarity~\citep{lin2023generating}, we use the Rouge-L as the similarity metric, and for SelfCKGPT~\citep{manakul2023selfcheckgpt}, we adopt the NLI version as recommended in their codebase\footnote{\url{https://github.com/potsawee/selfcheckgpt}}, which is a fine-tuned DeBERTa-v3-large model to measure the probability of ``entailment" or ``contradiction" between the most-likely generation and the sampled generations. For promoting-based baselines, we adopt the following prompt for Verbalize~\citep{lin2022teaching} on the open-book QA datasets:
\begin{center}
    \textit{Q: [question] A:[answer]. \textbackslash n The proposed answer is true with a confidence value (0-100) of },
\end{center}
and the prompt of
\begin{center}
    \textit{Context: [Context] Q: [question] A:[answer]. \textbackslash n The proposed answer is true with a confidence value (0-100) of },
\end{center}
for datasets with context. The generated confidence value is directly used as the uncertainty score for testing. For the Self-evaluation approach~\citep{kadavath2022language}, we follow the original paper and utilize the prompt for the open-book QA task as follows:
\begin{center}
    \textit{Question: [question] \textbackslash n Proposed Answer: [answer] \textbackslash n Is the proposed answer: \textbackslash n (A) True \textbackslash n (B) False \textbackslash n The proposed answer is:}
\end{center}
For datasets with context, we have the prompt of:
\begin{center}
    \textit{Context: [Context]  \textbackslash n  Question: [question] \textbackslash n Proposed Answer: [answer] \textbackslash n Is the proposed answer: \textbackslash n (A) True \textbackslash n (B) False \textbackslash n The proposed answer is:}
\end{center}
We use the log probability of output token ``A" as the uncertainty score for evaluating hallucination detection performance following the original paper.

\subsection{Distribution of the Membership Estimation Score}
\label{sec:score_dis_app}

\begin{wrapfigure}{r}{0.4\textwidth}
\vspace{-0.8cm}
\begin{center}
    \includegraphics[width=\linewidth]{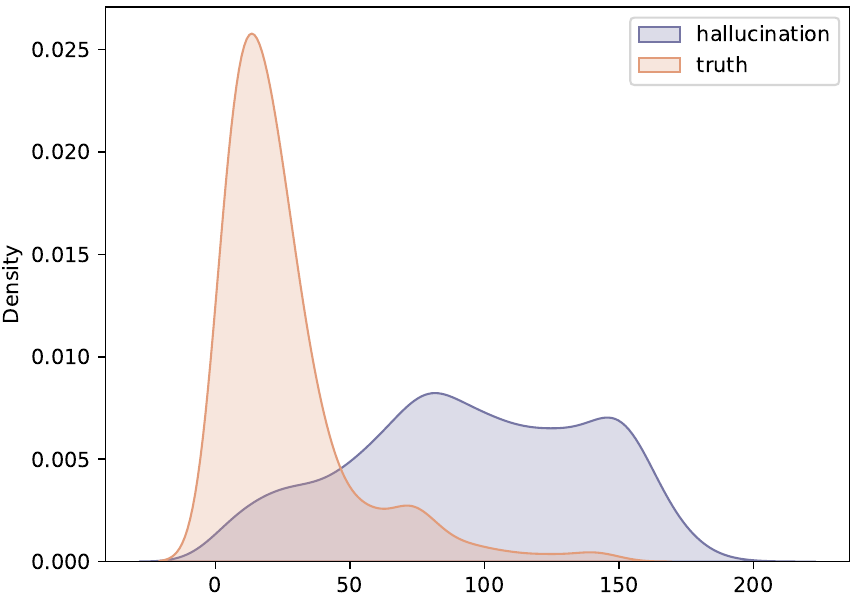}
  \end{center}
   \vspace{-1em}
    \caption[Distribution of membership estimation score]{ Distribution of membership estimation score.}
      \vspace{-1em}
    \label{fig:score_dis}
\end{wrapfigure}

We show in Figure~\ref{fig:score_dis} the distribution of the membership estimation score (as defined in Equation~\ref{eq:score} of the main paper) for the truthful and hallucinations in the unlabeled LLM generations of \textsc{TydiQA-GP}. Specifically, we visualize the score calculated using the LLM representations from the 14-th layer of LLaMA-2-chat-7b. The result demonstrates a reasonable separation between the two types of data, and can benefit the downstream training of the truthfulness classifier.

\subsection{Results with Rouge-L}
\label{sec:rouge_exp_app}
In our main paper, the generation is deemed truthful when the BLUERT score between the generation and the ground truth exceeds a given threshold. In this ablation, we show that the results are robust under a different similarity measure Rouge-L, following~\citep{kuhn2023semantic,chen2024inside}. 
Consistent with Section~\ref{sec:setup}, the threshold is set to be 0.5.
With the same experimental setup, the results on the LLaMA-2-7b-chat model are shown in Table~\ref{tab:rouge_metric_app}, where the effectiveness of our approach still holds.

\begin{table}[!h]
    \centering
    \footnotesize
   \scalebox{0.8}{ \begin{tabular}{ccc|cc}
    \hline

  Model & Method&  \makecell{Single \\ sampling} &  \textsc{TruthfulQA} &\textsc{TydiQA-GP} \\
    \hline
  
       \multirow{11}{*}{LLaMA-2-7b }  
     
     &   Perplexity~\citep{ren2023outofdistribution} & \checkmark &42.62  & 75.32  \\
    &  LN-Entropy~\citep{malinin2021uncertainty}  & \xmark & 44.77 & 73.90 \\
  & Semantic Entropy~\citep{kuhn2023semantic} & \xmark &  47.01 &  71.27 \\
 & Lexical Similarity~\citep{lin2023generating} & \xmark &  67.78 &45.63  \\
     &   EigenScore~\citep{chen2024inside} & \xmark & 67.31 & 47.90\\
     &   SelfCKGPT~\citep{manakul2023selfcheckgpt} & \xmark & 54.05  & 49.96 \\
 
   &Verbalize~\citep{lin2022teaching} & \checkmark &    53.71 & 55.29 \\
  & Self-evaluation~\citep{kadavath2022language} &\checkmark & 55.96  & 51.04   \\
   &CCS~\citep{burns2022discovering} & \checkmark &  59.07 & 71.62 \\ 
&CCS$^{*}$~\citep{burns2022discovering}& \checkmark & 60.12 & 77.35 \\ 
 &\cellcolor[HTML]{EFEFEF}\multirow{1}{*}{HaloScope (\textsc{Ours})}  & \cellcolor[HTML]{EFEFEF}\checkmark
&\cellcolor[HTML]{EFEFEF}\textbf{74.16} & \cellcolor[HTML]{EFEFEF}\textbf{91.53}  \\
\hline
    \end{tabular}}
    \vspace{1em}
    \caption[Main results with Rouge-L metric]{ \textbf{Main results with Rouge-L metric.} Comparison with competitive hallucination detection methods on different datasets. 
    All values
are percentages. ``Single sampling" indicates whether the approach requires multiple generations during inference. \textbf{Bold} numbers are superior results. }
 \vspace{-2em}
    \label{tab:rouge_metric_app}
\end{table}

\subsection{Results with a Different Dataset Split}
\label{sec:different_split_app}
We verify the performance of our approach using a different random split of the dataset. Consistent with our main experiment,  we randomly split 25\% of the available QA pairs for testing using a different seed. HaloScope can achieve similar hallucination detection performance to the results in our main Table~\ref{tab:main_table}. For example, on the LLaMA-2-chat-7b model, our method achieves an AUROC of 76.39\% and 94.89\% on \textsc{TruthfulQA} and \textsc{TydiQA-GP} datasets, respectively (Table~\ref{tab:different_split_app}). Meanwhile, HaloScope is able to outperform the baselines as well, which shows the statistical significance of our approach.
\begin{table}[!h]
    \centering
    \footnotesize
   \scalebox{0.8}{ \begin{tabular}{ccc|cc}
    \hline

  Model & Method&  \makecell{Single \\ sampling} &  \textsc{TruthfulQA} &\textsc{TydiQA-GP} \\
    \hline
  
       \multirow{11}{*}{LLaMA-2-7b }  
     
     &   Perplexity~\citep{ren2023outofdistribution} & \checkmark &  56.71& 79.39 \\
    &  LN-Entropy~\citep{malinin2021uncertainty}  & \xmark &59.18  & 74.85\\
  & Semantic Entropy~\citep{kuhn2023semantic} & \xmark & 56.62 &  73.29 \\
 & Lexical Similarity~\citep{lin2023generating} & \xmark & 55.69 & 46.44\\
     &   EigenScore~\citep{chen2024inside} & \xmark &47.40  & 45.87 \\
     &   SelfCKGPT~\citep{manakul2023selfcheckgpt} & \xmark & 55.53   & 51.03\\
 
   &Verbalize~\citep{lin2022teaching} & \checkmark &  50.29 & 46.83 \\
  & Self-evaluation~\citep{kadavath2022language} &\checkmark &  56.81  & 54.06 \\
   &CCS~\citep{burns2022discovering} & \checkmark &   63.78 & 77.61\\ 
&CCS$^{*}$~\citep{burns2022discovering}& \checkmark & 65.23 & 80.20 \\ 
 &\cellcolor[HTML]{EFEFEF}\multirow{1}{*}{HaloScope (\textsc{Ours})}  & \cellcolor[HTML]{EFEFEF}\checkmark
&\cellcolor[HTML]{EFEFEF}\textbf{76.39} & \cellcolor[HTML]{EFEFEF}\textbf{94.98}  \\
\hline
    \end{tabular}}
    \vspace{1em}
    \caption[Results with a different random split of the dataset]{ \textbf{Results with a different random split of the dataset.} Comparison with competitive hallucination detection methods on different datasets. 
    All values
are percentages. ``Single sampling" indicates whether the approach requires multiple generations during inference. \textbf{Bold} numbers are superior results. }
 \vspace{-2em}
    \label{tab:different_split_app}
\end{table}

\subsection{Ablation on Sampling Strategies}
\label{sec:sampling}
We evaluate the hallucination detection result when HaloScope identifies the hallucination subspace using LLM generations under different sampling strategies. In particular, our main results are obtained based on beam search, i.e., greedy sampling, which generates the next token based on the maximum likelihood.
In addition, we compare with multinomial sampling with a temperature of 0.5. Specifically, we sample one answer for each question and extract their embeddings for subspace identification (Section~\ref{sec:step_1}), and then keep the truthfulness classifier training the same as in Section~\ref{sec:step_2} for test-time hallucinations detection. The comparison in Table~\ref{tab:using_gene_subspace} shows similar performance between the two sampling strategies, with greedy sampling being slightly better.

\begin{table}[!h]
    \centering
    \begin{tabular}{c|cc}
    \toprule
    Unlabeled Data &   \textsc{TruthfulQA} &\textsc{TydiQA-GP} \\
    \hline
     Multinomial sampling   &76.62 & 93.68 \\
      \cellcolor[HTML]{EFEFEF}Greedy sampling (\textsc{Ours})  & \cellcolor[HTML]{EFEFEF}\textbf{78.64} & \cellcolor[HTML]{EFEFEF}\textbf{94.04}\\
     
         \bottomrule
    \end{tabular}
    \caption[Hallucination detection result under different sampling strategies]{ Hallucination detection result under different sampling strategies.  Results are based on the LLaMA-2-chat-7b model.}
    \label{tab:using_gene_subspace}
\end{table}

\subsection{Results with Less Unlabeled Data}
In this section, we ablate on the effect of the number of unlabeled LLM generations $N$. Specifically, on \textsc{TruthfulQA}, we randomly sample 100-500 generations from the current unlabeled split of the dataset ($N$=512) with an interval of 100, where the corresponding experimental result on LLaMA-2-chat-7b model is presented in Table~\ref{tab:num_gene}. We observe that the hallucination detection performance slightly degrades when $N$ decreases. Given that unlabeled data is easy and cheap to collect in practice, our results suggest that it's more desirable to leverage a sufficiently large sample size. 

\begin{table}[!h]
    \centering
    \begin{tabular}{c|c}
    \toprule
    $N$&   \textsc{TruthfulQA}  \\
    \hline
    100 &73.34 \\
    200 & 76.09\\
    300 & 75.61\\
    400 & 73.00\\
    500& 75.50\\
      512 &\textbf{78.64}  \\
         \bottomrule
    \end{tabular}
    \caption[The number of the LLM generations and its effect on the hallucination detection result]{ The number of the LLM generations and its effect on the hallucination detection result. }
    \label{tab:num_gene}
\end{table}

\subsection{Results of Using Other Uncertainty Scores for Filtering}
\label{sec:using_other_sciore_filter_app}
We compare our HaloScope with training the truthfulness classifier by membership estimation with other uncertainty estimation scores. We follow the same setting as HaloScope and select the threshold $T$ and other key hyperparameters using the same validation set. The comparison is shown in Table~\ref{tab:using_other_sciore_filter_app}, where the stronger performance of HaloScope vs. using other uncertainty scores for training can precisely highlight the benefits of our membership estimation approach by the hallucination subspace. The model we use is LLaMA-2-chat-7b.

\begin{table}[!h]
    \centering
    \begin{tabular}{c|cc}
    \toprule
    Method &   \textsc{TruthfulQA} &\textsc{TydiQA-GP} \\
    \hline
     Semantic Entropy  &65.98 &77.06 \\
     SelfCKGPT & 57.38 &52.47\\
     CCS$^*$ &69.13 & 82.83 \\
      \cellcolor[HTML]{EFEFEF}HaloScope (\textsc{Ours})  & \cellcolor[HTML]{EFEFEF}\textbf{78.64} & \cellcolor[HTML]{EFEFEF}\textbf{94.04}\\
     
         \bottomrule
    \end{tabular}
    \caption[Hallucination detection results leveraging other uncertainty scores]{ Hallucination detection results leveraging other uncertainty scores. }
    \label{tab:using_other_sciore_filter_app}
\end{table}

\begin{table}[!h]
    \centering
    \begin{tabular}{c|cc}
    \toprule
    Method &   Text continuation	& Text summarization  \\
    \hline
     Semantic Entropy  &69.88	&60.15 \\
     SelfCKGPT & 73.23	&69.91\\
     CCS$^*$ &76.79&	71.36 \\
      \cellcolor[HTML]{EFEFEF}HaloScope (\textsc{Ours})  & \cellcolor[HTML]{EFEFEF}\textbf{79.37} & \cellcolor[HTML]{EFEFEF}\textbf{75.84}\\
     
         \bottomrule
    \end{tabular}
    \caption[Hallucination detection results on different tasks]{ Hallucination detection results on different tasks. }
\end{table}

\subsection{Results on Additional Tasks}

We evaluate our approach on two additional tasks, which are (1) text continuation and (2) text summarization tasks. For text continuation, following~\citep{manakul2023selfcheckgpt}, we use LLM-generated articles for a specific concept from the WikiBio dataset. We evaluate under the sentence-level hallucination detection task and split the entire 1,908 sentences in a 3:1 ratio for unlabeled generations and test data. (The other implementation details are the same as in our main Table~\ref{tab:main_table}.)

For text summarization, we sample 1,000 entries from the HaluEval~\citep{li-etal-2023-halueval} dataset (summarization track) and split them in a 3:1 ratio for unlabeled generations and test data. We prompt the LLM with "[document] \textbackslash n Please summarize the above article concisely. A:" and record the generations while keeping the other implementation details the same as the text continuation task.

The comparison on LLaMA-2-7b with three representative baselines is shown below. We found that the advantage of leveraging unlabeled LLM generations for hallucination detection still holds.

\subsection{Software and Hardware}
\label{sec:computing_app}
We run all experiments with Python 3.8.5 and PyTorch 1.13.1, using NVIDIA RTX A6000 GPUs.

\end{app}

\bibliographystyle{mcbride}
\bibliography{citations}


\end{document}